\documentclass[sigconf,10pt]{acmart}
\usepackage{balance}
\usepackage{graphicx}
\newcommand{\eat}[1]{}
\newcommand{\scream}[1]{}
\newcommand{\reminder}[1]{}
\usepackage{amsmath}
\usepackage{amsfonts}
\usepackage{makecell}
\usepackage{subcaption}
\usepackage{multirow}
\usepackage{hhline}
\usepackage[font={bf}]{caption}
\usepackage[linesnumbered,ruled]{algorithm2e}
\usepackage{mathrsfs}
\usepackage{eufrak}
\usepackage{hyperref}
\usepackage{url}
\newtheorem{theorem}{Theorem}
\newtheorem{lemma}{Lemma}
\newtheorem{assumption}{Assumption}

\newcommand{\covmultidataset}{Cov}
\newcommand{\skindataset}{Skin}
\newcommand{\heartbeatdataset}{Heartbeat}
\newcommand{\SGEMMdataset}{SGEMM}
\newcommand{\higgsdataset}{HIGGS}
\newcommand{\rcvdataset}{RCV1}

\newcommand{\pro}{PrIU}
\newcommand{\proopt}{PrIU-opt}
\newcommand{\gbm}{GBM}

\newcommand{\cut}{Cut-off}

\newcommand{\std}{Std}
\newcommand{\iter}{BaseL}
\newcommand{\infl}{INFL}
\newcommand{\sgd}{SGD}
\newcommand{\gd}{GD}
\newcommand{\minisgd}{mb-SGD}

\newcommand{\subids}{\mathscr{B}^{(t)}}

\newcommand{\fliperr}{flipping-errors}
\newcommand{\newerr}{new-errors}
\newcommand{\featureerr}{feature-errors}

\newcommand{\w}{\textbf{w}}
\newcommand{\E}{\mathbb{E}}
\newcommand{\increw}{\textbf{w}_{_U}}
\newcommand{\linearw}{\textbf{w}_{_L}}
\newcommand{\linearincrew}{\textbf{w}_{_{LU}}}
\newcommand{\logistlinearincrew}{\textbf{w}_{_{RU}}}
\newcommand{\linearprov}{\mathcal{W}_{_L}}
\newcommand{\increprov}{\mathcal{W}_{_U}}
\newcommand{\linearincreprov}{\mathcal{W}_{_{LU}}}
\newcommand{\increB}{B_{_U}}
\newcommand{\closeform}{Closed-form}
\newcommand{\cifar}{cifar10}

\newcommand{\zeroprov}{0_{\mathrm{prov}}}
\newcommand{\oneprov}{1_{\mathrm{prov}}}

\newcommand\mygeqone{\stackrel{\mathclap{\normalfont\mbox{Equation \ref{eq: strong_convexity2}}}}{\geq}}

\newcommand\csineq{\stackrel{\mathclap{\normalfont\mbox{\text{Cauchy–Schwarz inequality}}}}{\leq}}
\newcommand\myleqone{\stackrel{\mathclap{\normalfont\mbox{Equation \ref{eq: piecewise_approx_rate}}}}{\leq}}
\newcommand\myleqtwo{\stackrel{\mathclap{\normalfont\mbox{Equation \ref{eq: R_property1}-\ref{eq: R_property2}}}}{\leq}}

\newcommand\myeqtwo{\stackrel{\mathclap{\normalfont\mbox{Lemma \ref{lm: matrix_norm_with_identity_matrix}}}}{=}}
\newcommand\myeqthree{\stackrel{\mathclap{\normalfont\mbox{Equation \ref{eq: g_expectation}}}}{=}}

\newcommand\numberthis{\addtocounter{equation}{1}\tag{\theequation}}

\newcommand{\REMARK}[2]{{\bf [*[#1:~#2]*]}}
\newcommand{\yinjun}[1]{}
\newcommand{\val}[1]{}
\newcommand{\susan}[1]{}

\newcommand{\struct}[1]{\REMARK{STRUCT}{#1}}
\def\BibTeX{{\rm B\kern-.05em{\sc i\kern-.025em b}\kern-.08emT\kern-.1667em\lower.7ex\hbox{E}\kern-.125emX}}
    
\copyrightyear{2020}
\acmYear{2020}
\setcopyright{acmcopyright}
\acmConference[SIGMOD '20] {2020 ACM SIGMOD International Conference on Management of Data}{June 14--19, 2020}{Portland, OR, USA}
\acmBooktitle{2020 ACM SIGMOD International Conference on Management of Data (SIGMOD'20), June 14--19, 2020, Portland, OR, USA}
\acmPrice{15.00}
\acmDOI{10.1145/3318464.3380571}
\acmISBN{978-1-4503-6735-6/20/06}
\begin{document}

\fancyhead{}
  % do not delete this code.

% The "title" command has an optional parameter, allowing the author to define a "short title" to be used in page headers.
\title{\pro: A Provenance-Based Approach for Incrementally Updating Regression Models}

\author{Yinjun Wu}
\affiliation{%
  \institution{University of Pennsylvania}
%   \streetaddress{Anonymous}
%   \city{Anonymous}
%   \country{Anonymous}
  }
\email{wuyinjun@seas.upenn.edu}
\author{Val Tannen}
\affiliation{%
  \institution{University of Pennsylvania}
%   \streetaddress{Anonymous}
%   \city{Anonymous}
%   \country{Anonymous}
  }
\email{val@cis.upenn.edu}
\author{Susan B. Davidson}
\affiliation{%
  \institution{University of Pennsylvania}
%   \streetaddress{Anonymous}
%   \city{Anonymous}
%   \country{Anonymous}
  }
\email{susan@cis.upenn.edu}
\begin{abstract}
The ubiquitous use of machine learning algorithms brings new challenges to traditional database problems such as incremental view update. Much effort
  is being put in better understanding and debugging machine learning models,
  as well as in identifying and repairing errors in training datasets.
  %since these can degrade the model quality and 
  %negatively affect prediction accuracy. 
  %In contrast, 
  Our focus is on how to assist these activities when they have to retrain the machine learning model after removing problematic training samples in cleaning or selecting different subsets of training data for interpretability. This paper presents an efficient provenance-based approach, \pro,
  %called \profull\ (\pro) 
  and its optimized version, \proopt, for
  %approximately \val{not sure why approximately}
  incrementally updating model parameters without sacrificing prediction accuracy.
  %compared to the way to retrain from the scratch. 
  We prove the correctness and convergence of the incrementally updated model parameters, 
  and 
  %verify 
  validate it experimentally. Experimental results show that up to two orders of magnitude speed-ups can be achieved by \proopt\ compared to simply retraining the model from scratch, yet obtaining highly similar models.  
\end{abstract}


\begin{CCSXML}
<ccs2012>
<concept>
<concept_id>10002951.10002952.10002953.10010820.10003623</concept_id>
<concept_desc>Information systems~Data provenance</concept_desc>
<concept_significance>500</concept_significance>
</concept>
<concept>
<concept_id>10002950.10003648.10003688.10003699</concept_id>
<concept_desc>Mathematics of computing~Exploratory data analysis</concept_desc>
<concept_significance>300</concept_significance>
</concept>
<concept>
<concept_id>10003752.10003809.10003716.10011138.10010043</concept_id>
<concept_desc>Theory of computation~Convex optimization</concept_desc>
<concept_significance>300</concept_significance>
</concept>
</ccs2012>
\end{CCSXML}

\ccsdesc[500]{Information systems~Data cleaning,Incremental maintenance,Data provenance}
\ccsdesc[300]{Mathematics of computing~Exploratory data analysis}
\ccsdesc[300]{Theory of computation~Convex optimization}



%
% Keywords. The author(s) should pick words that accurately describe the work being
% presented. Separate the keywords with commas.
\keywords{Data provenance, machine learning, deletion propagation}

%
% A "teaser" image appears between the author and affiliation information and the body 
% of the document, and typically spans the page. 
%   \Description{Enjoying the baseball game from the third-base seats. Ichiro Suzuki preparing to bat.}

%
% This command processes the author and affiliation and title information and builds
% the first part of the formatted document.
\maketitle

\section{Introduction}

% VAL INTRO revised to add data cleaning; remove GDPR argument for now.
In database terminology, this paper is about \emph{efficient incremental
view updates}, specifically about using provenance annotations to
propagate the effect of deletions from the input data to the output. However,
the views that we consider are \emph{regression models} (linear and
binomial/multinomial logistic regression) and the input data consists
of the samples used to train these models.

%Need for iteratively updating models:  data cleaning, GDPR, interpretability
The need for incremental techniques to efficiently update regression models arises in several contexts, for example data cleaning and interpretability.  
{\em Data cleaning} has been extensively studied by the database community~\cite{dasu2003exploratory, rahm2000data, fan2012quality, chu2016data}, and is typically an {iterative} and {interactive} process, allowing data analysts to alternate between analysis and cleaning tasks, as well as to interact with other parties such as IT staff and data curators~\cite{krishnan2016towards}.
Machine learning techniques are particularly sensitive to dirty data in training datasets, since it can result in erroneous models and  counter-intuitive predictions for test datasets~\cite{chu2016data}.
A number of techniques have therefore recently been proposed for detecting and repairing dirty data in machine learning, e.g., 
\cite{krishnan2017boostclean, heidari2019holodetect}. 
The work presented in this paper can be incorporated into these data cleaning pipelines by assuming that dirty data in the training set has already been detected, and addresses the next step by providing a solution for incrementally updating the machine learning model after the dirty data is removed.%\footnote{The work presented here is not applicable to incremental insertions of new training samples, but it may be useful for \emph{modifications}, i.e., changing the values of some of the features (see the discussion on future work  in Section~\ref{sec: conclusion}).} 
%Our focus is on propagating the effect of deleting subsets of training data to the corresponding updates to the model parameters, and doing so efficiently.

% \eat{
% %GDPR
% There has also been an increasing need for privacy protection from recent laws, such as the General Data Protection Regulation (GDPR) \cite{GDPR}, which require companies and institutions that process personal data to delete user data upon request. If the deleted data are involved in the machine learning training process, their effect should be also removed from the trained ML models and thus the ML model parameters should be ``decrementally'' updated \cite{schelteramnesia}.  }

%Interpretability
{\em Interpretability} is also a major concern in machine learning (see, for example, the general discussions
in~\cite{doshi2017roadmap,Lipton16mythos}, the extensive human
subjects experiments in~\cite{poursabzi2018manipulating}, as well the
many references in these papers). 
The problem is being studied from several different perspectives 
(see Sec.~\ref{sec: related_work}).
The data-driven approaches of ~\cite{doshi2017roadmap,krishnan2017palm}
discover \emph{factors of interpretability}
by performing \emph{repeated retraining}
of models using multiple different subsets of a training dataset to understand the relationship between samples with certain feature characteristics and the model behavior.
Such repeated retraining also occurs in 
\emph{model debugging}~\cite{krishnan2016activeclean,KangRBZ18,heidari2019holodetect} 
and \emph{deletion diagnostics}~\cite{cook1977detection}.
\eat{
\scream{Yinjun: feature engineering means something else in ML. I prefer removing the following sentence}when a fine-grained analysis of the training set is needed, 
perhaps leading to feature re-engineering.}

In this respect, our work shares goals with ~\cite{koh2017understanding}, 
which develops an {\em influence function} to approximately quantify the influence of a \emph{single} training sample on the model parameters and prediction
results; this can also be
used for estimating the model parameter change
after the removal of one training sample.
However, extending the influence function
approach to multiple training samples %(for example, all those coming from a given source) leads to a model significantly dissimilar to the expected updated model parameters, resulting in worse 
significantly weakens prediction accuracy. {\em In contrast, our techniques are not only efficient but significantly more accurate.}
\reminder{actually, as accurate as baseline retraining; but our models are also highly similar to the ones from baseline retraining, as opposed to INFL.}

%Connection to provenance
{\bf Connection to Provenance.} 
Note that the problem of incrementally updating the model after removing a subset of the training samples can be seen as a question of {\em data
provenance}~\cite{green2007provenance,buneman2001and,CheneyCT09}.  
Data provenance tracks the dependencies between
input and output data; in particular, the \emph{provenance semiring framework}
\cite{green2007provenance,GreenT17}
has been used for applying incremental updates (specifically deletions)
to views.

In the semiring framework, input data is annotated with provenance tokens which are carried through the operators performed on the data (e.g. select, project, join, union).
Output data is then annotated with \emph{provenance polynomials} expressed in terms of the provenance tokens.  When an input tuple is deleted, the effect on the output can be efficiently calculated by essentially ``zeroing out'' its token in the provenance polynomial.
Recently, the framework has been extended to
include basic linear algebra operations: matrix addition and
multiplication~\cite{yan2016fine}.
In this extension, the
provenance polynomials play the role of scalars and multiplication
with scalars plays the role of annotating matrices and vectors with
provenance. 

As an example, suppose that $p,q,r,s$ are provenance tokens
that annotate 
%some input data items, such as 
samples in a training dataset. Our methods
will show that vectors of interest (such as the vector of model parameters) can be expressed with provenance-annotated expressions such as:
$$
\mathbf{w} = (p^2q * \mathbf{u}) + (qr^4 * \mathbf{v}) + (ps * \mathbf{z})
$$
Here, $\mathbf{u,v,z}$ are numerical vectors signifying contributions to the
answer $\mathbf{w}$ and they are annotated (algebraic operation $*$) with
$p^2q, qr^4, ps$ which are provenance polynomials to be read as follows: the provenance $p^2q$ represents the use of both data items labeled $p$ and $q$ and, in fact, the first item is used twice.
Now suppose the data item annotated with $r$ is deleted while those annotated $p,q,s$
are retained. We can express the updated value of $\mathbf{w}$ under this deletion
by setting $r$ to the ``provenance 0 polynomial'', denoted $\zeroprov$
%$0_{_{\mathrm{prov}}}$ 
which signifies absence, and 
$p,q,s$ to the
``provenance 1 polynomial'', denoted $\oneprov$,
%1_{_{\mathrm{prov}}}$ 
which signifies ``neutral'' presence, no need to track further.
The algebraic properties of provenance polynomials and of their annotation of matrices/vectors ensure what one would expect, e.g,  
%$0_{_{\mathrm{prov}}}}\cdot r^4=0_{_{\mathrm{prov}}}$ 
$\zeroprov\cdot r^4 = \zeroprov$ as well as $\zeroprov*\mathbf{v} = \mathbf{0}$
(the all-zero vector)
and $\oneprov*\mathbf{z}=\mathbf{z}$. It follows that under this deletion 
$\mathbf{w} = \mathbf{u} + \mathbf{z}$.

%Overview of our approach
{\bf Approach.}  In this paper, we use the extension of the semiring framework to matrix operations
%of the semiring framework to basic linear algebra in order 
to track the provenance of input
samples through the training of logistic regression and linear
regression models using gradient descent and its variants.
\reminder{gradient descent, (mini-batch) stochastic gradient descent (called gradient-based methods (GBMs) below).}
In each iteration of the training phase, a gradient-based ``update rule'' updates the model parameters, \eat{. When the update rule is expressed using matrix addition and multiplication (as is the case in linear regression), we can annotate the rule with provenance polynomials. }which can be annotated with provenance polynomials. For logistic regression, we can achieve this via piecewise linear interpolation over the non-linear
components in the gradient update rule.

In addition to enabling provenance tracking, the linearization of the
gradient update rule allows us to separate the contributions of
the training samples from the contributions of the model parameters from the previous iteration.  As a result, the effect of deleting
training samples on the gradient update rule can be obtained by ``zeroing out''
the provenance tokens corresponding to those samples.

{\bf Challenges.} Reasoning over provenance to enable incremental updates introduces significant overhead in the gradient descent calculation. 
To speed up incremental updates over model parameters for dense datasets, we use several optimizations in our implementation, \pro:  First, between iterations during the training phase over the full training dataset,
%, for both linear and logistic regression, 
we cache intermediate results (some matrix expression) that
capture only the contribution of the training samples. These are annotated with
provenance. Then during the model update phase, the propagation of the deletion of a subset of samples comes down to a subtraction of the "zeroed-out" contributions of the removed samples.
%from the cached intermediate results. 
Second, \eat{if the feature space of the training dataset is large or the mini-batch size is small, }we apply singular value decomposition (SVD) over the intermediate results to reduce their dimensions. \reminder{Is it clear what the SVD is of?}
An optimized version of \pro, \proopt, is also designed for further optimizations over datasets with small feature sets using incremental updates to eigenvalues.
(For logistic regression, it is used by terminating provenance tracking early when provenance expressions stabilize. See Section \ref{sec: implementation} for more details). But the optimizations above cannot work for sparse datasets, for which we use only the linearization 
of the update rule for logistic regression.

\eat{An additional subtlety for logistic regression is that piecewise linear interpolation uses coefficients which are different in different iterations, and hence need to be efficiently tracked in the computation.}

\eat{
% optimizations for \sgd\ and \minisgd

We also noted that with stochastic gradient descent (\sgd) and mini-batch stochastic gradient descent (\minisgd), only a small portion of the training samples are used in each iteration, which can achieve significant speed-ups compared to traditional gradient descent method, thus reducing the performance gains brought by the idea of simply caching the intermediate results in the update rule. But we also observed that the rank of the cached intermediate results is small once the mini-batch size is small in \minisgd\ since the intermediate result represents the sum of multiple \textbf{rank-1} terms (each of which is contributed by one training sample from the mini-batch) and thus its rank should be no more than the mini-batch size. So we determined to decompose the intermediate result for \textbf{dense datasets} into the multiplications of two lower-rank matrix by only keeping the most important components of its singular value decomposition \cite{forsythe1967computer}.
}

%Why is our approach valuable?
%\scream{please examine carefully what I am saying here}
As we shall see, \pro\ and \proopt\ can lead to speed-ups of up to 2 orders of magnitude when compared to a baseline of retraining the model from the updated input data; however, for sparse datasets the speedup is only 10\%. 
While the practical impact of this speed-up may be small for an engineer who only deletes one subset of training samples, especially if retraining takes only a few minutes, the impact is much greater for an engineer who repeatedly removes multiple different subsets of training samples, e.g. when exploring factors of interpretability. In this case, even one order of magnitude speed-up reduces exploration from several hours to a few minutes.  

\eat{
%get rid of \pro?
One challenge of using provenance to enable incremental updates over machine learning models is the overhead, due not only to capturing and reasoning over provenance but to the model complexity in machine learning pipelines.
As a result, it is difficult to achieve {\em efficient} incremental updates by straighforwardly applying the update rule-based provenance strategies illustrated above. 
We therefore develop an {\em efficient incremental model update strategy}, \pro, and its optimized version, \proopt, for several simple machine learning models. The models include linear regression and (binary and multinomial) logistic regression.
}
\scream{can we simply talk about future work in the conclusion, the following paragraph is removed (in the eat environment)}
\eat{
Extending our approach to more complicated models than
linear and logistic regression (such as deep neural network models)
%\emph{generalized additive models} (GAM)~\cite{hastie1986gam}) 
is left for future work. }
\reminder{More in conclusions}

\textbf{Contributions} of this paper include:
\vspace*{-1mm}
\begin{enumerate}
    \item A theoretical framework which enables data provenance to be tracked and used for fast incremental model updates when subsets of training samples are removed.  The framework extends the approach in~\cite{green2007provenance, green2010provenance,yan2016fine} to  linear regression and (binary and multinomial) logistic regression models.
    \item Analytical results showing the convergence and accuracy of the updated model parameters for logistic regression, which are approximately computed by applying piecewise linear interpolation over the non-linear operations in the model parameter update rules. 
    \item Efficient provenance-based algorithms, \pro\ and \proopt, which achieve fast model updates after removing subsets of training samples.
    \eat{\scream{can we remove the following details since they have been mentioned above?} through efficient matrix operations, using the singular value decomposition to reduce the dimension of the intermediate results, using an incremental eigenvalue update approach to avoid expensive repetitive computation, and terminating provenance tracking early when provenance expressions stabilize. }
    \item Extensive experiments showing the effectiveness and accuracy of \pro\ and \proopt\ in incrementally updating the linear regression and logistic regression models compared to the straightforward approach of retraining from  scratch, as well as compared to implementing an extension of the influence function in \cite{koh2017understanding}.
    \item Enabling work on interpretability that seeks to understand
    the effect of removing {\em subsets} of the training data, rather than just of a single training sample. \reminder{We are not actually doing the interpretability work, just helping with the repeated retraining. And it's not clear that such work is currently taking place} 
\end{enumerate}

\vspace{-1mm}
The remainder of the paper is organized as follows. In Section \ref{sec: related_work}, we describe related work in incremental model maintenance, data provenance, data cleaning, and machine learning model interpretability. Section \ref{sec: prl} reviews the basic concepts of linear regression and logistic regression. The theoretical development of how to use provenance in the update rules of linear regression and logistic regression is presented in Section \ref{sec: model}, and its implementation provided in Section \ref{sec: implementation}. Experimental results comparing our approach to other solutions are presented in Section \ref{sec: experiment}. We conclude in Section \ref{sec: conclusion}.

To our knowledge, this is the first work to use provenance 
for the purpose of incrementally updating machine learning model parameters.

\eat{
\scream{our contributions, edited}

%connection to interpretability of machine learning model
%\struct{This is related to what is called "influence" of one or more items in the training data on the model parameters.}
Another contribution of our approach is a better, more fine-grained, understanding of how the parameters of the learned model depend on the training data. For example, suppose the training data comes from several different sources and we wish
to understand how one of these sources affects the model.
% which has become increasingly challenging with the use of complex machine learning models, e.g. deep neural networks.
This is related to the notion of {\em interpretability}~\cite{doshi2017roadmap, poursabzi2018manipulating}.  One approach to interpretability that has been proposed in the machine learning community~\cite{koh2017understanding} is to develop an {\em influence function} to approximately quantify the influence of a single training sample on the model parameters and prediction results; this can also be used to estimate the model parameter change after the removal of a training sample. 
However, our results show (see Section~\ref{sec: experiment}) that extending the influence function approach
to multiple training samples (for example, all those coming from a given source) leads to very inaccurate models.
Searching for another approach, we note that interpretability can be seen as a \emph{provenance} question.
\eat{However, we demonstrate experimentally (see Section \ref{sec: experiment}) that this approach
leads to very inaccurate models when extended
%their method naively 
to handle the deletion of multiple training samples.}
\eat{not sure if we should say here 10\%}
%\val{I thought there is a naive way to extend their definition to multiple samples
%but this sacrifices accuracy?}\yinjun{yes, modified}
\eat{
As our experiments will demonstrate  using influence functions to estimate the change of model parameters is very inaccurate 
\val{this needs to be better explained}\yinjun{done by explaining that their approach is based on Talyor expansion}
when multiple training samples (say 10\% of training samples) are removed.}
% compared to expected updated model parameters, which is even harmful for the following prediction step.

% brief intro to data provenance and its connections to data cleaning and interpretability problem
To achieve the goal of incrementally updating the model after the
removal of subsets of training samples, we use {\em data
provenance}~\cite{green2007provenance,buneman2001and,CheneyCT09}.  
In general, data provenance tracks the dependencies between
input and output data; in particular, the \emph{provenance semiring framework}
\cite{green2007provenance,GreenT17}
has been used for applying
incremental updates (specifically deletions)
to views, and more generally, through schema 
mappings~\cite{ives2008orchestra}.

In the semiring framework, output data is annotated with \emph{provenance
polynomials} expressed in terms of provenance tokens that annotate the
input data, thus enabling deletions of input data to be propagated to the output data.
The framework has proven quite versatile, allowing for extensions 
to powerful query languages (see references in~\cite{GreenT17}).
In particular, the framework has been extended to
basic linear algebra operations: matrix addition and
multiplication~\cite{yan2016fine} (also mentioned 
in~\cite{buneman2019data} as a first step in using data
provenance for interpretability of machine learning models).
In this extension, the provenance polynomials play the role of scalars and multiplication
with scalars plays the role of annotating matrices and vectors with
provenance; a similar idea can be found in ~\cite{amsterdamer2011provenance}.

In this paper we use the extension of the semiring framework to
basic linear algebra in order to track the provenance of input
samples through the training of logistic regression and linear
regression models using iterative gradient-based methods (GBMs).
In each iteration, a gradient-based ``rule'' updates the 
model parameters. We annotate the gradient update rule with 
provenance polynomials 
when it is expressed using matrix addition and multiplication. This is
the case when using GBM in linear regression. For logistic regression
we achieve this via piecewise linear interpolation of the non-linear
components in the gradient update rule.

In addition to allowing provenance tracking, the linearization of the
gradient update rule also allows us to separate the contributions of
the training samples from the contributions of the model parameters
from the previous iteration.  As a result, the effect of deleting
training samples on the gradient update rule can be obtained by
setting the provenance tokens corresponding to those samples to zero.
To speed up computation for both linear and logistic
regression, we are able to cache intermediate results between iterations that
capture only the contribution of the training samples, annotated with
provenance. Propagating the deletions of some of the samples then
comes down to subtracting a portion of these intermediate results.  An
additional subtlety for logistic regression is that the piecewise
linear interpolation uses coefficients which are different in
different iterations and hence need to be tracked in the computation.

\eat{We illustrate below how provenance can be used to incrementally update the model parameters after removing a subset of training data.

\begin{example}
%\subsection{Example}
Consider the UCI Skin Segmentation  \footnote{\url{https://www.csie.ntu.edu.tw/~cjlin/libsvmtools/datasets/binary.html#skin_nonskin}} dataset, which is used to classify pixels in images of faces as skin or not-skin using RGB values along with other features such as gender and race. A training dataset containing three samples 
%(which is unrealistically small) 
is shown in Figure \ref{fig: example} (ignore the last column for the moment). For simplicity, we only show the feature values for R, G, B attributes as well as the label (skin\_or\_not) for each sample. 

\begin{figure}[!htb]
\includegraphics[width=0.45\textwidth, height=0.25\textwidth]{Figures/example.png}
\caption{UCI Skin Segmentation training dataset example}\label{fig: example}
\end{figure}

Suppose that logistic regression with stochastic gradient descent (SGD) is used during the training phase, resulting in the vector of model parameters $(0.423, 0.6, -0.9,...)$ shown in the middle of the figure.
%Figure \ref{fig: example}.  
When running the model over the test dataset, it turns out that the accuracy is low, only 60\%. After consulting with experts who are familiar with the dataset, the first sample is discovered to be an error and is removed from the training phase. Rather than rerunning the logistic regression over the modified training dataset, we update the model parameters directly using provenance.

We start by annotating each training sample with a unique provenance token, e.g. $p_1, p_2, p_3$ for the three samples.
%in Figure \ref{fig: example}.
The value of $p_i(i=1,2,3)$ can be 1 or 0 representing the existence or non-existence of the corresponding training sample.  During the training phase over the original (dirty) dataset, the provenance tokens are propagated through the SGD computation. As a result, the model parameters (i.e. $(0.423, 0.6, -0.9,...)$) will be broken into different pieces, each of which is associated with some {\em joint} combination of the provenance tokens (the {\em tensor product} as will be seen later). For example, in Figure \ref{fig: example}, one such tensor product is $(0.134, 0.472, -0.28,\dots)*p_1$, which intuitively means that if only the first sample exists, the model parameters will be $(0.134, 0.472, -0.28,\dots)$. We can then ``remove'' the first sample by setting $p_1$ to 0 and $p_2$, $p_3$ to 1, and thus derive the updated model parameters, $(0.84, -0.06, 0.32,\dots)$ (shown at the bottom of the figure). Note that the model does not need to be retrained over the reduced training dataset; the model parameters are updated directly.  This new model now has a much higher test accuracy (90\%), similar to that achieved by retraining from scratch.
\end{example}}

One challenge of using provenance to enable incremental updates over machine learning models is the overhead, due not only to capturing and reasoning over provenance but to the model complexity in machine learning pipelines.
As a result, it is difficult to achieve {\em efficient} incremental updates by straighforwardly applying the update rule-based provenance strategies illustrated above. 
%Even for some simple machine learning models, such as logistic regression, designing full-fledged solutions is non-trivial. 
We therefore develop an {\em efficient incremental model update strategy}, \pro, and its optimized version, \proopt, for several simple machine learning models. The models include linear regression and (binary and multinomial) logistic regression.

Extending this to more complicated models is left for future work. A possible wider target, especially for the 
optimization ideas embodied in \proopt\ is the class of \emph{generalized additive models} (GAM)~\cite{hastie1986gam}.
Note also that 1) the multinomial logistic regression model represents the last layer, i.e. the softmax layer, 
of a typical deep neural network models; and 2) one of the most widely used activation functions in 
\emph{deep neural network models (DNNs)} is the sigmoid function, which is the same as the non-linear operation in the logistic regression model\eat{; and 3) another typical activation function in DNNs is the ReLU function, which is a segmented linear function and thus closely related to piecewise linear interpolation used in our approach}.
This suggests DNNs as another future target.

\textbf{Contributions} of this paper include:
%\vspace*{-1mm}
\begin{enumerate}
    \item A theoretical framework which enables data provenance to be tracked and used for very fast incremental model updates when subsets of training samples are removed.  The framework extends the approach in~\cite{green2007provenance, green2010provenance,yan2016fine} to  linear regression and (binary and multinomial) logistic regression models.
    
    \item Analytical results showing the convergence and accuracy of the updated model parameters which are approximately computed by applying piecewise linear interpolation over the non-linear operations (in logistic regression) in the model parameter update rules. 
    
    \item Efficient provenance-based algorithms, \pro\ and \proopt, which achieve very fast model updates after removing subsets of training samples through efficient matrix operations, using the singular value cdecomposition to reduce the dimenension of the intermediate results, using an incremental eigenvalue update approach to avoid expensive repetitive computation, and terminating provenance tracking early when provenance expressions stabilize. 
     
    \item Extensive experiments showing the effectiveness and accuracy of \pro\ and \proopt\ in incrementally updating the linear regression and logistic regression models compared to the straightforward approach of retraining from  scratch and compared to implementing an extension of the influence function in \cite{koh2017understanding}.
    
    \item Extending current work on interpretability by enabling the understanding the effect of removing {\em subsets} of the training data, rather than just a single training sample.
    
\eat{
- model of provenance for linear matrix operations -- WHAT ARE THESE
- use of this model to linear regression as well as a linearized version of logistic regression based on approximation
- correctness analysis: proof of convergence of resulting provenance expression whenever the model parameter w converges, proof on bound between the real and approximated model parameters
- application to cleaning, understanding influence of inputs on model
- speed up computation (optimization) of model update under deletion
}

\end{enumerate}

The remainder of the paper is organized as follows. In Section \ref{sec: related_work}, we describe related work in incremental model maintenance, data provenance, data cleaning, and machine learning model interpretability. Section \ref{sec: prl} reviews the basic concepts of linear regression, logistic regression, and piecewise linear interpolation. The theoretical development of how to use provenance in the update rules of linear regression and logistic regression is presented in Section \ref{sec: model}, and its implementation provided in Section \ref{sec: implementation}. Experimental results showing the feasibility of our approach compared to other solutions are presented in Section \ref{sec: experiment}. We conclude in Section \ref{sec: conclusion}.

To our knowledge, this is the first work to use provenance 
%(in particular, the provenance semiring model) 
for the purpose of incrementally updating machine learning model parameters.  
\eat{
through gradient-based method in the context of cleaning dirty data and interpretability of machine learning models.}

\eat{Extending this to more complicated models is left for future work. But considering the similarity between multinomial logistic regression and some portions of more complicated models, e.g. multinomial logistic regression is the same as the softmax layer of deep neural network models, we believe that the solutions proposed in this paper can be potential used in those models.}
%\susan{This is now repetitive of the above.}\yinjun{removed}

% \medskip

\eat{
\susan{This remains to be completed... contributions and outline of paper?}
\eat{If lot of parameters, the performance of our approach is bad.  Sort of related to the number of features.  Deep learning has a lot of parameters.}
%approaches: SGD/GD, approximation, 
highlight the novelty of our contributions. This is the first work to marry provenance and machine learning training phase for the purpose of efficient incremental updates

\struct{
\begin{itemize}
    \item Related work
    \item Preliminaries which learning model (linear, logistic, multinomial logistic); 
    which iterative algorithm for parameter learning (SGD, miniSGD,...)
    \item Model: provenance decoration of iterative algo; separate issue: linearize logistic regression.
    \item "Implementation": further detail the algorithm
    \item Experimental evaluation.
\end{itemize}
}

}
}
\vspace{-1mm}
\section{Related work}\label{sec: related_work}

% We introduce related work in this section, which expands 

{\em Incremental model maintenance.} There have been several proposals for materializing machine learning models for future reuse. \cite{deshpande2006mauvedb, gupta2015processing, nikolic2014linview} target the problem of efficiently updating the model as the training data changes, %\cite{deshpande2006mauvedb,gupta2015processing} 
which focus primarily on linear
regression and Naive Bayes models, and use closed-form solutions (rather than iterative algorithms, e.g., gradient-based approaches) of the model parameters to determine incremental updates in light of additions and deletions of training samples while 
%\scream{Yinjun: modified to cite paper requested by the reviewer 2} 
\cite{hasani2018efficient} deals with how to merge pre-materialized models to construct new models based on user requests. In addition, \cite{gupta2015processing} also deals with incremental updates of the model parameters based on the Mixture Weight Methods (a variant of gradient descent) for logistic regression. The method,
however, puts additional training samples into another batch and averages the pre-computed parameters derived from other batches (over the original data) with the parameters computed over the additional batch.
This cannot be used for incremental deletions which is our focus in this paper.
%, however, could not be used to update the model parameters by removing parts of input dataset.\val{The LINVIEW paper does not cite MauveDB and the Gupta paper cites neither MauveDB nor LINVIEW!!?!!}
%\val{But I think we must discuss all three and explain how "closed-form" solutions are different}\yinjun{can you take a look}

The basic ideas of \cite{deshpande2006mauvedb, gupta2015processing, nikolic2014linview} on how to incrementally update linear regression models are somewhat similar. Due to the existence of the matrix inverse operations in the closed-form solution for linear regression, only the intermediate results built with linear operations are maintained as views. They are updated when insertion or deletion happens in the input training data. After that, matrix inversion is used to compute the final updated model parameters. 
%\scream{Yinjun: have simplified the following sentences}
%which takes $O(m^3)$. 
In contrast, our approach proceeds directly to a gradient descent-based linear regression. As we shall see, our experiments show that
our approach is more efficient than the closed-form update.
%which takes $O(\min(k, m)\cdot m^2)$ 
%($k$ is a small number representing the number of removed samples), 
%thus more efficient than the closed-form solution.
%\scream{discuss this: only when $k\ll m$ but for SGEMM we have $m=18$ and $k$ can be as much as 24000. What's going on?!?}

\eat{But in practice, to avoid expensive inversion operations, iterative methods such as the Jacobi method \cite{saad2003iterative} are needed. We proceed directly to an iterative gradient descent-based linear regression in our approach noting that matrix inverting takes $O(m^3)$, which is less efficient than  for our solution. }

 \eat{To our knowledge, there is no existing work that solves the incremental model update problem for logistic regression and its counterpart for multiclasses in the case of removal of training samples.} \eat{Incremental updates over linear regression and naive bayes model is pretty straightforward since their model parameters can be solved with a closed-form formula, i.e. a formula taking the training data and model parameters as independent and dependent variables respectively, which, however, is not true for most of other machine learning models, e.g. logistic regression.} 

\eat{\val{Yinjun says:}
I think the main difference is that we focus on gradient descent, which is a more general method to derive the model parameters in machine learning domain.
In terms of your second question, I just checked their paper. They actually proposed different approaches for the model updates with closed-form solutions rather than for gradient-descent-based solutions and compare the performance between those approaches. I think we can highlight that our focus is for gradient-descent based machine learning models and we simply start from linear regression, which is the simplest one.
But I guess even in closed formula case, sometimes we may still need to use iterative method to derive the matrix inverse (something similar to gradient descent. We have discussed about this earlier), which aims at avoiding the expensive matrix inverse operations.
\val{we need to discuss}
}

{\em Data provenance.} Data provenance captures where data comes from and how it is processed.  Within the database community, various approaches have been proposed to track provenance through queries, e.g. where and why provenance \cite{buneman2001and}, and semiring 
provenance~\cite{green2007provenance,amsterdamer2011provenance}.
Provenance is used to identify the source of errors in computational processes, such as workflows~\cite{amsterdamer2011putting} and network diagnostics \cite{zhou2010efficient}. It is also used to support efficient incremental updates through database queries and schema mappings~\cite{green2007update,ives2008orchestra,green2010provenance} 
and workflow computation \cite{ellkvist2008using}. 
Provenance support for linear algebra 
operations in the context of machine learning tasks has also been recently studied~\cite{yan2016fine}. This work was mentioned 
in~\cite{buneman2019data} as a first step in \emph{using data
provenance for interpretability} of machine learning models.

{\em Data cleaning.} The goal of data cleaning is to detect and fix errors in data, and is a crucial step in preparing data for data analytics/machine learning tasks \cite{krishnan2016activeclean, dolatshah2018cleaning}. 
However, if erroneous/dirty data is detected after the model has been trained, the machine learning algorithm must be rerun to obtain the updated model parameters.
This repetitive training can cause significant delays when large volumes of data are processed.  One approach is to start each training phase by setting the initial model parameters to the ones generated by the previous training phase over the dirty data~\cite{krishnan2016activeclean}.  Our contribution is orthogonal to this approach, and  updates the machine learning model parameters directly by reasoning over provenance rather than retraining from scratch.

\eat{Most existing work focuses on efficient locate the errors of the data and thus improve the data quality. Once the cleaned data are obtained, the same machine learning algorithms are applied over them to derive the updated model parameters. Repetitive training process can cause significant delays especially when large volumes of data are processed although \cite{krishnan2016activeclean} starts each training phase by setting the initialized model parameters as the ones from previous training phase over the dirty data. Our contributions are thus orthogonal to the existing works in data cleaning domain, which aims at updating the machine learning model parameters more efficiently rather than retrain from the scratch when the dirty training samples are removed.}

\eat{{\em Adversarial learning} There is a common consensus in machine learning domain that to degrade the prediction performance of certain machine learning models, the adversaries do not necessarily hack into the programs. Instead, they can choose to attack the training datasets (called {\em poisoning}) or test datasets (called {\em evasion}) by introducing adversarial examples or polluting the existing feature values or labels, which have been studied in machine learning communities for years \cite{mei2015using, xiao2015feature, munoz2017towards, biggio2013evasion}. \eat{In this paper, we hope that even with the existence of {\em adversarial examples} in the training datasets, our approach can still produce }}

{\em Interpreting and understanding ML models.}
Fully understanding the behavior of ML models, especially deep neural network models,
%\val{speaking of DNNs we want to say more we should discuss}
is difficult due to their complexity. Moreover, there are different perspectives on
what we should understand. For example, one approach separates model components into ``shape''
functions, one for each feature, in generalized additive models, in particular for linear and logistic
regression~\cite{LouCG12,CaruanaLGKSE15}. Closest to our perspective is the idea of {\em influence function}~\cite{koh2017understanding} (similar problem is also mentioned in \cite{polyzotis2017data}), which originates from \emph{deletion diagnostics} in statistics \cite{cook1977detection}. \cite{koh2017understanding} estimates the effect of removing a 
\emph{single} training sample %(rather than multiple samples)
on the already obtained model,
%One essential component of {\em influence function} approximately estimates the change of the model parameters after removing the training sample of interest 
\emph{without} retraining the model. 
The influence function uses the Taylor expansion of the derivative of a customized objective function for the model parameter.  The calculation (and thus the approximation of model parameter change) is only based on lower-order terms in Taylor expansion.
% without using the higher-order terms.

This can be seen as a method for incremental model update for just one sample deletion. In fact, we have
observed that the method could be extended to deleting an arbitrary number of samples, which led us to compare it experimentally to our approach. 
%\scream{Yinjun: have simplified the following sentences, earlier version in the eat environment} 
The results (see Section \ref{sec: experiment}) show that this approach\eat{ to deleting multiple training samples} leads to very inaccurate results when multiple training samples are deleted. 
%\val{forward ref to experiments})
\eat{confirmed what we expected: 
%Intuitively, the smaller the expected change of the model parameters is, the more accurate the approximation will be. This is because 
the Taylor expansion approximates well around one point, i.e. the original model parameter,
and expanding this approach to deleting multiple training samples leads to very inaccurate results.}

\eat{Quantifying the influence of  errors in the training dataset has also recently been considered in the  database community \cite{polyzotis2017data}.}

\eat{There are also increasing concerns on how to interpret machine learning models. Due to the complexity of those models (especially deep neural network models), fully understanding the model behaviors is difficult. One possible way toward it is to use the idea of {\em influence function} proposed by \cite{koh2017understanding}, which aims at estimating the ''influence'' of some single training sample (rather than multiple samples) against the prediction performance. One essential component of {\em influence function} approximately estimates the change of the model parameters after removing the training sample of interest without retraining from the scratch. The sample to be removed might be potential errors in the training dataset. How to quantify the influence of such errors is also considered as a data management issue in database community \cite{polyzotis2017data}.}

\section{Preliminaries}\label{sec: prl}

%\scream{discuss this: now this has only one subsection maybe include in 4?}
%In this section, 
We give an overview of linear and logistic regression along 
with the gradient-based method for learning model parameters.
%in the 
%machine learning 
%training phase,
% We then review {\em piecewise linear interpolation},
%preparing readers for the model formalization in Section \ref{sec: model}.
% \subsection{Objective functions and learning methods}
\eat{The machine learning training phase centers around the definition of objective functions and their minimization through certain optimization methods over training datasets. We first review the objective functions for linear regression and logistic regression
as particular cases of the following general form:
%non-experts don't have to go to details to understand
%Suppose objective function of the machine learning algorithm is $h(\textbf{w})$, which has the following form:
\begin{equation}\label{eq: object_function_1}
h(\textbf{w}) = \frac{1}{n}\sum_{i=1}^n h_i(\textbf{w}) = \frac{1}{n}\sum_{i=1}^n (g_i(\textbf{w}) + l(\textbf{w}))
\end{equation}

where $\mathbf{w}$ represents the vector of model parameters,
$n$ is the number of training data samples, and 
$h_i(\textbf{w})$ represents the objective function evaluated over each training sample. In turn, this is composed of a loss function $g_i(\textbf{w})$ and regularization term $l(\textbf{w})$. 
%The model parameter $\textbf{w}$ is a vector of parameters. 
In linear regression, $h_i(\textbf{w})$ is the mean square error (MSE) \cite{lehmann2006theory} while in logistic regression, $h_i(\textbf{w})$ is the cross entropy loss \cite{goodfellow2016deep}. There are several choices for the regularization term $l(\textbf{w})$. }
%Suppose in linear regression and logistic regression,
Assume a training dataset ($\textbf{X}$, $\textbf{Y}$), where $\textbf{X}$ is an $n\times m$ matrix representing the feature matrix while $\textbf{Y}$ is an $n \times 1$ vector representing the labels, i.e.:
% \begin{equation}\label{eq: training_data_x}
% \textbf{X} = \begin{bmatrix}
% \textbf{x}_1^T\\
% \textbf{x}_2^T\\
%  \dots\\
% \textbf{x}_n^T
% \end{bmatrix},
% \textbf{Y} = \begin{bmatrix}
%     y_1\\
%     y_2\\
%     \dots\\
%     y_n
%     \end{bmatrix}
% \end{equation}
% \vspace{-2mm}
\begin{equation}\label{eq: training_data_x}
\textbf{X} = \begin{bmatrix}
\textbf{x}_1, \textbf{x}_2, \dots, \textbf{x}_n
\end{bmatrix}^T
\textbf{Y} = \begin{bmatrix}y_1, y_2, \dots,y_n \end{bmatrix}^T
\end{equation}
% \vspace{-0.2in}
\eat{\begin{equation}\label{eq: training_data_x}
\textbf{X} = \begin{bmatrix}
\textbf{x}_1^T\\
\textbf{x}_2^T\\
 \dots\\
\textbf{x}_n^T
\end{bmatrix}=\begin{bmatrix}
x_{11} & x_{12} & \dots & x_{1m}\\
x_{21} & x_{22} & \dots & x_{2m}\\ 
\dots\\
x_{n1} & x_{n2} & \dots & x_{nm}
\end{bmatrix}
\end{equation}
\begin{equation}\label{eq: training_data_y}
\textbf{Y} = \begin{bmatrix}
    y_1\\
    y_2\\
    \dots\\
    y_n
    \end{bmatrix}
\end{equation}
}
For both linear and logistic regression we only focus on a common case: $L2-$regularization. The objective functions of linear regression, binary logistic regression and multinomial logistic regression with $L2-$regularization are presented in Equations \ref{eq: objective_function_linear_regression}-\ref{eq: objective_function_multi_logistic_regression} respectively \footnote{Here we assume that the two possible labels in binary logistic regression are 1 and -1.}
% by instantiating $g_i(\textbf{w})$ in Equation \ref{eq: object_function_1} with specific loss functions and $l(\textbf{w})$ with $L2-$regularization terms:
\vspace{-2mm}
\begin{align}
\begin{split}
h(\textbf{w})& = \frac{1}{n}\sum_{i=1}^n(y_i - \textbf{x}_i^T\textbf{w})^2 +\frac{\lambda}{2}||\textbf{w}||^2_2
%\\ &= \frac{1}{n}||\textbf{Y} - \textbf{X}\textbf{w}||^2_2 + \frac{\lambda}{2}||\textbf{w}||^2_2
\label{eq: objective_function_linear_regression}
\end{split}
\end{align}
\vspace{-2mm}
\begin{align}
h(\textbf{w})& = \frac{1}{n}\sum_{i=1}^n \ln (1+\exp\{-y_i\textbf{w}^\top\textbf{x}_i\}) + \frac{\lambda}{2}||\textbf{w}||^2_2\label{eq: objective_function_logistic_regression}
\end{align}
\vspace{-2mm}
\begin{align}
\begin{split}\label{eq: objective_function_multi_logistic_regression}
h(\textbf{w})& = \frac{1}{n}\sum_{k=1}^q \sum_{y_i=k}(\ln (\sum_{j=1}^q e^{\textbf{w}_j^\top\textbf{x}_i})-\textbf{w}_k^T\textbf{x}_i)+ \frac{\lambda}{2}||\textbf{w}||^2_2\\
&\textbf{w} = vec([\textbf{w}_1, \textbf{w}_2, \dots, \textbf{w}_q])
\end{split}
\end{align}
% \vspace{-0.05in}
where $\textbf{w}$ is the vector of model parameters and 
$\lambda$ is the \emph{regularization rate}.
For simplicity, we denote $\textbf{w} = vec([\textbf{w}_1, \textbf{w}_2,$\\ $\dots, \textbf{w}_q])$ for multinomial logistic regression where $q$ represents the number of possible classes.
%\scream{discuss this: what is $q$}
Typical learning methods for computing $\textbf{w}$ are to apply gradient descent (\gd) or its variant, stochastic gradient descent (\sgd) or mini-batch stochastic gradient method (\minisgd) \cite{robbins1951stochastic}
to minimize the objective function $h(\textbf{w})$ iteratively. \gd, \sgd\ and \minisgd\ \eat{mini-SGD is not so good. It suggests there is something "mini" or "minimized" about SGD!  I thing mb-SGD is better}\eat{done}are the same in nature since \minisgd\ can be regarded as a generalization of \gd\ and \sgd. They are therefore called Gradient-based method (\gbm) for short, and hereafter we will only take \minisgd\ as an example. Considering the similarities between binary logistic regression and multinomial logistic regression and the complexity of the computation related to the latter one, we will only present the formulas related to binary logistic regression below. All the theorems that hold for binary logistic regression can be also proven to be true for multinomial logistic regression.

\eat{\begin{algorithm}[h!] 
% \small
\footnotesize
 \SetKwInOut{Input}{Input}
 \SetKwInOut{Output}{Output}
 \Input{Model parameter $\textbf{w}$, objective function $h(\textbf{w})$, Training set $(\textbf{X}, \textbf{Y})$, mini-batch size $B$}

 \Output{Model parameter $\textbf{w}$}

\While{$\textbf{w}$ is not converged}{\tcc{outer loop: super-iteration}

Randomly shuffle $(\textbf{X}, \textbf{Y})$ and partition $(\textbf{X}, \textbf{Y})$ into mini-batches of size $B$

\For{each mini-batch $\mathscr{B}$}{
    use Equation \ref{eq: mini-sgd} to update $\textbf{w}$
}

\Return $\textbf{w}$
}
%  \Return $\mathcal{M}$
 \caption{Overview of \minisgd}
 \label{alg: mini-sgd}
 \end{algorithm}}

% \gd, \minisgd\ and \sgd\ go through the entire training set multiple times (called super-iteration hereafter). 

At each iteration, \minisgd\ updates the $\textbf{w}^{(t)}$ by using the average gradient of $h(\textbf{w})$ over a randomly selected mini-batch from the training dataset.
\eat{picks up a mini-batch (size $B$) of training samples from the randomly shuffled input training dataset ($\textbf{X}, \textbf{Y}$), which is followed by updating the model parameter $\textbf{w}$ by subtracting the average gradient of $h(\textbf{w})$ over all the samples within the mini-batch at each iteration. \eat{The overview of \minisgd\ is presented in Algorithm \ref{alg: mini-sgd}, which}It relies on the following formula to update the model parameter $\textbf{w}$ in Equation \ref{eq: object_function_1} using mini-batch $\mathscr{B}^{(t)}$ at the $t^{_{\mathrm{th}}}$ iteration:
\begin{align}\label{eq: mini-sgd}
\begin{split}
    % \textbf{w}^{(t+1)} &\leftarrow \textbf{w}^{(t)} - \eta_t*\triangledown (\frac{1}{B}\sum_{i=r^{(t)}}^{r^{(t)}+B-1}h_i(\textbf{w}^{(t)}) + \lambda \textbf{w}^{(t)})\\
    \textbf{w}^{(t+1)} &\leftarrow \textbf{w}^{(t)} - \eta_t*\triangledown (\frac{1}{B}\sum_{i \in \mathscr{B}^{(t)}}h_i(\textbf{w}^{(t)}) + \lambda \textbf{w}^{(t)})\\
\end{split}
\end{align}}
\eat{ and the batch of training samples used at $t_{th}$ iteration starts from $r^{(t)}_{th}$ sample (inclusive) to $(r^{(t)} + B)_{th}$ sample (exclusive). 
 
Note that $r^{(t)}$ is associated with a superscript $t$, which is because of the change of the mini-batches (and thus the ids of the samples used) at different iteration $t$.}
Specifically, for linear regression and logistic regression, the rule for updating $\textbf{w}^{(t)}$ under \minisgd\ is presented below (Equations \ref{eq: mini_sgd_linear_regression} and  \ref{eq: mini_sgd_logistic_regression} respectively):
\begin{align}
\textbf{w}^{(t+1)}& \leftarrow (1-\eta_t\lambda)\textbf{w}^{(t)} - \frac{2\eta_t}{B} \sum_{i \in \mathscr{B}^{(t)}} \textbf{x}_i(\textbf{x}_i^T\textbf{w}^{(t)} - y_i)\label{eq: mini_sgd_linear_regression}\\
        \begin{split}
        \textbf{w}^{(t+1)}& \leftarrow (1-\eta_t\lambda)\textbf{w}^{(t)}\\
        & + \frac{\eta_t}{B} \sum_{i \in \mathscr{B}^{(t)}} y_i\textbf{x}_i (1-\frac{1}{1+\exp\{-y_i\textbf{w}^{(t)T}\textbf{x}_i\}})
        \end{split}
        \label{eq: mini_sgd_logistic_regression}
        % \textbf{w}^{(t+1)}& \leftarrow (1-\eta_t\lambda)\textbf{w}^{(t)} - \frac{2\eta_t}{B} \sum_{i=r^{(t)}}^{r^{(t)}+B-1} \textbf{x}_i(\textbf{x}_i^T\textbf{w}^{(t)} - y_i)\label{eq: mini_sgd_linear_regression}\\
        % \begin{split}
        % \textbf{w}^{(t+1)}& \leftarrow (1-\eta_t\lambda)\textbf{w}^{(t)}\\
        % & + \frac{\eta_t}{B} \sum_{i=r^{(t)}}^{r^{(t)}+B-1} y_i\textbf{x}_i (1-\frac{1}{1+\exp\{-y_i\textbf{w}^{(t)T}\textbf{x}_i\}})
        % \end{split}
        % \label{eq: mini_sgd_logistic_regression}
        % \begin{split}
% \textbf{w}^{(t+1)}&\leftarrow (1-\eta_t\lambda)\textbf{w}^{(t)} - vec([\frac{1}{B} (\sum_{i=r}^{r+B-1} \frac{\textbf{x}_{i}e^{\textbf{w}_1^{(t)T}\textbf{x}_i}}{\sum_{j=1}^q e^{\textbf{w}_j^{(t)T}\textbf{x}_i}} - \sum_{y_i=1}\textbf{x}_{i}),\\
% &\frac{1}{B} (\sum_{i=r}^{r+B-1} \frac{\textbf{x}_{i}e^{\textbf{w}_2^{(t)T}\textbf{x}_i}}{\sum_{j=1}^q e^{\textbf{w}_j^{(t)T}\textbf{x}_i}} - \sum_{y_i=2}\textbf{x}_{i}), \\
%     &\dots, \frac{1}{B} (\sum_{i=r}^{r+B-1} \frac{\textbf{x}_{i}e^{\textbf{w}_q^{(t)T}\textbf{x}_i}}{\sum_{j=1}^q e^{\textbf{w}_j^{(t)T}\textbf{x}_i}} - \sum_{y_i=q}\textbf{x}_{i})])\label{eq: mini_sgd_multi_logistic_regression}\\
    % \end{split}
\end{align}
where $\eta_t$ is called the {\em learning rate} and $\mathscr{B}^{(t)}$ represents a mini-batch of $B$ training samples. For \sgd, 
$\mathscr{B}^{(t)}$ includes only one sample ($B=1$), while for \gd, 
$\mathscr{B}^{(t)}$ includes all the training samples ($B = n$).

\eat{\subsection{Piecewise linear interpolation}

\scream{Do we need this subsection with generalities? Do we use anyting unusual?}

Due to the existence of  non-linear operations in the update rule of logistic regression (see the terms in the second line of Equation \ref{eq: mini_sgd_logistic_regression}), it is impossible to apply existing provenance models such as the provenance semiring model \cite{green2007provenance} and its extension in \cite{yan2016fine} since they only support linear operations. We therefore propose to use {\em Piecewise linear interpolation} as the first step to enabling provenance support for the update rules of logistic regression. Details on how to apply Piecewise linear interpolation for the update rules of logistic regression will be delayed until Section \ref{sec: model}; here we present some basic facts.% about Piecewise linear interpolation.

\eat{The first obstacle toward tracking provenance for machine learning training process is on how to deal with the non-linear operations. Specifically, the logistic function and softmax function appear in the update rule of binary logistic regression and multinomial logistic regression (See Equation \ref{eq: mini_sgd_instantiation}), which are obviously non-linear operations. Since in traditional provenance models such as provenance semiring model \cite{green2007provenance} and its extension in \cite{yan2016fine}, provenance only supports linear operations (i.e. multiplications and additions), it is challenging to extend existing provenance model for those non-linear operations.}

% put it into preliminary
In Piecewise linear interpolation~\cite{Kress1998}, we assume that the function to be approximated is a continuous function $f(x)$ where $x \in [a, b]$. Piecewise linear interpolation starts by picking up a series of {\em breaking points}, $x_i$ such that $a < x_1 < x_2 < \dots < x_p < b$ and then constructs a linear interpolant $s(x)$ over each interval $[x_{j-1}, x_{j})$ as follows:
\begin{align}\label{eq: piecewise_interpolant}
\begin{split}
    s(x) &= \frac{x-x_{j-1}}{x_j-x_{j-1}}f(x_j) + \frac{x_j-x}{x_j-x_{j-1}}f(x_{j-1})\\
    & = a_jx + b_j, x \in [x_{j-1}, x_j)
\end{split}
\end{align}

The following property holds on how close the value of $s(x)$ is compared to the original function $f(x)$:
\begin{align}\label{eq: piecewise_approx_rate}
    \begin{split}
        |f(x) - s(x)| &\leq \frac{1}{8}(\Delta x)^2 \max_{a\leq x \leq b}|f''(x)| = M_1 (\Delta x)^2\\
        |f'(x) - s'(x)| &\leq \frac{1}{2}(\Delta x) \max_{a\leq x \leq b}|f''(x)| = M_2 (\Delta x)
    \end{split}
\end{align}
\noindent
where $f'$ and $f''$ represent the first order and second order derivative of the function $f$ respectively, $\Delta x$ represents the length of the longest interval $[x_{j-1}, x_j)$, i.e. $max(x_j - x_{j-1})_{j=1,2,\dots, p-1}$, and $M_1$ and $M_2$ are constants. In terms of $s'(x)$, it can be derived as below:
\begin{align}
\begin{split}
    s'(x) &= a_j, x \in [x_{j-1}, x_j)
\end{split}
\end{align}
}

%tensor product
%provenance model

\section{Iteration models}\label{sec: model}
In this section, we first discuss the annotation with 
provenance of the gradient-bases update rules in our approach.
%of linear regression and logistic regression.  
%We start by reviewing the provenance semiring model for linear matrix operations proposed in \cite{yan2016fine}, and 
Next, we discuss for the non-linear operations in logistic regression the linearization that makes our provenance 
annotation framework usable. Finally, we give a rigorous theoretical analysis of the convergence of the iterative process with provenance-annotated
update rules for both linear and logistic regression model
\scream{added the following sentence}and the similarity to the expected results after linearization for logistic regression models.
%to show that the approximated updated models obtained by linearization are guaranteed to converge, and that the limits are close to the expected results. 
 \eat{In the end, we will discuss about the potential extensions to more complicated machine learning models, in particular, deep neural networks.}

%%%%%%%NEW SUBSECTION 4.1

\subsection{Provenance annotations for matrices}
\label{ssec: provenance_annotation}

In the semiring 
framework~\cite{green2007provenance,amsterdamer2011provenance,GreenT17}
one begins by annotating input data with elements of a set $T$ of 
\emph{provenance tokens}. These annotations are then propagated through 
query operators as they combine according to 
two operations: ``$+$'' that records \emph{alternative} use of information, 
as in relational union or projection, and  ``$\cdot$'', that records 
\emph{joint} use of information, as in relational join. With these, the 
annotations become  \emph{provenance polynomials} whose indeterminates are tokens and with coefficients in $\mathbb{N}$.
For example, the monomial $p^2q$ is the provenance 
of a result for which the data item annotated $p$ was used
\emph{twice} together with the item annotated $q$ used 
once. We denote the set of polynomials by $\mathbb{N}[T]$.

In the extension of the framework to matrix algebra~\cite{yan2016fine},
annotation formally becomes a 
\emph{multiplication of vectors with scalars} as in linear
algebra. The role of scalars is played by provenance polynomials and 
the role of vectors, of course, is played by matrices (generalizing their
row vectors and the transposes of these).

\reminder{Not sure if we want to include this: This further generalizes the
use of annotation as scalar multiplication in the work on provenance for 
aggregation~\cite{amsterdamer2011provenance}.}

\emph{Matrices annotated with provenance polynomials} form a nice
algebraic structure that extends matrix multiplication and
addition. We denote multiplication with scalars by ``$*$'' writing
$\mathfrak{p}*\textbf{A}$ for the matrix $\textbf{A}$ annotated with the
provenance polynomial $\mathfrak{p}$. For space reasons we cannot repeat
here the technical development in~\cite{yan2016fine}, however, we mention
a crucial algebraic property of annotated matrix multiplication, which
also illustrates combining provenance in joint use:
$$
(\mathfrak{p}_1 * \textbf{A}_1)(\mathfrak{p}_2 * \textbf{A}_2) = 
(\mathfrak{p}_1 \cdot \mathfrak{p}_2)*(\textbf{A}_1\textbf{A}_2)
$$

We apply this framework to tracking input training samples through
\gbm's in which the update involves only matrix multiplication and
addition. Let the training dataset be $(\textbf{X},\textbf{Y})$ where 
$\textbf{X}$ is an $n\times m$ feature matrix and $\textbf{Y}$ is
an $n\times 1$ column vector of sample labels.  For $i=1,\ldots, n$,
we annotate every sample
$(\textbf{x}_i,y_i)$ ($\textbf{x}_i$ and $[y_i]$ are the $i$'th rows
in $\textbf{X}$ respectively $\textbf{Y}$) with a distinct provenance 
token $p_i$. Next, we decompose $\textbf{X}$ and $\textbf{Y}$ as algebraic
expressions in terms of $p_1*\textbf{x}_1, \ldots, p_n*\textbf{x}_n, 
p_1*[y_1], \ldots, p_n*[y_n]$ and some matrices made up of the reals 0 and 1.
These ``helper'' matrices are annotated with the provenance polynomial 
%$1_{\mathbf{N}[X]}$ 
$\oneprov\in\mathbb{N}[T]$
(has only a term of degree zero which is the natural number 1) meaning
``always available, no need to track''.
We illustrate with the provenance-annotated $\textbf{X}$ when $n=2$:
\vspace{-1mm}
$$
\textbf{X} ~=~ (\oneprov * \begin{bmatrix}
                                 1\\
                                 0
                                 \end{bmatrix})
           (p_1*\textbf{x}_1) ~+~ 
               (\oneprov * \begin{bmatrix}
                                 0\\
                                 1\\
                                 \end{bmatrix}) 
          (p_2*\textbf{x}_2) ~=~
$$
\eat{
$$
\textbf{X} ~=~ (1_{\mathbf{N}[T]} * \begin{bmatrix}
                                 1\\
                                 0
                                 \end{bmatrix})
           (p_1*\textbf{x}_1) ~+~ 
               (1_{\mathbf{N}[T]} * \begin{bmatrix}
                                 0\\
                                 1\\
                                 \end{bmatrix} 
          (p_2*\textbff{x}_2) ~=~
$$
}
$$
~=~ (p_1 *\begin{bmatrix}
         \textbf{x}_1\\
         0\ldots0
         \end{bmatrix}) ~+~
    (p_2 *\begin{bmatrix}
         0\ldots0\\                
            \textbf{x}_2
         \end{bmatrix}) 
$$

When $\textbf{X}$ is transposed, a similar decomposition applies
in terms of the annotated column vectors $p_i*\textbf{x}_i^T$.
\reminder{Skipping this: The more general annotation technique 
of~\cite{yan2016fine} can annotate columns as well} 
We also note that the algebra of annotated matrices follows the
same laws as the usual matrix algebra. Consequently, we can perform
in the algebra of provenance-annotated matrices the calculations involved 
in the gradient-based update rules. For illustration, a calculation involving
$\textbf{X}$ that without
provenance takes the form $\sum_{i=1}^n \alpha_i\textbf{x}_i\textbf{x}_i^T$ 
(where $\alpha_i$ are some real numbers) becomes with provenance annotations
% \vspace{-2mm}
$$
%\sum_{i=1}^n (1_{\textbf{N}[X]}*[\alpha_i])(p_i*\textbf{x}_i)(p_i*\textbf{x}_i^T)
\sum_{i=1}^n (\oneprov *[\alpha_i])(p_i*\textbf{x}_i)(p_i*\textbf{x}_i^T)
~=~ 
\sum_{i=1}^n p_i^2*(\alpha_i\textbf{x}_i\textbf{x}_i^T).
$$

%for the input training dataset ($\textbf{X}, \textbf{Y}$), we only annotate every row (i.e. every sample) of $\textbf{X}$ and $\textbf{Y}$ with provenance tokens without annotating the columns, which is
%is a simplified version of the provenance model of \cite{yan2016fine} 
%but is enough for the purpose of incremental updates to the training of machine learning models. We also assume that for every sample ($\textbf{x}_i, \textbf{y}_i$), the provenance tokens $p_i$
%are same for both $\textbf{x}_i$ and $\textbf{y}_i$, which results in the following provenance annotated training dataset:

\eat{
\begin{align}\label{eq: training_data_x_with_provenance}
    \begin{split}
    \textbf{X} &= \begin{bmatrix}
p_1 * \textbf{x}_1, p_2 * \textbf{x}_2, \dots, p_n * \textbf{x}_n
\end{bmatrix}^T\\
\textbf{Y} &= \begin{bmatrix}
p_1 * y_1, p_2 * y_2, \dots, p_n * y_n \end{bmatrix}^T 
    \end{split}
\end{align}
% \begin{equation}\label{eq: training_data_x_with_provenance}
% \textbf{X} = \begin{bmatrix}
% p_1 * \textbf{x}_1^T\\
% p_2 * \textbf{x}_2^T\\
%  \dots\\
% p_n * \textbf{x}_n^T
% \end{bmatrix}
% \textbf{Y} = \begin{bmatrix}
% p_1 * y_1\\
% p_2 * y_2\\
%  \dots\\
% p_n * y_n
% \end{bmatrix}
% \end{equation}

%tensor product, semi-module
% Based on the ideas from \cite{yan2016fine}, we can further define provenance token propagation rules through the vector operations since all the iterative update rules in Equation \ref{eq: mini_sgd_instantiation_approx} only involve vector addition and multiplication
}

\eat{
Based on the annotations for $\textbf{X}$ and $\textbf{Y}$, the provenance expression for typical matrix operations can be derived, which is exemplified as below:
\scream{Yinjun: proposed one example on how to use semiring model below}
\begin{example}
The provenance expression of $\sum_{i=1}^n a_i\textbf{x}_i\textbf{x}_i^T$ and $\sum_{i=1}^n b_i\textbf{x}_i y_i$ is:
\begin{align}\label{eq: example_provenance}
    \begin{split}
        \mathscr{P}(\sum_{i=1}^n a_i\textbf{x}_i\textbf{x}_i^T) &= \sum_{i=1}^n p_i^2*(a_i\textbf{x}_i\textbf{x}_i^T)
    \end{split}
\end{align}
\begin{align}\label{eq: example_provenance2}
    \begin{split}
        \mathscr{P}(\sum_{i=1}^n b_i\textbf{x}_iy_i) &= \sum_{i=1}^n p_i^2*(b_i\textbf{x}_i y_i)
    \end{split}
\end{align}

where $\mathscr{P}(*)$ represents the provenance expression of some matrix expression. Since each $a_i\textbf{x}_i\textbf{x}_i^T$ is the multiplications between $a_i$, $\textbf{x}_i$ and $\textbf{x}_i^T$, which can be expressed as tensor products $1_k*a_i$ ($a_i$ is a constant), $p_i*\textbf{x}_i$ and $p_i*\textbf{x}_i^T$ as the consequence of provenance annotations, then Equation \ref{eq: example_provenance} can be regarded as the result of applying the rule (\ref{eq: associativity}), i.e.  $(1_k*a_i)(p_i*\textbf{x}_i)(p_i*\textbf{x}_i^T) = (1_k\cdot_K p_i \cdot_K p_i) (a_i \textbf{x}_i\textbf{x}_i^T) = p_i^2*(a_i \textbf{x}_i\textbf{x}_i^T)$. Equation \ref{eq: example_provenance2} can be derived in a similar way.
\end{example}
}

\eat{
Since matrices can be embedded in $K \bigotimes \mathcal{M}$, the update rules under \gbm\ 
correspond to update rules in $K \bigotimes \mathcal{M}$ involving matrices "with provenance".
}

And here is the provenance-annotated expression for the update rule 
of linear regression (i.e. Equation~\ref{eq: mini_sgd_linear_regression}):

\topskip=1pt
% \vspace{-2mm}
{
\begin{align}\label{eq: mini_sgd_linear_regression_provenance}
    \begin{split}
        &\mathcal{W}^{(t+1)} \leftarrow [(1-\eta_t\lambda)(1_k * \eat{\cdot} \textbf{I}) \\ &-\frac{2\eta_t}{\mathcal{P}^{(t)}} \sum_{i \in \mathscr{B}^{(t)}} p_i^2 * \eat{\cdot}
        \textbf{x}_i\textbf{x}_i^T] \mathcal{W}^{(t)} + \frac{2\eta_t}{\mathcal{P}^{(t)}} \sum_{i\in \mathscr{B}^{(t)}}p_i^2 * \eat{\cdot} 
        \textbf{x}_iy_i
    \end{split}
\end{align}
}
% Ideally, with the iteration proceeds, the provenance for the model parameter can be also computed iteratively. To conduct incremental updates by removing the samples of interest (say $\textbf{x}_i$), we can simply setting $p_i$ as 0 and get the updated 
%apply idempodence over *_K, map the simplified version of provenance expression to the incremental updated model parameters.

\noindent
where $\mathcal{W}^{(t)}$ represents the provenance-annotated expression for the
vector $\textbf{w}^{(t)}$ of model parameters while $\mathcal{P}^{(t)}$ represents a provenance-annotated expression for the number of samples in the min-batch $\mathscr{B}^{(t)}$,
%at the $t^{\mathrm{th}}$ iteration, 
for example, following the approach to aggregation in~\cite{amsterdamer2011provenance},
$\mathcal{P}^{(t)} = \sum_{i \in \mathscr{B}^{(t)}} p_i * \eat{\cdot} 1$.
%\val{Three kinds of B may be too much. Consider using $\mathcal{P}^{(t)}$ instead of $\mathcal{B}^{(t)}$.}
%\yinjun{done}

%\val{we need to explain what is division by a provenance expression 
%since $\mathcal{B}$ appears in a denominator}. 
In the semiring framework there is no division operation so we used fractions with
denominator $\mathcal{P}^{(t)}$ in Equation~\ref{eq: mini_sgd_linear_regression_provenance}
only for notational purposes. As we shall see immediately below, in incremental
update $\mathcal{P}^{(t)}$ can be replaced with an integer.

As with the other applications of the semiring framework, deletion propagation 
is done by "zeroing-out" the deleted samples. That is, if sample $i$ is deleted
we set the corresponding provenance token 
%$p_i=0_{\mathbb{N}[X]}$. 
$p_i=\zeroprov\in\mathbb{N}[T]$ (has only a term of degree zero which is the natural number 0).
The challenge,
as detailed in the following section is how to do this efficiently throughout
the gradient descent.

For the samples that remain we obtain (after we stop the iterations)
a provenance-annotated expression that can be put in the form
$\mathcal{W} = \sum\mathfrak{m}_k*\mathbf{u}_k$ where $\mathfrak{m}_k$ is a
\emph{monomial} in the provenance tokens and each $\mathbf{u}_k$ is a vector
of contributions to the model parameters. 
\reminder{Further analysis of the contributions
of the remaining samples is possible here eg using the expoenents}
To get the updated vector of model parameters we set each remaining provenance
token to 
%$1_{\mathbb{N}[X]}$ 
$\oneprov$ obtaining 
$\mathbf{w}^{^{\mathrm{upd}}} = \sum\mathbf{u}_k$.
%\scream{shouldn't it be $\mathbf{w}^{^{\mathrm{upd}}} = \sum\mathbf{u}_k$? YES}
And, as promised, we notice that when all the provenance tokens are set to 
$\zeroprov$ or $\oneprov$ the provenance expression $\mathcal{P}^{(t)}$
comes down to an integer.
%\scream{the text that is below is hard to understand I would like to delete it
%because I think that what I wrote above is enough.
%Except that Equation\ref{eq: mini_sgd_linear_regression_provenance_update}
%is heavily referenced later! Yinjun: Done. Can you take a look? Val: Looks fine}
Denoting this integer by $\increB^{(t)}$ and
denoting the set of the indexes of the removed training samples by $\mathcal{R}$,
the provenance-annotated update rule for $\mathcal{W}^{(t+1)}$ becomes:
% \vspace{-1mm}
{
\begin{align}\label{eq: mini_sgd_linear_regression_provenance_update}
    \begin{split}
        &\increprov^{(t+1)} \leftarrow [(1-\eta_t\lambda)(\oneprov * \eat{\cdot} \textbf{I}) \\
        &-\frac{2\eta_t}{\increB^{(t)}} \sum_{\substack{ i \in \mathscr{B}^{(t)} \\, i \not \in \mathcal{R}}} p_i^2 * \eat{\cdot} \textbf{x}_i\textbf{x}_i^T] \increprov^{(t)}+ \frac{2\eta_t}{\increB^{(t)}} \sum_{\substack{ i \in \mathscr{B}^{(t)} \\, i \not \in \mathcal{R}}}p_i^2 * \eat{\cdot} \textbf{x}_iy_i
    \end{split}
\end{align}
}

%\scream{Val: is there a typo with $y\in\Delta Y$? But I don't think we should show this.RESOLVED}

\eat{for provenance expressions,
although a division-like symbol has been used in Equation~\ref{eq: mini_sgd_linear_regression_provenance}
for notational purposes. In fact, 
%in the context of the incremental update problem, 
once the subset of training samples (say ($\Delta \textbf{X}, \Delta \textbf{Y}$)) 
\scream{discuss this: maybe we don't need this notation}
to be removed is determined, the corresponding provenance tokens will be set to
$\zeroprov$ and the remaining ones to $\oneprov$ so 
% (say all the provenance tokens except $\{p_{i_1}, p_{i_2}, \dots, p_{i_z}\}$ are zeros due to the updates), 
%and thus the number of non-zero tensor products in 
$\mathcal{P}^{(t)}$ becomes the number of the remaining samples in the batch.
%the size of the remaining samples in each min-batch 
%after the removal operation 
(denoted by $\increB^{(t)}$ \scream{discuss this: maybe we do not need this}). So we can replace $\mathcal{P}^{(t)}$ with $\increB^{(t)}$ in the update rule that
gives $\mathcal{W}^{(t+1)}$, i.e.:}

\eat{Now suppose $\increprov^{(0)} = 1_K * \eat{\cdot} \textbf{w}^{(0)}$ \eat{represent number of non-zero provenance tokens in current batch}, then by induction we can observe that $\increprov^{(t)}$ should be the sum of multiple tensor products, each of which should be in the form of $P*\textbf{u}$ where $\textbf{u}$ is a matrix associated with a provenance monomial $P$ ($P= p_{j_1}^{r_1}p_{j_2}^{j_2}p_{j_3}^{r_3}\dots p_{j_k}^{r_k}$ and $r_i (i=1,2,\dots,k)$ is a natural number). Intuitively speaking, $P*\textbf{u}$ represents how the combinations of $r_1$ copies of ${j_1}^{\mathrm{th}}$ sample, $r_2$ copies of ${j_2}^{\mathrm{th}}$ sample, \dots and $r_k$ copies of ${j_k}^{\mathrm{th}}$ contribute to the model parameter where ${j_1}^{\mathrm{th}}, {j_2}^{\mathrm{th}}\dots, {j_k}^{\mathrm{th}}$ sample should not be in $\Delta \textbf{X}$.

Then by setting $p_i$ as 1, we can get the update rule to compute the updated model parameters as below:
\begin{align}\label{eq: mini_sgd_linear_regression_para_update}
    \begin{split}
        \increw^{(t+1)}& \leftarrow [(1-\eta_t\lambda)\textbf{I} -\frac{2\eta_t}{\increB^{(t)}} \sum_{\substack{ i \in \mathscr{B}^{(t)} \\, \textbf{x}_i \not\in \Delta \textbf{X}}} \eat{\cdot} \textbf{x}_i\textbf{x}_i^T] \increw^{(t)}\\
        &+ \frac{2\eta_t}{\increB^{(t)}} \sum_{\substack{ i \in \mathscr{B}^{(t)} \\, \textbf{x}_i \not\in \Delta \textbf{X}, y_i \not\in \Delta \textbf{Y}}}\eat{\cdot} \textbf{x}_iy_i
    \end{split}
\end{align}

which simply sum up the matrices from each tensor product on the right-hand side of Equation \ref{eq: mini_sgd_linear_regression_provenance_update}.}

\eat{. are updated iteratively as the iteration proceeds. To conduct incremental deletion over the model parameters $\textbf{w}^{(t)}$ by removing the samples of interest (say $\textbf{x}_{i_1}, \textbf{x}_{i_2}, \dots, \textbf{x}_{i_z}$), we can simply setting $p_{i_1}, p_{i_2}, \dots, p_{i_z}$ as $0_K$ to eliminate the effect of those samples and setting other provenance tokens as $1_K$ at the same time.}

\eat{
Unlike the more general annotation technique of~\cite{yan2016fine}, for the input training dataset ($\textbf{X}, \textbf{Y}$), we only annotate every row (i.e. every sample) of $\textbf{X}$ and $\textbf{Y}$ with provenance tokens without annotating the columns, which is
%is a simplified version of the provenance model of \cite{yan2016fine} but is 
enough for the purpose of incremental updates to the training of machine learning models. We also assume that for every sample ($\textbf{x}_i, \textbf{y}_i$), the provenance tokens $p_i$~\cite{green2007provenance,yan2016fine}
are same for both $\textbf{x}_i$ and $\textbf{y}_i$, which results in the following provenance annotated training dataset:

\begin{equation}\label{eq: training_data_x_with_provenance}
\textbf{X} = \begin{bmatrix}
p_1 * \textbf{x}_1^T\\
p_2 * \textbf{x}_2^T\\
 \dots\\
p_n * \textbf{x}_n^T
\end{bmatrix}
\textbf{Y} = \begin{bmatrix}
p_1 * y_1\\
p_2 * y_2\\
 \dots\\
p_n * y_n
\end{bmatrix}
\end{equation}

%tensor product, semi-module
% Based on the ideas from \cite{yan2016fine}, we can further define provenance token propagation rules through the vector operations since all the iterative update rules in Equation \ref{eq: mini_sgd_instantiation_approx} only involve vector addition and multiplication. 

Since matrices can be embedded in $K \bigotimes \mathcal{M}$, the update rules under \gbm\ 
correspond to update rules in $K \bigotimes \mathcal{M}$ involving matrices "with provenance". 
For example, after annotating 
each training sample in $(\textbf{X}, \textbf{Y})$ with provenance tokens as above for the update rule of linear regression, Equation~\ref{eq: mini_sgd_linear_regression} becomes:
\begin{align}\label{eq: mini_sgd_linear_regression_provenance}
    \begin{split}
        \mathcal{W}^{(t+1)}& \leftarrow ((1-\eta_t\lambda)(1_k * \eat{\cdot} \textbf{I}) -\frac{2\eta_t}{\mathcal{P}^{(t)}} \sum_{i \in \mathscr{B}^{(t)}} p_i^2 * \eat{\cdot}
        \textbf{x}_i\textbf{x}_i^T) \mathcal{W}^{(t)}\\
        &+ \frac{2\eta_t}{\mathcal{P}^{(t)}} \sum_{i\in \mathscr{B}^{(t)}}p_i^2 * \eat{\cdot} 
        \textbf{x}_iy_i\\
    \end{split}
\end{align}
% Ideally, with the iteration proceeds, the provenance for the model parameter can be also computed iteratively. To conduct incremental updates by removing the samples of interest (say $\textbf{x}_i$), we can simply setting $p_i$ as 0 and get the updated 
%apply idempodence over *_K, map the simplified version of provenance expression to the incremental updated model parameters.
\noindent
where $\mathcal{W}^{(t)}$ represents the provenance-annotated expression for model parameter $\textbf{w}^{(t)}$ while $\mathcal{P}^{(t)}$ represents the provenance-annotated expression
for the batch number $B$ at the $t^{\mathrm{th}}$ iteration, i.e. $\mathcal{P}^{(t)} = \sum_{i \in \mathscr{B}^{(t)}} p_i * \eat{\cdot} 1$.
%\val{Three kinds of B may be too much. Consider using $\mathcal{P}^{(t)}$ instead of $\mathcal{B}^{(t)}$.}
%\yinjun{done}

%\val{we need to explain what is division by a provenance expression 
%since $\mathcal{B}$ appears in a denominator}. 
In the provenance semiring framework there is no division operation for provenance expressions,
although a division-like symbol has been used in Equation~\ref{eq: mini_sgd_linear_regression_provenance}
for notational purposes. In fact, in the context of the incremental update problem, once the subset of training samples to be removed is determined, the corresponding provenance tokens will be set as $0_{K}$ (say all the provenance tokens except $\{p_{i_1}, p_{i_2}, \dots, p_{i_z}\}$ are zeros due to the updates), and thus the number of non-zero tensor products in $\mathcal{P}^{(t)}$ can be determined, which also represents the size of the remaining samples in each batch 
%after the removal operation 
(denoted by $\increB^{(t)}$). So we can replace $\mathcal{P}^{(t)}$ with $\increB^{(t)}$ in the update rule that
gives $\mathcal{W}^{(t+1)}$, i.e.:

\begin{align}\label{eq: mini_sgd_linear_regression_provenance_update}
    \begin{split}
        \increprov^{(t+1)}& \leftarrow ((1-\eta_t\lambda)(1_K * \eat{\cdot} \textbf{I}) \\
        &-\frac{2\eta_t}{\increB^{(t)}} \sum_{\substack{ i \in \mathscr{B}^{(t)} \\ i \in \{{i_1}, {i_2}, \dots, {i_z}\}}} p_i^2 * \eat{\cdot} \textbf{x}_i\textbf{x}_i^T) \increprov^{(t)}\\
        &+ \frac{2\eta_t}{\increB^{(t)}} \sum_{\substack{ i \in \mathscr{B}^{(t)} \\ i \in \{{i_1}, {i_2}, \dots, {i_z}\}}}p_i^2 * \eat{\cdot} \textbf{x}_iy_i
    \end{split}
\end{align}

Now suppose $\increprov^{(0)} = 1_K * \eat{\cdot} \textbf{w}^{(0)}$ \eat{represent number of non-zero provenance tokens in current batch}, then by induction we can observe that $\increprov^{(t)}$ should be the sum of multiple tensor products, each of which should be in the form of $P*\textbf{u}$ where $\textbf{u}$ is a matrix associated with a provenance monomial $P$ ($P= p_{j_1}^{r_1}p_{j_2}^{j_2}p_{j_3}^{r_3}\dots p_{j_k}^{r_k}$ and $p_{j_t} \in \{p_{i_1}, p_{i_2}, \dots, p_{i_z}\}$ and $r_i (i=1,2,\dots,k)$ is a natural number). Intuitively speaking, $P*\textbf{u}$ represents how the combinations of $r_1$ copies of ${j_1}^{\mathrm{th}}$ sample, $r_2$ copies of ${j_2}^{\mathrm{th}}$ sample, \dots and $r_k$ copies of ${j_k}^{\mathrm{th}}$ contribute to the model parameter.
}

\eat{. are updated iteratively as the iteration proceeds. To conduct incremental deletion over the model parameters $\textbf{w}^{(t)}$ by removing the samples of interest (say $\textbf{x}_{i_1}, \textbf{x}_{i_2}, \dots, \textbf{x}_{i_z}$), we can simply setting $p_{i_1}, p_{i_2}, \dots, p_{i_z}$ as $0_K$ to eliminate the effect of those samples and setting other provenance tokens as $1_K$ at the same time.}

%%%%%%%%% END OF NEW SECTION 4.1

\eat{
%%%%%% THIS IS THE OLD SUBSECTION 4.1
\subsection{Provenance annotations and semantics}\label{ssec: provenance_annotation}
We follow the ideas of the provenance model in \cite{yan2016fine}, which is based on a
generalization of the standard construction of {\em tensor product}.

Let $(\mathcal{M}, +, \cdot, 0, I)$ be the {\em many-sorted} ring of matrices of various dimensions
and let $(K, +_K, \cdot_K, 0_K, 1_K)$ be a commutative semiring (we only use the semiring of provenance polynomials~\cite{green2007provenance}). We denote by $K \bigotimes \mathcal{M}$ the generalized
tensor product of $K$ and $\mathcal{M}$~\cite{yan2016fine}. 
The construction yields a (non-commutative) ring $(K \bigotimes \mathcal{M},+\cdot,0,I)$ together with a multiplication with scalars from $K$
that we denote by $*$  such that $\forall k, k_1, k_2 \in K, A_1, A_2, A_3 \in K \bigotimes \mathcal{M}$, the following equalities hold:
\begin{enumerate}
    \item $k * (A_1 + A_2) = k* A_1 + k * A_2$
    \item $k * 0 = 0$
    \item $(k_1 +_K k_2) * A = k_1* A + k_2* A$
    \item $0_K * A = 0$
    \item $(k_1 \cdot_K k_2) * A = k_1 * (k_2 * A)$
    \item $1_K * A = A$
    \item $(k_1 * A_1)(k_2 * A_2) = (k_1 \cdot_K k_2)*(A_1A_2)$ \label{eq: associativity}
\end{enumerate}
The basic elements of $K \bigotimes \mathcal{M}$ have (modulo the standard equivalences used 
in the tensor product construction) the form $k * M$ where $k\in K$ is,
for example, a provenance polynomial, and $M$ is a matrix (including $m\times1$, $1\times n$ and even $1\times 1$ matrices). Note that the ring of matrices $\mathcal{M}$ is \emph{embedded} into
$K \bigotimes \mathcal{M}$ as every matrix $M$ can be seen as $1_K * M \in K \bigotimes \mathcal{M}$.

%provenance tokens for each row of the input, leading to the provenance tokens for the entire matrix or vector

Unlike the more general annotation technique of~\cite{yan2016fine}, for the input training dataset ($\textbf{X}, \textbf{Y}$), we only annotate every row (i.e. every sample) of $\textbf{X}$ and $\textbf{Y}$ with provenance tokens without annotating the columns, which is
%is a simplified version of the provenance model of \cite{yan2016fine} but is 
enough for the purpose of incremental updates to the training of machine learning models. We also assume that for every sample ($\textbf{x}_i, \textbf{y}_i$), the provenance tokens $p_i$~\cite{green2007provenance,yan2016fine}
are same for both $\textbf{x}_i$ and $\textbf{y}_i$, which results in the following provenance annotated training dataset:
\begin{align}\label{eq: training_data_x_with_provenance}
    \begin{split}
    \textbf{X} &= \begin{bmatrix}
p_1 * \textbf{x}_1, p_2 * \textbf{x}_2, \dots, p_n * \textbf{x}_n
\end{bmatrix}^T\\
\textbf{Y} &= \begin{bmatrix}
p_1 * y_1, p_2 * y_2, \dots, p_n * y_n \end{bmatrix}^T 
    \end{split}
\end{align}
% \begin{equation}\label{eq: training_data_x_with_provenance}
% \textbf{X} = \begin{bmatrix}
% p_1 * \textbf{x}_1^T\\
% p_2 * \textbf{x}_2^T\\
%  \dots\\
% p_n * \textbf{x}_n^T
% \end{bmatrix}
% \textbf{Y} = \begin{bmatrix}
% p_1 * y_1\\
% p_2 * y_2\\
%  \dots\\
% p_n * y_n
% \end{bmatrix}
% \end{equation}

%tensor product, semi-module
% Based on the ideas from \cite{yan2016fine}, we can further define provenance token propagation rules through the vector operations since all the iterative update rules in Equation \ref{eq: mini_sgd_instantiation_approx} only involve vector addition and multiplication. 
Based on the annotations for $\textbf{X}$ and $\textbf{Y}$, the provenance expression for typical matrix operations can be derived, which is exemplified as below:
\scream{Yinjun: proposed one example on how to use semiring model below}
\begin{example}
The provenance expression of $\sum_{i=1}^n a_i\textbf{x}_i\textbf{x}_i^T$ and $\sum_{i=1}^n b_i\textbf{x}_i y_i$ is:
\begin{align}\label{eq: example_provenance}
    \begin{split}
        \mathscr{P}(\sum_{i=1}^n a_i\textbf{x}_i\textbf{x}_i^T) &= \sum_{i=1}^n p_i^2*(a_i\textbf{x}_i\textbf{x}_i^T)
    \end{split}
\end{align}
\begin{align}\label{eq: example_provenance2}
    \begin{split}
        \mathscr{P}(\sum_{i=1}^n b_i\textbf{x}_iy_i) &= \sum_{i=1}^n p_i^2*(b_i\textbf{x}_i y_i)
    \end{split}
\end{align}

where $\mathscr{P}(*)$ represents the provenance expression of some matrix expression. Since each $a_i\textbf{x}_i\textbf{x}_i^T$ is the multiplications between $a_i$, $\textbf{x}_i$ and $\textbf{x}_i^T$, which can be expressed as tensor products $1_k*a_i$ ($a_i$ is a constant), $p_i*\textbf{x}_i$ and $p_i*\textbf{x}_i^T$ as the consequence of provenance annotations, then Equation \ref{eq: example_provenance} can be regarded as the result of applying the rule (\ref{eq: associativity}), i.e.  $(1_k*a_i)(p_i*\textbf{x}_i)(p_i*\textbf{x}_i^T) = (1_k\cdot_K p_i \cdot_K p_i) (a_i \textbf{x}_i\textbf{x}_i^T) = p_i^2*(a_i \textbf{x}_i\textbf{x}_i^T)$. Equation \ref{eq: example_provenance2} can be derived in a similar way.

\end{example}

Since matrices can be embedded in $K \bigotimes \mathcal{M}$, the update rules under \gbm\ 
correspond to update rules in $K \bigotimes \mathcal{M}$ involving matrices "with provenance". 
Similarly, the provenance expression for the update rule of linear regression (i.e. Equation~\ref{eq: mini_sgd_linear_regression}) becomes:
\begin{align}\label{eq: mini_sgd_linear_regression_provenance}
    \begin{split}
        &\mathcal{W}^{(t+1)} \leftarrow [(1-\eta_t\lambda)(1_k * \eat{\cdot} \textbf{I}) \\ &-\frac{2\eta_t}{\mathcal{P}^{(t)}} \sum_{i \in \mathscr{B}^{(t)}} p_i^2 * \eat{\cdot}
        \textbf{x}_i\textbf{x}_i^T] \mathcal{W}^{(t)} + \frac{2\eta_t}{\mathcal{P}^{(t)}} \sum_{i\in \mathscr{B}^{(t)}}p_i^2 * \eat{\cdot} 
        \textbf{x}_iy_i
    \end{split}
\end{align}
% Ideally, with the iteration proceeds, the provenance for the model parameter can be also computed iteratively. To conduct incremental updates by removing the samples of interest (say $\textbf{x}_i$), we can simply setting $p_i$ as 0 and get the updated 
%apply idempodence over *_K, map the simplified version of provenance expression to the incremental updated model parameters.
\noindent
where $\mathcal{W}^{(t)}$ represents the provenance-annotated expression for model parameter $\textbf{w}^{(t)}$ while $\mathcal{P}^{(t)}$ represents the provenance-annotated expression
for the batch number $B$ at the $t^{\mathrm{th}}$ iteration, i.e. $\mathcal{P}^{(t)} = \sum_{i \in \mathscr{B}^{(t)}} p_i * \eat{\cdot} 1$.
%\val{Three kinds of B may be too much. Consider using $\mathcal{P}^{(t)}$ instead of $\mathcal{B}^{(t)}$.}
%\yinjun{done}

%\val{we need to explain what is division by a provenance expression 
%since $\mathcal{B}$ appears in a denominator}. 
In the provenance semiring framework there is no division operation for provenance expressions,
although a division-like symbol has been used in Equation~\ref{eq: mini_sgd_linear_regression_provenance}
for notational purposes. In fact, in the context of the incremental update problem, once the subset of training samples (say ($\Delta \textbf{X}, \Delta \textbf{Y}$)) to be removed is determined, the corresponding provenance tokens will be set as $0_{K}$ 
% (say all the provenance tokens except $\{p_{i_1}, p_{i_2}, \dots, p_{i_z}\}$ are zeros due to the updates), 
and thus the number of non-zero tensor products in $\mathcal{P}^{(t)}$ can be determined, which also represents the size of the remaining samples in each min-batch 
%after the removal operation 
(denoted by $\increB^{(t)}$). So we can replace $\mathcal{P}^{(t)}$ with $\increB^{(t)}$ in the update rule that
gives $\mathcal{W}^{(t+1)}$, i.e.:
\begin{align}\label{eq: mini_sgd_linear_regression_provenance_update}
    \begin{split}
        &\increprov^{(t+1)} \leftarrow [(1-\eta_t\lambda)(1_K * \eat{\cdot} \textbf{I}) \\
        &-\frac{2\eta_t}{\increB^{(t)}} \sum_{\substack{ i \in \mathscr{B}^{(t)} \\, \textbf{x}_i \not\in \Delta \textbf{X}}} p_i^2 * \eat{\cdot} \textbf{x}_i\textbf{x}_i^T] \increprov^{(t)}+ \frac{2\eta_t}{\increB^{(t)}} \sum_{\substack{ i \in \mathscr{B}^{(t)} \\, \textbf{x}_i \not\in \Delta \textbf{X}\\, y_i \in \Delta \textbf{Y}}}p_i^2 * \eat{\cdot} \textbf{x}_iy_i
    \end{split}
\end{align}

Now suppose $\increprov^{(0)} = 1_K * \eat{\cdot} \textbf{w}^{(0)}$ \eat{represent number of non-zero provenance tokens in current batch}, then by induction we can observe that $\increprov^{(t)}$ should be the sum of multiple tensor products, each of which should be in the form of $P*\textbf{u}$ where $\textbf{u}$ is a matrix associated with a provenance monomial $P$ ($P= p_{j_1}^{r_1}p_{j_2}^{j_2}p_{j_3}^{r_3}\dots p_{j_k}^{r_k}$ and $r_i (i=1,2,\dots,k)$ is a natural number). Intuitively speaking, $P*\textbf{u}$ represents how the combinations of $r_1$ copies of ${j_1}^{\mathrm{th}}$ sample, $r_2$ copies of ${j_2}^{\mathrm{th}}$ sample, \dots and $r_k$ copies of ${j_k}^{\mathrm{th}}$ contribute to the model parameter where ${j_1}^{\mathrm{th}}, {j_2}^{\mathrm{th}}\dots, {j_k}^{\mathrm{th}}$ sample should not be in $\Delta \textbf{X}$.

Then by setting $p_i$ as 1, we can get the update rule to compute the updated model parameters as below:
\begin{align}\label{eq: mini_sgd_linear_regression_para_update}
    \begin{split}
        \increw^{(t+1)}& \leftarrow [(1-\eta_t\lambda)\textbf{I} -\frac{2\eta_t}{\increB^{(t)}} \sum_{\substack{ i \in \mathscr{B}^{(t)} \\, \textbf{x}_i \not\in \Delta \textbf{X}}} \eat{\cdot} \textbf{x}_i\textbf{x}_i^T] \increw^{(t)}\\
        &+ \frac{2\eta_t}{\increB^{(t)}} \sum_{\substack{ i \in \mathscr{B}^{(t)} \\, \textbf{x}_i \not\in \Delta \textbf{X}, y_i \in \Delta \textbf{Y}}}\eat{\cdot} \textbf{x}_iy_i
    \end{split}
\end{align}

which simply sum up the matrices from each tensor product on the right-hand side of Equation \ref{eq: mini_sgd_linear_regression_provenance_update}.

\eat{. are updated iteratively as the iteration proceeds. To conduct incremental deletion over the model parameters $\textbf{w}^{(t)}$ by removing the samples of interest (say $\textbf{x}_{i_1}, \textbf{x}_{i_2}, \dots, \textbf{x}_{i_z}$), we can simply setting $p_{i_1}, p_{i_2}, \dots, p_{i_z}$ as $0_K$ to eliminate the effect of those samples and setting other provenance tokens as $1_K$ at the same time.}
%%%%% END OF OLD SUBSECTION 4.1
}

\vspace{-5mm}
\subsection{Linearization for logistic regression}
\eat{In linear regression, the update rule after applying \gbm\ is only composed of linear operations, for which the existing provenance model in \cite{yan2016fine} is usable.}
The model in~\cite{yan2016fine} supports tracking provenance through matrix addition
and multiplication. In order to apply it to \gbm\ for logistic expression, 
we linearize, using {\em piecewise linear interpolation},
the non-linear operations in the corresponding update rules, i.e. Equation \ref{eq: mini_sgd_logistic_regression}.

In Equation \ref{eq: mini_sgd_logistic_regression}, the non-linear operations can be abstracted as $f(x) = 1- \frac{1}{1+e^{-x}}$\eat{, i.e. 1 minus the Sigmoid function}, where the value of the product $y_i\textbf{w}^{(t)T}\textbf{x}_i$ is assigned to the variable $x$ in Equation \ref{eq: mini_sgd_logistic_regression}. Then $f(x)$ can be approximated by applying 1-D piecewise linear interpolation ~\cite{Kress1998}. So for each $\textbf{x}_i$ and $\textbf{w}^{(t)}$, $f(y_i\textbf{w}^{(t)T}\textbf{x}_i)$ can be approximated by $s(y_i\textbf{w}^{(t)T}\textbf{x}_i) = a^{i,(t)}y_i\textbf{w}^{(t)T}\textbf{x}_i + b^{i,(t)}$, where $a^{i,(t)}$ and $b^{i, (t)}$ are the linear coefficients produced by the linearizations, which depends on which sub-interval (defined by piecewise linear interpolation) the value of $y_i\textbf{w}^{(t)T}\textbf{x}_i$ locates and thus should be varied between different $\textbf{x}_i$ and different $\textbf{w}^{(t)T}$ (see the associated superscript).
% $f(x)$ is approximated by a piecewise linear function $s(x)$ and each specific assignment of the value $y_i\textbf{w}^{(t)T}\textbf{x}_i$ to the variable $x$ will end up with linear coefficients, $a^{i,(t)}$ and $b^{i,(t)}$ such that .

Throughout the paper, we will consider the case in which the variable $x$ in $f(x)$ is defined within an interval $[-a,a]$ ($a = 20$) that is equally partitioned into $10^6$ sub-intervals; for $x$ outside $[-a, a]$, we assume that $s(x)$ is a constant since when $|x| > a$, the value of $f(x)$ is very close to its bound (0 or 1). We will show that the length of each sub-interval influences the approximation rate.
% but has little influence on the time performance. 

In terms of multinomial logistic regression, the non-linear operations in its update rule is the softmax function, which is a vector-valued function and thus requires piecewise linear interpolation in multiple dimensions, which can be achieved by\scream{removed (see eat environment), since it is related to the discussion about the approx rate for 1-D interpolation, which has been removed}\eat{. A similar approximation rate to the 1-D piecewise linear interpolation can be obtained} using the interpolation method proposed in \cite{weiser1988note}. 
%Details are omitted due to the space limit.

After the interpolation step over the update rules for binary logistic regression, Equation \ref{eq: mini_sgd_logistic_regression} is approximated as:
\scream{PUT EQUATIONS 9,10, 11 in a SEPARATE SOMETHING}
\vspace{-1mm}
{ 
\begin{align}\label{eq: mini_sgd_instantiation_approx}
    \begin{split}
        \linearw^{(t+1)}& \approx [(1-\eta_t\lambda)\textbf{I} + \frac{\eta_t}{B}\sum_{i\in \mathscr{B}^{(t)}}a^{i, (t)}\textbf{x}_i\textbf{x}_i^T]\linearw^{(t)}\\
        & + \frac{\eta_t}{B} \sum_{i\in \mathscr{B}^{(t)}} b^{i, (t)}y_i\textbf{x}_i
    \end{split}
\end{align}
}

\vspace{-2mm}
\noindent
in which $\linearw^{(t)}$ represents the model parameter after linearization at $t^{\mathrm{th}}$ iteration. By annotating each training sample $\textbf{x}_i$ with provenance token $p_i$
% \eat{, the expressions on how to propagate those provenance tokens through the approximated update rule in Equation \ref{eq: mini_sgd_instantiation_approx} is presented below:
% \vspace{-5mm}
% {\small 
% \begin{align}\label{eq: mini-SGD_logistic_regression_prov}
%     \begin{split}
%         \linearprov^{(t+1)} & \leftarrow ((1-\eta_t\lambda)(1_K*\textbf{I})\\
%         & + \frac{\eta_t}{\mathcal{P}^{(t)}}\sum_{i\in \mathscr{B}^{(t)}}a^{i, (t)}(p_i^2*\textbf{x}_i\textbf{x}_i^T))\linearprov^{(t)}\\
%         & + \frac{\eta_t}{\mathcal{P}^{(t)}} \sum_{i\in \mathscr{B}^{(t)}} ({p_i^2}*b^{i, (t)}y_i\textbf{x}_i)   
%     \end{split}
% \end{align}
% }
% \noindent
% where $\linearprov^{(t)}$ represents the provenance expression for $\linearw^{(t)}$. }
and by taking the similar derivation of Equation \ref{eq: mini_sgd_linear_regression_provenance_update}, after the removal of the subset of training samples the provenance expression
% and the update rule for the updated model parameter $\linearincrew^{(t)}$ 
becomes:

\vspace{-4mm}
{
\begin{align}\label{eq: mini-SGD_logistic_regression_prov_update}
    \begin{split}
        &\hspace{-1.5cm}\linearincreprov^{(t+1)} \leftarrow [(1-\eta_t\lambda)(\oneprov*\textbf{I})\\
        & + \frac{\eta_t}{\increB^{(t)}}\sum_{\substack{ i \in \mathscr{B}^{(t)}, i \not \in \mathcal{R}}}p_i^2*(a^{i, (t)}\textbf{x}_i\textbf{x}_i^T)]\linearincreprov^{(t)}\\
        & + \frac{\eta_t}{\increB^{(t)}} \sum_{\substack{ i \in \mathscr{B}^{(t)}, i \not \in \mathcal{R}}} {p_i^2}*(b^{i, (t)}y_i\textbf{x}_i) 
    \end{split}
\end{align}
}
% \vspace{-4mm}
% by summing up the matrices in all the tensor products of $\linearincreprov^{(t+1)}$, 
\noindent
By setting all the $p_i$ in Equation \ref{eq: mini-SGD_logistic_regression_prov_update} as $\oneprov$, we can get the update rule for the updated model parameter $\linearincrew^{(t)}$, i.e.:

\vspace{-5mm}
% \topskip
\begin{align}\label{eq: mini-SGD_logistic_regression_para_update}
    \begin{split}
        &\linearincrew^{(t+1)} \approx [(1-\eta_t\lambda)\textbf{I}\\
        & + \frac{\eta_t}{\increB^{(t)}}\sum_{\substack{ i \in \mathscr{B}^{(t)} \\, i \not \in \mathcal{R}}}a^{i, (t)}\textbf{x}_i\textbf{x}_i^T]\linearincrew^{(t)} + \frac{\eta_t}{\increB^{(t)}} \sum_{\substack{ i \in \mathscr{B}^{(t)} \\, i \not \in \mathcal{R}}} b^{i, (t)}y_i\textbf{x}_i 
    \end{split}
\end{align}

% \noindent
% in which Equation \ref{eq: mini-SGD_logistic_regression_para_update} is derived by setting all the $p_i$ in Equation \ref{eq: mini-SGD_logistic_regression_prov_update} as $\oneprov$.
% , the update rule for the updated model parameter $\linearincrew^{(t)}$ is:

\eat{In the following discussion, we use $\increprov^{(t)}(p_{j_1}, p_{j_2},p_{j_3}, \dots, p_{j_k})$ ($\linearincreprov^{(t)}(p_{j_1}, p_{j_2},p_{j_3}, \dots, p_{j_k})$ resp.) to represent the sum of all matrices which come from the tensor products in $\increprov^{(t)}$ ($\linearincreprov^{(t)}$ resp.) with provenance tokens from $\{p_{j_1}, p_{j_2},p_{j_3}, \dots, p_{j_k}\}$, where $\{p_{j_1}, p_{j_2},p_{j_3}, \dots, p_{j_k}\}$ is an arbitrary subset of $\{p_{i_1}, p_{i_2},p_{i_3}, \dots, p_{i_z}\}$. For example, $\increprov^{(t)}(p_{j_1}, p_{j_2})$ is the sum of matrices which are from the tensor products of the form $p_{j_1}^{r_1}p_{j_2}^{r_2}$, where $r_1$ and $r_2$ are natural numbers (which can be 0). It is worth noting that $\increprov^{(t)}(p_{i_1}, p_{i_2}, \dots, p_{i_z})$ ($\linearincreprov^{(t)}(p_{i_1}, p_{i_2}, \dots, p_{i_z})$ resp.) should be equal to the updated model parameter (denoted by $\increw^{(t)}$ and $\linearincrew^{(t)}$ respectively) when all training samples other than  the ${i_1}^{\mathrm{th}}$, ${i_2}^{\mathrm{th}}$, \dots, ${i_z}^{\mathrm{th}}$ samples are removed, i.e.: 
\begin{align}
    \increprov^{(t)}(p_{i_1}, p_{i_2}, \dots, p_{i_z}) = \increw^{(t)}\\
    \linearincreprov^{(t)}(p_{i_1}, p_{i_2}, \dots, p_{i_z}) = \linearincrew^{(t)}\label{eq: provenanace_eq_update_model_logistic}
\end{align}}
\noindent
\eat{in which the update rules for $\increw^{(t)}$ and  $\linearincrew^{(t)}$ are:
\eat{Then if we set $p_{j_1}, p_{j_2}, \dots, p_{j_s}$ to $1_k$ and other provenance tokens to $0_k$, the resulting $\mathcal{W}^{(t)}(\{p_{j_1}, p_{j_2}, \dots, p_{j_s}\})$ will be a single tensor product $P'*\textbf{u}'$ where $P' = 1_k$ and $\textbf{u}'$ equals to the updated $\linearw^{(t)}$ (denoted by $\linearincrew^{(t)}$) after removing all the other samples except $\textbf{x}_{j_1}, \textbf{x}_{j_2}, \dots, \textbf{x}_{j_s}$,}
\begin{align}\label{eq: mini-SGD_updated_model_parameters}
    \begin{split}
        \linearincrew^{(t+1)} &\leftarrow [(1-\eta_t\lambda)\textbf{I} + \frac{\eta_t}{z^{(t)}}\sum_{\substack{i\in \mathscr{B}^{(t)},\\ i \in \{i_1, i_2,\dots, i_z\}}}a^{i, (t)}\textbf{x}_i\textbf{x}_i^T]\linearincrew^{(t)}\\
        & + \frac{\eta_t}{z^{(t)}} \sum_{\substack{i\in \mathscr{B}^{(t)},\\ i \in \{i_1, i_2,\dots, i_z\}}} b^{i, (t)}y_i\textbf{x}_i
    \end{split}\\
    \begin{split}
        \increw^{(t+1)}& \leftarrow [(1-\eta_t\lambda)\textbf{I} -\frac{2\eta_t}{\increB^{(t)}} \sum_{\substack{ i \in \mathscr{B}^{(t)} \\ i \in \{{i_1}, {i_2}, \dots, {i_z}\}}} \textbf{x}_i\textbf{x}_i^T] \increw^{(t)}\\
        &+ \frac{2\eta_t}{\increB^{(t)}} \sum_{\substack{ i \in \mathscr{B}^{(t)} \\ i \in \{{i_1}, {i_2}, \dots, {i_z}\}}} \textbf{x}_iy_i
    \end{split}
\end{align}}
% while the update rules for multi-nomial logistic regression is provided in \ref{}.
\vspace{-5mm}
\subsection{Convergence analysis for provenance-annotated iterations}\label{ssection: convergence}

One concern in using \gbm\ is whether the model parameters ultimately converge.  This has been extensively studied in the machine learning community \cite{karimi2016linear, she2017linear, kumar2017convergence, schmidt2014convergence, bottou2018optimization}. In \cite{bottou2018optimization}, convergence conditions have been provided for GD and SGD 
% various gradient methods (GD, SGD,  Greedy Coordinate Descent and so on)
over strong convex objective functions. 
% without constraints under the Polyak-Lojasiewicz Inequality (PL Inequality)\cite{polyak1963gradient}\eat{, a variant of Lojasiewicz Inequality \cite{ji1992global}}. 
Those convergence conditions can exactly fit linear regression and logistic regression with L2-regularization because their objective functions are strong convex.
% , which satisfy \eat{{\em strong convexity} and thus satisfy }the PL Inequality \cite{karimi2016linear}.

A similar concern occurs when \gbm\ is coupled with provenance, i.e. whether the provenance expression $\increprov^{(t)}$ in Equation \ref{eq: mini_sgd_linear_regression_provenance_update} and $\linearincreprov^{(t)}$ 
\eat{check that we still ahve this equation and notation}
in Equation \ref{eq: mini-SGD_logistic_regression_prov_update} converge in the case when the original model parameter $\textbf{w}^{(t)}$ 
% (rather than the approximate model parameter $\linearw^{(t)}$) 
converges. We propose the following definition
for the convergence of provenance-annotated expressions. 

\begin{definition}{\bf Convergence of provenance-annotated expressions.}\label{def: convergence_tensor_prod}
The expression $\mathcal{W}^{(t)}=
\sum_{i} \mathfrak{p}_i^{(t)}*\textbf{u}_i^{(t)}$ 
converges when 
$t \rightarrow \infty$ iff every matrix $\textbf{u}_i^{(t)}$ converges when $t \rightarrow \infty$.
%which is a sum of multiple tensor products, written as
\end{definition}

%\val{Why does "Definition" have a different font than "Theorem?}

As mentioned before, we hope that the convergence of $\increprov^{(t)}$ and $\linearincreprov^{(t)}$ can be achieved when $\textbf{w}^{(t)}$ can converge. The convergence conditions of $\textbf{w}^{(t)}$ are presented below:

\begin{lemma}\textbf{Convergence conditions for general \minisgd.}~\cite{bottou2018optimization}\label{lemma: convergence_conditions}
Given an objective function $h(\textbf{w})$, which is $L-$Lipschitz continuous and \eat{under the assumption in Lemma \ref{lemma: sgd_assumption}, by applying \sgd\ or \minisgd\ over $h(\textbf{w})$}
$\lambda-$strong convex
% has $L2-$regularization term $\lambda ||\textbf{w}||_2^2$, 
once the learning rate $\eta_t$ satisfies: 1) $\eta_t < \frac{1}{L}$; 2) $\eta_t$ is a constant across all the iterations (denoted by $\eta$), then $\textbf{w}^{(t)}$ converges when \minisgd\ is used. \eat{and we can obtain a linear convergence rate up to a solution level that is proportional to $\eta$, i.e.:
\begin{align}
\begin{split}
&E(h(\textbf{w}^{(t)}) - h^*) \leq
% (1-2\lambda\eta)^k(h(\textbf{w}^{(0)}) - h^*) + \frac{LC^2\eta}{4\lambda}\\
% & = 
(1-2\lambda\eta)^t(h(\textbf{w}^{(0)}) - h^*) + O(\eta)
\end{split}
% \begin{split}
% &E(\textbf{w}^{(t)} - \textbf{w}^*)
% \leq (1-2\lambda\eta)^k(\textbf{w}^{(0)} - \textbf{w}^*) + \frac{C^2\eta}{2\lambda}\\
% \leq (1-2\lambda\eta)^t(\textbf{w}^{(0)} - \textbf{w}^*) + O(\eta)
% \end{split}
\end{align}
where $h^*$ represents the optimal value of $h(\textbf{w})$.}
% and $\textbf{w}^*$ represents the corresponding model parameters.
\end{lemma}

\eat{We will focus on convergence analysis for $\increprov^{(t)}$ in Equation \ref{eq: mini_sgd_linear_regression_provenance_update} and $\linearincreprov^{(t)}$ in Equation \ref{eq: mini-SGD_logistic_regression_prov_update} in the following presentation while the convergence of $\linearprov^{(t)}$ can be easily derived since $\linearprov^{(t)}$ can be viewed as a special case of $\linearincreprov^{(t)}$ where the number of removed samples is 0. Similarly, $\mathcal{W}^{(t)}$ in Equation \ref{eq: mini_sgd_linear_regression_provenance} can be regarded as a special case of $\increprov^{(t)}$ in the case of linear regression. }

Unfortunately, our theoretical analysis shows that there is no convergence guarantee for $\increprov^{(t)}$ and $\linearincreprov^{(t)}$ under the convergence conditions from Lemma \ref{lemma: convergence_conditions}, i.e.:
\begin{theorem}\label{theorem: non_convergence}
$\increprov^{(t)}$ in Equation \ref{eq: mini_sgd_linear_regression_provenance_update} and $\linearincreprov^{(t)}$ in Equation \ref{eq: mini-SGD_logistic_regression_prov_update} need not
converge under the conditions in Lemma \ref{lemma: convergence_conditions}.
\footnote{Due to space limitations the proofs of the theorems
are omitted. They will appear in the full version of the paper.}
\end{theorem}

\eat{\proofsketch

Let us take linear regression as an example. In order to prove Theorem \ref{theorem: non_convergence}, we need to show that there exists a case where $\increprov^{(t)}$ cannot converge under the conditions in Lemma \ref{lemma: convergence_conditions}.  This is achieved by considering gradient descent (\gd) without excluding any original training samples, i.e. $\{p_{i_1}, p_{i_2}, \dots, p_{i_z}\} = \{1,2,\dots,n\}$, every $\increB^{(t)}=\;n$ in Equation \ref{eq: mini_sgd_linear_regression_provenance_update} and every $\mathscr{B}^{(t)}$ includes all $n$ samples in Equations \ref{eq: mini_sgd_linear_regression} and \ref{eq: mini_sgd_linear_regression_provenance_update}. \eat{$\mathcal{W}^{(t)}$ in Equation \ref{eq: mini_sgd_linear_regression_provenance} since $\mathcal{W}^{(t)}$ is a special case of $\increprov^{(t)}$. Plus, we only consider gradient descent (\gd) here.}
We can then apply the update rule in Equations \ref{eq: mini_sgd_linear_regression} and \ref{eq: mini_sgd_linear_regression_provenance_update} recursively, which ends up with:
\begin{align}
\begin{split}
&\textbf{w}^{(t+1)} \label{eq: w_expansion}
% = ((1-\eta_t\lambda)\textbf{I} - \frac{2\eta_t}{n} \sum_{i=1}^{n} \textbf{x}_i\textbf{x}_i^T)\textbf{w}^{(t)} + \frac{2\eta_t}{n} \sum_{i=1}^{n} \textbf{x}_i y_i\\&
= ((1-\eta\lambda)\textbf{I} - \frac{2\eta}{n} \sum_{i=1}^{n} \textbf{x}_i\textbf{x}_i^T)^{t+1}\textbf{w}^{(0)}\\
& + (\sum_{j=1}^t ((1-\eta\lambda)\textbf{I} - \frac{2\eta}{n} \sum_{i=1}^{n} \textbf{x}_i\textbf{x}_i^T)^{j}) \frac{2\eta}{n} \sum_{i=1}^{n} \textbf{x}_i y_i
\end{split}\\
\begin{split}\label{eq: w_prov_expansion}
&\increprov^{(t)}
% = ((1-\eta_t\lambda)\textbf{I}\cdot 1_k - \frac{2\eta_t}{n} \sum_{i=1}^{n} \textbf{x}_i\textbf{x}_i^T \cdot p_i^2)\mathcal{W}^{(t)} \\&
% + \frac{2\eta_t}{n} \sum_{i=1}^{n} \textbf{x}_i y_i \cdot p_i^2\\&
= ((1-\eta\lambda)1_k*\textbf{I}- \frac{2\eta}{n} \sum_{i=1}^{n} p_i^2 * \textbf{x}_i\textbf{x}_i^T )^{t}\increprov^{(0)}\\
& + (\sum_{j=1}^t ((1-\eta\lambda)1_k*\textbf{I} - \frac{2\eta}{n} \sum_{i=1}^{n} p_i^2 *\textbf{x}_i\textbf{x}_i^T)^{j}) \frac{2\eta}{n} \sum_{i=1}^{n} p_i^2*\textbf{x}_i y_i
\end{split}
\end{align}

Observe that the update rule in $\textbf{w}^{(t)}$ in Equation \ref{eq: w_expansion} becomes a linear system. In order to make sure that $\textbf{w}^{(t)}$ converges, the condition, $||((1-\eta\lambda)\textbf{I} - \frac{2\eta}{n} \sum_{i=1}^{n} \textbf{x}_i\textbf{x}_i^T)||_2 \leq 1$, should be satisfied, where $||*||_2$ represents the $L2-$norm of certain matrix. Since every $\textbf{x}_i\textbf{x}_i^T$ is a semi-positive definite matrix, under the convergence conditions, for every $i$, $||\frac{2\eta}{n} \textbf{x}_i\textbf{x}_i^T||_2 \leq 1$ (details omitted).

After expanding the first term in the right-hand side of Equation \ref{eq: w_prov_expansion}, the tensor product with provenance monomial $p_i^{t}$ should be $p_i^t*{t\choose \frac{t}{2}}(1-\eta\lambda)^{\frac{t}{2}}(-\frac{2\eta}{n} \textbf{x}_i\textbf{x}_i^T)^{\frac{t}{2}}$. According to the convergence conditions in Lemma \ref{lemma: convergence_conditions}, $\eta < \frac{1}{2\lambda}$ and thus $||{t\choose \frac{t}{2}}(1-\eta\lambda)^{\frac{t}{2}}(-\frac{2\eta}{n} \textbf{x}_i\textbf{x}_i^T)^{\frac{t}{2}}||_2 \geq {t\choose \frac{t}{2}}||(-\frac{\eta}{n} \textbf{x}_i\textbf{x}_i^T)^{\frac{t}{2}}||_2$. According to \cite{sun2001convergence, dasbrief}, when $t\rightarrow \infty$, ${t\choose \frac{t}{2}}$ should be very close to $2^t$ and thus ${t\choose \frac{t}{2}}||(-\frac{\eta}{n} \textbf{x}_i\textbf{x}_i^T)^{\frac{t}{2}}||_2 \rightarrow \infty$, which means that the tensor product with provenance monomial $p_i^{t}$ cannot converge and thus $\increprov^{(t)}$ cannot converge. \qed}

\eat{The proof is presented in Appendix \ref{sec: non_convergency_proof}. However, if we assume ``idempotence'' over $*_K$ in tensor product, 
% Equation \ref{eq: mini-SGD_prov} will be simplified as:
% \begin{align}\label{eq: mini-SGD_prov_idem}
%     \begin{split}
%         \mathcal{W}^{(t+1)'} & \leftarrow ((1-\eta_t\lambda)(1_K*\textbf{I})\\
%         &+ \frac{\eta_t}{\mathcal{B}}\sum_{i=r^{(t)}}^{r^{(t)}+B-1}a^{i, (t)}(p_i*\textbf{x}_i\textbf{x}_i^T))\mathcal{W}^{(t)'}\\
%         & + \frac{\eta_t}{\mathcal{B}} \sum_{i=r^{(t)}}^{r^{(t)}+B-1} ({p_i}*b^{i, (t)}y_i\textbf{x}_i)\\        
%     \end{split}
% \end{align}
the convergence of $\linearprov^{(t)}$ can be achieved (Detailed proof is presented in Appendix \ref{sec: theorem_convergency_proof}), i.e.:}

However, $\increprov^{(t)}$ in Equation \ref{eq: mini_sgd_linear_regression_provenance_update} and $\linearincreprov^{(t)}$ in Equation \ref{eq: mini-SGD_logistic_regression_prov_update}  
converge under the conditions in Lemma \ref{lemma: convergence_conditions} with one more assumption about the provenance expression, i.e.:

\begin{theorem}\label{theorem: convergence_res}
The expectation of $\increprov^{(t)}$ in Equation \ref{eq: mini_sgd_linear_regression_provenance_update} and of $\linearincreprov^{(t)}$ in Equation \ref{eq: mini-SGD_logistic_regression_prov_update}, 
%i.e. $E(\increprov^{(t)})$ and $E(\linearincreprov^{(t)})$, 
converge when $t \rightarrow \infty$ if we also assume that provenance polynomial
multiplication is \emph{idempotent}.
\end{theorem}
Intuitively speaking, the assumption of multiplication idempotence 
for provenance polynomials means that we do not track \emph{multiple joint uses 
of the same data sample}, which is not problematic for deletion propagation.
\scream{delete the following sentence in the eat environment since I guess readers may be confused about the notion of ``trust''}
\eat{although it might complicate future work attempting to quantify trust in the training samples.}
\eat{
the tensor products with the same set of provenance tokens will be merged together using the identities in Section \ref{ssec: provenance_annotation}, which indicates that the existence (rather than every combination) of the provenance tokens is captured in each tensor product.  
\eat{\val{We must comment on the meaning of idempotence for the "joint use"
provenance operation}\yinjun{done, can you take a look}
}
% where $i_j$ is randomly picked up from all the training samples and the entire training samples are assumed to be uniformly distributed. In what follows, we will denote $E_{i_j}(\cdot)$ by $E(\cdot)$ for simplicity to represent the expected values over the sample distributions.

To prove the convergence of $\increprov^{(t)}$ and $\linearincreprov^{(t)}$, we need to prove that each tensor product in $\increprov^{(t)}$ and $\linearincreprov^{(t)}$ converges. Due to the ``idempotence'' over $*_K$, each tensor product in $\increprov^{(t)}$ and $\linearincreprov^{(t)}$ will be in the form of $P*\textbf{u}$ where the provenance monomial $P$ is a product of provenance tokens all with exponent 1, i.e. $P = p_{j_1}p_{j_2}\dots p_{j_s}$, in which $\textbf{u}$ can be computed by using $\increprov^{(t)}(p_{j_1}, p_{j_2},\dots, p_{j_s})$ 
\\($\linearincreprov^{(t)}(p_{j_1}, p_{j_2},\dots, p_{j_s})$ resp.) as below:

\begin{align}\label{eq: provenance_equation}
\begin{split}
& \textbf{u} = \increprov^{(t)}(p_{j_1}, p_{j_2},\dots, p_{j_s})\\
&- \sum_{\substack{\{p_{k_1}, p_{k_2},\dots, p_{k_t}\}\\ \subsetneq \{p_{j_1}, p_{j_2},\dots, p_{j_s}\}}}\increprov^{(t)}(p_{k_1}, p_{k_2},\dots, p_{k_t})
\end{split}
\end{align}

Intuitively, $P*\textbf{u}$ represents the tensor product where $P$ is \\$(p_{j_1}p_{j_2}\dots p_{j_s})$ while $\increprov^{(t)}(p_{j_1}, p_{j_2},\dots, p_{j_s})$ may include matrices whose corresponding tensor products only partially include tokens from $\{p_{j_1}, p_{j_2},\dots, p_{j_s}\}$, which are subtracted from $\increprov^{(t)}(p_{j_1}, p_{j_2},\dots, p_{j_s})$ to obtain $\textbf{u}$. According to Equation \ref{eq: provenance_equation}, the convergence of each $P*\textbf{u}$ can be derived by the convergence of each $\increprov^{(t)}(p_{j_1}, p_{j_2},\dots, p_{j_s})$, which is presented below.
\\}

\eat{\proofsketch

In \minisgd, a typical assumption is used for the convergence analysis of the model parameter $\textbf{w}^{(t)}$ in Equation \ref{eq: mini-sgd}, i.e.:

\begin{lemma}\label{lemma: sgd_assumption}
For any randomly selected sample $i_j$ in some batch, the expectation of its gradient should be the same as the gradient over the all the samples, i.e.:

$E(\triangledown h_{i_j}(\textbf{w})) = \triangledown h(\textbf{w})$

which also implies that the following equality holds for \minisgd:

$E(\triangledown( \frac{1}{B} \sum_{i\in \mathscr{B}^{(t)}}h_{i}(\textbf{w}))) = \triangledown h(\textbf{w})$

where $E$ is the expectation value with respect to the sampling over the entire training samples.
\end{lemma}

The following proofs will be described using logistic regression. By denoting $\linearincreprov^{(t)}(p_{j_1}, p_{j_2},\dots, p_{j_s})$ as $\textbf{v}$ (a matrix), its update rule should be as follows:
% First of all, due to the ``idempotence'' of $*_K$, the expectation of Equation \ref{eq: mini-SGD_logistic_regression_prov_update} is:
% \begin{align}\label{eq: w_prov_expansion}
%     \begin{split}
%         & E(\linearincreprov^{(t+1)}) \leftarrow [(1-\eta_t\lambda)(1_K*\textbf{I})\\
%         & + \frac{\eta_t}{n}\sum_{i=1}^n a^{i, (t)}(p_i*\textbf{x}_i\textbf{x}_i^T)]\linearincreprov^{(t)}+ \frac{\eta_t}{n} \sum_{i=1}^n {p_i}*b^{i, (t)}y_i\textbf{x}_i
%     \end{split}
% \end{align}
% \begin{align}
% \begin{split}\label{eq: w_prov_expansion}
% &E(\increprov^{(t)})
% = ((1-\eta\lambda)\textbf{I}\cdot 1_k - \frac{2\eta}{n} \sum_{i=1}^{n} \textbf{x}_i\textbf{x}_i^T \cdot p_i)^{t}\increprov^{(0)}\\
% & + (\sum_{j=1}^t ((1-\eta\lambda)\textbf{I}\cdot 1_k - \frac{2\eta}{n} \sum_{i=1}^{n} \textbf{x}_i\textbf{x}_i^T \cdot p_i)^{j}) \frac{2\eta}{n} \sum_{i=1}^{n} \textbf{x}_i y_i \cdot p_i
% \end{split}
% \end{align}
% Plus, we also observe that $\linearincreprov^{(t)}(p_{j_1}, p_{j_2},\dots, p_{j_s}) = \linearincreprov^{(t)}|_{\substack{p_i = 0, p_i \not \in \{p_{j_1} \\, p_{j_2},\dots, p_{j_s}\}}}$. So we start from the simplest case where $\{p_{j_1}, p_{j_2},\dots, p_{j_s}\}$ is a singleton set. So Equation \ref{eq: w_prov_expansion} can be modified as below:
\begin{align}\label{eq: mini-SGD_logistic_regression_prov_update_singleton}
    \begin{split}
        & E(\textbf{v}^{(t+1)}) = 
        % ((1-\eta_t\lambda) + \frac{\eta_t}{n}\sum_{r \in \{j_1, j_2, \dots, j_s\}}a^{r, (t)}(\textbf{x}_r\textbf{x}_r^T))\textbf{v}^{(t)}\\
        % & + \frac{\eta_t}{n} \sum_{r \in \{j_1, j_2, \dots, j_s\}}(b^{r, (t)}y_r\textbf{x}_r)\\
        \textbf{v}^{(t)}\\
        & - \eta_t (\lambda \textbf{v}^{(t)} - \frac{1}{n}\sum_{r \in \{j_1, j_2, \dots, j_s\}} (a^{r,(t)}\textbf{x}_r\textbf{x}_r^T \textbf{v}^{(t)} + b^{r, (t)}y_r\textbf{x}_r))\\
        &=  \textbf{v}^{(t)} - \eta_t \triangledown^{(t)} R_{j_1,j_2,\dots, j_s}(\textbf{v}^{(t)})
    \end{split}
\end{align}
\noindent
where $\triangledown^{(t)} R_{j_1,j_2,\dots, j_s}(\textbf{v}^{(t)})$ represents the term in the second from the last line in Equation \ref{eq: mini-SGD_logistic_regression_prov_update_singleton}, which should satisfy the following two inequalities.
\begin{align}
        &||\triangledown^{(t)} R_{j_1,j_2,\dots, j_s}(\textbf{v}^{(t)})||_2 < C'\label{eq: R_property1} \\
        & <\textbf{v}^{(t)} - \textbf{v}^*, \triangledown^{(t)} R_{j_1,j_2,\dots, j_s}(\textbf{v}^{(t)})> \geq \lambda ||\textbf{v}^{(t)} - \textbf{v}^*||_2^2 \label{eq: R_property2}
\end{align}
\noindent
where $C'$ is some constant and $\textbf{v}^*$ represents the matrix that satisfies $\triangledown^{(t)}R_{j_1,j_2,\dots, j_s}(\textbf{v}^*) = \textbf{0}$. By combining the equations above, we obtain the following inequality:
\begin{align}\label{eq: u_gap}
    \begin{split}
        &E(||\textbf{v}^{(t+1)} - \textbf{v}^*||_2^2)\\
        & = E(||\textbf{v}^{(t)} - \eta \triangledown^{(t)} R^{(t)}(\textbf{v}^{(t)}) - \textbf{v}^*||_2^2)\\
        & = ||\textbf{v}^{(t)} - \textbf{v}^*||_2^2 - 2\eta <\triangledown R_{j_1,j_2,\dots, j_s}(\textbf{v}^{(t)}), \textbf{v}^{(t)} - \textbf{v}^*>\\
        & + \eta^2 ||R_{j_1,j_2,\dots, j_s}^{(t)}(\textbf{v}^{(t)})||_2^2\\
        & \myleqtwo (1-2\eta \lambda )||\textbf{v}^{(t)} - \textbf{v}^*||_2^2 + C'^2\eta^2
    \end{split}
\end{align}

By deriving Equation \ref{eq: u_gap} recursively and under the condition that $\eta_t \leq \frac{1}{2\lambda}$ we can conclude that $\textbf{v}^{(t)}$ has the same convergence rate as $\textbf{w}^{(t)}$. 

Then according to Equation \ref{eq: provenanace_eq_update_model_logistic}, due to the equality between $\linearincreprov^{(t)}(p_{i_1}, p_{i_2},\dots, p_{i_z})$ and $\linearincrew^{(t)}$, we can conclude that $\linearincrew^{(t)}$ should be converged under the convergence conditions of the original model parameters $\textbf{w}^{(t)}$. \qed}

\eat{Its convergence depends on that of the matrix part of $\linearincreprov^{(t+1)}(p_r)$, i.e.:

\begin{align}\label{eq: mini-SGD_logistic_regression_prov_update_singleton_matrix}
    \begin{split}
        & E(\linearincrew^{(t+1)}(p_r)) \leftarrow ((1-\eta_t\lambda)(1_K*\textbf{I})\\
        & + \frac{\eta_t}{n}a^{r, (t)}(p_r*\textbf{x}_r\textbf{x}_r^T))\linearincreprov^{(t)}(p_r) + \frac{\eta_t}{n} ({p_r}*b^{r, (t)}y_r\textbf{x}_r)
    \end{split}
\end{align}

According to Weyl's inequality \cite{weyl1912asymptotische}, the largest eigenvalue of

We then define $\mathcal{W}^{(t)}(\{p_{j_1}, p_{j_2}, \dots, p_{j_s}\}|\{p_{i_1}, p_{i_2}, \dots, p_{i_z}\})$ as the tensor product with provenance monomial $p_{j_1}p_{j_2}\dots p_{j_s}$ where all the provenance tokens except $\{p_{i_1}, p_{i_2}, \dots, p_{i_z}\}$ are $0_k$. Obviously, $\mathcal{W}^{(t)}(\{p_{j_1}, p_{j_2}, \dots, p_{j_s}\}|\{p_{i_1}, p_{i_2}, \dots, p_{i_z}\})$  is not \textbf{0} iff $\{p_{j_1}, p_{j_2}, \dots, p_{j_s}\}$ is a subset of $\{p_{i_1}, p_{i_2}, \dots, p_{i_z}\}$. We also represent $\mathcal{W}^{(t)}(\{p_{i_1}, p_{i_2}, \dots, p_{i_z}\})$ as \\$\sum_{\substack{\{p_{j_1}, p_{j_2}, \dots, p_{j_s}\}\\ \subset \{p_{i_1}, p_{i_2}, \dots, p_{i_z}\}}} \mathcal{W}^{(t)}(\{p_{j_1}, p_{j_2}, \dots, p_{j_s}\}|\{p_{i_1}, p_{i_2}, \dots, p_{i_z}\})$. We use 
}

\eat{In the context of incremental updates, the provenance expression $\mathcal{W}^{(t)}(\{p_{j_1}, p_{j_2}, \dots, p_{j_s}\}|\{p_{i_1}, p_{i_2}, \dots, p_{i_z}\})$ should converge under the convergence conditions of the original model parameter $\textbf{w}^{(t)}$:

The proof of Theorem \ref{theorem: convergence_res} is sketched below:

\proofsketch

Let us take logistic regression as an example.

\qed

\begin{lemma}\label{lm: convergence_another_form}

$\mathcal{W}^{(t)}(\{p_{j_1}, p_{j_2}, \dots, p_{j_x}\})$ converges iff every $\mathcal{W}^{(t)}(\{p_{j_1}, p_{j_2}, \dots, p_{j_s}\}|\{p_{i_1}, p_{i_2}, \dots, p_{i_z}\})$ converges.
% For any tensor product $P\cdot \textbf{u}$ in $\mathcal{W}^{(t)}$ where $P = p_{i_1}p_{i_2}\dots p_{i_z}$, $\textbf{u}$ converges iff  converges for any $\{p_{j_1}, p_{j_2}, \dots, p_{j_x}\}$ which is a subset of $\{p_{i_1}, p_{i_2}, p_{i_3}, \dots, p_{i_z}\}$.
\end{lemma}

Besides, $\mathcal{W}^{(t)}(\{p_{i_1}, p_{i_2}, p_{i_3}, \dots, p_{i_z}\})$ 

where $z^{(t)}$ represents the number of samples that belongs to the batch in the $t_{th}$ iteration and falls in the subset of the remaining training samples after the deletion process (i.e. $\{\textbf{x}_{i_1}, \textbf{x}_{i_2}, \dots, \textbf{x}_{i_z}\}$). }

\subsection{Accuracy analysis for linearized logistic regression}\label{ssec: accuracy_proof}
The next question is whether the approximated model parameters after linearization of Equation \ref{eq: mini_sgd_logistic_regression} (i.e. $\linearw^{(t)}$ in Equation \ref{eq: mini_sgd_instantiation_approx}) is close enough to the real model parameters from Equation \ref{eq: mini_sgd_logistic_regression}. By following the approximation property of piecewise linear interpolation, we can prove that the distance between $\textbf{w}^{(t)}$ and $\linearw^{(t)}$ is very small.

\scream{Probably we can remove this theorem if we have space problem}
\begin{theorem}\label{theorem: aproximation_bound}
$||E(\textbf{w}^{(t)} - \linearw^{(t)})||_2$ is bounded by $O((\Delta x)^2)$ where $\Delta x$ is an arbitrarily small value representing the length of the longest sub-interval used in piecewise linear interpolations. 
\end{theorem}

\eat{\proofsketch

Based on the analysis in Section \ref{ssection: convergence}, $\linearw^{(t)}$ should be equal to $\linearincreprov(p_1, p_2, \dots, p_n)$ when no tokens are removed.  Then by referencing Equation \ref{eq: mini-SGD_logistic_regression_prov_update_singleton}, $\linearw^{(t)}$ can be expressed as:
\begin{align}
    \begin{split}
        \linearw^{(t+1)} = \linearw^{(t)} -\eta_t \triangledown^{(t)} R_{1,2,3,\dots, n}(\linearw^{(t)})
    \end{split}
\end{align}
\noindent
where $\triangledown^{(t)} R_{1,2,3,\dots,n}(\linearw^{(t)})$ still satisfies Equation \ref{eq: R_property1} and Equation \ref{eq: R_property2}. Then we can compare $\linearw^{(t)}$ and $\textbf{w}^{(t)}$ explicitly:

\begin{align}\label{eq: w_gap}
    \begin{split}
        &E(||\textbf{w}^{(t+1)} - \linearw^{(t+1)}||_2^2)\\
        & = E(||\textbf{w}^{(t)} -\eta_t \triangledown h^{(t)}(\textbf{w}^{(t)})\\
        &- (\linearw^{(t)} -\eta_t \triangledown R_{1,2,\dots,n}^{(t)}(\linearw^{(t)}))||_2)\\
        & = ||\textbf{w}^{(t)} - \linearw^{(t)}||_2^2\\
        & - 2\eta_t <\textbf{w}^{(t)} - \linearw^{(t)}, \triangledown^{(t)} R_{1,2,\dots,n}(\textbf{w}^{(t)})\\
        &- \triangledown^{(t)} R_{1,2,\dots,n}(\linearw^{(t)})>\\
        & - 2\eta_t <\textbf{w}^{(t)} - \linearw^{(t)}, \triangledown h(\textbf{w}^{(t)}) - \triangledown R_{1,2,\dots,n}^{(t)}(\textbf{w}^{(t)})>\\
        & + \eta_t^2 E(||\triangledown h^{(t)}(\textbf{w}^{(t)}) - \triangledown R_{1,2,\dots,n}^{(t)}(\linearw^{(t)})||_2^2)
    \end{split}
\end{align}

Since $\triangledown R_{1,2,\dots, n}^{(t)}(*)$ is the linear approximation of $\triangledown h(*)$, we can prove that the last two terms are very small, whose norms can be bounded by $O((\Delta x))$ (details omitted). We can also apply Equation \ref{eq: R_property2} over the term in the third to the last line, and thus Equation \ref{eq: w_gap} can be simplified as:
\begin{align}
    \begin{split}
        &E(||\textbf{w}^{(t+1)} - \linearw^{(t+1)}||_2^2)\\
        & \leq (1-2 \eta_t \lambda)||\textbf{w}^{(t)} - \linearw^{(t)}||_2^2 + O((\Delta x))
    \end{split}
\end{align}
\noindent
which can be derived recursively, and thus we can conclude that $E(||\textbf{w}^{(t+1)} - \linearw^{(t+1)}||_2^2$ is bounded by $O((\Delta x))$. \qed}
% where $E(\cdot)$ represents the expectation with respect to the mini-batches.

Furthermore, in terms of the updated model parameters for logistic regression, we also need to guarantee that the updated parameters $\linearincrew^{(t)}$ are close to the real updated model parameters without linearization (denoted by $\logistlinearincrew$), i.e.:
\vspace{-0.5mm}
\begin{align}\label{eq: mini-SGD_updated_model_parameters_expected}
\begin{split}
    &\hspace{-3mm}\logistlinearincrew^{(t+1)} \leftarrow (1-\eta_t\lambda)\logistlinearincrew^{(t)} +\frac{\eta_t}{\increB^{(t)}} \sum_{\substack{ i \in \mathscr{B}^{(t)} \\,i \not \in \mathcal{R}}} y_i\textbf{x}_i f(y_i\logistlinearincrew^{(t)}\textbf{x}_i)
\end{split}
\end{align}
\vspace{-0.5mm}
Recall that $f(x) = 1- \frac{1}{1+e^{-x}}$\eat{ represents one minus the Sigmoid function}. Note that the linear coefficients \eat{ $a^{i,(t)}$ and $b^{i, (t)}$ in Equation \ref{eq: mini-SGD_updated_model_parameters} are computed by assigning the result $y_i\linearincrew^{(t)}\textbf{x}_i$ to the variable $x$ in $f(x)$, then the distance between $\linearincrew^{(t)}$ and $\logistlinearincrew^{(t)}$ should be still bounded by $O((\Delta x))$, which is same as before. However, }$a^{i,(t)}$ and $b^{i, (t)}$ in Equation \ref{eq: mini-SGD_logistic_regression_para_update} are actually derived \eat{by replacing $x$ in $f(x)$ with $y_i\textbf{w}^{(t)T}\textbf{x}_i$ (instead of $y_i\logistlinearincrew^{(t)}\textbf{x}_i$) in Equation \ref{eq: mini_sgd_logistic_regression} }in the training phase where all samples exist (rather than in the model update phase), which implies that a larger difference between $\linearincrew^{(t)}$ and $\logistlinearincrew^{(t)}$ should be expected. Surprisingly, we can prove that the distance between $\linearincrew^{(t)}$ and $\logistlinearincrew^{(t)}$ is still small enough.
\begin{theorem}\label{theorem: aproximation_bound_change}
$||E(\linearincrew^{(t)}- \logistlinearincrew^{(t)})||_2$ is bounded by $O(\frac{\Delta n}{n}\Delta x) + O((\frac{\Delta n}{n})^2) + O((\Delta x)^2)$, where $\Delta n$ is the number of the removed samples and $\Delta x$ is defined in Theorem \ref{theorem: aproximation_bound}.
\end{theorem}

\eat{The proof of Theorem \ref{theorem: aproximation_bound} and Theorem \ref{theorem: aproximation_bound_change} are provided in Appendix \ref{sec: aproximation_bound_proof} and \ref{sec: aproximation_bound_change_proof} respectively.}

\eat{\proofsketch

$E(||\linearincrew^{(t)}- \logistlinearincrew^{(t)}||_2)$ can be expressed as below:
\begin{align}
    \begin{split}
        &E(||\linearincrew^{(t)}- \logistlinearincrew^{(t)}||_2)\\
        % & = E(||\linearincrew^{(t)}- \linearw^{(t)} + \linearw^{(t)} - \textbf{w}^{(t)} + \textbf{w}^{(t)} - \logistlinearincrew^{(t)}||_2)\\
        &\leq E(||\linearincrew^{(t)}-  \linearw^{(t)}||_2) +  E(||\linearw^{(t)} - \textbf{w}^{(t)}||_2)\\
        & + E(||\textbf{w}^{(t)} - \logistlinearincrew^{(t)}||_2)
    \end{split}
\end{align}

According to Theorem \ref{theorem: aproximation_bound}, $E(||\linearw^{(t)} - \textbf{w}^{(t)}||_2) \leq O((\Delta x))$. We can also prove that $E(||\linearincrew^{(t)}-  \linearw^{(t)}||_2)$ and $E(||\textbf{w}^{(t)} - \logistlinearincrew^{(t)}||_2)$ are bounded by $O(\eta)$. So $E(||\linearincrew^{(t)}- \logistlinearincrew^{(t)}||_2)$ should be bounded by $O((\Delta x)) + O(\eta_t)$ (details omitted).}

\section{Implementation}\label{sec: implementation}
%Based on the theoretical results of the previous section, 
We now discuss how the ideas in the previous section are implemented in \pro\ and \proopt, for both linear and logistic
regression.
%We start with the version for linear regression before moving to the version for logistic regression. After introducing either \pro\ or \proopt, 
Along the way, time and space complexity analyses,
as well as theorems that justify our approximation strategies used in \pro\ and \proopt, 
are provided.

\subsection{\pro: Linear regression}
% \subsubsection{For dataset with small feature space}
\eat{Since the update rules of the linear regression model use linear operations, there is no need to use interpolation. }
In Equation \ref{eq: mini_sgd_linear_regression_provenance_update}, by setting all the $p_i$ as $\oneprov$, the expression $\sum_{\substack{ i \in \mathscr{B}^{(t)}, i \not \in \mathcal{R}}} p_i^2 * \eat{\cdot} \textbf{x}_i\textbf{x}_i^T$ becomes $\sum_{\substack{ i \in \mathscr{B}^{(t)}}} \textbf{x}_i\textbf{x}_i^T - \sum_{\substack{ i \in \mathscr{B}^{(t)}, i  \in \mathcal{R}}}$\\ $\textbf{x}_i\textbf{x}_i^T$, in which the first term, $\sum_{\substack{ i \in \mathscr{B}^{(t)}}} \textbf{x}_i\textbf{x}_i^T$, can be regarded as provenance information and thus cached as an intermediate result for each mini-batch during the training phase for the original model parameter $\textbf{w}^{(t)}$. Thus we only need to compute the latter term during the incremental update phase. $\sum_{\substack{ i \in \mathscr{B}^{(t)}, i \not \in \mathcal{R}}}\eat{\cdot} \textbf{x}_iy_i$ can be computed in a similar way. In the end, Equation \ref{eq: mini_sgd_linear_regression_provenance_update} is then rewritten as follows for the purpose of incremental updates:
\eat{n Equation \ref{eq: mini_sgd_linear_regression} can be rewritten in the form of matrix operations as shown below.  Note that these operations are linear, and there is no need to use linear interpolation.
\begin{align}\label{eq: gbm_linear_regression_matrix_form}
    \begin{split}
        % \textbf{w}^{(t+1)}& \leftarrow ((1-\eta_t\lambda)\textbf{I} - \frac{2\eta_t}{B}\textbf{X}^T_{\subids}\textbf{X}_{\subids})\textbf{w}^{(t)}\\ &+\frac{2\eta_t}{B}\textbf{X}^T_{\subids}\textbf{Y}_{\subids}
        \textbf{w}^{(t+1)}& \leftarrow ((1-\eta_t\lambda)\textbf{I} - \frac{2\eta_t}{B}\textbf{X}^T_{\mathscr{B}^{(t)}}\textbf{X}_{\mathscr{B}^{(t)}})\textbf{w}^{(t)}\\ &+\frac{2\eta_t}{B}\textbf{X}^T_{\mathscr{B}^{(t)}}\textbf{Y}_{\mathscr{B}^{(t)}}
    \end{split}
\end{align}
\noindent
Here, $X_{\subids}$ and $Y_{\subids}$ represent the sub-matrices composed of entries from the mini-batch $\subids$.   
During the \gbm\ over the original dirty dataset, the matrix multiplication results $\textbf{X}^T_{\subids}\textbf{X}_{\subids}$ and $\textbf{X}^T_{\subids}\textbf{Y}_{\subids}$ are recorded as ``provenance information'' for each mini-batch. Now suppose that $\Delta \textbf{X}_{\subids}$ and $\Delta \textbf{Y}_{\subids}$ represent the samples to be removed from each mini-batch during the data cleaning process.  Then we can use the formula below to update the model parameters, which modifies Equation \ref{eq: gbm_linear_regression_matrix_form} by subtracting $\Delta \textbf{X}^T_{\subids}\Delta \textbf{X}_{\subids}$ and $\Delta \textbf{X}^T_{\subids}\Delta \textbf{Y}_{\subids}$ from $\textbf{X}^T_{\subids}\textbf{X}_{\subids}$ and $\textbf{X}^T_{\subids}\textbf{Y}_{\subids}$ respectively:}
% \noindent
% \topskip
\vspace{-1mm}
\begin{align}\label{eq: gbm_linear_regression_incremental_updates}
    \begin{split}
         &\increw^{(t+1)} \leftarrow [(1-\eta_t\lambda)\textbf{I} -\frac{2\eta_t}{\increB^{(t)}} \sum_{\substack{ i \in \mathscr{B}^{(t)}}} \textbf{x}_i\textbf{x}_i^T\\
         &\hspace{-4mm} - \sum_{\substack{ i \in \mathscr{B}^{(t)} \\, i  \in \mathcal{R}}} \textbf{x}_i\textbf{x}_i^T] \increw^{(t)} + \frac{2\eta_t}{\increB^{(t)}} (\sum_{\substack{ i \in \mathscr{B}^{(t)}}}\eat{\cdot} \textbf{x}_iy_i - \sum_{\substack{ i \in \mathscr{B}^{(t)}, i  \in \mathcal{R}}}\eat{\cdot} \textbf{x}_iy_i)
    \end{split}
\end{align}
\vspace{-1mm}

\eat{To reduce the overhead of computing $\increw^{(t)}$ using Equation \ref{eq: gbm_linear_regression_incremental_updates}, }
\noindent
Note that $\sum_{\substack{ i \in \mathscr{B}^{(t)}, i  \in \mathcal{R}}} \textbf{x}_i\textbf{x}_i^T$ can be rewritten into matrix form, i.e. $\Delta \textbf{X}^T_{\mathscr{B}^{(t)}} \Delta \textbf{X}_{\mathscr{B}^{(t)}}$ where $\Delta \textbf{X}_{\mathscr{B}^{(t)}}$ is a matrix consisting of the removed samples in the mini-batch $\mathscr{B}^{(t)}$. 
% which can be further rewritten such that the
The associativity property of matrix multiplication can also be used to avoid expensive matrix-matrix multiplications (i.e. $\Delta \textbf{X}^T_{\mathscr{B}^{(t)}}\Delta \textbf{X}_{\mathscr{B}^{(t)}}$) by conducting more efficient matrix-vector multiplications instead (e.g. computing $\Delta \textbf{X}_{\mathscr{B}^{(t)}}\increw^{(t+1)}$ first and then multiplying the result by $\Delta \textbf{X}^T_{\mathscr{B}^{(t)}}$). \eat{i.e.:

\begin{align}\label{eq: gbm_linear_regression_incremental_updates2}
    \begin{split}
        &\increw^{(t+1)} \leftarrow ((1-\eta_t\lambda)\textbf{I} - \frac{2\eta_t}{\increB^{(t)}}(\textbf{X}^T_{\subids}\textbf{X}_{\subids})\increw^{(t)}\\
        &-\Delta \textbf{X}^T_{\subids}\Delta \textbf{X}_{\subids}\increw^{(t)})\\ &+\frac{2\eta_t}{\increB^{(t)}}(\textbf{X}^T_{\subids}\textbf{Y}_{\subids}-\Delta \textbf{X}^T_{\subids}\Delta \textbf{Y}_{\subids})\\
    \end{split}
\end{align}}

Suppose that $\Delta B$ samples are removed from each mini-batch on average, then the time complexity of updating the model parameters in each iteration using Equation \ref{eq: gbm_linear_regression_incremental_updates} will be $O(\Delta B m + m^2)$ 
(recall that the dimension of $\textbf{X}$ is $n \times m$).
% and thus the total time complexity will be $O(\tau m^2 + \tau \Delta B m)$ if there are $\tau$ iterations in total.  
In contrast, the time complexity for retraining from scratch (i.e. not caching $\sum_{\substack{ i \in \mathscr{B}^{(t)}}} \textbf{x}_i\textbf{x}_i^T$ in Equation \ref{eq: gbm_linear_regression_incremental_updates}) will be $O((B - \Delta B) m)$. Of course, performance predictions based on asymptotic complexity give only
very rough guidance, and we conduct experiments for realistic assessments.
Still, the bounds above suggest, for example, that for small
$\Delta B$ and $m\ll B$ incremental deletions with \pro\ work better than retraining
(and our experiments verify this, see Section~\ref{sec: experiment}).

\reminder{It is risky to say anything because it is not clear that
asymptotic complexity applies here. Basically the comparison is between $m+\Delta B$ and $B-\Delta B$, ignoring constants.
But can we ignore constants when $m,B,\Delta B$ are smallish numbers? Ultimately it's the experiments that provide valid comparisons. Alternatively we can do a finer analysis of the multiplications and additions involved that would be better than Big-Oh.}

\eat{\textbf{Space complexity analysis} For efficient computations in \pro, we need to cache $\textbf{X}^T_{\subids}\textbf{X}_{\subids}$ and $\textbf{X}^T_{\subids}\textbf{Y}_{\subids}$ for each step, which have dimensions $m \times m$ and $m \times 1$ respectively. So the total space complexity will be $O(\tau m^2)$, which is acceptable when the feature space is small.

\scream{Can space be reused in each epoch? So we can avoid multiplying by $\tau$.}}

%\susan{We have a problem with t and T.  You used T here for the number of iterations (which I changed to t), and in the beginning of the next section switched to t.  However, I think you are using T for "transpose" in the preceding text.  This problem persists in 5.3 where you use T for the number of iterations and for transpose (?).  I'm confused by the notation here.}\yinjun{I changed it to $\tau$, which is also a common notation for iteration in practice}

% \begin{algorithm}[h!] 
% % \small
% \footnotesize
%  \SetKwInOut{Input}{Input}
%  \SetKwInOut{Output}{Output}
%  \Input{Training dataset with errors: $(X, Y)$, the erroneous subset of the dataset: $(\Delta X, \Delta Y)$, initialized model parameter $\theta_0$}

%  \Output{Updated model parameter $\theta$}

%  \caption{Overview of model update algorithms for }
%  \label{covering_sets_calculation}
%  \end{algorithm}

% \subsubsection{For dataset with large feature space}
Typically, however, a smaller mini-batch size $B$ is used. To deal with the case in which $m>B$,
% the time complexity of \pro\ will be worse than the solution to retraining from the scratch. But at the same time, 
we notice that the rank of the intermediate result, $\sum_{\substack{ i \in \mathscr{B}^{(t)}}} \textbf{x}_i\textbf{x}_i^T$ should be no more than $B$, thus smaller than $m$ when $B$ is smaller than $m$.
%\scream{Whose dimension are you referring to below?}\scream{Yinjun: Done}
This motivates us to reduce the dimension of the intermediate results using SVD, i.e. $\sum_{\substack{ i \in \mathscr{B}^{(t)}}} \textbf{x}_i\textbf{x}_i^T = \textbf{U}^{(t)}\textbf{S}^{(t)}\textbf{V}^{T,(t)}$, where $\textbf{S}^{(t)}$ is a diagonal matrix whose diagonal elements represent the singular values, while $\textbf{U}^{(t)}$ and $\textbf{V}^{(t)}$ are the left and right singular vectors. Suppose after SVD, we only keep the $r$ largest singular values and the corresponding singular vectors where $r \ll B$, then $\sum_{\substack{ i \in \mathscr{B}^{(t)}}} \textbf{x}_i\textbf{x}_i^T$ is approximated by $\textbf{U}^{(t)}_{1..r}\textbf{S}^{(t)}_{1..r}\textbf{V}^{(t)T}_{1..r}$ ($\textbf{U}^{(t)}_{1..r},\textbf{V}^{(t)}_{1..r}$ represents the submatrix composed of the first $r$ columns and $\textbf{S}^{(t)}_{1..r}$ is a diagonal matrix composed of the first $r$ eigenvalues in $\textbf{S}^{(t)}$).  Thus Equation \ref{eq: gbm_linear_regression_incremental_updates} is rewritten as:

% \vspace{-1cm}
% \noindent
% \topskip

\begin{align}\label{eq: gbm_linear_regression_incremental_updates_svd}
    \begin{split}
        &\increw^{(t+1)} \leftarrow [(1-\eta_t\lambda)\textbf{I} - \frac{2\eta_t}{\increB^{(t)}}(\textbf{U}^{(t)}_{1..r}\textbf{S}^{(t)}_{1..r}\textbf{V}^{(t)T}_{1..r}\\
        &\hspace{-2mm}-\Delta \textbf{X}^T_{\subids}\Delta \textbf{X}_{\subids})]\increw^{(t)}+\frac{2\eta_t}{\increB^{(t)}}(\sum_{\substack{ i \in \mathscr{B}^{(t)}}}\textbf{x}_i y_i - \sum_{\substack{ i \in \mathscr{B}^{(t)}\\, i  \in \mathcal{R}}}\eat{\cdot} \textbf{x}_i y_i)
    \end{split}
\end{align}

Here we can cache the results of $\textbf{U}^{(t)}_{1..r}\textbf{S}^{(t)}_{1..r}$ (denoted by $\textbf{P}_{1..r}^{(t)}$) and $\textbf{V}^{(t)}_{1..r}$ for efficient updates, both of which have dimensions $m \times r$. 

\textbf{Time complexity.} The time complexity to update the model parameters using this approach is $O(rm + \Delta B m)$ for each iteration since the computation time is dominated by the matrix-vector computation, e.g. the multiplications of $\textbf{V}_{1..r}^{T, (t)}$ and $\increw^{(t)}$.  This is more efficient than retraining from scratch, which has time complexity $O((B-\Delta B) m)$. So the total complexity for \pro\ is $O(\tau rm + \tau \Delta B)$, where $\tau$ is the number of iterations in the training phase.

\textbf{Space complexity.} Using this approximation, at each iteration we only need to cache $\textbf{P}_{1..r}^{(t)}$ and $\textbf{V}_{1..r}^{(t)}$, which require space $O(rm)$. So the total space complexity will be $O(\tau rm)$ for $\tau$ iterations.

%\scream{Can space be reused in each epoch? So we can avoid multiplying by $\tau$. Currently, no ,since the intermediate results vay between each iteration and thus this space complexity is necessary}

%\scream{The effect of the approximations}
%\scream{What is $\delta \%$ below?}
% \textbf{The approximation ratio}
% Then we provide a theoretical results to show the effect of the approximation through SVD.
\begin{theorem}\textbf{Approximation ratio}\label{theorem: aproximation_bound_svd}
Under the convergence conditions for $\textbf{w}^{(t)}$, $||\textbf{w}^{(t)}||$ should be bounded by some constant $C$. Suppose $\frac{||\textbf{U}^{(t)}_{1..r}\textbf{S}^{(t)}_{1..r}\textbf{V}^{T,(t)}_{1..r}||_2}{||\textbf{U}^{(t)}\textbf{S}^{(t)}\textbf{V}^{T,(t)}||_2} \geq 1-\epsilon$ where $\epsilon$ is a small value, then the change of model parameters caused by the approximation will be bounded by $O(\epsilon)$.
\end{theorem}

This shows that with proper choice of $r$ in the SVD approximation\eat{ and small portions of removed samples}, the updated model parameters computed by \pro\ or \proopt\ should be still very close to the expected result. So in our implementations, $r$ is chosen based on $\epsilon$ (say 0.01) such that the inequality in Theorem \ref{theorem: aproximation_bound_svd} is satisfied.

\subsection{\proopt: Optimizations for linear regression}\label{sec: opt_linear_regression}

When the feature space is small, additional optimizations can be used for linear regression.
\eat{However, in practice the gradient-based method (\gbm) uses repetitive for-loop iterations and thus the term $O(\tau m^2)$ will become the dominant factor both for \pro\ and for retraining from scratch.  This means that \pro\ will not end up with a significant performance advantage.  }
Note that according to \cite{karimi2016linear}, the model parameters derived by both \sgd\ and \minisgd\ will end up with statistically the same results as \gd.  This means that the update rule in Equation \ref{eq: mini_sgd_linear_regression} and Equation \ref{eq: gbm_linear_regression_incremental_updates} could be approximated by its alternative using \gd, i.e.:
\vspace{-1.6mm}
\begin{align}\label{eq: gd_linear_regression_matrix_form}
    \begin{split}
    % \vspace{-5mm}
        \textbf{w}^{(t+1)}& \leftarrow ((1-\eta_t\lambda)\textbf{I} - \frac{2\eta_t}{n}\textbf{X}^T\textbf{X})\textbf{w}^{(t)}+\frac{2\eta_t}{n}\textbf{X}^T\textbf{Y}
        \end{split}
\end{align}
\vspace{-4.5mm}
\begin{align}
\begin{split}
        \increw^{(t+1)}& \leftarrow ((1-\eta_t\lambda)\textbf{I} - \frac{2\eta_t}{n-\Delta n}(\textbf{X}^T\textbf{X}-\Delta\textbf{X}^T\Delta \textbf{X}))\increw^{(t)}\\
        &+\frac{2\eta_t}{n - \Delta n}(\textbf{X}^T\textbf{Y} - \Delta\textbf{X}^T\Delta\textbf{Y})
    \end{split}
\end{align}

% the following equality should hold for linear regression (by replacing $\triangledown h(\textbf{w})$ with the gradient of linear regression):
% \begin{equation}
%     E(\lambda\textbf{w} + \frac{2}{B} \sum_{i=r^{(t)}}^{r^{(t)}+B-1} \textbf{x}_i(\textbf{x}_i^T\textbf{w} - y_i)) = \lambda\textbf{w} + \frac{2}{n} \sum_{i=1}^{n} \textbf{x}_i(\textbf{x}_i^T\textbf{w} - y_i)
% \end{equation}
% i.e. by computing $\textbf{w}^{(t)}$ recursively such that all the batches are covered, Equation \ref{eq: gbm_linear_regression_matrix_form} can be further rewritten as:
% \begin{align}\label{eq: gbm_linear_regression_matrix_form_inductively}
%     \begin{split}
%         \textbf{w}^{(t+1)}& \leftarrow \Pi_{i=1}^{\tau}((1-\eta_t\lambda)\textbf{I} - \frac{2\eta_t}{B}\textbf{M}_{\subids})\textbf{w}^{(t)}+\frac{2\eta_t}{B}\textbf{N}_{\subids}
%     \end{split}
% \end{align}

% By going through all the batches, 
\noindent
in which $(\Delta \textbf{X}, \Delta \textbf{Y})$ represent the removed samples while $\Delta n$ represents the number of those samples. Let $\textbf{M}$ and $\textbf{N}$ denote $\textbf{X}^T\textbf{X}$ and $\textbf{X}^T\textbf{Y}$ respectively. Then eigenvalue decomposition can be applied over $\textbf{M}$, i.e. $\textbf{M} = \textbf{Q}\; diag\; (\{c_{i}\}_{i=1}^n)\; \textbf{Q}^{-1}$ (where $c_{i}$ represents the eigenvalues of $\textbf{M}$).  This is then plugged into Equation \ref{eq: gd_linear_regression_matrix_form} and computed recursively, which results in the following formula:
{
\begin{align}\label{eq: gd_linear_regression_matrix_form_eigen}
    \begin{split}
        &\hspace{-2mm} \textbf{w}^{(t+1)} 
        % =\textbf{Q}\; diag \;(\{1-\eta_t\lambda - \frac{2\eta_t}{B}c_i\}_{i=1}^n) \; \textbf{Q}^{-1}\textbf{w}^{(t)} + \frac{2\eta_t}{B}\textbf{N}\\&
        =\textbf{Q} \; diag  \;(\{\Pi_{j=1}^t (1-\eta_j\lambda - \frac{2\eta_j}{n}c_i)\}_{i=1}^n) \; \textbf{Q}^{-1}\textbf{w}^{(0)}\\
        &+ \textbf{Q} diag(\sum_{l=1}^{t-1}\eta_l\{\Pi_{j=l+1}^t (1-\eta_j\lambda - \frac{2\eta_j}{n}c_i)\}_{i=1}^n) \; \textbf{Q}^{-1}\frac{2\textbf{N}}{n}\\
    \end{split}
\end{align}

}%

This indicates that once the eigenvalues and eigenvectors of each $\textbf{M}$ are given, we can derive $\textbf{w}^{(t)}$ by simply computing the product, $\Pi_{j=1}^t (1-\eta_j\lambda - \frac{2\eta_j}{n}c_i)$, and the sum of the product, $\sum_{l=1}^{t-1}\eta_l\Pi_{j=l+1}^t (1-\eta_j\lambda - \frac{2\eta_j}{n}c_i)$, on diagonal entries. The overhead of this is only $O(\tau m)$ (recall that $\tau$ represents the total iteration number), and thus we avoid the repetitive matrix multiplication operations through the for-loops. 
% Unfortunately, given a subset $\Delta \textbf{X}$ of $\textbf{X}$ to be removed, applying this over the updated $\textbf{M}$, i.e. $\textbf{M}'=\textbf{X}^T\textbf{X} - \Delta \textbf{X}^T\Delta \textbf{X}$, can still be very expensive since eigenvalue decomposition is very expensive.\footnote{The time complexity is $O(n^3)$ where $n$ is the size of the input square matrix, which is far more expensive than the matrix multiplication operation over square matrix with dimension $n\times n$.} The time complexity of the naive solution is $O(tB m)$
Also, observe that $\textbf{M}' = \textbf{X}^T\textbf{X} - \Delta \textbf{X}^T\Delta \textbf{X}$ can be regarded as a small change over $\textbf{M}$ when $\Delta n$ is small.  Thus we can use the results on incremental updates over eigenvalues in ~\cite{ning2010incremental}, i.e. when the difference between the eigenvectors of $\textbf{M}'$ and that of $\textbf{M}$ is negligible then the eigenvalues of $\textbf{M}'$ are estimated as:
\begin{align}\label{eq: eigen_value_incremental}
    \begin{split}
        \textbf{Q}^{-1}\textbf{M}'\textbf{Q} = diag(\{ c_i'\}_{i=1}^n)
    \end{split}
\end{align}
\noindent
Here, $c_i'$ represents the approximated $i^{th}$ eigenvalue of $\textbf{M}'$. 
\eat{So the eigenvalue decomposition of $\textbf{M}'$ can be approximated by:
\begin{equation}
    \textbf{M}' = \textbf{Q}\; diag(\{c_i + \Delta c_i\}_{i=1}^n)\textbf{Q}^{-1}
\end{equation}}
\noindent
It indicates that we can apply eigenvalue decomposition over $\textbf{M}$ {\em offline} before the model incremental update phase and use Equation \ref{eq: eigen_value_incremental} to get the updated eigenvalues {\em online}.  
% This incurs the time overhead for two matrix multiplication operations, i.e. $O(m^3)$, which is far more efficient than eigenvalue decomposition over $\textbf{M}'$. 

\textbf{Time complexity.}
The time complexity for updating the model parameters is dominated by the computation of $c_i'$\eat{, i.e. $\textbf{Q}^{-1}\Delta \textbf{X}^T \textbf{X} \textbf{Q}$}, which is followed by the computation over each $c_i'$ as Equation \ref{eq: gd_linear_regression_matrix_form_eigen} does.
% and the power computation over each updated eigenvalues.  
\eat{What is the "power computation"? Yinjun: have explained it in another way}
These have time complexities $O(\min\{\Delta n, m\} m^2)$ and $O(\tau m)$, respectively. So the total time complexity is $O(\min\{\Delta n, m\} m^2) + O(\tau m)$, which can be
%\scream{discuss this: maybe we should say "can be" instead of "is"}
more efficient than the closed-form solution
(see experiments in 
Section \ref{sec: experiment}).

\textbf{Space complexity.}
The method avoids caching the provenance information at each iteration, and only requires caching $Q$, $Q^{-1}$ and all the eigenvalues $c_i$, which takes space $O(m^2)$

% The approximation effect of this optimization strategy is as follows:
%\scream{The effect of the approximations}
\begin{theorem}(\textbf{Approximation ratio})\label{theorem: aproximation_bound_eigen}
The approximation of \proopt\ over the model parameters is bounded by $O(||\Delta \textbf{X}^T \Delta \textbf{X}||)$
\end{theorem}

This shows that with small number of removed samples, the approximation ratio should be very small.

% In the case of small $m$, significant performance gains can be achieved compared to \pro. We call this strategy \proopt, and experimentally evaluate
% and take the advantage of the performance gains achieved by Equation \ref{eq: gbm_linear_regression_matrix_form_inductively_eigen}. 
% its efficiency in Section \ref{sec: experiment}. 

\subsection{\pro: Logistic regression}\label{sec: priu_logistic_regression}

% \subsubsection{For dense dataset with small feature space}
As the first step of the implementation of \pro\ for logistic regression,\eat{ is more complicated than that for linear regression due to the non-linear operations.} non-linear operations are linearized using  piecewise linear interpolation. Then, based on the analysis in Section \ref{sec: model}, given the ids of the samples to be removed in dense datasets, $\mathcal{R}$, Equation \ref{eq: mini-SGD_logistic_regression_para_update} is rewritten as follows:

% \topskip
\begin{align}\label{eq: provenance_update_formular}
    \begin{split}
         \linearincrew^{(t+1)}& \leftarrow [(1-\eta_t\lambda)\textbf{I} + \frac{\eta_t}{\increB^{(t)}} (\textbf{C}^{(t)}-\Delta \textbf{C}^{(t)})]\linearincrew^{(t)} \\
         &+\frac{\eta_t}{\increB^{(t)}} (\textbf{D}^{(t)}-\Delta \textbf{D}^{(t)})\\
    \end{split}
\end{align}

where $\textbf{C}^{(t)}, \textbf{D}^{(t)}, \Delta \textbf{C}^{(t)}, \Delta \textbf{D}^{(t)}$ are:

{
\begin{align*}
    % \begin{split}
        &\textbf{C}^{(t)}  = \sum_{i \in \mathscr{B}^{(t)}}a^{i, (t)}\textbf{x}_i\textbf{x}_i^T, \Delta \textbf{C}^{(t)} = \sum_{\substack{i  \in \mathcal{R}, i \in  \mathscr{B}^{(t)}}}a^{i, (t)}\textbf{x}_i\textbf{x}_i^T
    \end{align*}
    \begin{align*}
        &\textbf{D}^{(t)} = \sum_{i \in \mathscr{B}^{(t)}} b^{i, (t)}y_i\textbf{x}_i, \Delta\textbf{D}^{(t)}  = \sum_{\substack{i  \in \mathcal{R}, i \in  \mathscr{B}^{(t)}}} b^{i, (t)}y_i\textbf{x}_i
    % \end{split}
\end{align*}
}

Similar to linear regression, the intermediate results $\textbf{C}^{(t)}$ and $\textbf{D}^{(t)}$ are cached and the dimension of $\textbf{C}^{(t)}$ can be reduced by using SVD before the model update phase, which can happen offline. \eat{Notice that the dimension of $\textbf{C}^{(t)}$ is $m \times m$: To further improve performance in the model update phase and also reduce the memory footprint, we can apply SVD over $\textbf{C}^{(t)}$ to reduce its dimension, i.e.}Suppose after SVD, $\textbf{C}^{(t)} \approx \textbf{P}_{1..r}^{(t)}\textbf{V}_{1..r}^{T, (t)}$, in which $\textbf{P}_{1..r}^{(t)}$ and $\textbf{V}_{1..r}^{(t)}$ are two matrices with dimension $m \times r$.
% , cached for further use in the model update phase. 
In the end, Equation \ref{eq: provenance_update_formular} is modified as below for incremental model updates:
\vspace{-1mm}
\begin{align}\label{eq: provenance_update_formular_svd}
    \begin{split}
         \linearincrew^{(t+1)}& \leftarrow [(1-\eta_t\lambda)\textbf{I} + \frac{\eta_t}{\increB^{(t)}} (\textbf{P}_{1..r}^{(t)}\textbf{V}_{1..r}^{T, (t)}-\Delta \textbf{C}^{(t)})]\linearincrew^{(t)} \\
         &+\frac{\eta_t}{\increB^{(t)}} (\textbf{D}^{(t)}-\Delta \textbf{D}^{(t)})\\
    \end{split}
\end{align}

\textbf{Time complexity.} To apply Equation \ref{eq: provenance_update_formular_svd} in the model update phase, the computation of $\textbf{P}_{1..r}^{(t)}\textbf{V}_{1..r}^{T, (t)}\linearincrew^{(t)}$ and $\Delta \textbf{C}^{(t)}\linearincrew^{(t)}$ become the major overhead\eat{.  These consist of a series of matrix-vector multiplications, e.g. $\textbf{V}_{1..r}^{T, (t)}\linearincrew^{(t)}$, and}, which have time complexity $O(rm)$ and $O(\Delta B m)$, respectively. Suppose there are $\tau$ iterations in total, then the total time complexity is $O(\tau (rm + \Delta B m))$.  In comparison, the time complexity of retraining from scratch is $O(\tau ((B-\Delta B) m + C_{non} m))$, where $C_{non}$ represents the overhead of the non-linear operations. When $r \ll B$ and $\Delta B \ll B$, we can therefore expect  \pro\ to be more efficient than retraining from scratch.

\eat{The implementation therefore consists of two phases, an offline phase and an online phase. The former is used to capture and compute the provenance information of the \gbm\ process, while the latter is used to load the provenance information and derive the updated model parameters. In what follows, we will only discuss details of the implementations of \pro\ and \proopt\ for binary logistic regression and note that the implementation of multinomial logistic regression can be designed in a similar way.}

\eat{\textbf{Offline phase}
When \gbm\ is applied over the original training dataset in the training phase, the model parameters derived at every iteration, $\textbf{w}^{(t)}$, are recorded.  The computation of provenance information is delayed until the end of the \gbm, which minimizes the effect of capturing provenance during the training phase.

%\susan{What about binary logistic regression?}\yinjun{yes..it should be binary logistic regression, modified}

Piecewise linear interpolation is then applied over the non-linear operation in the update rules of binary logistic regression, i.e., $f(x) = 1-\frac{1}{1+e^{-x}}$, where $x = y_i\textbf{w}^{(t)T}\textbf{x}_i (i=1,2,\dots, n; t = 1,2,\dots, T)$ to obtain the linear coefficients $a^{i,(t)}$ and $b^{i,(t)}$ presented in Equation \ref{eq: mini_sgd_instantiation_approx}.  This is followed by pre-computing the following two terms from Equation \ref{eq: mini_sgd_instantiation_approx}:
\begin{align}
    \begin{split}
        \textbf{C}^{(t)} & = \sum_{i \in \mathscr{B}^{(t)}}a^{i, (t)}\textbf{x}_i\textbf{x}_i^T, \textbf{D}^{(t)}  = \sum_{i \in \mathscr{B}^{(t)}} b^{i, (t)}y_i\textbf{x}_i
    \end{split}
\end{align}

As we discussed for the solutions of linear regression, when $m$ is larger than $B$, the time overhead of the solution above will be dominated by the overhead of the matrix-vector multiplications between $\textbf{C}^{(t)}$ and $\linearincrew^{(t)}$, i.e. $O(m^2)$, thus failing to beat the method to retrain from the scratch, which can be even worse for multi-nomial logistic regression where the total number of model parameters is typically very large, equal to the multiplication of the feature number and the number of all possible classes. To tackle this issue, we still consider reducing the dimensionality of $\textbf{C}^{(t)}$ with SVD. So instead of caching $\textbf{C}^{(t)}$ in the offline phase, we determined to cache a pair of its lower-rank matrices (denoted by $\textbf{P}_{1..r}^{(t)}$ and $\textbf{V}_{1..r}^{(t)}$ such that $\textbf{P}_{1..r}^{(t)}\textbf{V}_{1..r}^{T, (t)} \approx \textbf{C}^{(t)}$)

% (similarly, we will pre-compute $\textbf{C}^{(t)}$ and $\textbf{D}^{(t)}$ in Equation \ref{eq: mini_sgd_instantiation_approx} for multi-nomial logistic regression)

At the end of each iteration, the linear coefficients $a^{i,(t)}$ and $b^{i,(t)}$ along with the intermediate results $\textbf{C}^{(t)}$ and $\textbf{D}^{(t)}$ are cached for use during the online phase.

\textbf{Online phase.}
At the beginning of this phase, the subset of training data to be removed (denoted by $(\Delta \textbf{X}, \Delta \textbf{Y})$) is specified. 
\eat{ and can be provided by the experts who points out which piece of training samples are errors in the context of data cleaning application or by the users who expect to see the influence of the training samples of interest}
%Once obtaining $(\Delta \textbf{X}, \Delta \textbf{Y})$, 
Using this as input, the next step is to derive the update rules in the same form as Equation \ref{eq: mini_sgd_instantiation_approx} but without $(\Delta \textbf{X}, \Delta \textbf{Y})$, which is achieved by subtracting the terms from $\textbf{C}^{(t)}$ and $\textbf{D}^{(t)}$, i.e.:

\begin{align}
    \begin{split}
        \Delta \textbf{C}^{(t)} & 
        % = \sum_{\substack{\textbf{x}_{i_j} \in \Delta \textbf{X}, \\ i_j \in  \mathscr{B}^{(t)}}}a^{i_j, (t)}\textbf{x}_{i_j}\textbf{x}_{i_j}^T 
        = [\textbf{x}_{i_1}, \textbf{x}_{i_2}, \dots]_{\substack{\textbf{x}_{i_j} \in \Delta \textbf{X},\\ i_j \in  \mathscr{B}^{(t)}}} \times  diag(a^{i_1, (t)}, a^{i_2, (t)}, \dots)_{i_j \in  \mathscr{B}^{(t)}} \\ 
        & \times [\textbf{x}_{i_1}, \textbf{x}_{i_2}, \dots]^T_{\substack{\textbf{x}_{i_j} \in \Delta \textbf{X},\\ i_j \in  \mathscr{B}^{(t)}}}\\ 
        \Delta\textbf{D}^{(t)} &
        % = \sum_{\substack{(\textbf{x}_{i_j}, y_{i_j}) \in (\Delta \textbf{X}, \Delta \textbf{Y}), \\ i_j \in  \mathscr{B}^{(t)}}} b^{i_j, (t)}y_{i_j}\textbf{x}_{i_j} 
        = [\textbf{x}_{i_1}, \textbf{x}_{i_2}, \dots]_{\substack{\textbf{x}_{i_j} \in \Delta \textbf{X}, \\ i_j \in  \mathscr{B}^{(t)}}} \times diag(a^{i_1, (t)}, a^{i_2, (t)}, \dots)_{\substack{i_j \in  \mathscr{B}^{(t)}}}\\
        & \times [y_{i_1}, y_{i_2}, \dots]^T_{\substack{y_{i_j} \in \Delta \textbf{Y},\\ i_j \in  \mathscr{B}^{(t)}}}
    \end{split}
\end{align}

\eat{which, however, can be effectively computed by matrix operations for all $t$ simultaneously, i.e.:
\begin{align}\label{eq: implement_matrix_form}
    \begin{split}
        &\begin{bmatrix}\textbf{C}^{(t_1)'} & \textbf{C}^{(t_2)'} & \dots & \textbf{C}^{(t_s)'} &\dots \end{bmatrix}\\
        &= \begin{bmatrix}\textbf{x}_{i_1}\textbf{x}_{i_1}^T  & \dots & \textbf{x}_{i_j}\textbf{x}_{i_j}^T & \dots \end{bmatrix}\begin{bmatrix}a^{i_1, (t_1)}& a^{i_1, (t_2)}&\dots & a^{i_1, (t_s)} & \dots \\ a^{i_2, (t_1)}&a^{i_2, (t_2)}& \dots & a^{i_2, (t_s)} & \dots \\\dots \\ a^{i_j, (t_1)}&a^{i_j, (t_2)}& \dots & a^{i_j, (t_s)} & \dots \\\end{bmatrix}\\
        &\begin{bmatrix}\textbf{D}^{(t_1)'} & \textbf{D}^{(t_2)'} & \dots & \textbf{D}^{(t_s)'} &\dots \end{bmatrix}\\
        &= \begin{bmatrix}\textbf{x}_{i_1} & \dots & \textbf{x}_{i_j} & \dots \end{bmatrix} \begin{bmatrix}b^{i_1, (t_1)}y_{i_1}&b^{i_1, (t_2)}y_{i_1}& \dots & b^{i_1, (t_s)}y_{i_1}& \dots \\b^{i_2, (t_1)}y_{i_2} & b^{i_2, (t_2)}y_{i_2}& \dots & b^{i_2, (t_s)}y_{i_2} & \dots \\ \dots \\ b^{i_j, (t_1)}y_{i_j} & b^{i_j, (t_2)}y_{i_j} & \dots & b^{i_j, (t_s)}y_{i_j} & \dots \\ \dots \end{bmatrix}
    \end{split}
\end{align}

where $t_1, t_2, \dots, t_s, \dots$ represents the iterations where some batch $(\textbf{X}_{r..r+B}, \textbf{Y}_{r..r+B})$ is used ($r = r^{(t_1)} = r^{(t_2)} = \dots = r^{(t_s)}$). The formula in Equation \ref{eq: implement_matrix_form} is far more efficient than the naive approach by iterating every remaining sample in $(\textbf{X}_{r..r+B}, \textbf{Y}_{r..r+B})$ for every iteration $t$, whose time complexity is bounded by $O(tB m^2)$}

\eat{we  to obtain $\Delta \textbf{C}^{(t)}$ and $\Delta \textbf{D}^{(t)}$, the value of each $a^{i_j, (t)}\textbf{x}_{i_j}\textbf{x}_{i_j}^T$ and $b^{i_j, (t)}y_{i_j}\textbf{x}_{i_j}$ is necessary, which can be computed 
\eat{ Since within each super-iteration (recall that in \minisgd, there are multiple super-iterations, in each of which each mini-batch is traversed) every $(\textbf{x}_{i}, y_{i})$ is traversed, all the $a^{i_j, (t)}\textbf{x}_{i_j}\textbf{x}_{i_j}^T$ in a super-iteration $\tau$ can be computed}in parallel using matrix multiplication:
\eat{ \footnote{\url{ https://docs.scipy.org/doc/numpy/user/basics.broadcasting.html#module-numpy.doc.broadcasting}}}
\begin{align}\label{eq: implement_matrix_form}
    \begin{split}
        % &\begin{bmatrix}\textbf{C}^{(t_1)'} & \textbf{C}^{(t_2)'} & \dots & \textbf{C}^{(t_s)'} &\dots \end{bmatrix}\\
        &\begin{bmatrix}vec(\textbf{x}_{i_1}\textbf{x}_{i_1}^T) & \dots & vec(\textbf{x}_{i_j}\textbf{x}_{i_j}^T) & \dots \end{bmatrix}^T \times \begin{bmatrix}a^{i_1, (t)}\\ a^{i_2, (t)}\\\dots \\ a^{i_j, (t)}\end{bmatrix}\\
        % &\begin{bmatrix}\textbf{D}^{(t_1)'} & \textbf{D}^{(t_2)'} & \dots & \textbf{D}^{(t_s)'} &\dots \end{bmatrix}\\
        % &= \begin{bmatrix}\textbf{x}_{i_1} & \dots & \textbf{x}_{i_j} & \dots \end{bmatrix} \begin{bmatrix}b^{i_1, (t_1)}y_{i_1}&b^{i_1, (t_2)}y_{i_1}& \dots & b^{i_1, (t_s)}y_{i_1}& \dots \\b^{i_2, (t_1)}y_{i_2} & b^{i_2, (t_2)}y_{i_2}& \dots & b^{i_2, (t_s)}y_{i_2} & \dots \\ \dots \\ b^{i_j, (t_1)}y_{i_j} & b^{i_j, (t_2)}y_{i_j} & \dots & b^{i_j, (t_s)}y_{i_j} & \dots \\ \dots \end{bmatrix}
    \end{split}
\end{align}
\noindent
Here $vec(\textbf{x}_{i_1}\textbf{x}_{i_1}^T)$ is the vectorized version of $\textbf{x}_{i_1}\textbf{x}_{i_1}^T$, which is thus an $m^2 \times 1$ vector. If we assume that there are $\Delta B$ samples removed from each iteration, then the two matrices used in Equation \ref{eq: implement_matrix_form} have dimension $m^2 \times \Delta B$ and $\Delta B \times 1$ respectively, which thus incurs time complexity $O(\Delta B m^2)$. In the end, by reshaping the result of Equation \ref{eq: implement_matrix_form}, we can get the result for $\Delta \textbf{C}^{(t)}$.\eat{will be:
\begin{align}\label{eq: implement_matrix_form_res}
    \begin{split}
        &\begin{bmatrix}\sum_{}vec(a^{i_1, (t)}\textbf{x}_{i_1}\textbf{x}_{i_1}^T) & \dots & vec(a^{i_j, (t)}\textbf{x}_{i_j}\textbf{x}_{i_j}^T) & \dots \end{bmatrix}\\
    \end{split}
\end{align}}

$\Delta \textbf{D}^{(t)}$ can be computed in a similar way. 
% Then $\Delta \textbf{C}^{(t)}$ and $\Delta \textbf{D}^{(t)}$ can be computed by summing up $vec(a^{i_j, (t)}\textbf{x}_{i_j}\textbf{x}_{i_j}^T)$ over all the samples in the mini-batch at $t_{th}$ iteration. Since the dimension of the two matrices for the element-wise operation in Equation \ref{eq: implement_matrix_form} is $|\Delta \textbf{X}| \times m^2$ and $|\Delta \textbf{X}| \times 1$ respectively, the time complexity to compute all the $\Delta \textbf{C}^{(t)}$ and $\Delta \textbf{D}^{(t)}$ within certain super-iteration $\tau$ will be $O(|\Delta \textbf{X}| m^2)$ ($|\Delta \textbf{X}|$ represents the number of samples in $\Delta \textbf{X}$). 
To obtain further performance gains, we can  modify Equation \ref{eq: implement_matrix_form} to compute $\Delta \textbf{C}^{(t)}$ and $\Delta \textbf{D}^{(t)}$ across all the super-iterations simultaneously.}
%which is omitted due to the space limit.\eat{, which is called Equation \ref{eq: implement_matrix_form}-opt (details omitted) and requires $O(T |\Delta \textbf{X}| m^2)$ operations in total where $T$ is the total number of super-iterations.}

The update rule of the \gbm\ for the training data without $(\Delta \textbf{X}, \Delta \textbf{Y})$ can then be computed as follows using the result of $\textbf{C}^{(t)}$, $\Delta \textbf{C}^{(t)}$, $\textbf{D}^{(t)}$ and $\Delta \textbf{D}^{(t)}$:

where $z^{(t)}$  represents the number of remaining training samples in the mini-batch at the $t^{th}$ iteration. 

Note that instead of computing $\Delta \textbf{C}^{(t)}$ explicitly, similar to \pro\ for linear regression, we can multiply $[\textbf{x}_{i_1}, \textbf{x}_{i_2}, \dots]^T_{\substack{\textbf{x}_{i_j} \in \Delta \textbf{X},\\ i_j \in  \mathscr{B}^{(t)}}}$ and $\linearincrew^{(t)}$ first to avoid less efficient matrix-matrix multiplications.

\textbf{Time complexity}
Suppose there are $\tau$ iterations in total, then the total time complexity of the online phase is $O(\tau\Delta B m + \tau m^2)$, where the major overhead comes from two matrix-vector multiplications, i.e. $\textbf{C}^{(t)}$ times $\linearincrew^{(t)}$ and $\Delta \textbf{C}^{(t)}$ times $\linearincrew^{(t)}$. In comparison, the time complexity of the method by retraining from the scratch is $O(\tau ((B-\Delta B) m + C_{non} m))$ where $C_{non}$ represents the overhead on the non-linear operations. As we can see, with small $m$ and large $B$, \pro\ can be more efficient than retraining from the scratch.

% \scream{the complexity is calculated as below: in logistic regression, the first step is to multiply $\textbf{X}$ with \textbf{w}, which incurs time complexity $O(t(B-\delta B)m)$ and ends up with a vector of dimension $m \times 1$. Then for the resulting vector, we apply the non-linear operations over each element}

\textbf{Space complexity}
Similar to \pro\ for linear regression, at each iteration, we need to cache the two intermediate results $\textbf{C}^{(t)}$ and $\textbf{D}^{(t)}$, which requires $O(\tau m^2)$ space in total. Plus, the space overhead for the linear coefficients will be $O(n \lceil \frac{\tau B}{n} \rceil)$. With small feature space, the memory footprint of \pro\ should be still acceptable. 

\subsubsection{For dense dataset with large feature space}

\textbf{Offline phase}
Similar to the offline phase for dense dataset with small feature space, the intermediate results $\textbf{C}^{(t)}$ and $\textbf{D}^{(t)}$ are computed. Plus, 

\textbf{Online phase} In the online phase, the update rule for $\linearincrew^{(t)}$ will be:

\begin{align}\label{eq: provenance_update_formular}
    \begin{split}
         \linearincrew^{(t+1)}& \leftarrow ((1-\eta_t\lambda)\textbf{I} + \frac{\eta_t}{z^{(t)}} (\textbf{P}_{1..r}^{(t)}\textbf{V}_{1..r}^{(t), T}-\Delta \textbf{C}^{(t)}))\linearincrew^{(t)} \\
         &+\frac{\eta_t}{z^{(t)}} (\textbf{D}^{(t)}-\Delta \textbf{D}^{(t)})\\
    \end{split}
\end{align}

\textbf{Time complexity analysis} By utilizing the same techniques as before to avoid the expensive matrix-matrix computations in Equation \ref{eq: provenance_update_formular}, the total time complexity will be $O(\tau (rm + \Delta B m))$, which will be more efficient compared to the way to retrain from the scratch. For multi-nomial logistic regression, if there are $k$ possible classes in total, the total number of model parameters will be $km$ and thus the time complexity for \pro\ and the method to retrain from the scratch will be $O(\tau (rkm + \Delta B km))$ and $O(\tau ((B-\Delta B)km + C_{non}km))$ respectively.}

\textbf{Space complexity analysis} Through this approximation, we need to cache $\textbf{P}_{1..r}^{(t)}$ and $\textbf{V}_{1..r}^{(t)}$ at each iteration, which requires $O(\tau rm)$ space in total. Plus, $O(n \lceil \frac{\tau B}{n} \rceil)$ extra space is necessary to cache the linear coefficients. So the total space complexity will be $O(\tau rm) + O(n \lceil \frac{\tau B}{n} \rceil)$.
% , which will become $O(\tau rkm + n \lceil \frac{\tau B}{n} \rceil k^2)$ if there are $k$ possible classes.

%\scream{what's the effect of the approximation by SVD?}
\begin{theorem} \textbf{(Approximation ratio)}\label{theorem: aproximation_bound_svd_logistic}
Similar to Theorem \ref{theorem: aproximation_bound_svd}, the deviation caused by the SVD approximation will be bounded by $O(\epsilon)$, given \eat{the upper bound of $||\textbf{w}^{(t)}||$ and }the ratio  $\frac{||\textbf{P}^{(t)}_{1..r}\textbf{V}^{T,(t)}_{1..r}||_2}{||\textbf{P}^{(t)}\textbf{V}^{T,(t)}||_2} \geq 1-\epsilon$. So using Theorem \ref{theorem: aproximation_bound_change}, $||E(\linearincrew^{(t)}- \logistlinearincrew^{(t)})||_2$ is bounded by $O(\frac{\Delta n}{n}\Delta x) + O((\frac{\Delta n}{n})^2) + O((\Delta x)^2) + O(\epsilon)$.
\end{theorem}

This indicates that $\linearincrew^{(t)}$ should be very close to $\logistlinearincrew^{(t)}$ (similar to the discussion after Theorem \ref{theorem: aproximation_bound_svd}).

%\scream{how to talk about sparse dataset}
\textbf{Discussion} Notice that for sparse datasets with large feature space, we can utilize the efficient sparse matrix operations by retraining from scratch. Also note that the intermediate result $\textbf{C}^{(t)}$ will be a sparse matrix for such datasets. However, after SVD, there is no guarantee that $\textbf{P}^{(t)}$ and $\textbf{V}^{(t)}$ are sparse matrices.
%which thus cannot leverage the efficient sparse matrix operations. 
Therefore, for sparse training datasets, we will simply use the linearized update rule, i.e. Equation \ref{eq: mini-SGD_logistic_regression_para_update} directly, without considering the strategies above. 
% in this subsection. 

\vspace{-2mm}
\subsection{\proopt: Optimizations for logistic regression}\label{ssec: proopt_logistic_regression}

\eat{Although Equation \ref{eq: implement_matrix_form}-opt provides an efficient way to compute $\Delta \textbf{C}^{(t)}$ and $\Delta \textbf{D}^{(t)}$, which, however, is very memory-inefficient when the total number of super-iteration $T$ is very large. Plus, the size of the matrix $\textbf{C}^{(t)}$ and $\textbf{D}^{(t)}$ for each iteration $t$ depends on the total number of parameters in $\textbf{w}$, which equals to the feature number in binary logistic regression and equals to the product of the feature number and the number of classes in multinomial logistic regression. That means that the space complexity to maintain all the $\textbf{C}^{(t)}$ and $\textbf{D}^{(t)}$ will be up to $O(T k^2)$ and $O(T k)$ where $k$ represents the number of model parameters, which can blow up the memories in the case of large $k$ (i.e. large number of model parameters) and large $T$ (i.e. long training process). }

%As in the case of linear regression
Again, when the feature space is small additional optimizations are possible. 
\eat{however, when the total number of iterations $\tau$ is large, the performance gains of \pro\ will not be significant. In order to alleviate this effect to make sure that computation of Equation \ref{eq: implement_matrix_form}-opt can fit in memory} 
In particular, we observe that for each sample $i$ the change in the coefficients $a^{i,(t)}$ and $b^{i, (t)}$ from one iteration to the next becomes smaller and smaller as $\textbf{w}^{(t)}$ converges. This suggests that we can stop capturing new provenance information at some earlier iteration, call it $t_s$, and continue with the same provenance until convergence.
%than at the end when convergence of %$\textbf{w}^{(t)}$ is achieved, and assume that %the remaining iterations share the same %provenance information.
% By stopping early, we can reduce the time overhead to compute 
% $\Delta \textbf{C}^{(t)}$ and $\Delta \textbf{D}^{(t)}$. 
% The criterion for when to stop capturing the provenance can be determined in a similar way to other convergence conditions. In our implementation, when the difference of the model parameters between two iterations is small enough (say smaller than $10^{-4}$) at iteration $t_s$, the provenance capturing process terminates. 
Suppose that for each sample $i$ we approximate $a^{i,(t)},b^{i,(t)}$ by $a^{i,*}$ and $b^{i,*}$ after the iteration $t_s$.
%which are computed over the $i_{th}$ sample in the latest pass over the training dataset before iteration $t_s$. 
Therefore the matrices $\textbf{C}^{(t)}, \textbf{D}^{(t)}, \Delta \textbf{C}^{(t)}$ and $\Delta \textbf{D}^{(t)}$  will be approximated using the coefficients $a^{i,*}$ and $b^{i,*}$ 
%where $\textbf{C}^{(t)}$ and $\textbf{D}^{(t)}$ are computed over the full training set $\textbf{X}$ while $\Delta \textbf{C}^{(t)}$ and $\Delta \textbf{D}^{(t)}$ are computed over all the removed training set $\Delta \textbf{X}$. 
and will remain the same for all iterations 
$t \geq t_s$ allowing us to avoid their recomputation. In the experiments, we found
that a rule of thumb that takes $t_s$ to be
70\% of the total number of iterations works well.
%The consequence of the approximation is that the remaining iterations after iteration 
%\scream{discuss this: bad sentence}
%$t_s$ share the same $\textbf{C}^{(t)}, \textbf{D}^{(t)}, \Delta \textbf{C}^{(t)}$ and $\Delta \textbf{D}^{(t)}$.  

This has the same form as for linear regression, motivating us to use the same techniques from \proopt\ for linear regression, i.e. conducting eigenvalue decomposition over $\textbf{C}^{(t)}$, followed by incrementally updating the eigenvalues given the changes $\Delta \textbf{C}^{(t)}$, thus avoiding recomputations after the iteration $t_s$.

\textbf{Time complexity.}
Before and after the iteration $t_s$, the total time complexity is $O(t_s (rm + \Delta B m))$ and $O(\min\{\Delta n, m\} m^2) + O((\tau - t_s) m)$ (see the time complexity analysis in Section \ref{sec: opt_linear_regression}) respectively. Thus the total time complexity is $O(t_s (rm + \Delta B m)) + O(\min\{\Delta n, m\} m^2) + O((\tau - t_s) m)$.

\textbf{Space complexity.}
After the iteration $t_s$, we only need to keep the eigenvectors of $\textbf{C}^{(t)}$, which requires $O(m^2)$ space. Including the space overhead for the first $t_s$ iterations, the total space complexity is $O(m^2) + O(t_s rm) + O(n \lceil \frac{t_s B}{n} \rceil)$.
%\scream{The effect of the approximations}
\begin{theorem}\textbf{(Approximation ratio)}\label{theorem: aproximation_bound_eigen_logistic}
Suppose that after the iteration $t_s$ the gradient of the objective function is smaller than $\delta$, then the approximations of \proopt\ can lead to deviations of the model parameters bounded by $O((\tau - t_s)\delta) + O(||\Delta \textbf{X}^T\Delta \textbf{X}||)$. By combining the analysis in Theorem \ref{theorem: aproximation_bound_change}, $||E(\linearincrew^{(t)}- \logistlinearincrew^{(t)})||_2$ is bounded by $O(\frac{\Delta n}{n}\Delta x) + O((\frac{\Delta n}{n})^2) + O((\Delta x)^2) + O((\tau - t_s)\delta) + O(||\Delta \textbf{X}^T\Delta \textbf{X}||)$
\end{theorem}
\vspace{-2mm}
% Similarly\eat{ to discussion after Theorem \ref{theorem: aproximation_bound_eigen}}, 
This thus indicates that $\linearincrew^{(t)}$ should be very close to $\logistlinearincrew^{(t)}$.
% the model parameters updated by \pro\ and \proopt\ should be still very close to the expected results computed by retraining from scratch

\textbf{Discussion.} Our current framework handles linear and logistic models with L2 regularization.
Our solutions cannot handle L1 regularization since in this case the gradient of the objective function is not continuous, thus invalidating some of the error bound analysis above. How to handle L1 regularization will be our future work.

%scream{why we list these here and not after the experiments?}
%\scream{I agree that this should be moved to after the experiments (conclusions?) . Also I deleted some of this from the introduction, maybe you want to move it here.}

% Equation \ref{eq: provenance_update_formular} then becomes the following form:
% \begin{align}\label{eq: gbm_cut_off}
%     \begin{split}
%         \linearincrew^{(t+1)}=
%         \begin{cases}
%         ((1-\eta_t\lambda)\textbf{I} + \frac{\eta_t}{z^{(t)}} (\textbf{C}^{(t)}-\Delta \textbf{C}^{(t)}))\linearincrew^{(t)} \\
%          +\frac{\eta_t}{z^{(t)}} (\textbf{D}^{(t)}-\Delta \textbf{D}^{(t)}) & \text{if } t\leq t_s\\
%         ((1-\eta_t\lambda)\textbf{I} + \frac{\eta_t}{z^{(t)}} (\textbf{C}^{(t_s)}-\Delta \textbf{C}^{(t_s)}))\linearincrew^{(t)} \\
%          +\frac{\eta_t}{z^{(t)}} (\textbf{D}^{(t_s)}-\Delta\textbf{D}^{(t_s)}) & \text{otherwise}
%         \end{cases}
%     \end{split}
% \end{align}

\eat{Note that $(\textbf{C}^{(t_s)}-\Delta\textbf{C}^{(t_s)})$ and $(\textbf{D}^{(t_s)}-\Delta\textbf{D}^{(t_s)})$ are constant after iteration $t_s$, for which the eigenvalue decomposition optimization strategy in Section \ref{sec: opt_linear_regression} can be applied. Suppose the eigenvalue decomposition over $\textbf{C}^{(t_s)}$ is $\textbf{Q} \; diag(\{c_i\}_{i=1}^n)\textbf{Q}^{-1}$, and the updates of the eigenvalues after the change over $\textbf{C}^{(t_s)}$ are computed using Equation \ref{eq: eigen_value_incremental}, i.e. $\textbf{Q}^{-1}\Delta \textbf{C}^{(t_s)}\textbf{Q} = diag \; (\{\delta c_i\}_{i=1}^n)$, then the update rule in Equation \ref{eq: gbm_cut_off} after iteration $t_s$ becomes:
\begin{align}\label{eq: gbm_cut_off}
    \begin{split}
        &\linearincrew^{(t+1)}=\\
        % ((1-\eta_t\lambda)\textbf{I} + \frac{\eta_t}{B-\Delta B} (\textbf{Q}diag(\{c_i-\delta c_i\}_{i=1}^n)\textbf{Q}^{-1}))\textbf{w}^{(t)'} \\
        % & +\frac{\eta_t}{B-\Delta B} (\textbf{D}^{(t_0)}-\textbf{D}^{(t_0)'})\\
       &(\textbf{Q}(diag(\{\Pi_{j=t_s}^t(1-\eta_j\lambda) + \frac{\eta_j}{z^{(t)}}(c_i-\delta c_i)\}_{i=1}^n))\textbf{Q}^{-1})\linearincrew^{(t)} \\
        & + \frac{\eta_t}{z^{(t)}}\textbf{Q}(diag(\{\sum_{l=t_s + 1}^{t-1}\Pi_{j=l}^t(1-\eta_j\lambda) + \frac{\eta_j}{z^{(t)}}(c_i-\delta c_i)\}_{i=1}^n))\\
        &\textbf{Q}^{-1}(\textbf{D}^{(t_s)}-\Delta\textbf{D}^{(t_s)})
    \end{split}
\end{align}
Finally, the total time complexity will be $O(t_s\Delta B m^2) + O(\tau-t_s)$ when there are $\tau$ iterations in total, which is much more efficient when $m$ is small compared to \pro.}
\section{Experiments}\label{sec: experiment}

%basic setup
\subsection{Experimental setup}

\textbf{Platform.}~
We conduct extensive experiments in Python 3.6 and use PyTorch 1.3.0 \cite{paszke2017automatic} for the experiments for dense datasets and scipy 1.3.1 \cite{scipylib} for the experiments for sparse datasets. All experiments were conducted on a Linux server with an Intel(R) Xeon(R) CPU E5-2630 v4 @ 2.20GHz and 64GB of main memory. 

\textbf{Datasets.}~ 
Six datasets were used in our experiments: (1) the UCI SGEMM GPU dataset\footnote{\url{https://archive.ics.uci.edu/ml/datasets/SGEMM+GPU+kernel+performance}}; 
(2) the UCI Covtype dataset \footnote{\url{https://archive.ics.uci.edu/ml/datasets/covertype}}; (3) the UCI HIGGS dataset \footnote{\url{https://archive.ics.uci.edu/ml/datasets/HIGGS}}; (4) the RCV1 dataset \footnote{simplified version from https://scikit-learn.org/0.18/datasets/rcv1.html} (5) the Kaggle ECG Heartbeat Categorization Dataset\footnote{\url{https://www.kaggle.com/shayanfazeli/heartbeat}}; (6) the CIFAR-10 dataset \footnote{\url{https://www.cs.toronto.edu/~kriz/cifar.html}}, which are referenced as \SGEMMdataset, \covmultidataset, \higgsdataset,  \rcvdataset, \heartbeatdataset\ and \cifar\ hereafter. 

\SGEMMdataset\ has continuous label values, therefore we use it in experiments with \emph{linear
regression} while the rest of them have values that are appropriate for classification.
% , hence we use \skindataset\ for \emph{binary logistic regression} % (\skindataset\ falls into only 2 classes),
% and  \heartbeatdataset\ for \emph{multinomial logistic regression}.
\eat{
classification, hence we use them for 
\emph{binary} (\skindataset\ falls into only 2 classes), respectively \emph{multinomial logistic regression}.}
Each dataset is partitioned into \emph{training} (90\% of the samples)
and \emph{validation} (10\% of the samples)
datasets, the latter used for measuring the accuracy of 
models trained from the former.

The characteristics of these datasets are listed in Table \ref{Table: datasets_summary}, which indicates that \rcvdataset\ and \cifar\ have extremely large feature space (over 30k model parameters) while other datasets have much fewer parameters (\heartbeatdataset\ has around 1000 while others have less than 500).
%\val{in table 1 fix columns for size of training set and size of validation set}
%\yinjun{maybe we can explicitly say that 90\% data are used for training while 10\% are used for validation}
%\val{done}
% \begin{table}
% \centering
% \small
% \caption{Summary of datasets}
% % \vspace*{-0.2cm}
% \begin{tabular}[!h]{|>{\centering\arraybackslash}p{1.5cm}|>{\centering\arraybackslash}p{1.2cm}|>{\centering\arraybackslash}p{1cm}|>{\centering\arraybackslash}p{1.5cm}|>{\centering\arraybackslash}p{1.5cm}|} \hline
% %dataset
% name & \# features  & \# classes & \# samples\\ \hline
% \SGEMMdataset & 18 &  &241600 \\ \hline
% \skindataset & 3 & 2&245,057 \\ \hline
% % \covmultidataset& 54 & 10&581,012 \\ \hline
% \heartbeatdataset& 188 &7 &87.6K \\ \hline
% \end{tabular}
% \label{Table: datasets_summary}
% \end{table}

\subsection{Experiment design}\label{ssec: exp_design}

%\scream{Also, have you removed "data influence"?}

We conduct two sets of experiments, the first of which aims to evaluate the performance of \pro\ and \proopt\ with
respect to the deletion of \emph{one subset} of the training samples. We do this over different types of datasets (dense VS sparse, large feature space VS small) with varied configurations (how many samples to be removed, mini-batch size, iteration numbers etc.), and compare against retraining from scratch.  
The second set of experiments simulate the scenario where users \emph{repetitively} remove different subsets of training samples. 

In the first set of experiments we simulate the cleaning scenario. To specify the samples to be removed from the training datasets, we introduce dirty samples, which are a selected subset of samples from the original dataset $\mathcal{T}$ that are modified to incorrect values by rescaling. The resulting dataset is denoted  $\mathcal{T}_{_\mathrm{dirty}}$, over which the initial model $\mathcal{M}_{_\mathrm{init}}$ is constructed.  The dirty samples are then removed in the model update phase. The goal is to compare the robustness of \pro, \proopt\ and the \emph{influence function} \cite{koh2017understanding} method \eat{(denoted by \infl) }when dirty data exists. \eat{\scream{Do we still do this? Can we deemphasize this?}
to construct an augmented \scream{How is it augmented?}
training datasets $\mathcal{T}_{_\mathrm{dirty}}$ from the original training dataset $\mathcal{T}$, which were removed later in the model update phase.  (\scream{only use earlier feature errors} \scream{So we need to edit this paragraph?}) The dirty data were designed such that the feature values of some existing samples were changed to unreasonable values by rescaling. We also designed other types of errors, which ends up with similar results \scream{This will not be noticed by any reader}
and thus omitted here.}
In the experiment we vary
the number of erroneous samples generated. The ratio between the erroneous samples and the original training dataset is called 
the \textbf{deletion rate},
and we give it values ranging from $0.0001$ (i.e. $0.01\%$) to $0.2$ (i.e. $20\%$).

% \textbf{Data influence experiments}
In the second set of the experiments, we simulate the scenario in which users debug or interpret models by removing different subsets of samples, necessitating repeated incremental model update operations.  We assume that the datasets are very large; to simulate this,
% the scenario where the users removed some small amount of data from the {\em large training set} for the purpose of understanding the effect of those removed samples, 
we create three synthetic datasets $\mathcal{T}_{_\mathrm{cat}}$ by concatenating 4 copies of \higgsdataset, 20 copies of \covmultidataset\ and 130 copies of \heartbeatdataset\ such that the total number of training samples is around 40 million, 11 million and 11 million, respectively, which are denoted \higgsdataset\ (extended), \covmultidataset\ (extended) and \heartbeatdataset\ (extended), respectively.
In the experiments, ten different subsets are removed and for each of them the deletion rate is about 0.1\% of randomly picked samples out of the full training set. The hyperparameters for this set of experiments are listed in Table \ref{Table: exp_summary}.

\textbf{Baseline.}~ For both $\mathcal{T}_{_\mathrm{dirty}}$ and $\mathcal{T}_{_\mathrm{cat}}$, we simulate what users (presumably
unaware of errors) would do, and train an initial model $\mathcal{M}_{_\mathrm{init}}$ using the following standard method:
%(1) the autograd library in PyTorch (which we call the \std method hereafter), or 
Manually derive the formula for the gradient of the objective function 
%of linear regression and logistic regression and 
and then program explicitly the \gbm\ iterations. The erroneous or chosen samples  are then removed from $\mathcal{T}_{_\mathrm{dirty}}$ or $\mathcal{T}_{_\mathrm{cat}}$. 
% The set of removed samples 
%(called the \iter method hereafter). 
% In data cleaning scenarios, the subsequent failures of  $\mathcal{M}_{_\mathrm{init}}$ would presumably induce the users to 
% apply data cleaning
% and get rid of the erroneous samples. Assuming that all the errors generated are found, we 
% delete the erroneous samples 
%simulate this by deleting the erroneous samples 
%(we kept track of them when we introduced them) 
% from $\mathcal{T}_{_\mathrm{dirty}}$ while in the data influence experiments, the users remove the subset of data of interest. 
% are denoted by $\mathcal{T}_{_\mathrm{del}}$. 
\eat{I have added closed-form solutions for comparison in linear regression experiments} 
For linear regression (except for \gbm), we also compare \pro\ and \proopt\ against close-form formula solutions for incremental updates~\cite{deshpande2006mauvedb, gupta2015processing, nikolic2014linview, hasani2018efficient}, denoted by \closeform.
\eat{
, for linear regression, except the \gbm, 
\cite{deshpande2006mauvedb, gupta2015processing, nikolic2014linview, hasani2018efficient} also uses closed-form formula for incremental updates (denoted by \closeform), which is also compared against \pro\ and \proopt\ in this section.}

\textbf{Incrementality.}~ 
%\scream{These two paragraphs are very convoluted... I haven't reworded. Yinjun: modified a little bit. VAl: some more edits looks firn to me now}
To update the model $\mathcal{M}_{_\mathrm{init}}$, the straightforward solution is to retrain from the scratch by using the same standard method as before but exclude the removed samples from each mini-batch. We denote this solution by \iter. 
% assume that users will use the same mini-batch as the one used for training $\mathcal{M}_{_\mathrm{init}}$ in each iteration to do the \emph{retraining} from scratch without including the removed samples by using the same standard method, denoted by \iter.
% without being interfered by the effect of random sampling in \minisgd, the baseline approach (referred to as \iter)
% does \emph{retraining}  by using the same mini-batches in each iteration as the training process for $\mathcal{M}_{_\mathrm{init}}$ except that the gradient of the samples from $\mathcal{T}_{_\mathrm{del}}$ in each iteration is zeroed out, i.e. using Equation \ref{eq: mini-SGD_updated_model_parameters_expected}. 
In contrast, our approach uses \pro\ or \proopt\ to incrementally update the model. 
% In the end, an  is produced.
The time taken by \iter, \pro\ or \proopt\ to produce the updated model \eat{$\mathcal{M}_{_\mathrm{upd}}$ }
is reported in the experiments as the \emph{update time}, and is compared over the 
two solutions: retraining with \iter\ vs.~incremental update with \pro\ or \proopt.

\eat{\scream{delete the following paragraph}
In contrast, our approach, \pro\ (and \proopt\ for dataset with small feature space), obtains from $\mathcal{T}_{_\mathrm{dirty}}$ or $\mathcal{T}_{_\mathrm{cat}}$ an initial construct $\mathcal{P}_{_\mathrm{init}}$ that encompasses the provenance collected during \gbm\ iterations.
Afterwards, in the \emph{incremental update} phase, it propagates the deletion of the selected samples from the training dataset to
\eat{to the model. Therefore, it \emph{incrementally updates}  (a \emph{provenance-instrumented} version of)}
$\mathcal{P}_{_\mathrm{init}}$ producing its own updated model $\mathcal{M}_{_\mathrm{upd}}$. 
The time to do this incremental update is
the \emph{update time} for our approach. In the experiments we compare the update time
for our approach to the update time of the baseline approach. }
%\val{should we already say here that our update time can be two orders of magnitude faster than the baseline update time?}

Note that our \pro/ \proopt\ approach uses
provenance information collected from the 
whole training dataset. This phase is \emph{offline} for the \pro/ \proopt\ algorithms
and is \emph{not} included in their reported running times. In practice, for the first set
of experiments (cleaning of erroneous samples)
provenance collection is done during the training of  $\mathcal{M}_{_\mathrm{init}}$ from $\mathcal{T}_{_\mathrm{dirty}}$. For the second set of experiments (repeated deletions of subsets for debugging or interpretability)
provenance collection is done
during an initial training of $\mathcal{M}_{_\mathrm{init}}$ from the entire dataset $\mathcal{T}_{_\mathrm{cat}}$, which only
needs to be done \emph{once} even if many deletions of subsets are performed subsequently.

\eat{
construct $\mathcal{P}_{_\mathrm{init}}$ that encompasses the provenance during training $\mathcal{M}_{_\mathrm{init}}$ over $\mathcal{T}_{_\mathrm{dirty}}$ or $\mathcal{T}_{_\mathrm{cat}}$
is \emph{offline}.
}

\reminder{Yet it is relevant to compare them in order to understand the \emph{overhead} of provenance collection. sometimes it will take longer to finish the provenance preparation step... We do not present graphs due to space limitations, but report that the overhead of collecting provenance for \pro\ is
less than 90\% while, remarkably, the overhead for \proopt\ is only between 10\% and 30\%.}

Since \proopt\ is the optimized version for datasets with small feature space we only record the update time of \pro\ over \rcvdataset\ and \cifar, which have very large feature spaces.

\textbf{Accuracy.}~ We compare the quality of the updated model obtained by \iter\ and \pro/\proopt. \eat{retraining (the
baseline approach) with the quality of the updated model obtained by incremental update in our approach.} The goal is to show that the improvement in update time is not achieved at the expense of accuracy. For experiments with linear regression, we 
%employ the widely used 
use the \emph{mean squared error (MSE)} over the validation datasets as a measure for accuracy. A lower MSE corresponds to higher accuracy over the validation set.
%\val{over the validation set of SGEMM? does even SGEMM even have it? 
%you said: not all the datasets we used are associated with test datasets, 
%we thus only recorded validation accuracy for logistic regression.}
%\yinjun{we partitioned validation sets for all three datasets}
For experiments with binary or multinomial logistic regression, we use the updated model to
classify the samples in the validation datasets and report their \emph{validation accuracy}.
% i.e., the proportion of validation samples that are correctly classified by the models. 

\eat{\scream{delete this paragraph?}\textbf{Algorithmic parameters.}~ As described above, for both linear regression and logistic regression, we use iterative convergence of formulas obtained
with gradient-based methods in: 1) \iter, the standard training of the initial model $\mathcal{M}_{_\mathrm{init}}$ from the training set
$\mathcal{T}_{_\mathrm{dirty}}$ or $\mathcal{T}_{_\mathrm{cat}}$, 2) \iter, the standard retraining of the updated model
$\mathcal{M}_{_\mathrm{upd}}$ from the training set $\mathcal{T}_{_\mathrm{dirty}}$ or $\mathcal{T}_{_\mathrm{cat}}$ by removing gradient formula of samples in $\mathcal{T}_{_\mathrm{del}}$ in the iterative computations,
and 3) the computation in \pro\ and \proopt\ of the provenance-instrumented 
$\mathcal{P}_{_\mathrm{init}}$ from $\mathcal{T}_{_\mathrm{dirty}}$
which is later subjected to deletion propagation.
\eat{In {all} of these, the learning rate $\eta_t$, the regularization parameter $\lambda$ 
in Equation \ref{eq: mini_sgd_linear_regression}, and the iteration convergence criterion are the \emph{same}.}
\val{it occurs to me that although we do not measure the provenance intrumentation it's offline for us the reviewer might want to know that it is not MUCH slower than the usual muil of M-init}}

\textbf{Model comparison.}~ 
\eat{In addition to comparing the accuracy of the updated models}
We also compare the updated models \emph{structurally} by comparing the vector of updated model parameters obtained via \pro/\proopt\ against the ones by using \iter.
% incremental update in our approach to the (\iter) vector of model parameters obtained via retraining in the baseline approach.  
This is done in two different ways: 1) Using \emph{distance}, that is, the \emph{L2-norm} of the difference between the two vectors, for both linear and logistic regression, and 2) Using \emph{similarity}, that is, the \emph{cosine}
of the angle between the two vectors.  The latter is only done for logistic regression since the angle is only relevant for classification techniques. For both linear regression and logistic regression, we also record the changes of the signs and magnitude of individual coordinate of the updated model parameters by \pro\ and \proopt\ compared to the ones obtained by \iter.
\eat{The definitions are
 \begin{align}
    \begin{split}
        &\mbox{L2-dist}(\textbf{v}_1, \textbf{v}_2) = ||\textbf{v}_1-\textbf{v}_2||_2\\
        &\mbox{cosine}(\textbf{v}_1, \textbf{v}_2) = \frac{||\textbf{v}_1 \cdot \textbf{v}_2||_2}{||\textbf{v}_1||_2||\textbf{v}_2||_2}
    \end{split}
\end{align}
where $\textbf{v}_1 \cdot \textbf{v}_2$ represents the inner product of $\textbf{v}_1$ and $\textbf{v}_2$. }%\val{we can assume people know this?}

\eat{
Specifically, we record the difference the \proopt\ and \infl\ do not compute exact updated model parameters, the distance (or similarity) of their results to the expected results (computed by \std\ or \iter) are also recorded. \val{edit} Since the machine learning model is composed of a vector of parameters, we measure the distances between any pair of them ($\textbf{v}_1$ and $\textbf{v}_2$) using two different measures, i.e. the L2-norm of the absolute errors $L2-error$ (for logistic regression and linear regression) 
\val{why is this called "error". This *very* overloaded word in our paper...}
and cosine similarity $cosine$ (for logistic regression only), which can be expressed as:
}

\eat{
against the methods to update the model by retraining from the scratch over the training dataset after samples of interest are removed, for which there can be two different implementations in Pytorch. The first one is to use its autograd library to automatically compute the gradient of the objective function iteratively in \gbm\ (called {\em \std} method hereafter) while the other one is to manually derive the formula of the gradient of the objective function of linear regression and logistic regression and program it explicitly (called {\em \iter} method hereafter). 
}

\textbf{Comparison with influence function.}~
%\scream{Yinjun: simplified this paragraph since there are some overlaps between Section 2 and this paragraph for the introduction of \infl}
As indicated in Section \ref{sec: related_work}, the \emph{influence function} method in \cite{koh2017understanding} \eat{provides a formula, derived and approximated through Taylor expansion, for estimating the change of the model parameters after removing {\em one} training sample, which is}can be extended to handle the removal of multiple training samples by us
% We follow the same analysis 
% to derive a similar formula to handle the removal of multiple training samples 
(details omitted).
We denote the resulting method {\infl} and compare it against \pro/\proopt\ in the experiments. \eat{Since 
\infl\ can estimate the change of model parameters by using Taylor expansion without going through the gradient-based iteration process, this should be very efficient. However, }We predicted and verified experimentally that this approach produces models with poor validation accuracy since the derivation of \infl\ relies on the approximation of the Taylor expansion, which can be inaccurate. We also notice that the Taylor expansion used in \infl\ involves the computation of the Hessian matrix, which is very expensive for datasets with extremely large feature space. So we did not run \infl\ over \rcvdataset\ and \cifar\ in the experiments; the comparison between \pro/\proopt\ and \infl\ over other datasets is enough to show the benefits of our approaches.

\textbf{Effect of the hyperparameters and feature space size}
% \noindent
As discussed in Section \ref{sec: implementation}, the performance of \pro\ and \proopt\ is influenced by the mini-batch size, the number of iterations and the size of the feature space. To explore the effect of the first two parameters for logistic regression, three different combinations of mini-batch size and number of iterations are used  over \covmultidataset, denoted \covmultidataset\ (small), \covmultidataset\  (large 1) and \covmultidataset\ (large 2) (see Table \ref{Table: exp_summary}). Since the datasets used for logistic regression have different  feature space sizes,  the performance difference with respect to feature space size is also compared. Since there is only one dataset for linear regression, \SGEMMdataset, we extend this dataset by adding 1500 random features for each sample to determine the effect of  feature space size. The extended version of \SGEMMdataset\ is denoted \SGEMMdataset\ (extended) (see Table \ref{Table: exp_summary}). Other hyperparameters used in the experiments are shown in Table \ref{Table: exp_summary}. Note that since erroneous samples exist in the training datasets for the first set of experiments, some values of the learning rate need to be very small to make sure that the convergence can be reached.
\eat{discuss this: where is the explanation for HIGGS (extended)}

\eat{\textbf{Benefit of optimizations in \proopt.}~ 
%\pro/\proopt
We also compare the two implementations of our approach, \pro\ and \proopt, in terms of update time and accuracy as a way of confirming the benefit of the optimizations described in Section \ref{sec: implementation}. Due to space limitations we only report
experiments with linear regression using \SGEMMdataset\ and \featureerr.}

% \begin{figure}[h]
%     \centering
%     \includegraphics[width=0.3\textwidth, height=0.2\textwidth]{Figures/linear_regression_less_change_data_value_time.jpg}
%     \caption{Update time using \SGEMMdataset\ with linear regression}
%     \label{fig: pro_VS_proopt_tradeoff_less}
% \end{figure}

% \begin{figure}[h]
%     \centering
%     \includegraphics[width=0.3\textwidth, height=0.2\textwidth]{Figures/linear_regression_more_change_data_value_time.jpg}
%     \caption{Update time using extended \SGEMMdataset\ with linear regression}
%     \label{fig: pro_VS_proopt_tradeoff_more}
% \end{figure}

% \begin{figure}[h]
%     \centering
%     \includegraphics[width=0.4\textwidth, height=0.2\textwidth]{Figures/mem_usage.png}
%     \caption{Memory usage across different datasets}
%     \label{fig: mem_usage}
% \end{figure}

\eat{
\val{This was changed from HERE}
In the experiments, we want to solve the following problems before and after training samples of interest are removed from the training dataset: 
\begin{enumerate}
    % \item What is the best choice for the threshold in \cut? 
    \item Will the number of iterations in \gbm\ influence the relative time performance between the five approaches (\std, \iter, \pro, \proopt, \infl)?
    \item How far is the approximated model parameters computed by \pro, \proopt\ and \infl\ from the expected model parameters by baseline (i.e. \std\ and \iter)?
    % \item How much does \cut\ influence the approximation rate in \pro?
    \item Will the approximated model parameters in \pro, \proopt\ and \infl\ influence the prediction performance?
    \item What's the relative time performance of \pro\ and \proopt\ compared to other approaches?
\end{enumerate}
}

%\scream{WHERE IS Table 2 discussed? large 1 large 2 extended in two different ways...}

In the experiments, we answer the following questions:
% \vspace{-0.05in}
\begin{enumerate}
    % \item What is the best choice for the threshold in \cut? 
    \item[\textbf{(Q1)}]
\eat{    The optimizations that led to \proopt\ from \pro\ 
    rely on ideas that might be
    applicable elsewhere.} Do the optimizations used in \proopt\ compared to \pro\ lead to a significant improvement in update time
    without sacrificing accuracy when the number of features in the training set is small?
    \item[\textbf{(Q2)}] Do \pro\ and \proopt\ afford significant gains in efficiency compared to \iter?
    \item[\textbf{(Q3)}]
    Are the efficiency gains provided by \pro\ and \proopt\ achieved without
    sacrificing the accuracy of the updated model?
    \item[\textbf{(Q4)}]
    Can we experimentally validate the theoretical analysis in Sections \ref{ssec: accuracy_proof} and  \ref{sec: implementation}, i.e. that the updated model derived through the approximations in \pro\ and \proopt\ is very close to the one obtained by \iter?
    \eat{How different are the updated model parameters computed by \iter\ \eat{    Are there important structural differences between \iter\ }
    %the baseline updated model 
    and \proopt\ if the two approaches lead to very close validation accuracy? }
    %\susan{I still find it confusing.}\yinjun{can you take a look}\val{I still object to the use of "expected"}
    %How far is the approximated model parameters computed by \pro, \proopt\ and \infl\ from the expected model parameters by baseline (i.e. \std\ and \iter)?
    % \item How much does \cut\ influence the approximation rate in \pro?
    \item[\textbf{(Q5)}]
    Does the influence function approach, \infl, provide a competitive alternative to \pro\ and \proopt?
    %Will the approximated model parameters in \pro, \proopt\ and \infl\ influence the prediction performance?
\eat{    \item[\textbf{(Q6)}] Do the comparisons between \pro, \proopt\  and \iter\ vary among the different targets
    (linear regression, binary or multinomial logistic regression)?
    \scream{Below are new added}}
    \item[\textbf{(Q6)}] Can we experimentally show the effect of the hyperparameters, such as mini-batch size and iteration numbers over the performance gains of \pro\ and \proopt?
    \item[\textbf{(Q7)}] Can we experimentally show the effect of the feature space size (i.e. the number of model parameters, which equals to the feature number times the number of classes for multi-nomial logistic regression)?
    \item[\textbf{(Q8)}] What is the memory overhead of \pro\ and \proopt\ for caching the provenance information?
    %What's the relative time performance of \pro\ and \proopt\ compared to other approaches?
\end{enumerate}

\begin{figure}[b]
\centering
\begin{subfigure}{0.5\textwidth}
\centering
        \includegraphics[width=0.65\textwidth, height=0.5\textwidth]{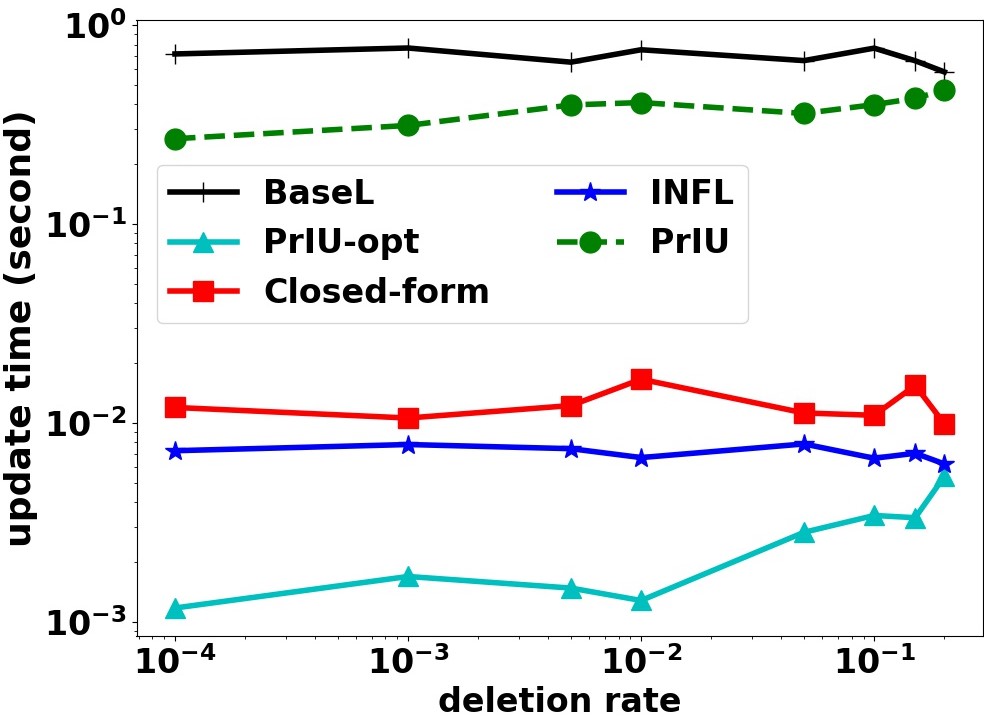}
        \vspace{-1mm}
        \caption{\SGEMMdataset\ (original)}
    \label{fig: pro_VS_proopt_tradeoff_less}
    \vspace{2mm}
    \end{subfigure}
    \begin{subfigure}{0.5\textwidth}
    \centering
        \includegraphics[width=0.7\textwidth, height=0.5\textwidth]{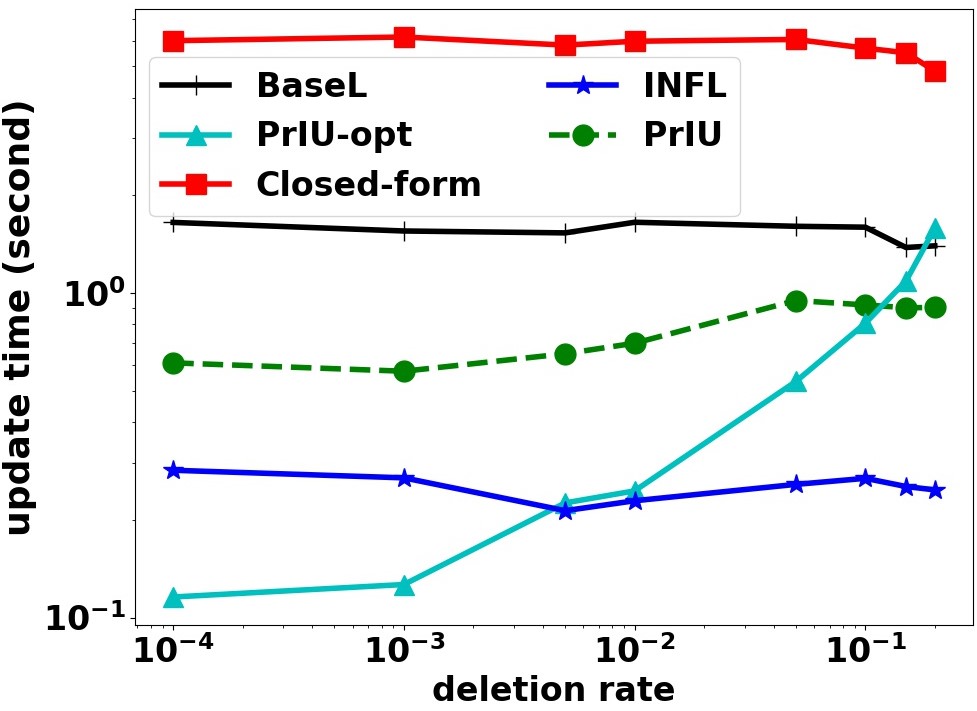}
        \vspace{-1mm}
        \caption{\SGEMMdataset\ (extended)}
    \label{fig: pro_VS_proopt_tradeoff_more}
    \end{subfigure}
    \vspace{-3mm}
    \caption{Update time using linear regression}
    \label{fig: pro_VS_proopt_linear_regression}
\end{figure}

%error generator

\eat{
 After the standard cleaning process, i.e. the removal of those errors (we know where the errors are and it's pretty easy for us to remove errors), the users can use one of the five approaches (\std, \iter, \pro, \proopt\ and \infl), which are compared against each other in terms of the update time and improvement of predication performance, }

 \eat{
 We summarized the notations that will be used in Table \ref{Table: notation_summary}.
\val{I don't think we need this anymore}
}

%experiment design
% \begin{figure}
%     \centering
% \includegraphics[width=0.4\textwidth,height = 0.25\textwidth]{Figures/skin_epoch_num_time_varied_epochs.jpg}
%     \caption{$t_{up}$ with varied epoch numbers}
%     \label{fig:skin_time_varied_epoch}
% \end{figure}
\eat{
\begin{figure*}[!htb]
\captionsetup[subfigure]{width=0.9\textwidth, height=0.25\textwidth}
     \centering
        \begin{subfigure}{0.30\textwidth}
    \hspace*{-0.4cm}
        \raisebox{-\height}{\includegraphics[height = 0.7\textwidth, width=1.2\textwidth]{Figures/sgemm_change_data_value_time.jpg}}
        \caption{Update time comparison}
        \label{fig:sge_time}
    \end{subfigure}
    \hfill
    \begin{subfigure}{0.30\textwidth}
    \hspace*{-0.8cm}
        \raisebox{-\height}{\includegraphics[height = 0.7\textwidth, width=1.2\textwidth]{Figures/sgemm_change_data_value_accuracy.jpg}}
        \caption{Accuracy comparison}     \label{fig:sge_accuracy}
    \end{subfigure}
    \hfill
    \begin{subfigure}{0.30\textwidth}
    \hspace*{-0.8cm}
        \raisebox{-\height}{\includegraphics[height = 0.7\textwidth, width=1.2\textwidth]{Figures/sgemm_change_data_value_distance.jpg}}
        \caption{Distance to model parameters obtained by \iter}
        % \caption{$t_{up}$ with varied epoch \val{iteration?} numbers using \SGEMMdataset\ and the \featureerr\ scenario}
        \label{fig:sge_distance}
    \end{subfigure}
    \caption{Experimental results using \SGEMMdataset\ in the \featureerr\ scenario with linear regression}\label{fig:sge_flipping_errors}
\end{figure*}
\begin{figure*}
\captionsetup[subfigure]{width=0.9\textwidth, height=0.25\textwidth}
     \centering
         \begin{subfigure}{0.30\textwidth}
    \hspace*{-0.8cm}
        \raisebox{-\height}{\includegraphics[height = 0.7\textwidth, width=1.2\textwidth]{Figures/skin_add_noise_time.jpg}}
        \caption{Update time comparison}
        \label{fig:skin_time}
    \end{subfigure}
    \hfill
    \begin{subfigure}{0.30\textwidth}
    \hspace*{-0.4cm}
        \raisebox{-\height}{\includegraphics[height = 0.7\textwidth, width=1.2\textwidth]{Figures/skin_add_noise_accuracy.jpg}}
        \caption{\eat{Validation } Accuracy comparison}
        \label{fig:skin_accuracy}
    \end{subfigure}
    \hfill
\begin{subfigure}{0.30\textwidth}
    \hspace*{-0.8cm}
        \raisebox{-\height}{\includegraphics[height = 0.7\textwidth, width=1.2\textwidth]{Figures/skin_add_noise_distance.jpg}}
        \caption{Distance and similarity to the model parameters obtained by \iter}
        \label{fig:skin_distance}
    \end{subfigure}
    \caption{Experimental results using \skindataset\ in the \newerr\ scenario with binary logistic regression}\label{fig:skin_flipping_errors}
\end{figure*}}

\begin{figure*}[!htb]
\captionsetup[subfigure]{width=1\textwidth}
     \begin{subfigure}{0.33\textwidth}
    % \hspace*{-0.1cm}
        \raisebox{-\height}{\includegraphics[height = 0.8\textwidth, width=1\textwidth]{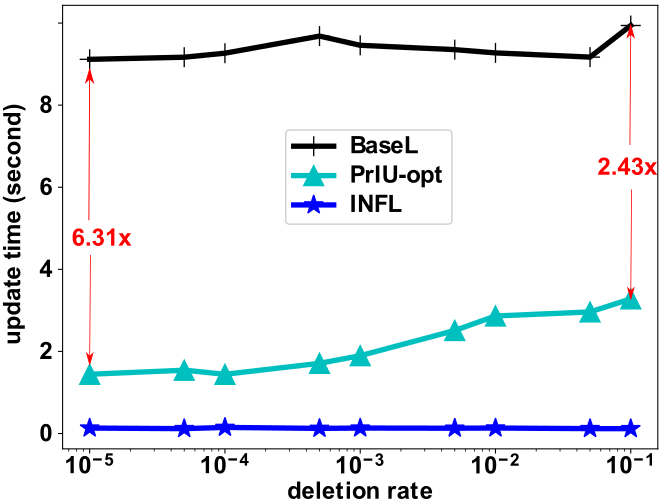}}
        \caption{\covmultidataset\ (small)}
        \label{fig:cov_small_time}
    \end{subfigure}
    \hfill
    \begin{subfigure}{0.33\textwidth}
    % \hspace*{0.1cm}
        \raisebox{-\height}{\includegraphics[height = 0.8\textwidth, width=1\textwidth]{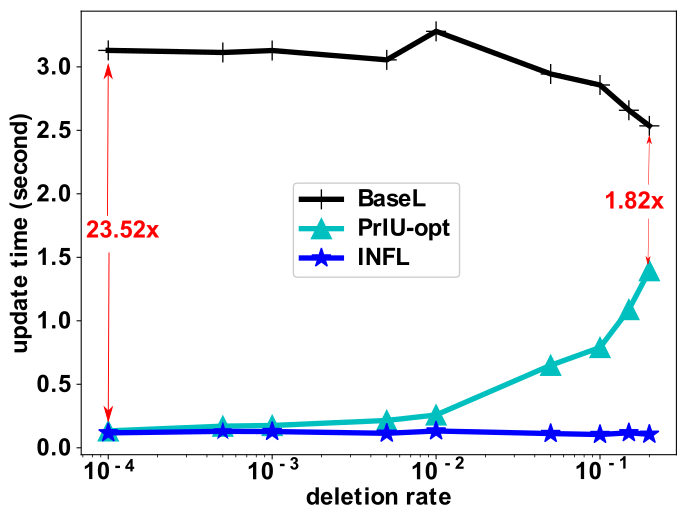}}
        \caption{\covmultidataset\ (large 1)}
        \label{fig:cov_large_time}
    \end{subfigure}
    \hfill
    \begin{subfigure}{0.33\textwidth}
    % \hspace*{0.2cm}
        \raisebox{-\height}{\includegraphics[height = 0.8\textwidth, width=1\textwidth]{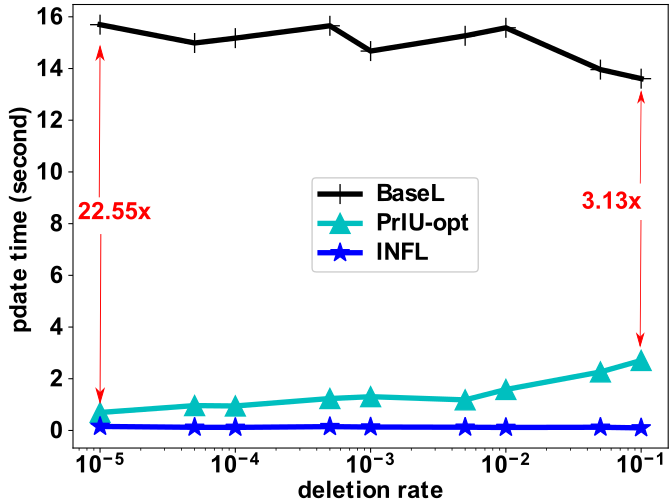}}
        \caption{\covmultidataset\ (large 2)}     \label{fig:cov_large_time_2}
    \end{subfigure}
    \vspace{-4mm}
    \caption{Update time using logistic regression over \covmultidataset\ and the hyperparameters from Table \ref{Table: exp_summary}}\label{fig:Update_time_vary_cov}
    \vspace{-4mm}
\end{figure*}

\begin{figure*}[!htb]
\captionsetup[subfigure]{width=1\textwidth}
     \centering
     \begin{subfigure}{0.33\textwidth}
        \raisebox{-\height}{\includegraphics[height = 0.8\textwidth, width=1\textwidth]{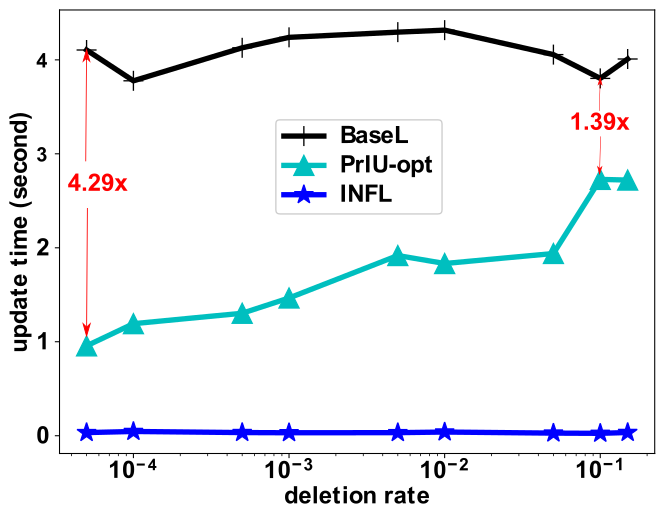}}
        \caption{\heartbeatdataset}
        \label{fig:heart_time}
    \end{subfigure}
    \hfill
    \begin{subfigure}{0.33\textwidth}
        \raisebox{-\height}{\includegraphics[height = 0.8\textwidth, width=1\textwidth]{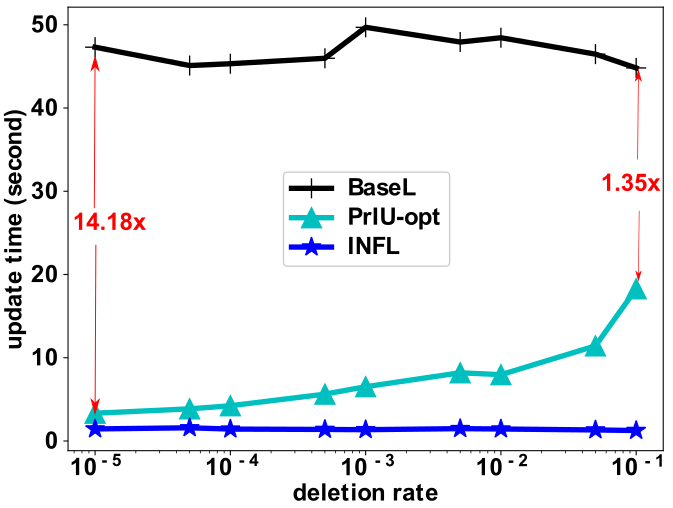}}
        \caption{\higgsdataset}     \label{fig:higgs_time}
    \end{subfigure}
    \hfill
    \centering
     \begin{subfigure}{0.33\textwidth}
        \raisebox{-\height}{\includegraphics[height = 0.8\textwidth, width=1\textwidth]{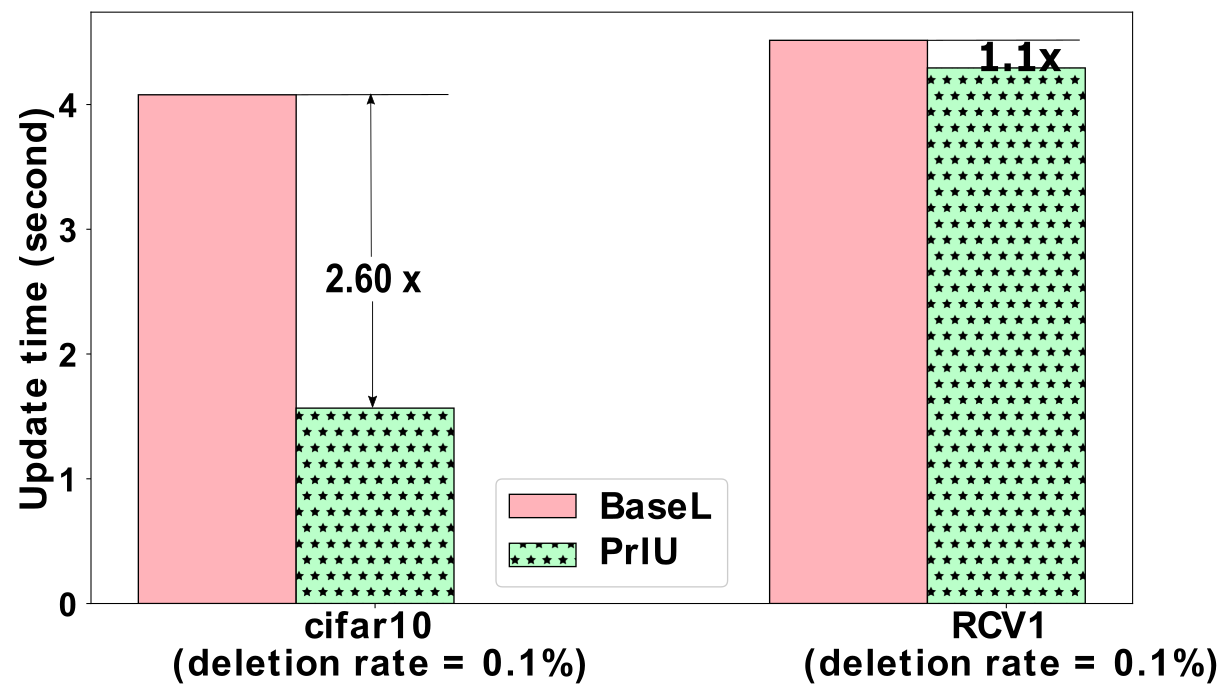}}
        \caption{\rcvdataset\ and \cifar}
        \label{fig:rcv_time}
    \end{subfigure}
    \vspace{-4mm}
    \caption{Update time using logistic regression}\label{fig:Update_time_small_feature}
\end{figure*}

\eat{
\val{I will get rid of this because has changed and it is repeated anyway above}

\textbf{Exp1} This experiment is designed for linear regression to compare \pro\ and \proopt\ against the baseline, in which the model updated time, distance to the results computed by baseline and the test accuracy are recorded under the varied deletion rate (0-0.2). The widely used accuracy metrics for linear regression is Mean Square Errors (MSE), which is measured for all the five approaches in all the three error scenarios. The experiments are conducted over \SGEMMdataset. \eat{how to talk about batch size} We also varied the size of each mini-batch in \gbm\ to see its effect over the performance for the five approaches. Proper hyper-parameters used in \gbm\ (i.e. learning rate $\eta_t$ and regularization parameter $\lambda$ in Equation \ref{eq: mini_sgd_linear_regression}) are chosen such that \gbm\ reaches convergence ultimately.

\textbf{Exp2} In \gbm, the time to produce or update the model parameters is proportional to the number of iterations (except \infl) and thus its effect over the five approaches is explored in this experiment, which is varied from 0 to 10000. We report the time to update the logistic regression model over \skindataset\ with the five approaches after the data cleaning process, which should be independent from the choice of three error scenarios theoretically. So the experiments were conducted in \fliperr\ with fixed deletion rate.

\textbf{Exp3} The goal of this experiment is to experimentally compare the five approaches in binary logistic regression and multinomial logistic regression over \skindataset\ and \heartbeatdataset\ under varied deletion rate (0-0.25) in both the \fliperr\ and \featureerr. The hyper-parameters are fixed for the experiments over the same datasets, in which the maximal number of iterations is set as 10000 to guarantee the termination of the \gbm.
}

\eat{
\begin{table}
\centering
\small
\caption{Summary of notations}
% \vspace*{-0.2cm}
\begin{tabular}[!h]{|>{\centering\arraybackslash}p{2cm}|>{\centering\arraybackslash}p{5cm}|} \hline
Notation & Meaning\\ \hline
% $t_{up}$& time to update the model parameters \\ \hline
$acc\%$ & accuracy over validation data \\ \hline
$MSE$ & Mean Square Errors\\ \hline
$L2-error$& L2-norm of the absolute errors between two vectors \\ \hline
$cosine$& cosine similarity between two vectors \\ \hline
\end{tabular}
\label{Table: notation_summary}
\end{table}
}

\subsection{Experimental results}

\begin{table}
\centering
\small
\caption{Summary of datasets}
% \vspace*{-0.2cm}
\begin{tabular}[!h]{|>{\centering\arraybackslash}p{1.5cm}|>{\centering\arraybackslash}p{1.5cm}|>{\centering\arraybackslash}p{1.2cm}|>{\centering\arraybackslash}p{2cm}|} \hline
%dataset
name & \# features  & \# classes & \# samples\\ \hline
\SGEMMdataset & 18 &  &241,600\\ \hline
% \skindataset & 3 & 2&245,057& \\ \hline
\makecell{\covmultidataset}& 54 & 7&581,012\\ \hline
\higgsdataset & 28 & 2 & 11,000,000\\\hline
\rcvdataset &47,236 & 2 & 23,149 \\\hline
\heartbeatdataset& 188 &7 &87,553 \\ \hline
\cifar& 3072 & 10 & 50,000 \\ \hline
% \covmultidataset & 
\end{tabular}
\label{Table: datasets_summary}
\end{table}

\begin{table}
\centering
\small
\caption{Summary of hyperparameters used in the experiments}
% \vspace*{-0.2cm}
\begin{tabular}[!h]{|>{\centering\arraybackslash}p{2.5cm}|>{\centering\arraybackslash}p{0.8cm}|>{\centering\arraybackslash}p{1.3cm}|>{\centering\arraybackslash}p{2.3cm}|} \hline
%dataset
name & mini-batch size & \# of iterations & other hyper-parameters ($\eta, \lambda$) \\ \hline
\SGEMMdataset\ (original) & 200 & 2000 & ($5 \times 10^{-3}, 0.1$) \\ \hline
\SGEMMdataset\ (extended) & 200 & 2000 & ($5 \times 10^{-3}, 0.1$) \\ \hline
% \skindataset & 3 & 2&245,057& \\ \hline
\makecell{\covmultidataset\ (small)}& 200 & 10000 & ($1\times 10^{-4}, 0.001$) \\ \hline
\makecell {\covmultidataset\ (large 1)}& 10000 & 500 & ($1\times 10^{-4}, 0.001$) \\ \hline
\makecell {\covmultidataset\ (large 2)}& 10000 & 3000 & ($1 \times 10^{-4}, 0.001$) \\ \hline
\higgsdataset & 2000 & 20000 & ($1\times 10^{-5}, 0.01$) \\\hline
\makecell {\covmultidataset\ (extended)}& 1000 & 40000 & ($1 \times 10^{-4}, 0.001$) \\ \hline
\higgsdataset & 2000 & 20000 & ($1\times 10^{-5}, 0.01$) \\\hline
\makecell{\higgsdataset\ (extended)} & 2000 & 60000 & ($1 \times 10^{-5}, 0.01$) \\\hline
\heartbeatdataset & 500 & 5000 & ($1 \times 10^{-5}, 0.1$) \\ \hline
\makecell{\heartbeatdataset\\ (extended)} & 500 & 40000 & ($1 \times 10^{-5}, 0.1$) \\ \hline
\rcvdataset & 500 & 3000 & ($1 \times 10^{-6}, 0.5$) \\\hline
\cifar & 500 & 1000 & ($0.001, 0.1$) \\\hline
% \heartbeatdataset& 188 &7 &87.6K \\ \hline
% \covmultidataset & 
\end{tabular}
\label{Table: exp_summary}
\end{table}

%The experimental results are reported in this subsection. 
We report the results of our experiments in this subsection. 
\eat{First, the experimental results for linear regression are shown in Figure 
comparison between \pro\ and \proopt\ (Figure \ref{fig: pro_VS_proopt_tradeoff}),  followed by comparisons between \proopt\ and other approaches (Figures \ref{fig:sge_flipping_errors}-\ref{fig:heart_feature_errors} for the \SGEMMdataset, \skindataset\ and \heartbeatdataset\ datasets respectively). Based on these results, we answer the questions 
asked above.}  %at the end of Section \ref{ssec: exp_design}.

\begin{table}
\centering
\small
\caption{Memory consumption summary (GB)}
% \vspace*{-0.2cm}
\begin{tabular}[!h]{|>{\centering\arraybackslash}p{3cm}|>{\centering\arraybackslash}p{1.2cm}|>{\centering\arraybackslash}p{1cm}|>{\centering\arraybackslash}p{1.5cm}|} \hline
%dataset
Dataset & \iter\ & \pro\ & \proopt\ \\ \hline
\covmultidataset\ (small) &0.71  & 4.30 & 4.34\\ \hline
\covmultidataset\ (large 1) & 0.87 & 4.02 & 3.49 \\ \hline
\covmultidataset\ (large 2) & 1.34 &  21.0 & 17.4 \\ \hline
\higgsdataset & 5.09 & 8.40 & 8.40\\ \hline
\SGEMMdataset\ (original) & 2.43 & 2.45 & 2.48 \\\hline
\SGEMMdataset\ (extended) & 4.94 & 6.66 & 5.74\\\hline
\heartbeatdataset\ & 0.46 & 6.01 & 5.69 \\ \hline
\rcvdataset\ & 0.28 & 0.3 & - \\ \hline
\cifar\ & 0.79 & 26.59& - \\ \hline
% \closeform & 28 & 2 & 11,000,000 \\\hline
% \rcvdataset & 23149& 2 & 47236 \\\hline
% \heartbeatdataset& 188 &7 &87.6K \\ \hline
% \covmultidataset & 
\end{tabular}
\label{Table: mem_res}
\end{table}

\textbf{(Q1)}~~\eat{Q1 not yet modified}
\eat{We should discuss this}
\eat{The optimizations that led to \proopt\ from \pro\ rely on some ideas that might be
    applicable also elsewhere. Do these optimizations lead to a significant improvement in update time
    without sacrificing accuracy? }
    \eat{
    We present the trade-off between \pro\ and \proopt\ in Figure \ref{fig: pro_VS_proopt_tradeoff_more}, which shows the update time for linear regression models using \SGEMMdataset\ (origin) and \SGEMMdataset\ (extended) respectively. }
    We compare the update time of \pro\ and \proopt\ for linear regression using \SGEMMdataset\ (extended)  in Figure \ref{fig: pro_VS_proopt_tradeoff_more}.
    The results show that
    % . We also show the 
    % As seen in Figure \ref{fig: pro_VS_proopt_tradeoff}, the updated model computed by \proopt\ gives only a slightly bigger mean square error than that computed by \pro\ hence only a small deterioration in accuracy, yet 
    the update time of \proopt\ is significantly better than that of \pro\ except when the deletion rate is approaching 20\%. We also see from 
    %compare the updated models produced by \pro\ and \proopt\ against the one obtained by \iter\ in 
    Table \ref{Table: accuracy_comparison} that
    \proopt\ and \iter\ yield models that 
    have exactly the same validation accuracy.
    Therefore, although \proopt\ uses additional approximations for optimization, they do not hurt the predictive power of the updated models.
    %since \pro, \proopt\ and \iter\ share the same validation accuracy. 
    This shows that the optimization strategies in Sections
    \ref{sec: opt_linear_regression} and \ref{ssec: proopt_logistic_regression}
    are 
    %very much 
    worth the design and implementation effort. 
    % due to space, we omit results that compare \pro\ with \proopt\ for binary and multinomial regression but report that they show very similar trends. 
    Consequently, we will only compare \proopt\ 
    against other approaches except for
    %\scream{Yinjun: added one sentence for \cifar\ and \rcvdataset} 
    \cifar\ and \rcvdataset\ which
    %we use \pro\ instead of \proopt\ since they 
    have extremely large feature spaces.
    %(as mentioned in Section \ref{ssec: exp_design}).

    % \item What is the best choice for the threshold in \cut? 
\textbf{(Q2)}
\eat{Does the \proopt\ incremental approach afford significant gains in efficiency compared to the retraining \iter\ approach?}
         Figures \ref{fig: pro_VS_proopt_tradeoff_less}-\ref{fig: pro_VS_proopt_tradeoff_more} compare the update time in \iter\ and \proopt\ using linear regression (ignore the \infl\ lines for the moment), while Figures \ref{fig:Update_time_vary_cov}-\ref{fig:Update_time_small_feature} 
         \eat{The captions of Figure 4 are unclear}
         show the same results for logistic regression for single model update operation. Observe that for both linear 
         %regression 
         and logistic regression, when the deletion rate is small (<0.01), \proopt\ can achieve significant speed-up compared to \iter:
         %which can be 
         up to two orders of magnitude for linear regression and up to around 23x for logistic regression (for \covmultidataset\ (large 1) and \covmultidataset\ (large 2) with low deletion rate). Even when the feature spaces are extremely large, with deletion rate 0.1\%, there is around a 2.6x speed-up for dense datasets (\cifar\ in Figure \ref{fig:rcv_time}) and only 10\% for sparse datasets (\rcvdataset\ in Figure \ref{fig:rcv_time}), respectively (similar speed-ups were observed for other small deletion rates). The former shows the effectiveness of the optimization strategies in \pro\ over dense datasets with a large feature space while the latter is due to the fact that the optimization strategies for dense datasets were not applied over the sparse ones. Notice that for linear regression, \proopt\ is always faster than \closeform. Figure \ref{fig: data_infl_exp} shows the results of repetitive model updates; \proopt\ achieves an order of magnitude speed-up for  \higgsdataset\ (extended). 
        %  \emph{between one and three orders of magnitude speed-ups} are achieved by our
    %  \proopt\ incremental approach compared to the retraining  \iter\ approach.
     
\eat{     Not surprisingly, the update time of \proopt\ slowly increases while that of \iter\ keeps stable as the deletion rate increases (more dirty samples are deleted). }
     
    % Figure \ref{fig:skin_time_varied_epoch}, which indicates that the time to update the model 
    %parameters in \iter\  increases rapidly as the increasing of the number of iterations. 
    %\val{???} In contrast, the update time in \proopt\ keeps steady (less than 0.01 seconds) 
    %even with large number of iterations. \val{???}

\textbf{(Q3)}
%\scream{This English is quite tortured... haven't corrected it.}
\eat{Are the efficiency gains provided by the \proopt\  incremental approach achieved without
    sacrificing the accuracy of the updated model?}
    Table \ref{Table: accuracy_comparison} (validation accuracy for \pro\ and \proopt\ column) compares the quality of the models obtained by \pro/\proopt\ 
%    and those obtained by \infl\ 
    with that of the models obtained by \iter. For these results we chose
    the highest deletion rate in the experiments, i.e. 20\%.
    %the comparisons of \iter\ and \proopt\ in terms of prediction accuracy over the validation dataset by %using the updated model parameters are shown in 
    % (ignore the \infl\ lines for the moment). 
    %(see the columns starting with ``Validation accuracy'', ignore the columns for \infl\ temporarily). 
    For all the experiments, the validation accuracy (MSE in the case of linear regression)
    of the updated models obtained by \pro\ and \proopt\ \emph{match exactly} the accuracy of the ones obtained by \iter. Combined with the answer to \textbf{Q2}, we can conclude that \emph{\proopt\ speeds up the model update time by up to two orders of magnitude without sacrificing any validation accuracy}.

\textbf{(Q4)}
\eat{Are there important structural differences between the baseline updated model and the incrementally updated model?}         We investigate why \proopt\ has the same validation accuracy as \iter\ by measuring the distance and similarity between the updated models computed by \proopt\ and \iter.  The results are presented in Table \ref{Table: accuracy_comparison} 
        %for the three datasets respectively 
        (again, ignore the columns for \infl). The results indicate that the updated model parameters computed by \proopt\ are very close to the ones obtained by \iter\ since the cosine similarity is almost 1 (see the ``similarity'' column) while the $\mbox{L2-dist}$ is very small (see the ``distance'' column).
        %thus explaining the close match in validation accuracy between \proopt\ and \iter. These results can be seen as experimentally verifying the theoretical analysis in Section \ref{ssec: accuracy_proof} and Section \ref{sec: implementation}, i.e. the closeness of the updated model parameters by \proopt\ to the ones obtained by \iter. 
        An even finer-grained analysis, comparing the signs and magnitude of each coordinate in the model parameters updated by  \proopt\ and \iter\
        shows that there is no sign flipping and only negligible magnitude changes for \proopt\ compared to \iter\ when the deletion rate is small. Even with a large deletion rate of 20\% in \higgsdataset, only 2 out of 58 coordinates flip their signs with small magnitude change.

    %How far is the approximated model parameters computed by \pro, \proopt\ and \infl\ from the expected model parameters by baseline (i.e. \std\ and \iter)?
    % \item How much does \cut\ influence the approximation rate in \pro?
\textbf{(Q5)}
\eat{Does the influence function approach provide a competitive alternative to our
    \proopt\ approach based on provenance?}
    The model update time of \infl\ is also included in Figures \ref{fig:Update_time_vary_cov} and \ref{fig:Update_time_small_feature}. 
     Note that it can be up to one order of magnitude better than \proopt, which is expected since using \infl\ to update the model parameters does not require an iterative computation. However, there is a significant drop in validation accuracy of the updated model derived by \infl\ compared to \iter\ and \proopt\ (see Table \ref{Table: accuracy_comparison}), which is due to the significantly higher $\mbox{L2-dist}$ (see the ``distance'' column) and lower cosine similarity (see the ``similarity'' column) of its updated model compared to the model derived by \iter. We conclude that \pro\ and \proopt\ produce much better models than \infl\ yet can still achieve comparable speed-ups.

\eat{    In contrast, \infl\ incurs much higher $L2-error$ than the other approaches according to Figure \ref{fig:sge_distance} although its response time is pretty short, which thus demonstrates that \pro\ and \proopt\ are more tolerant to the errors than \infl\ without incurring significantly more time overhead.
    
    In terms of \infl,  For \pro\ and \proopt, the performance difference demonstrates that the repetitive matrix multiplications in \gbm\ in \pro\ dominate the overhead in the iterations after the \cut\ threshold, which thus shows the feasibility of the approximation in \proopt\ for speed-ups. 
    
    On contrast, \infl\ fails to provide accurate updated model parameters, which leads to the negative influence over the validation accuracy (i.e. $acc\%$, see Figure \ref{fig:skin_accuracy}). In Figure \ref{fig:skin_accuracy}, as more and more samples are incorrectly labeled in the original training datasets, the validation accuracy ($acc\%$) drops drastically from 80\% to almost 20\% (see the pink line in Figure \ref{fig:skin_accuracy}). After the cleaning process, \std, \iter, \pro\ and \proopt\ can boost $acc\%$ back to about 80\% while \infl\ still mislabeled 10\% more samples in the validation dataset. 

 \proopt\ and is surprisingly on a par with the update time of \infl.}
 
 %\val{I had the feeling that the previous writeup had too much about INFL}
 
    %Will the approximated model parameters in \pro, \proopt\ and \infl\ influence the prediction performance?
\eat{\textbf{(Q6)}\scream{possible to be deleted}
\scream{I do not understand it let's talk about it}
\eat{Do the comparisons between \proopt\  and \iter\ vary among the different targets
    (linear regression, binary or multinomial logistic regression)?
    }
    By comparing the results across different datasets, we observe that although \pro\ and \proopt\ can always achieve orders of magnitude improvement for model update time, it can cause significant delay compared to \infl\ in \heartbeatdataset\ (see Figure \ref{fig:heart_time}), which is due to the much higher number of model parameters used in multinomial logistic regression for \heartbeatdataset\ (i.e. 188 features $\times$ 7 classes = 1316) compared to the other two datasets. This also agrees with the time complexity analysis in Section \ref{ssec: proopt_logistic_regression}, i.e. more model parameters can result in worse time performance in \proopt.
}

\textbf{(Q6)} \textbf{Effect of mini-batch size}. The effect of mini-batch size is seen by comparing \covmultidataset\ (large 1) and \covmultidataset\ (small). One observation is that with larger mini-batch size, the maximal speed-up of \proopt\ is around 23x, while with the smaller mini-batch size it is only about 6x, see Figures \ref{fig:cov_small_time} and \ref{fig:cov_large_time}
%\scream{Sue: Is there a figure/table showing this? Yinjun: shown in the figure now}
This confirms the analysis in Section \ref{sec: implementation}.
%, which shows that with larger mini-batch size, \pro\ and \proopt\ can gain more significant speed-ups. 
In the second set of experiments, we used a small mini-batch size for \covmultidataset\ (1000) and \heartbeatdataset\ (500), resulting in only 4.62x and 3.2x speed-ups by \proopt, respectively (see Figure \ref{fig: data_infl_exp}). 
%\scream{Why did use a small mini-batch size?  Eliminated "So better performance of our approach should be expected if users choose larger mini-batch size." since it is already said above.}
\begin{figure}[h]
    \centering
    \includegraphics[width=0.5\textwidth, height=0.15\textwidth]{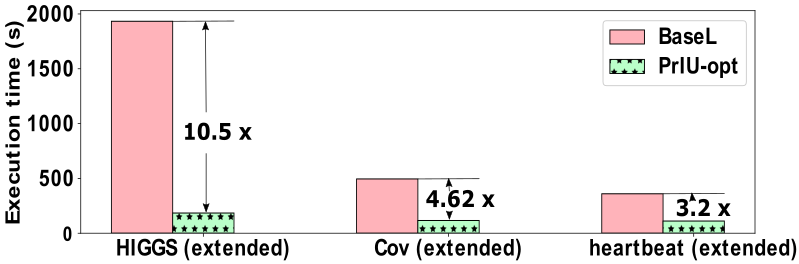}
    \vspace{-6mm}
    \caption{The execution time of repetitively removing 10 different subsets}
    \label{fig: data_infl_exp}
\end{figure}

\textbf{Effect of number of iterations.}
A comparison of \covmultidataset\ (large 1) and \covmultidataset\ (large 2), which have the same mini-batch size but a different number of iterations, can be found in Figures \ref{fig:cov_large_time} and  \ref{fig:cov_large_time_2}. 
We observe that no matter how many iterations the program runs for, at the same deletion rate
\proopt\ achieves a similar speed-up against \iter.
%\scream{Sue: not really similar, it doesn't remain flat. The graph b is also confusing because the y-axis is a different scale, made a note in the text. Yinjun: we compared figure a and b with the same deletion rate but with  different iteration numbers. To avoid confusion, I added ``with the same deletion rate''}
 For example, we have up to around 23x speed-up for small deletion rates and smaller speed-up for higher deletion rates (note the difference in y-axis scale between Figures \ref{fig:cov_large_time} and  \ref{fig:cov_large_time_2}). However, increasing the number of iterations  increases the amount of provenance information cached for \proopt, thus requiring more memory. As Table \ref{Table: mem_res} indicates, since there are 6x iterations for \covmultidataset\ (large 2) compared to \covmultidataset\ (large 1), roughly 6x memory is needed,  confirming the analysis in Section \ref{sec: implementation}. 
 However for \covmultidataset, with a large mini-batch size and 500 iterations, convergence is achieved and we do not observe a difference in validation accuracy between \covmultidataset\ (large 1) and \covmultidataset\ (large 2).
 %\scream{add the discussions for the number of iterations} 
 Note that according to \cite{darzentas1984problem, polyak1992acceleration, lecun2012efficient}, the theoretical optimal number of passes for logistic regression using \minisgd\ (one pass equals to the total number of iterations divided by the number of iterations used for going through the full training set) is quite small. 
 \eat{one and one rule of thumb for the number of passes \cite{lecun2012efficient} is $\left\lceil 10^6/n \right\rceil$ ($n$ is the size of training dataset). So the number of passes used in practice is typically small, }
 However, for \covmultidataset\ (large 2) the number of passes over the full training set is quite large ($3000 /(581012/10000) \approx 60$). Such a high memory usage should therefore not arise in practice.

\textbf{(Q7)}
%\scream{Yinjun: removed this paragraph (in the eat environment) since it does not help answer Q7 and the result has been mentioned in Q2. Val: good.}
\eat{In terms of the update time for datasets with extremely large feature space, Figure \ref{fig:rcv_time} shows the update times of \rcvdataset\ and \cifar\ with a deletion rate of 0.1\%; similar speed-ups were observed for other small deletion rates. Note that \pro\ only achieves a speed-up of around 10\% for \rcvdataset, due to the fact that the optimization strategies for dense datasets were not applied.  In contrast, for \cifar\ the optimization strategies in \pro\ still work, resulting in about a 2.6x speed-up despite the significant memory usage (see Table \ref{Table: mem_res}).}
\eat{such that the idea of reducing dimensions for the intermediate result fails to provide speed-ups (the decomposed matrices for the sparse matrix will become dense matrix) and thus the only optimizations that we can do it to linearize the non-linear operations, which ends up with the marginal speed-ups.
\scream{down from 3X :)}}
In terms of the update time for experiments over datasets with a comparable mini-batch size but with different feature space sizes (\heartbeatdataset\ VS \higgsdataset),
%with varied number of model parameters, 
we notice that a larger number of model parameters leads to poorer performance by \proopt\ (compare Figures \ref{fig:heart_time} and \ref{fig:higgs_time}). This is also validated through a second set of experiments in which \higgsdataset\ (extended) achieves significant speed-up compared to \heartbeatdataset\ (extended) (see Figure \ref{fig: data_infl_exp}). This confirms the analysis in Section \ref{sec: implementation}, where we show how the asymptotic execution time of \pro\ and \proopt\ depends on the number of the model parameters.
% the large memory footprint of the provenance information in \pro\ (over 10x more than \iter, see the memory consumption for \cifar\ in Table \ref{Table: lin_regression_results}). This result matches our complexity analysis in Section \ref{sec: implementation}, in which the space complexity analysis shows that the memory consumption will be proportional to the number of parameters. 

\textbf{(Q8)}
Table \ref{Table: mem_res} shows that in most cases, both \pro\ and \proopt\ only consume no more than 5x memory compared to \iter\ (ignore the number for \covmultidataset\ (large 2) since, as discussed earlier, it is a rare case in practice). However, with a large number of model parameters (like \cifar\ and \heartbeatdataset) there is over 10x memory consumption for \pro\ and \proopt. How to decrease the memory usage for dense datasets with large feature space is left for future work.

    %\val{what should we put here; certainly that increased number of model parameters 
    %i.e. feature numbers decresed benefit}
    
\begin{table}
\centering
\small
\caption{Accuracy and similarity comparison between \proopt\ and \infl\ with deletion rate 0.2}
% \vspace*{-0.2cm}
\begin{tabular}[!h]{|>{\centering\arraybackslash}p{1.4cm}|>{\centering\arraybackslash}p{1.3cm}|>{\centering\arraybackslash}p{0.8cm}|>{\centering\arraybackslash}p{0.7cm}|>{\centering\arraybackslash}p{0.6cm}|>{\centering\arraybackslash}p{0.7cm}|>{\centering\arraybackslash}p{0.8cm}|} \hline
%dataset
\multirow{2}{*}{Dataset} & \multicolumn{2}{c|}{\makecell{Validation \\ accuracy}} & \multicolumn{2}{c|}{distance}& \multicolumn{2}{c|}{similarity} \\ \hhline{~------}
 &\iter\ = \proopt & \infl & \proopt &\infl & \proopt &\infl\\ \hline
\makecell{\covmultidataset\ \\ (small)}& 48.76\% & 36.93\% & 0.184& 1.287 & 0.992& 0.624\\ \hline %
\makecell{\covmultidataset\ \\ (large 1)} & 48.76\% &37.99\%  & 0.0016 &1.047 & 1.0 & 0.738 \\ \hline
\makecell{\covmultidataset\ \\ (large 2)} & 48.76\% &46.38\%  & 0.0003 &1.430 & 1.0 & 0.471 \\ \hline
\higgsdataset & 52.99\% & 47.99\% & 0.0004 & 0.006 & 0.979& -0.040\\ \hline
\heartbeatdataset & 82.78\% &74.34\% & 0.0016& 0.583& 1.00 & 0.143\\ \hline
\SGEMMdataset\ (origin) &0.001 & 0.002 & 0.027 & 0.140 & - & - \\\hline
\SGEMMdataset\ (extended) &0.001 & 0.002 & 0.029 & 0.141 & - & - \\\hline
% \rcvdataset\ &  97.11\% & 97.11\% & - & - & 0.0002 &- &-  &1.0 &  -& - & 0& -&$2\times10^{-8}$ &-\\ \hline
% \closeform & 28 & 2 & 11,000,000 \\\hline
% \rcvdataset & 23149& 2 & 47236 \\\hline
% \heartbeatdataset& 188 &7 &87.6K \\ \hline
% \covmultidataset & 
\end{tabular}
\label{Table: accuracy_comparison}
\end{table}

\eat{
\val{In any case we have to change the list below}

\textbf{Exp1} The experimental results over \SGEMMdataset\ in \fliperr\ are presented in Figure \ref{fig:sge_distance} and Figure \ref{fig:sge_time}. Similar results happen in the other two error scenarios and thus are not included here due to the space limit. Figure \ref{fig:sge_distance} indicates that the absolute distance to the updated model parameters by \iter\ (i.e. $L2-error$) in \proopt\ is much larger than that in \pro\ due to the approximation in \proopt, which, however, does not harm the prediction performance (i.e. Mean Square Errors over validation dataset) and can achieve at least one order of magnitude speed-ups compared to \pro\ and \iter\ and at least three orders of magnitude speed-ups compared to \std\ as Figure \ref{fig:sge_time} shows. In contrast, \infl\ incurs much higher $L2-error$ than the other approaches according to Figure \ref{fig:sge_distance} although its response time is pretty short, which thus demonstrates that \pro\ and \proopt\ are more tolerant to the errors than \infl\ without incurring significantly more time overhead. Plus, it is worth noting that the use of \iter\ can lead to significant speed-ups compared to \std, which indicates that the repetitive gradient derivation in \std\ can slow down the computation process in \gbm.

\textbf{Exp2} The experimental results over \skindataset\ in \fliperr\ using binary logistic regression with varied number of iterations are shown in Figure \ref{fig:skin_time_varied_epoch}, which indicates that the time to update the model parameters in \std, \iter\ and \pro\ increases rapidly as the increasing of the number of iterations. In contrast, the update time in \proopt\ and \infl\ keeps steady (less than 0.01 seconds) even with large number of iterations. In terms of \infl, the result is expected since its computation to update the model parameters is free of the iterative computation. For \pro\ and \proopt, the performance difference demonstrates that the repetitive matrix multiplications in \gbm\ in \pro\ dominate the overhead in the iterations after the \cut\ threshold, which thus shows the feasibility of the approximation in \proopt\ for speed-ups.

\textbf{Exp3} Due to the space limit, we only report one set of experimental results from logistic regression and multinomial logistic regression respectively. Figure \ref{fig:skin_flipping_errors} shows the experimental results over \skindataset\ in \fliperr\ using binary logistic regression (similar results happen in \featureerr\ and thus not shown here). Despite of the approximation in \pro\ and \proopt, the resulting updated model parameters are still very close to the ones derived from \std\ and \iter\ (the cosine similarity is almost 1 while the $L2-error$ is very small, see Figure \ref{fig:skin_distance}) even when the deletion rate is very high. On contrast, \infl\ fails to provide accurate updated model parameters, which leads to the negative influence over the validation accuracy (i.e. $acc\%$, see Figure \ref{fig:skin_accuracy}). In Figure \ref{fig:skin_accuracy}, as more and more samples are incorrectly labeled in the original training datasets, the validation accuracy ($acc\%$) drops drastically from 80\% to almost 20\% (see the pink line in Figure \ref{fig:skin_accuracy}). After the cleaning process, \std, \iter, \pro\ and \proopt\ can boost $acc\%$ back to about 80\% while \infl\ still mislabeled 10\% more samples in the validation dataset. 

In terms of the time to update the model parameters, similar to \textbf{Exp1}, \iter\ is still faster than \std\ but the performance gap between them is much smaller than \textbf{Exp1} due to the existence of the non-linear operations in logistic regressions, which, however, is still about one order of magnitude slower than \pro\. The effect of the optimization from \proopt\ is more significant, which leads to one order of magnitude speed-ups compared to \proopt\ and is surprisingly on a par with the update time of \infl.

We also present the experimental results over \heartbeatdataset\ in Figure \ref{fig:heart_feature_errors} under \featureerr, which shows similar trends between the five approaches in terms of similarity measure, $acc\%$ and update time compared to the experimental results over \skindataset. Notice that according to Figure \ref{fig:heart_time}, the gap of update time between \pro\ and \proopt\ is smaller than that over \skindataset. The reason is that due to the significantly higher number of model parameters in \heartbeatdataset\ (i.e. 188 features $\times$ 7 classes = 1316) than that in \skindataset\ (i.e. 3), the overhead of computing the term $C^{(t)'}$ and $D^{(t)'}$ in Equation \ref{eq: implement_matrix_form} becomes the major overhead, which is a necessary step for both \pro\ and \rpoopt\ and is proportional to the size of the erroneous data to be removed, leading to the slow grow-up of update time in \proopt. 
}

\textbf{Discussion.} Extensive experiments using linear regression and logistic regression over the datasets above show the feasibility of our approach. \pro\ and \proopt\
can achieve up to two orders of magnitude speed-up for incrementally updating model parameters compared to the baseline, especially for large datasets with a small feature space.  This is done without sacrificing the correctness of the results 
(measured by similarity to the updated model parameters by \iter) 
and the prediction performance.
% (measured by the accuracy over the validation datasets).
%\susan{add the speculation that the optimizations used in can be reused in other settings (e.g. neural networks)}
The experiments also show that the optimizations used in \proopt\ give significant performance gains compared to \pro\ with only a small loss of accuracy. 
% and we speculate that these optimizations can be used in other settings, e.g. neural networks.
We observe that 
% \proopt\ incurs significant time overhead when there are a large number of model parameters by comparing to \infl, which matches our time complexity analysis in Section \ref{ssec: proopt_logistic_regression}. However,
\infl\ is not a good solution because of the poor quality of models produced when more than one sample is removed. 
% How to further gain performance for \proopt\ when the number of model parameters is large is left for future research.

\textbf{Limitations.} Our experiments also show the limitations of our solutions. They concern the memory footprint when the feature space or the number of iterations is large (anticipated by several analyses in Section \ref{sec: implementation}) and the marginal speed-up for
large sparse datasets 
(See Section \ref{sec: priu_logistic_regression}).
%is high compared to retraining from scratch, especially 
%\scream{Sue: The next bit is hard to understand. I tried to reword but might be wrong. Yinjun: the sentences below have been mentioned earlier Section 5.3. So I point back to Section 5.3 without talking more about details}
%Besides, as mentioned in  we only use linearize the non-linear operations in the update rule for large sparse datasets, \eat{the idea of reducing the dimension of intermediate results will also not work since after SVD, the sparse matrix will be decomposed into the product of two dense matrices.  Matrix operations over dense matrices cannot beat fast matrix operations over sparse matrices. So for large sparse datasets, we can only obtain }which only results in marginal speed-up as the experimental results show.
We shall endeavor to approach these limitations in future work. 
% \begin{enumerate}
%     \item 
%     \item For sparse datasets, the optimization techniques for dense datasets will not work...We can only utilize the benefit of linear operations compared to non-linear counterpart and obtain marginal speed-ups..
% \end{enumerate}

%1. vary the cut_off_para

%2. vary the time epoch

%3. fix cut_off_para and time epoch, vary the noise_ratio

% \begin{figure*}
% \captionsetup[subfigure]{width=1\textwidth}
%      \centering
%     \begin{subfigure}{0.30\textwidth}
%     \hspace*{-0.8cm}
%         \raisebox{-\height}{\includegraphics[height = 0.8\textwidth, width=1.2\textwidth]{Figures/skin_random_flipping_distance.jpg}}
%         \caption{Similarity against updated model parameters from \std}     \label{fig:skin_distance}
%     \end{subfigure}
%     \hfill
%     \begin{subfigure}{0.30\textwidth}
%     \hspace*{-0.8cm}
%         \raisebox{-\height}{\includegraphics[height = 0.8\textwidth, width=1.2\textwidth]{Figures/skin_random_flipping_accuracy.jpg}}
%         \caption{test accuracy comparison}
%         \label{fig:skin_accuracy}
%     \end{subfigure}
%     \hfill
%     \begin{subfigure}{0.30\textwidth}
%     \hspace*{-0.8cm}
%         \raisebox{-\height}{\includegraphics[height = 0.8\textwidth, width=1.2\textwidth]{Figures/skin_random_flipping_time.jpg}}
%         \caption{training time comparison}
%         \label{fig:skin_time}
%     \end{subfigure}
%     \caption{Experimental results in \fliperr}
% \end{figure*}

\section{Conclusions}\label{sec: conclusion}
%\scream{simplified the conclusion, please take a look}
%Talking about the usage of our work
In this paper, we build a connection between data provenance and incremental machine learning model updates, which is useful in many machine learning and data science applications.
% such as data cleaning and interpretability. 
Building on an extension of the provenance semiring framework~\cite{GreenT17} to include basic linear algebra operations~\cite{yan2016fine},
we capture provenance in the training phase of linear regression and (binary and multinomial) logistic regression 
% This enables  model parameters to be efficiently updated after a subset of training samples are removed. We assume that gradient descent and its variants (i.e. stochastic gradient descent and mini-batch stochastic gradient descent) are used,
and address non-linear operations in logistic regression using piecewise linear interpolation. We prove that linearization does not harm convergence of the updated parameters and similarity to the expected results. 
%talking about approach
Based on these theoretical results, we construct solutions, \pro\ and \proopt, which are optimized to reduce the time and space overhead. %, which have already been a non-trivial task. 
%talking about experiments
The benefits of our solutions are experimentally verified through extensive evaluations over various datasets.  
Looking forward, we believe that these solutions for simpler machine learning models are likely to extend to \emph{generalized additive models}~\cite{hastie1986gam} and they also pave the 
way toward solutions for more complicated machine learning models such as deep neural networks.

\begin{acks}
This material is based upon work that is in part supported by the Defense Advanced Research Projects Agency (DARPA) under Contract No. HR001117C0047.
Partial support was provided by NSF Awards 1547360 and 1733794. Tannen's work at the National University of Singapore was supported in part by the Kwan Im Thong Hood Cho Temple/Avalokite\'{s}vara.
\end{acks}

\newpage
\bibliographystyle{ACM-Reference-Format}
\bibliography{vldb_sample}

%appendix
\balance
\clearpage
\begin{appendix}
\onecolumn

\section{Appendix}

\subsection{Notations}
\subsubsection{Notations for objective functions, gradients and update rule}

The objective functions for linear regression, binary logistic regression and multinomial logistic regression are shown as below (They are Equation \eqref{eq: objective_function_linear_regression}-\eqref{eq: objective_function_multi_logistic_regression} in the paper):

\begin{align}
\begin{split}
h(\textbf{w})& = \frac{1}{n}\sum_{i=1}^n(y_i - \textbf{x}_i^T\textbf{w})^2 +\frac{\lambda}{2}||\textbf{w}||^2_2
%\\ &= \frac{1}{n}||\textbf{Y} - \textbf{X}\textbf{w}||^2_2 + \frac{\lambda}{2}||\textbf{w}||^2_2
% \label{eq: objective_function_linear_regression}
\end{split}
\end{align}
\vspace{-2mm}
\begin{align}
h(\textbf{w})& = \frac{1}{n}\sum_{i=1}^n \ln (1+\exp\{-y_i\textbf{w}^\top\textbf{x}_i\}) + \frac{\lambda}{2}||\textbf{w}||^2_2
% \label{eq: objective_function_logistic_regression}
\end{align}
\vspace{-2mm}
\begin{align}
\begin{split}
% \label{eq: objective_function_multi_logistic_regression}
h(\textbf{w})& = \frac{1}{n}\sum_{k=1}^q \sum_{y_i=k}(\ln (\sum_{j=1}^q e^{\textbf{w}_j^\top\textbf{x}_i})-\textbf{w}_k^T\textbf{x}_i)+ \frac{\lambda}{2}||vec([\textbf{w}_1, \textbf{w}_2, \dots, \textbf{w}_q])||^2_2\\
&\textbf{w} = vec([\textbf{w}_1, \textbf{w}_2, \dots, \textbf{w}_q])
\end{split}
\end{align}

Note that $h(\textbf{w})$ can be rewritten as $h(\textbf{w}) = \frac{1}{n}\sum_{i=1}^n h_i(\textbf{w}) + \frac{\lambda}{2}\|\w^{(t)}\|$. For example for Equation \eqref{eq: objective_function_logistic_regression}, $h_i(\textbf{w}) = (y_i - \textbf{x}_i^T\textbf{w})^2$ 

For linear regression and logistic regression, the rule for updating $\textbf{w}^{(t)}$ under \minisgd\ is presented below (They are Equation \eqref{eq: mini_sgd_linear_regression}-\eqref{eq: mini_sgd_logistic_regression} in the paper and we use $\nabla^{(t)} h(\w^{(t)})$ to denote the average gradients evaluated over the mini-batch at the $t_{th}$ iteration):
\begin{align}
\begin{split}\label{eq: mini_sgd_linear_regression2}
\textbf{w}^{(t+1)}& \leftarrow (1-\eta_t\lambda)\textbf{w}^{(t)} - \frac{2\eta_t}{B} \sum_{i \in \mathscr{B}^{(t)}} \textbf{x}_i(\textbf{x}_i^T\textbf{w}^{(t)} - y_i)\\
& = (1-\eta_t\lambda)\textbf{w}^{(t)} - \frac{\eta_t}{B} \sum_{i \in \mathscr{B}^{(t)}} \nabla h_i(\w^{(t)}) = (1-\eta_t\lambda)\textbf{w}^{(t)} - \eta_t\nabla^{(t)} h(\w^{(t)})
\end{split}\\
        \begin{split}\label{eq: mini_sgd_logistic_regression2}
        \textbf{w}^{(t+1)}& \leftarrow (1-\eta_t\lambda)\textbf{w}^{(t)} + \frac{\eta_t}{B} \sum_{i \in \mathscr{B}^{(t)}} y_i\textbf{x}_i (1-\frac{1}{1+\exp\{-y_i\textbf{w}^{(t)T}\textbf{x}_i\}})\\
        & = (1-\eta_t\lambda)\textbf{w}^{(t)} - \frac{\eta_t}{B} \sum_{i \in \mathscr{B}^{(t)}} \nabla h_i(\w^{(t)}) = (1-\eta_t\lambda)\textbf{w}^{(t)} - \eta_t\nabla^{(t)} h(\w^{(t)})
        \end{split}
\end{align}

Note that in Equation \eqref{eq: mini_sgd_logistic_regression2}, the non-linear part can be abstracted as $f(x) = 1-\frac{1}{1+e^{-x}}$. So this formula can be also represented as:
\begin{align}
    \begin{split}\label{eq: mini_sgd_logistic_regression3}
        \textbf{w}^{(t+1)}& \leftarrow (1-\eta_t\lambda)\textbf{w}^{(t)} + \frac{\eta_t}{B} \sum_{i \in \mathscr{B}^{(t)}} y_i\textbf{x}_i (1-\frac{1}{1+\exp\{-y_i\textbf{w}^{(t)T}\textbf{x}_i\}})\\
        % & = (1-\eta_t\lambda)\textbf{w}^{(t)} - \frac{\eta_t}{B} \sum_{i \in \mathscr{B}^{(t)}} \nabla h_i(\w^{(t)})\\
        & = (1-\eta_t\lambda)\textbf{w}^{(t)} + \frac{\eta_t}{B} \sum_{i \in \mathscr{B}^{(t)}} y_i\textbf{x}_i f(y_i\textbf{w}^{(t)T}\textbf{x}_i)\\
        & = (1-\eta_t\lambda)\textbf{w}^{(t)} - \frac{\eta_t}{B} \sum_{i \in \mathscr{B}^{(t)}} \nabla h_i(\w^{(t)}) = (1-\eta_t\lambda)\textbf{w}^{(t)} - \eta_t\nabla^{(t)} h(\w^{(t)})
        \end{split}
\end{align}

% By representing $\lambda\textbf{w}^{(t)} + \frac{1}{B} \sum_{i \in \mathscr{B}^{(t)}} y_i\textbf{x}_i (1-\frac{1}{1+\exp\{-y_i\textbf{w}^{(t)T}\textbf{x}_i\}})$ in Equation \eqref{eq: mini_sgd_logistic_regression2} as $\triangledown h^{(t)}(\textbf{w})$, which is the gradient formula at $t_{th}$ iteration, then
So $$\triangledown h^{(t)}(\textbf{w}) = -\frac{1}{B}\sum_{i \in \mathscr{B}^{(t)}} y_i\textbf{x}_i f(y_i\textbf{w}^{(t)T}\textbf{x}_i)$$
Also we can explicitly evaluate $\triangledown^2 h^{(t)}(\textbf{w})$ as:
\begin{align}\label{eq: second_derivative_h_logistic_regression}
    \begin{split}
        &\triangledown^2 h^{(t)}(\textbf{w}^{(t)})
        % = \triangledown_{\textbf{w}} (\triangledown h_{r..r+B}(\textbf{w}))\\
        % & = \triangledown_{\textbf{w}} (\frac{1}{B}\sum_{i=r}^{r+B-1} y_i\textbf{x}_i(1-\frac{1}{1 + exp\{-y_i\textbf{w}^T\textbf{x}_i\}}) + \lambda \textbf{w})\\
        % &= \triangledown_{\textbf{w}} (\frac{1}{B}\sum_{i=r}^{r+B-1} y_i\textbf{x}_if(y_i\textbf{w}^T\textbf{x}_i) + \lambda \textbf{w})\\
        = -\frac{1}{B}\sum_{\mathscr{B}^{(t)}} \textbf{x}_i\textbf{x}_i^Tf'(y_i\textbf{w}^{(t)T}\textbf{x}_i)
    \end{split}
\end{align}

in which $-\sum_{i \in \mathscr{B}^{(t)}} \textbf{x}_i\textbf{x}_i^Tf'(y_i\textbf{w}^{(t)T}\textbf{x}_i)$ should be a semi-definite matrix since $f(x)  =1 -\frac{1}{1+exp\{-x\}}$ is a monotonically decreasing function and thus $f'(x)$ should be negative for any $x$.

\subsubsection{Notations for the linearized update rule}
After the interpolation step over the update rules for binary logistic regression, Equation \eqref{eq: mini_sgd_logistic_regression2} can be approximated as (It is Equation \eqref{eq: mini_sgd_instantiation_approx} in the paper):
\begin{align}
    \begin{split}
        \linearw^{(t+1)}& \approx [(1-\eta_t\lambda)\textbf{I} + \frac{\eta_t}{B}\sum_{i\in \mathscr{B}^{(t)}}a^{i, (t)}\textbf{x}_i\textbf{x}_i^T]\linearw^{(t)} + \frac{\eta_t}{B} \sum_{i\in \mathscr{B}^{(t)}} b^{i, (t)}y_i\textbf{x}_i
    \end{split}
\end{align}

which can be also represented as:
\begin{align}\label{eq: mini_sgd_instantiation_approx2}
    \begin{split}
        \linearw^{(t+1)}& \approx [(1-\eta_t\lambda)\textbf{I} + \frac{\eta_t}{B}\sum_{i\in \mathscr{B}^{(t)}}a^{i, (t)}\textbf{x}_i\textbf{x}_i^T]\linearw^{(t)} + \frac{\eta_t}{B} \sum_{i\in \mathscr{B}^{(t)}} b^{i, (t)}y_i\textbf{x}_i \\
        & = (1-\eta_t\lambda)\linearw^{(t)} + \frac{\eta_t}{B} \sum_{i \in \mathscr{B}^{(t)}} y_i\textbf{x}_i s(y_i\linearw^{(t)T}\textbf{x}_i)
    \end{split}
\end{align}

in which $s(x) = a^{i,(t)}x + b^{i,(t)}$.

Suppose after removing certain subset (the number of those samples is $\Delta n$ and the corresponding indices are $\mathcal{R}$), Equation \eqref{eq: mini_sgd_instantiation_approx2} becomes (It is Equation \eqref{eq: mini-SGD_logistic_regression_para_update} in the paper):

\begin{align}\label{eq: mini-SGD_logistic_regression_para_update2}
    \begin{split}
        &\linearincrew^{(t+1)} \approx [(1-\eta_t\lambda)\textbf{I} + \frac{\eta_t}{\increB^{(t)}}\sum_{\substack{ i \in \mathscr{B}^{(t)}, i \not\in \mathcal{R}}}a^{i, (t)}\textbf{x}_i\textbf{x}_i^T]\linearincrew^{(t)} + \frac{\eta_t}{\increB^{(t)}} \sum_{\substack{ i \in \mathscr{B}^{(t)}, i \not\in \mathcal{R}}} b^{i, (t)}y_i\textbf{x}_i \\
        & = (1-\eta_t\lambda)\linearincrew^{(t)} + \frac{\eta_t}{\increB^{(t)}} \sum_{i \in \mathscr{B}^{(t)}, i \not\in \mathcal{R}} y_i\textbf{x}_i s(y_i\linearincrew^{(t)T}\textbf{x}_i)
    \end{split}
\end{align}

For the linearized version of the update rule of logistic regression in Equation \eqref{eq: mini_sgd_instantiation_approx2} and the update rule in Equation \eqref{eq: mini-SGD_logistic_regression_para_update2}, we represent $\triangledown T^{(t)}(\linearw^{(t)})$ and $\triangledown R^{(t)}(\linearincrew^{(t)})$ as:
\begin{align}
\begin{split}
&\triangledown T^{(t)}_i(\linearw^{(t)}) = -y_i \textbf{x}_i s(y_i \linearw^{(t)T}\textbf{x}_i) = (- a^{i, (t)}\textbf{x}_i\textbf{x}_i^T)\linearw^{(t)}- b^{i, (t)}y_i\textbf{x}_i    
\end{split}\\
\begin{split}
&\triangledown T^{(t)}(\linearw^{(t)}) = \frac{1}{B}\sum_{i \in \mathscr{B}^{(t)}}\triangledown T^{(t)}_i(\linearw^{(t)})
\end{split}
\end{align}
\begin{align}\label{eq: T_second_derivative_logistic_regression}
\begin{split}
        \triangledown R^{(t)}_i(\linearincrew^{(t)}) &= -y_i \textbf{x}_i s(y_i \linearincrew^{(t)T}\textbf{x}_i) = -a^{i, (t)}\textbf{x}_i\textbf{x}_i^T\linearincrew^{(t)} - b^{i, (t)}y_i\textbf{x}_i 
        % (- \frac{1}{z^{(t)}}\sum_{\substack{ i \in \mathscr{B}^{(t)},\\i \in \{i_1, i_2, \dots, i_z\}}}a^{i, (t)}\textbf{x}_i\textbf{x}_i^T + \lambda\textbf{I})\linearincrew^{(t)}\\
        % &- \frac{1}{z^{(t)}} \sum_{\substack{ i \in \mathscr{B}^{(t)},\\i \in \{i_1, i_2, \dots, i_z\}}} b^{i, (t)}y_i\textbf{x}_i
\end{split}\\
\begin{split}
        \triangledown R^{(t)}(\linearincrew^{(t)}) &= \frac{1}{\increB^{(t)}}\sum_{\substack{ i \in \mathscr{B}^{(t)},i \not\in \mathcal{R} }}\triangledown R^{(t)}_i(\linearincrew^{(t)})
        % (- \frac{1}{z^{(t)}}\sum_{\substack{ i \in \mathscr{B}^{(t)},\\i \in \{i_1, i_2, \dots, i_z\}}}a^{i, (t)}\textbf{x}_i\textbf{x}_i^T + \lambda\textbf{I})\linearincrew^{(t)}\\
        % &- \frac{1}{z^{(t)}} \sum_{\substack{ i \in \mathscr{B}^{(t)},\\i \in \{i_1, i_2, \dots, i_z\}}} b^{i, (t)}y_i\textbf{x}_i
\end{split}
\end{align}

where $\triangledown T^{(t)}(\linearw^{(t)})$ and $\triangledown R^{(t)}(\linearincrew^{(t)})$ can be considered as pseudo-derivative in Equation \eqref{eq: mini_sgd_instantiation_approx2}. So Equation \eqref{eq: mini_sgd_instantiation_approx2} and Equation \eqref{eq: mini-SGD_logistic_regression_para_update2} can be rewritten as:
\begin{align}\label{eq: update_rule_w_l}
    & \linearw^{(t+1)} = (1-\eta_t\lambda)\linearw^{(t)} + \frac{\eta_t}{\increB^{(t)}} \sum_{i \in \mathscr{B}^{(t)}, i \not\in \mathcal{R}} y_i\textbf{x}_i s(y_i\linearw^{(t)T}\textbf{x}_i) = (1-\eta_t\lambda)\linearw^{(t)} - \eta_t\triangledown T^{(t)}(\linearw^{(t)})
\end{align}
\begin{align}
    & \linearincrew^{(t+1)} =  (1-\eta_t\lambda)\linearincrew^{(t)} + \frac{\eta_t}{\increB^{(t)}} \sum_{i \in \mathscr{B}^{(t)}, i \not\in \mathcal{R}} y_i\textbf{x}_i s(y_i\linearincrew^{(t)T}\textbf{x}_i) = (1-\eta_t\lambda)\linearincrew^{(t)} - \eta_t\triangledown R^{(t)}(\linearincrew^{(t)})
\end{align}

In contrast, by computing the model parameter from the scratch for logistic regression after removing the same set of training samples, the update rule is (It is Equation \eqref{eq: mini-SGD_updated_model_parameters_expected} in the paper):
\begin{align}\label{eq: mini-SGD_updated_model_parameters_expected2}
\begin{split}
    &\logistlinearincrew^{(t+1)} \leftarrow (1-\eta_t\lambda)\logistlinearincrew^{(t)} +\frac{\eta_t}{\increB^{(t)}} \sum_{\substack{ i \in \mathscr{B}^{(t)},i \not\in \mathcal{R}}} y_i\textbf{x}_i f(y_i\logistlinearincrew^{(t)}\textbf{x}_i) = (1-\eta_t\lambda)\logistlinearincrew^{(t)} - \eta_t \nabla^{(t)} g(\logistlinearincrew^{(t)})
\end{split}
\end{align}

which aims at minimizing the following objective function:
\begin{align}
\begin{split}
g(\textbf{w})& = \frac{1}{n-\Delta n}\sum_{\substack{i \not\in \mathcal{R}}}h_i(\w) +\frac{\lambda}{2}||\textbf{w}||^2_2
\end{split}
\end{align}

in which $\mathcal{R}$ represents the ids of the samples that are removed and $\Delta n$ represents the number of the removed samples and $$\nabla^{(t)} g(\logistlinearincrew^{(t)}) = \frac{1}{\increB^{(t)}}\sum_{i\in \mathscr{B}^{(t)}, i \not\in \mathcal{R}}\nabla h_i(\logistlinearincrew^{(t)}).$$

Similarly after removing certain subset, the update rule for linear regression model is (It is Equation \eqref{eq: gbm_linear_regression_incremental_updates} in the paper):
\begin{align}\label{eq: gbm_linear_regression_incremental_updates2}
    \begin{split}
         &\increw^{(t+1)} \leftarrow [(1-\eta_t\lambda)\textbf{I} -\frac{2\eta_t}{\increB^{(t)}} \sum_{\substack{ i \in \mathscr{B}^{(t)}}} \textbf{x}_i\textbf{x}_i^T\\
         &\hspace{-4mm} - \sum_{\substack{ i \in \mathscr{B}^{(t)}, i \in \mathcal{R}}} \textbf{x}_i\textbf{x}_i^T] \increw^{(t)} + \frac{2\eta_t}{\increB^{(t)}} (\sum_{\substack{ i \in \mathscr{B}^{(t)}}}\eat{\cdot} \textbf{x}_iy_i - \sum_{\substack{ i \in \mathscr{B}^{(t)}, i \in \mathcal{R}}}\eat{\cdot} \textbf{x}_iy_i)
    \end{split}
\end{align}

The provenance expression for the model parameters of linear regression model and logistic regression model after removing subset of training samples are (They are Equation \eqref{eq: mini_sgd_linear_regression_provenance_update} and Equation \eqref{eq: mini-SGD_logistic_regression_prov_update} in the paper):
\begin{align}\label{eq: mini_sgd_linear_regression_provenance_update2}
    \begin{split}
        &\increprov^{(t+1)} \leftarrow [(1-\eta_t\lambda)(\oneprov * \eat{\cdot} \textbf{I}) \\
        &-\frac{2\eta_t}{\increB^{(t)}} \sum_{\substack{ i \in \mathscr{B}^{(t)}, i \not\in \mathcal{R}}} p_i^2 * \eat{\cdot} \textbf{x}_i\textbf{x}_i^T] \increprov^{(t)}+ \frac{2\eta_t}{\increB^{(t)}} \sum_{\substack{ i \in \mathscr{B}^{(t)}, i \not\in \mathcal{R}}}p_i^2 * \eat{\cdot} \textbf{x}_iy_i
    \end{split}
\end{align}

\begin{align}\label{eq: mini-SGD_logistic_regression_prov_update2}
    \begin{split}
        &\linearincreprov^{(t+1)} \leftarrow [(1-\eta_t\lambda)(\oneprov*\textbf{I}) + \frac{\eta_t}{\increB^{(t)}}\sum_{\substack{ i \in \mathscr{B}^{(t)}, i \not\in \mathcal{R}}}p_i^2*(a^{i, (t)}\textbf{x}_i\textbf{x}_i^T)]\linearincreprov^{(t)}\\
        & + \frac{\eta_t}{\increB^{(t)}} \sum_{\substack{ i \in \mathscr{B}^{(t)}, i \not\in \mathcal{R}}} {p_i^2}*(b^{i, (t)}y_i\textbf{x}_i)
    \end{split}
\end{align}

\eat{Plus, we can apply Singular Value Decomposition (SVD) over $-\sum_{i=1}^{n} \textbf{x}_i\textbf{x}_i^Tf'(y_i\textbf{w}^T\textbf{x}_i)$, i.e. $-\sum_{i=1}^{n} \textbf{x}_i\textbf{x}_i^Tf'(y_i\textbf{w}^T\textbf{x}_i) =M_r^T diag(\{C_i\}_{i=1}^n) M_r$ where $M_r$ is an orthogonal matrix and $C_i$ is the singular value and eigenvalue.

The equation above can be plugged into Equation \eqref{eq: L_continuous_lemma}, i.e.:
\begin{align}\label{eq: second_derivative_expression_1}
\begin{split}
&||E_{\textbf{w}}(\triangledown^2 h_{r..r+B}(\textbf{w}))||_2 = ||\triangledown^2 h(\textbf{w})||_2\\ & = ||-\frac{1}{n}\sum_{i=1}^{n} \textbf{x}_i\textbf{x}_i^Tf'(y_i\textbf{w}^T\textbf{x}_i) + \lambda \textbf{I}||_2\\
& \myeqtwo \lambda + \frac{1}{B}max\{C_i\}_{i=1}^n \leq L
\end{split}
\end{align}
% Since $f(x)  =1 -\frac{1}{1+exp\{x\}}$, which is a monotonically decreasing function, then $f'(x)$ should be negative for any $x$ and thus $-\sum_{i=r}^{r+B-1} \textbf{x}_i\textbf{x}_i^Tf'(y_i\textbf{w}^T\textbf{x}_i)$ should be a semi-definite matrix, for which we can apply Singular Value Decomposition (SVD), i.e. $-\sum_{i=r}^{r+B-1} \textbf{x}_i\textbf{x}_i^Tf'(y_i\textbf{w}^T\textbf{x}_i) =$\\ $M_r^T diag(\{C_i\}_{i=1}^n) M_r$ where $M_r$ is an orthogonal matrix and $C_i$ is the singular value and eigenvalue at the same time. The decomposition results can be plugged back into 
% Equation \eqref{eq: second_derivative_expression_1}, i.e.:
% \begin{align}\label{eq: second_derivative_expression_2}
% \begin{split}
% &||\triangledown^2 h(\textbf{w})||_2 = ||\frac{1}{B}\sum_{i=r}^{r+B-1} \textbf{x}_i\textbf{x}_i^Tf'(y_i\textbf{w}^T\textbf{x}_i) + \lambda \textbf{I}||_2\\
% &=||-\frac{1}{B}M_r^T diag(\{C_i\}_{i=1}^n) M_r +\lambda \textbf{I}||_2\\
% & = ||-\frac{1}{B}M_r^T diag(\{C_i\}_{i=1}^n) M_r +\lambda M_r^TM_r||_2\\
% & = ||M_r^T diag(\{\lambda-\frac{1}{B}C_i\}_{i=1}^n) M_r||_2\\
% \end{split}
% \end{align}

% in which $M_r^T diag(\{\lambda-\frac{1}{B}C_i\}_{i=1}^n) M_r$ also represents an eigenvalue decomposition, which means that:
% \begin{align}
%     ||\triangledown^2 h(\textbf{w})||_2 = 
% \end{align}

In terms of the piecewise linear interpolant $s(x) = a_i x + b_i$ in Equation \eqref{eq: piecewise_interpolant} over $f(x) = 1-\frac{1}{1+exp\{-x\}}$, the following Lemma holds:
\begin{lemma}
$a_i \leq 0$
\end{lemma}}

\subsection{proof preliminary}

There are some useful properties related to matrix theory, matrix norm, real analysis and SGD convergence, which will be used in the follow-up proof.

\begin{lemma}[SGD convergence, \cite{bottou2018optimization}]\label{lemma: convergence_conditions2} (Full version of Lemma \eqref{lemma: convergence_conditions} in the paper)
Suppose that the stochastic gradient estimates are correlated with the true gradient, and bounded in the following way. There exist two scalars $J_1 \geq J_2 > 0$ such that for arbitrary $\mathscr{B}_{t}$, the following two inequalities hold:

\begin{align}
    & \nabla h\left(\w_{t}\right)^T \E\frac{1}{B_{t}}\sum_{i \in \mathscr{B}_{t}} \nabla h_i\left(\w_{t}\right) \geq J_2 \|\nabla h\left(\w_{t}\right)\|^2 \label{eq: sgd_exp_bound1}\\
    & \|\E\frac{1}{B_{t}}\sum_{i \in \mathscr{B}_{t}} \nabla h_i\left(\w_{t}\right)\| \leq J_1 \|\nabla h\left(\w_{t}\right)\|\label{eq: sgd_exp_bound2}
\end{align}

Also, for two scalars $J_3, J_4 \geq 0$ we have:
\begin{align}
    Var\left(\frac{1}{B_{t}}\sum_{i \in \mathscr{B}_{t}} \nabla h_i\left(\w_{t}\right)\right) \leq J_3 + J_4 \|\nabla h\left(\w_{t}\right)\|^2\label{eq: sgd_var_bound}
\end{align}

By combining equations \eqref{eq: sgd_exp_bound1}-\eqref{eq: sgd_var_bound}, the following inequality holds:

\begin{align}
    \begin{split}
        \E\|\frac{1}{B_{t}}\sum_{i \in \mathscr{B}_{t}} \nabla F_i\left(\w_{t}\right)\|^2 \leq J_3 + J_5 \|\nabla F\left(\w_{t}\right)\|^2 
    \end{split}
\end{align}

where $J_5 = J_4 + J_1^2 \geq J_2^2 \geq 0$.

Then stochastic gradient descent with fixed step size $\eta_t = \eta \leq \frac{J_2}{L J_5}$ has the convergence rate:

\begin{align}
\begin{split}
    & \E\left[h\left(\w_{t}\right) - h\left(\w^*\right)\right]\\
    & \leq \frac{\eta L J_3}{2\mu J_2} + \left(1-\eta\mu J_2\right)^{t-1}\left(h\left(\w_{1}\right) - h\left(\w^*\right) - \frac{\eta L J_3}{2\mu J_2}\right) \rightarrow \frac{\eta L J_3}{2\mu J_2}
\end{split}
\end{align}

If the gradient estimates are unbiased, then $\E\frac{1}{B_{t}}\sum_{i \in \mathscr{B}_{t}} \nabla h_i\left(\w_{t}\right) $ $= \frac{1}{n} \sum_{i=1}^n \nabla h_i\left(\w_{t}\right) = \nabla h\left(\w_{t}\right)$ and thus $J_1 = J_2 = 1$. Moreover, $J_3\sim 1/B$, where $B$ is the minibatch size, because $J_2$ is the variance of the stochastic gradient.

So the convergence condition for fixed step size becomes $\eta_t = \eta \leq \frac{1}{LJ_5}$, in which $J_5 = J_4 +J_1^2 = J_4 + 1 \geq 1$. So $\eta_t = \eta \leq \frac{1}{LJ_5} \leq \frac{1}{L}$ suffices to ensure convergence.

\end{lemma}

So in what follows, we will simply consider the case where the learning rate is a constant across all the iterations as Lemma \ref{lemma: convergence_conditions2} indicates.

\begin{lemma}\label{lm: matrix_norm}
For a matrix $\textbf{A}$, its $L2-$norm equals to its largest singular value and the maximal eigenvalue of matrix $\textbf{A}^T\textbf{A}$, i.e.:
$||\textbf{A}||_2 = \sigma_{max}(\textbf{A}) = \sqrt{C_{max}(\textbf{A}^T\textbf{A})}$.

where $\sigma_{max}$ and $C_{max}$ represents the largest singular value and the largest eigenvalue of certain matrix.

If $\textbf{A}$ is a semi-definite matrix, its eigenvalue is the same as its singular value, then the equation above can be rewritten as:
$||\textbf{A}||_2 = \sigma_{max}(\textbf{A}) = C_{max}(\textbf{A})$.
\end{lemma}

% \begin{lemma}\label{lm: matrix_norm_with_identity_matrix}
% For a semi-definite matrix $\textbf{A}$, for any real value $\lambda$, the following equality holds:
% \begin{equation}
%     ||\textbf{A} + \lambda \textbf{I}||_2 = \sigma_{max}(\textbf{A}) + \lambda =  C_{max}(\textbf{A}) + \lambda
% \end{equation}
% \end{lemma}

\begin{lemma}
If an $n \times n $ matrix $\textbf{A}$ is a real symmetric matrix, then we can find $n$ mutually orthogonal eigenvectors for $\textbf{A}$.
\end{lemma}

\begin{lemma}\label{lm: linear_system_convergence}
Given an iteration formula $\textbf{u}^{(t+1)} = \textbf{A}\textbf{u}^{(t)} + \textbf{b}$ where $\textbf{A}$ is a matrix while $\textbf{u}^{(t)}$ is a vector to be derived iteratively, if $\textbf{I}-\textbf{A}$ is invertible, then the following statements are equivalent:
\begin{enumerate}
    \item $\textbf{u}^{(t)}$ will get converged
    \item $||\textbf{B}||_p < 1$ for some matrix norm $||||_p$
\end{enumerate}
\end{lemma}

\begin{lemma}{\bf Cauchy schwarz inequality}\label{lm: cs_inequality}
For any two matrix $\textbf{A}$ and $\textbf{B}$, their norm should satisfy the {\em Cauchy schwarz inequality}, i.e.:
$||\textbf{A}\textbf{B}||_x \leq ||\textbf{A}||_x||\textbf{B}||_x$ where $||\cdot||_x$ represents any matrix norm
\end{lemma}

{\begin{lemma}\label{lm: Weyl's_inequality}\textbf{Weyl's inequality}
For any three Hermitian matrices, $\textbf{M}, \textbf{N}, \textbf{P}$ satisfying $\textbf{M} = \textbf{N} + \textbf{P}$, the eigenvalues of $\textbf{M}$ is: $\mu_1 \geq \mu_2 \geq \mu_3 \dots \geq \mu_n$;

the eigenvalues of $\textbf{N}$ is: $v_1 \geq v_2 \geq v_3 \dots \geq v_n$;

and the eigenvalues of $\textbf{P}$ is: $\rho_1 \geq \rho_2 \geq \rho_3 \dots \geq \rho_n$;

the following inequalities hold:
$v_i + \rho_n \leq \mu_i \leq v_i + \rho_1$

\end{lemma}
}

The following lemma requires the definition of Lipschitz-continuity and Strong-convexity, which are provided below:

{\bf Lipschitz-continuous}
A function $f(x)$ is Lipschitz-continuous ($L-$continuous) if there exists a constant $L$ such that the following inequality is satisfied for all $x,y$:
\begin{equation}\label{eq: L_continuous}
    |f(y) - f(x)|\leq L||y-x||^2_2
\end{equation}
Another form of Equation \eqref{eq: L_continuous} is:
\begin{equation}\label{eq: L_continuous_2}
    f(y) \leq f(x) + <\triangledown f(x), y-x> + \frac{L}{2}||y-x||_2^2
\end{equation}

{\bf Strong convexity}
A function $f(x)$ is $\lambda-$strong convexity iff there exists a constant $\lambda$ such that the following inequality is satisfied for all $x,y$:
\begin{equation}\label{eq: strong_convexity}
    f(y) \geq f(x) + <\triangledown f(x), y-x> + \frac{\lambda}{2}||y-x||^2
\end{equation}
Other equivalent forms of Equation \eqref{eq: strong_convexity} are:
\begin{equation}\label{eq: strong_convexity2}
    (\triangledown f(x) - \triangledown f(y))(x - y) \geq \lambda ||x - y||_2^2
\end{equation}
\begin{equation}\label{eq: strong_convexity3}
    \triangledown^2 f(x) \geq \lambda
\end{equation}

Then there is a useful lemma about $\lambda-$strong convexity, i.e:
\begin{lemma}
a function $f(x)$ is a strong convex function iff $f(x)-\frac{\lambda}{2}||x||^2_2$ is a convex function, 
\end{lemma}

\begin{lemma} \textbf{Piecewise linear interpolation} In Piecewise linear interpolation~\cite{Kress1998}\label{lemma: piecewise_interpolation_bound}, we assume that the function to be approximated is a continuous function $f(x)$ where $x \in [a, b]$. Piecewise linear interpolation starts by picking up a series of {\em breaking points}, $x_i$ such that $a < x_1 < x_2 < \dots < x_p < b$ and then constructs a linear interpolant $s(x)$ over each interval $[x_{j-1}, x_{j})$ as follows:
\begin{align}\label{eq: piecewise_interpolant}
\begin{split}
    s(x) &= \frac{x-x_{j-1}}{x_j-x_{j-1}}f(x_j) + \frac{x_j-x}{x_j-x_{j-1}}f(x_{j-1})\\
    & = a_jx + b_j, x \in [x_{j-1}, x_j)
\end{split}
\end{align}

The following property holds on how close the value of $s(x)$ is compared to the original function $f(x)$:
\begin{align}\label{eq: piecewise_approx_rate}
    \begin{split}
        |f(x) - s(x)| &\leq \frac{1}{8}(\Delta x)^2 \max_{a\leq x \leq b}|f''(x)| = O((\Delta x)^2)\\
        |f'(x) - s'(x)| &\leq \frac{1}{2}(\Delta x) \max_{a\leq x \leq b}|f''(x)| = O((\Delta x))
    \end{split}
\end{align}

\end{lemma}

% In what follows, we assume that 
\begin{lemma}\textbf{Expectation of the number of the removed samples}
Because of the randomness from \minisgd, the $\Delta n$ removed samples can be viewed as uniformly distributed within all $n$ training samples, which can be considered as a $0-1$ Bernoulli distribution with probability $\frac{\Delta n}{n}$. In other words, we can define a random variable $\textbf{S}_i$ for each sample, which is 1 with probability $\frac{\Delta n}{n}$ and 0 with probability $1-\frac{r}{n}$. So within a single mini-batch $\mathscr{B}^{t}$, we can have \begin{align*}
    \E(\sum_{i\in \mathscr{B}^{t}} \textbf{S}_i) = \E(\Delta B_t) = B\frac{r}{n}
\end{align*} and 
\begin{align*}
Var(\sum_{i\in \mathscr{B}^{t}} \textbf{S}_i) = B \frac{r}{n}(1-\frac{r}{n})    
\end{align*}

So in terms of the random variable $\frac{\Delta B_t}{B}$, its expectation and variance will be 
\begin{align}\label{eq: removed_num_exp}
    \E(\frac{\Delta B_t}{B}) = \frac{r}{n}
\end{align} and 
\begin{align}\label{eq: removed_num_var}
Var(\frac{\Delta B_t}{B}) = \frac{r}{Bn}(1-\frac{r}{n})    
\end{align}
\end{lemma}

In \minisgd, a typical assumption is used for the convergence analysis of the model parameter $\textbf{w}^{(t)}$ in Equation \eqref{eq: mini_sgd_linear_regression2}-\eqref{eq: mini_sgd_logistic_regression2} and the update rules for other general models, i.e.:

\begin{lemma}\label{lemma: sgd_assumption}
For any randomly selected sample $i_j$ in some batch, the expectation of its gradient should be the same as the gradient over the all the samples, i.e.:

$E(\triangledown h_{i_j}(\textbf{w})) = \triangledown h(\textbf{w})$

which also implies that the following equality holds for \minisgd:

$E(\triangledown( \frac{1}{B} \sum_{i\in \mathscr{B}^{(t)}}h_{i}(\textbf{w}))) = \triangledown h(\textbf{w})$

where $E$ is the expectation value with respect to the sampling over the entire training samples.
\end{lemma}

In what follows, our analysis is based on the following assumptions:
\begin{assumption}\label{assp: hessian_property}
every $h_i(\textbf{w})$ ($i=1,2,\dots,n$) is $L-$Lipschitz continuous. Since $h_i(\textbf{w})$ has $L2-$norm regularization term, then we also know that $h_i(\textbf{w})$ is $\lambda-$strong convex.
\end{assumption}

\begin{assumption}\label{assp: bounded_grad}
each $\nabla h_i (\w^{(t)})$ is bounded by some constant $c_1$ for each $\w^{(t)}$.
\end{assumption}

\begin{assumption}\label{assp: continuous}
The function $f'(*)$ is $c_2-$Lipschitz continuous, which means that the following inequality holds:
\begin{align*}
    |f'(x) - f'(y)| \leq c_2 \|x-y\|
\end{align*}
\end{assumption}

\subsection{Main results and proofs}

\begin{theorem}\label{theorem: non_convergence2} (It is Theorem \ref{theorem: non_convergence} in the paper)
$\increprov^{(t)}$ in Equation \eqref{eq: mini_sgd_linear_regression_provenance_update2} and $\linearincreprov^{(t)}$ in Equation \eqref{eq: mini-SGD_logistic_regression_prov_update2} need not
converge under the conditions in Lemma \ref{lemma: convergence_conditions2}.
\end{theorem}

\begin{proof}
% \subsection{Proof of Theorem \ref{theorem: non_convergence2}}\label{sec: non_convergency_proof}
Let us take linear regression as an example. Note that we can explicitly evaluate the second order derivative of $h(\textbf{w})$ for linear regression, i.e. $\triangledown^2 h(\textbf{w})$, then according to Assumption \ref{assp: hessian_property}, $\triangledown^2 h(\textbf{w})$ should satisfy the following inequality:
\begin{align}\label{eq: hessian_linear_regression_bound}
    \begin{split}
        \lambda \leq ||\triangledown^2 h(\textbf{w})||_2 = ||\frac{2}{n} \sum_{i=1}^{n}\textbf{x}_i\textbf{x}_i^T + \lambda\textbf{I}||_2 \leq L 
    \end{split}
\end{align}

In order to prove Theorem \ref{theorem: non_convergence2}, we need to show that there exists a case where $\increprov^{(t)}$ cannot converge under the conditions in Lemma \ref{lemma: convergence_conditions2}.  This is achieved by considering gradient descent (\gd) without excluding any original training samples, i.e. $\{p_{i_1}, p_{i_2}, \dots, p_{i_z}\} = \{1,2,\dots,n\}$, every $\increB^{(t)}=\;n$ in Equation \eqref{eq: mini_sgd_linear_regression_provenance_update2} and every $\mathscr{B}^{(t)}$ includes all $n$ samples in Equations \eqref{eq: mini_sgd_linear_regression2} and \ref{eq: mini_sgd_linear_regression_provenance_update2}. \eat{$\mathcal{W}^{(t)}$ in Equation \eqref{eq: mini_sgd_linear_regression_provenance} since $\mathcal{W}^{(t)}$ is a special case of $\increprov^{(t)}$. Plus, we only consider gradient descent (\gd) here.}
We can then apply the update rule in Equations \eqref{eq: mini_sgd_linear_regression2} and \eqref{eq: mini_sgd_linear_regression_provenance_update2} recursively, which ends up with:
\begin{align}
\begin{split}
&\textbf{w}^{(t+1)} \label{eq: w_expansion}
% = ((1-\eta_t\lambda)\textbf{I} - \frac{2\eta_t}{n} \sum_{i=1}^{n} \textbf{x}_i\textbf{x}_i^T)\textbf{w}^{(t)} + \frac{2\eta_t}{n} \sum_{i=1}^{n} \textbf{x}_i y_i\\&
= ((1-\eta\lambda)\textbf{I} - \frac{2\eta}{n} \sum_{i=1}^{n} \textbf{x}_i\textbf{x}_i^T)^{t+1}\textbf{w}^{(0)}\\
& + (\sum_{j=1}^t ((1-\eta\lambda)\textbf{I} - \frac{2\eta}{n} \sum_{i=1}^{n} \textbf{x}_i\textbf{x}_i^T)^{j}) \frac{2\eta}{n} \sum_{i=1}^{n} \textbf{x}_i y_i
\end{split}\\
\begin{split}\label{eq: w_prov_expansion}
&\increprov^{(t)}
% = ((1-\eta_t\lambda)\textbf{I}\cdot 1_k - \frac{2\eta_t}{n} \sum_{i=1}^{n} \textbf{x}_i\textbf{x}_i^T \cdot p_i^2)\mathcal{W}^{(t)} \\&
% + \frac{2\eta_t}{n} \sum_{i=1}^{n} \textbf{x}_i y_i \cdot p_i^2\\&
= ((1-\eta\lambda)\oneprov*\textbf{I}- \frac{2\eta}{n} \sum_{i=1}^{n} p_i^2 * \textbf{x}_i\textbf{x}_i^T )^{t}\increprov^{(0)}\\
& + (\sum_{j=1}^t ((1-\eta\lambda)\oneprov*\textbf{I} - \frac{2\eta}{n} \sum_{i=1}^{n} p_i^2 *\textbf{x}_i\textbf{x}_i^T)^{j}) \frac{2\eta}{n} \sum_{i=1}^{n} p_i^2*\textbf{x}_i y_i
\end{split}
\end{align}

According to Assumption \ref{assp: hessian_property}, the following inequality should be satisfied:
\begin{align}\label{eq: hessian_linear_regression_bound2}
    \begin{split}
        ||\eta\triangledown^2 h(\textbf{w})||_2 = ||\frac{2\eta}{n} \sum_{i=1}^{n}\textbf{x}_i\textbf{x}_i^T + \lambda\eta\textbf{I}||_2 \leq \eta L \leq 1 
    \end{split}
\end{align}

which implies that 
\begin{align}\label{eq: hessian_linear_regression_bound3}
    \begin{split}
     ||(1-\eta\lambda)\textbf{I} - \frac{2\eta}{n} \sum_{i=1}^{n}\textbf{x}_i\textbf{x}_i^T||_2 \leq 1 
    \end{split}
\end{align}

Also since every $\textbf{x}_i\textbf{x}_i^T$ is a semi-positive definite matrix, by using Lemma \ref{lm: matrix_norm}, Equation \eqref{eq: hessian_linear_regression_bound2} also implies that:
\begin{align}\label{eq: hessian_linear_regression_bound4}
    \begin{split}
     & 1 \geq ||\frac{2\eta}{n} \sum_{i=1}^{n}\textbf{x}_i\textbf{x}_i^T + \lambda\eta\textbf{I}||_2 = C_{max}(\frac{2\eta}{n} \sum_{i=1}^{n}\textbf{x}_i\textbf{x}_i^T + \lambda\eta\textbf{I})\\
     & \geq C_{max}(\frac{2\eta}{n} \sum_{i=1}^{n}\textbf{x}_i\textbf{x}_i^T) = ||\frac{2\eta}{n} \sum_{i=1}^{n}\textbf{x}_i\textbf{x}_i^T||_2
    \end{split}
\end{align}

Then by applying Equation \eqref{lm: Weyl's_inequality}, Equation \eqref{eq: hessian_linear_regression_bound4} also leads to:
\begin{equation}\label{eq: x_i_bound}
    ||\frac{2\eta}{n}\textbf{x}_i\textbf{x}_i^T||_2 = C_{max}(\frac{2\eta}{n}\textbf{x}_i\textbf{x}_i^T) < C_{max}(\frac{2\eta}{n} \sum_{i=1}^{n}\textbf{x}_i\textbf{x}_i^T) \leq 1
\end{equation}

\eat{Observe that the update rule in $\textbf{w}^{(t)}$ in Equation \eqref{eq: w_expansion} is actually a linear system. According to the convergence conditions for iterative linear system \scream{give citations}, since Equation \eqref{eq: hessian_linear_regression_bound3} holds, then Equation \eqref{eq: w_expansion} should be converged. In order to make sure that $\textbf{w}^{(t)}$ converges, the condition, $||((1-\eta\lambda)\textbf{I} - \frac{2\eta}{n} \sum_{i=1}^{n} \textbf{x}_i\textbf{x}_i^T)||_2 \leq 1$, should be satisfied, where $||*||_2$ represents the $L2-$norm of certain matrix. Since every $\textbf{x}_i\textbf{x}_i^T$ is a semi-positive definite matrix, under the convergence conditions, for every $i$, $||\frac{2\eta}{n} \textbf{x}_i\textbf{x}_i^T||_2 \leq 1$ (details omitted).
}

Then we can expand the first term in the right-hand side of Equation \eqref{eq: w_prov_expansion}, the tensor product with provenance monomial $p_i^{t}$ should be $p_i^t*{t\choose \frac{t}{2}}(1-\eta\lambda)^{\frac{t}{2}}(-\frac{2\eta}{n} \textbf{x}_i\textbf{x}_i^T)^{\frac{t}{2}}$. According to the convergence conditions in Lemma \ref{lemma: convergence_conditions2}, $\eta < \frac{1}{L}$ and thus $||{t\choose \frac{t}{2}}(1-\eta\lambda)^{\frac{t}{2}}(-\frac{2\eta}{n} \textbf{x}_i\textbf{x}_i^T)^{\frac{t}{2}}||_2 \geq {t\choose \frac{t}{2}}||(-\frac{\eta}{n} \textbf{x}_i\textbf{x}_i^T)^{\frac{t}{2}}(\frac{2(L-\lambda)}{L})^{\frac{t}{2}}||_2$. According to \cite{sun2001convergence, dasbrief}, when $t\rightarrow \infty$, ${t\choose \frac{t}{2}}$ should be very close to $2^t$ and thus when $||-\frac{\eta}{n} \textbf{x}_i\textbf{x}_i^T||_2 \geq \frac{L}{2(L-\lambda)}$ (note that $||-\frac{2\eta}{n} \textbf{x}_i\textbf{x}_i^T||_2$ can be any value between 0 and 1 according to Equation \eqref{eq: x_i_bound}), ${t\choose \frac{t}{2}}||(-\frac{\eta}{n} \textbf{x}_i\textbf{x}_i^T)^{\frac{t}{2}}||_2 \rightarrow \infty$, which means that the tensor product with provenance monomial $p_i^{t}$ cannot converge and thus $\increprov^{(t)}$ cannot converge.

\eat{If we apply gradient descent for linear regression, the model parameter and provenance expression without idempotence over $*_k$ is as below, which can be derived recursively (Suppose we use constant learning rate, i.e. $\eta_t = \eta$ for all the iterations): 
\begin{align}
\begin{split}
&\textbf{w}^{(t+1)} 
% = ((1-\eta_t\lambda)\textbf{I} - \frac{2\eta_t}{n} \sum_{i=1}^{n} \textbf{x}_i\textbf{x}_i^T)\textbf{w}^{(t)} + \frac{2\eta_t}{n} \sum_{i=1}^{n} \textbf{x}_i y_i\\&
= ((1-\eta\lambda)\textbf{I} - \frac{2\eta}{n} \sum_{i=1}^{n} \textbf{x}_i\textbf{x}_i^T)^{t+1}\textbf{w}^{(0)}\\
& + (\sum_{j=1}^t ((1-\eta\lambda)\textbf{I} - \frac{2\eta}{n} \sum_{i=1}^{n} \textbf{x}_i\textbf{x}_i^T)^{j}) \frac{2\eta}{n} \sum_{i=1}^{n} \textbf{x}_i y_i
\end{split}\\
\begin{split}
&\mathcal{W}^{(t+1)} 
% = ((1-\eta_t\lambda)\textbf{I}\cdot 1_k - \frac{2\eta_t}{n} \sum_{i=1}^{n} \textbf{x}_i\textbf{x}_i^T \cdot p_i^2)\mathcal{W}^{(t)} \\&
% + \frac{2\eta_t}{n} \sum_{i=1}^{n} \textbf{x}_i y_i \cdot p_i^2\\&
= ((1-\eta\lambda)\textbf{I}\cdot 1_k - \frac{2\eta}{n} \sum_{i=1}^{n} \textbf{x}_i\textbf{x}_i^T \cdot p_i^2)^{t+1}\mathcal{W}^{(0)}\\
& + (\sum_{j=1}^t ((1-\eta\lambda)\textbf{I}\cdot 1_k - \frac{2\eta}{n} \sum_{i=1}^{n} \textbf{x}_i\textbf{x}_i^T \cdot p_i^2)^{j}) \frac{2\eta}{n} \sum_{i=1}^{n} \textbf{x}_i y_i \cdot p_i^2
\end{split}
\end{align}

Then according to the convergence conditions for iterative linear system, in order for $\textbf{w}^{(t)}$ to be converged, $||((1-\eta\lambda)\textbf{I} - \frac{2\eta}{n} \sum_{i=1}^{n} \textbf{x}_i\textbf{x}_i^T)||_2 \leq 1$, which matches the convergence conditions for gradient descent in \cite{karimi2016linear}. Since $(\textbf{x}_i\textbf{x}_i^T)$ is a semi-definite positive matrix, then for every $i$, $||\frac{2\eta}{n} \textbf{x}_i\textbf{x}_i^T||_2 \leq 1$

Note that $\mathcal{W}^{(0)} = \textbf{w}^{(0)} \cdot 1_k$, then after applying binomial expansion over the term in the last but one line of the formula above, the tensor product with provenance monomial $p_i^{2k}$ will be: $\frac{t!}{(t-k)!k!}(1-\eta \lambda)^{t-k}(-\frac{2\eta}{n}{\textbf{x}_i\textbf{x}_i^T})^k\cdot p_i^{2k}$. Note that when $k = \frac{t}{2}$, $\frac{t!}{(t-k)!k!} \geq (\frac{t}{k})^{k} = 2^{\frac{t}{2}}$ and thus $\frac{t!}{(t-k)!k!}(1-\eta \lambda)^{t-k}||(-\frac{2\eta}{n}{\textbf{x}_i\textbf{x}_i^T})^k||_2 \geq  2^{\frac{t}{2}}(1-\eta \lambda)^{t-k}||(-\frac{2\eta}{n}{\textbf{x}_i\textbf{x}_i^T})^k||_2$. Due to the existence of $2^{\frac{t}{2}}$, there will be no convergence guarantee when $t \rightarrow \infty$.} \end{proof}

\begin{theorem}\label{theorem: convergence_res2} (It is Theorem \ref{theorem: convergence_res} in the paper)
The expectation of $\increprov^{(t)}$ in Equation \eqref{eq: mini_sgd_linear_regression_provenance_update2} and of $\linearincreprov^{(t)}$ in Equation \eqref{eq: mini-SGD_logistic_regression_prov_update2}, 
%i.e. $E(\increprov^{(t)})$ and $E(\linearincreprov^{(t)})$, 
converge when $t \rightarrow \infty$ if we also assume that provenance polynomial
multiplication is \emph{idempotent}.
\end{theorem}

\eat{The following proofs will be described using logistic regression. By denoting $\linearincreprov^{(t)}(p_{j_1}, p_{j_2},\dots, p_{j_s})$ as $\textbf{v}$ (a matrix), its update rule should be as follows:
% First of all, due to the ``idempotence'' of $*_K$, the expectation of Equation \eqref{eq: mini-SGD_logistic_regression_prov_update2} is:
% \begin{align}\label{eq: w_prov_expansion}
%     \begin{split}
%         & E(\linearincreprov^{(t+1)}) \leftarrow [(1-\eta_t\lambda)(1_K*\textbf{I})\\
%         & + \frac{\eta_t}{n}\sum_{i=1}^n a^{i, (t)}(p_i*\textbf{x}_i\textbf{x}_i^T)]\linearincreprov^{(t)}+ \frac{\eta_t}{n} \sum_{i=1}^n {p_i}*b^{i, (t)}y_i\textbf{x}_i
%     \end{split}
% \end{align}
% \begin{align}
% \begin{split}\label{eq: w_prov_expansion}
% &E(\increprov^{(t)})
% = ((1-\eta\lambda)\textbf{I}\cdot 1_k - \frac{2\eta}{n} \sum_{i=1}^{n} \textbf{x}_i\textbf{x}_i^T \cdot p_i)^{t}\increprov^{(0)}\\
% & + (\sum_{j=1}^t ((1-\eta\lambda)\textbf{I}\cdot 1_k - \frac{2\eta}{n} \sum_{i=1}^{n} \textbf{x}_i\textbf{x}_i^T \cdot p_i)^{j}) \frac{2\eta}{n} \sum_{i=1}^{n} \textbf{x}_i y_i \cdot p_i
% \end{split}
% \end{align}
% Plus, we also observe that $\linearincreprov^{(t)}(p_{j_1}, p_{j_2},\dots, p_{j_s}) = \linearincreprov^{(t)}|_{\substack{p_i = 0, p_i \not \in \{p_{j_1} \\, p_{j_2},\dots, p_{j_s}\}}}$. So we start from the simplest case where $\{p_{j_1}, p_{j_2},\dots, p_{j_s}\}$ is a singleton set. So Equation \eqref{eq: w_prov_expansion} can be modified as below:
\begin{align}\label{eq: mini-SGD_logistic_regression_prov_update_singleton}
    \begin{split}
        & E(\textbf{v}^{(t+1)}) = 
        % ((1-\eta_t\lambda) + \frac{\eta_t}{n}\sum_{r \in \{j_1, j_2, \dots, j_s\}}a^{r, (t)}(\textbf{x}_r\textbf{x}_r^T))\textbf{v}^{(t)}\\
        % & + \frac{\eta_t}{n} \sum_{r \in \{j_1, j_2, \dots, j_s\}}(b^{r, (t)}y_r\textbf{x}_r)\\
        \textbf{v}^{(t)}\\
        & - \eta_t (\lambda \textbf{v}^{(t)} - \frac{1}{n}\sum_{r \in \{j_1, j_2, \dots, j_s\}} (a^{r,(t)}\textbf{x}_r\textbf{x}_r^T \textbf{v}^{(t)} + b^{r, (t)}y_r\textbf{x}_r))\\
        &=  \textbf{v}^{(t)} - \eta_t \triangledown^{(t)} R_{j_1,j_2,\dots, j_s}(\textbf{v}^{(t)})
    \end{split}
\end{align}
\noindent
where $\triangledown^{(t)} R_{j_1,j_2,\dots, j_s}(\textbf{v}^{(t)})$ represents the term in the second from the last line in Equation \eqref{eq: mini-SGD_logistic_regression_prov_update_singleton}, which should satisfy the following two inequalities.
\begin{align}
        &||\triangledown^{(t)} R_{j_1,j_2,\dots, j_s}(\textbf{v}^{(t)})||_2 < C'\label{eq: R_property1} \\
        & <\textbf{v}^{(t)} - \textbf{v}^*, \triangledown^{(t)} R_{j_1,j_2,\dots, j_s}(\textbf{v}^{(t)})> \geq \lambda ||\textbf{v}^{(t)} - \textbf{v}^*||_2^2 \label{eq: R_property2}
\end{align}
\noindent
where $C'$ is some constant and $\textbf{v}^*$ represents the matrix that satisfies $\triangledown^{(t)}R_{j_1,j_2,\dots, j_s}(\textbf{v}^*) = \textbf{0}$. By combining the equations above, we obtain the following inequality:
\begin{align}\label{eq: u_gap}
    \begin{split}
        &E(||\textbf{v}^{(t+1)} - \textbf{v}^*||_2^2)\\
        & = E(||\textbf{v}^{(t)} - \eta \triangledown^{(t)} R^{(t)}(\textbf{v}^{(t)}) - \textbf{v}^*||_2^2)\\
        & = ||\textbf{v}^{(t)} - \textbf{v}^*||_2^2 - 2\eta <\triangledown R_{j_1,j_2,\dots, j_s}(\textbf{v}^{(t)}), \textbf{v}^{(t)} - \textbf{v}^*>\\
        & + \eta^2 ||R_{j_1,j_2,\dots, j_s}^{(t)}(\textbf{v}^{(t)})||_2^2\\
        & \myleqtwo (1-2\eta \lambda )||\textbf{v}^{(t)} - \textbf{v}^*||_2^2 + C'^2\eta^2
    \end{split}
\end{align}

By deriving Equation \eqref{eq: u_gap} recursively and under the condition that $\eta_t \leq \frac{1}{2\lambda}$ we can conclude that $\textbf{v}^{(t)}$ has the same convergence rate as $\textbf{w}^{(t)}$. 

Then according to Equation \eqref{eq: provenanace_eq_update_model_logistic}, due to the equality between $\linearincreprov^{(t)}(p_{i_1}, p_{i_2},\dots, p_{i_z})$ and $\linearincrew^{(t)}$, we can conclude that $\linearincrew^{(t)}$ should be converged under the convergence conditions of the original model parameters $\textbf{w}^{(t)}$. \qed}

% \subsection{Proof of Theorem \ref{theorem: convergence_res2}}\label{sec: theorem_convergency_proof}

% which can also be rewritten as below by referencing Lemma \ref{lm: matrix_norm}:
% \begin{equation}\label{eq: L_continuous_lemma_2} 
%     ||\triangledown^2 h(\textbf{w})||_2 =  C_{max}(\triangledown^2 h(\textbf{w})^T\triangledown^2 h(\textbf{w})) = \sigma_{max}(\triangledown^2 h(\textbf{w}))\leq L
% \end{equation}

\eat{which is proved as below:
\begin{proof}
By applying SVD, $\textbf{A} = \textbf{M}^T diag(\{c_i\}_{i=1}^n) \textbf{M}$ where $c_i$ is the singular value of $\textbf{A}$. Since $\textbf{A}$ is a semi-definte matrix, $\textbf{M}$ should be a orthogonal matrix and each $c_i$ should be also eigenvalue of $\textbf{A}$ according to Lemma \ref{lm: matrix_norm}.

$\textbf{A} + \lambda \textbf{I} = \textbf{M}^T diag(\{c_i\}_{i=1}^n) \textbf{M} + \lambda \textbf{I} = \textbf{M}^T diag(\{c_i\}_{i=1}^n) \textbf{M} + \lambda \textbf{M}^T \textbf{M} = \textbf{M}^T diag(\{c_i + \lambda\}_{i=1}^n)\textbf{M}$

which is thus the eigendecomposition for $\textbf{A} + \lambda \textbf{I}$ whose eigenvalues are $\{c_i + \lambda\}_{i=1}^n$. So $||\textbf{A} + \lambda \textbf{I}||_2 = ||\textbf{M}^T diag(\{c_i + \lambda\}_{i=1}^n)\textbf{M}||_2 = \max(c_i + \lambda) = \max(c_i) + \lambda = C_{max}(A) + \lambda$
\end{proof}
}

\eat{The proof of the convergence of $\mathcal{W}^{(t)}$ in Equation \eqref{eq: mini_sgd_linear_regression_provenance} and $\linearprov^{(t)}$ in Equation \eqref{eq: mini_sgd_instantiation_approx2} also relies on convergence analysis from \cite{karimi2016linear}, which is briefly presented below.

\eat{\begin{lemma}\label{lemma: convergence_conditions2}
Given an objective function $h(\textbf{w})$, under the assumption in Lemma \ref{lemma: sgd_assumption}, by applying \sgd\ and \minisgd\ over $h(\textbf{w})$ with learning rate $\eta_t$ at $t_{th}$ iteration and $L2-$regularization term $\lambda ||\textbf{w}||_2^2$,  once the condition $\eta_t < \frac{1}{2\lambda}$ hold and $\eta_t$ is a constant across all the iterations, $\textbf{w}^{(t)}$ converges, we can obtain a linear convergence rate up to a solution level that is proportional to $\eta$, i.e.:
\begin{align}
\begin{split}
&E(h(\textbf{w}^{(t)}) - h^*) \leq (1-2\lambda\eta)^k(h(\textbf{w}^{(0)}) - h^*) + \frac{LC^2\eta}{4\lambda}\\
& = (1-2\lambda\eta)^k(h(\textbf{w}^{(0)}) - h^*) + O(\eta)
\end{split}\\
\begin{split}
&E(\textbf{w}^{(t)} - \textbf{w}^*) \leq (1-2\lambda\eta)^k(\textbf{w}^{(0)} - \textbf{w}^*) + \frac{C^2\eta}{2\lambda}\\
& = (1-2\lambda\eta)^k(\textbf{w}^{(0)} - \textbf{w}^*) + O(\eta)
\end{split}
\end{align}
where $C^2$ is the upper bound of $E[||\triangledown h_i(\textbf{w}^{(t)})||_2^2]$ for all $\textbf{w}^{(t)}$ and any randomly selected sample $\textbf{x}_i$ at $t_{th}$ iteration.
\end{lemma}
}
\eat{In terms of the assumption from Lemma \ref{lemma: sgd_assumption}, by applying linearization over the update rule of logistic regression, suppose the objective function becomes $T(\textbf{\textbf{w}})$, then at each iteration, the following equality should hold:
\begin{lemma}\label{lemma: expectation_linearization}
\begin{equation}
E(\triangledown T^{(t)}(\linearw^{(t)})) = \triangledown  T(\linearw^{(t)})    
\end{equation}
\begin{equation}
E(\triangledown^2 T^{(t)}(\linearw^{(t)})) = \triangledown T(\linearw^{(t)})    
\end{equation}
\end{lemma}

Also after the removal of subset of training samples, the property above should also hold, i.e.:
\begin{lemma}\label{lemma: expectation_incremental}
\begin{align}
    \begin{split}
    E(\triangledown \frac{1}{n-\Delta n}\sum_{i=1}^{n-\Delta n - 1}h_i(\textbf{w})) = \triangledown  h(\textbf{w}^{(t)})    
    \end{split}\\
    \begin{split}
    E(\triangledown^2 \frac{1}{n-\Delta n}\sum_{i=1}^{n-\Delta n - 1}h_i(\textbf{w})) = \triangledown^2 h(\textbf{w}^{(t)})
    \end{split}
\end{align}
\end{lemma}

This is because $E$ can be regarded as sampling over the entire training set and each sample should satisfy $E(\triangledown h_i(\textbf{w})) = \triangledown h(\textbf{w})$ and $E(\triangledown^2 h_i(\textbf{w})) = \triangledown^2 h(\textbf{w})$}

After introducing those lemmas above, we start the formal proof for Theorem \ref{theorem: convergence_res2}, which starts by proving the convergence property of the provenance expression for linear regression, i.e. $\mathcal{W}^{(t)}$ in Equation \eqref{eq: mini_sgd_linear_regression_provenance}. }

\textbf{Convergence proof for linear regression}
We simply need to consider whether $\mathcal{W}^{(t)}(\{p_{i_1}, p_{i_2}, \dots, p_{i_z}\})$ converges (suppose there are $\Delta n$ provenance tokens in total that are set as 0, which corresponds to the deletion of $\Delta n$ samples), which equals to the update rule in Equation \eqref{eq: gbm_linear_regression_incremental_updates2} and leads to a new objective function without the $\Delta n$ removed samples (denoted by $(\Delta \textbf{X}, \Delta \textbf{Y})$), i.e:

% By applying \gbm\ over $g(\textbf{w})$, we can get updated model parameter $\increw^{(t)}$ at $t_{th}$ iteration. Notice that the expectation of $g^{(t)}(\textbf{w}), \triangledown g^{(t)}(\textbf{w})$ and $\triangledown^2 g^{(t)}(\textbf{w})$ should satisfy the following equation for any $\textbf{w}$ i.e.:
% \begin{align}\label{eq: g_expectation}
%     \begin{split}
%     &E(g^{(t)}(\textbf{w})) = h(\textbf{w})\\
%     &E(\triangledown g^{(t)}(\textbf{w})) = \triangledown h(\textbf{w})\\
%     &E(\triangledown^2 g^{(t)}(\textbf{w})) = \triangledown^2 h(\textbf{w})
%     \end{split}
% \end{align}
\eat{
Suppose the optimal model parameter of $g(\textbf{w})$ is $\increw^*$, then:
\begin{align}\label{eq: linear_regression_convergence}
    \begin{split}
        & E(||\increw^{(t+1)} - \increw^*||_2^2) \\
        & = E(||\increw^{(t)} - \eta_t \triangledown g^{(t)}(\increw^{(t)}) - \increw^*||_2)\\
        & = ||\increw^{(t)} - \increw^*||_2^2 \\
        & -2\eta_t <\increw^{(t)} - \increw^*, E(\triangledown g^{(t)}(\increw^{(t)}))>\\
        & + \eta_t^2 E||\triangledown g^{(t)}(\increw^{(t)})||_2^2
    \end{split}
\end{align}

Since $\triangledown g^{(t)}(\increw^*) = 0$, then the term in the last but one line of Equation \eqref{eq: linear_regression_convergence} can be bounded as below:
\begin{align}\label{eq: linear_regression_derivation}
    \begin{split}
        & <\increw^{(t)} - \increw^*, E(\triangledown g^{(t)}(\increw^{(t)}))> \\
        & = <\increw^{(t)} - \increw^*, E(\triangledown g^{(t)}(\increw^{(t)})) - E(\triangledown g^{(t)}(\increw^*))> \\
        & \myeqthree <\increw^{(t)} - \increw^*, \triangledown h(\increw^{(t)}) - \triangledown h(\increw^*)>\\
        & \mygeqone \lambda ||\increw^{(t)} - \increw^*||_2^2
    \end{split}
\end{align}

\eat{since $E$ can be regarded as sampling over the entire training set. 
\eat{Based on this result, then we can compute expectation and apply $L2-norm$ on both sides of Equation \eqref{eq: gbm_linear_regression_matrix_form} and Equation \eqref{eq: gbm_linear_regression_incremental_updates2} (see Equation \eqref{eq: gbm_linear_regression_matrix_form_expectation} and Equation \eqref{eq: gbm_linear_regression_incremental_updates_expectation} respectively):
\begin{align}\label{eq: gbm_linear_regression_matrix_form_expectation}
    \begin{split}
        ||E(\textbf{w}^{(t+1)})||_2 & \leq  ||(1-\eta_t\lambda)\textbf{I} - \frac{2\eta_t}{n}\textbf{X}^T\textbf{X}||_2||\textbf{w}^{(t)}||_2\\
        & +\frac{2\eta_t}{n}||\textbf{X}^T\textbf{Y}||_2
    \end{split}
\end{align}
\begin{align}\label{eq: gbm_linear_regression_incremental_updates_expectation}
    \begin{split}
        ||E(\increw^{(t)})||_2 &\leq ||(1-\eta_t\lambda)\textbf{I} - \frac{2\eta_t}{n}\textbf{X}^T\textbf{X}||_2||\increw^{(t)}||_2\\
        & +||\frac{2\eta_t}{n}(\textbf{X}^T\textbf{Y})||_2\\
    \end{split}
\end{align}}

Also we can give the explicit expression of $\triangledown^2 h(\textbf{w})$ and $\triangledown^2 g(\textbf{w})$, i.e.:
\begin{align}
\begin{split}
    &\triangledown^2 h(\textbf{w}) = \frac{2}{n}\textbf{X}^T\textbf{X} + \lambda \textbf{I}
\end{split}\\
\begin{split}
&\triangledown^2 g(\textbf{w}) = \frac{2}{n -\Delta n}(\textbf{X}^T\textbf{X} - \Delta \textbf{X}^T \Delta \textbf{X}) + \lambda \textbf{I}\\
& = (\frac{2}{n} + \frac{2\Delta n}{n(n-\Delta n)})(\textbf{X}^T\textbf{X} - \Delta \textbf{X}^T \Delta \textbf{X}) + \lambda \textbf{I}
\end{split}
\end{align}

So the difference between $\triangledown^2 h(\textbf{w})$ and $\triangledown^2 g(\textbf{w})$ is bounded by:
\begin{align}
\begin{split}
    & ||\frac{2}{n}\Delta \textbf{X}^T \Delta \textbf{X} - \frac{2\Delta n}{n(n-\Delta n)}(\textbf{X}^T \textbf{X} - \Delta \textbf{X}^T \Delta \textbf{X})||_2\\
    & \leq ||\frac{2\Delta n}{n(n-\Delta n)}(\textbf{X}^T \textbf{X} - \Delta \textbf{X}^T \Delta \textbf{X})||_2\\
    & + ||\frac{2}{n}\Delta \textbf{X}^T \Delta \textbf{X}||_2 \leq \frac{4\Delta n}{n}\sqrt{m}U
\end{split}
\end{align}

where $U$ is the upper bound of each entry value in $\textbf{X}$, which is usually 1 after normalization or standardization over $\textbf{X}$. Considering the fact that $\frac{\Delta n}{n}$ is small due to the small changes over the input, $||\frac{2}{n}\Delta \textbf{X}^T \Delta \textbf{X} - \frac{2\Delta n}{n(n-\Delta n)}(\textbf{X}^T \textbf{X} - \Delta \textbf{X}^T \Delta \textbf{X})||_2 \leq \frac{4\Delta n}{n}U^2$.
% , which is small compared to the upper bound of $||\triangledown^2 h(\textbf{w})||_2$, $L$. So $||\triangledown^2 h(\textbf{w})-\triangledown^2g(\textbf{w})|| \leq o(L)$. 
Thus the $L2-$norm of $\triangledown^2g(\textbf{w})$ can be bounded by:
\begin{align}
\begin{split}
& ||\triangledown^2 g(\textbf{w})||_2 = ||\triangledown^2 g(\textbf{w}) - \triangledown^2 h(\textbf{w}) + \triangledown^2 h(\textbf{w})||_2\\
& \leq ||\triangledown^2 g(\textbf{w}) - \triangledown^2 h(\textbf{w})||_2 + ||\triangledown^2 h(\textbf{w})||_2\\
& \leq L +\frac{4\Delta n}{n}U^2 = L'
\end{split}
\end{align}

in which $L'$ is a constant.}
Besides, the following equation should hold for $E||\triangledown h_i(\textbf{w}^{(t)})||^2_2$ and $E||\triangledown g_i(\textbf{w}^{(t)})||^2_2$ for any $\textbf{w}^{(t)}$, i.e.:
\begin{equation}\label{eq: linear_regression_second_derivative_bound}
    E[||\triangledown h_i(\textbf{w}^{(t)})||^2_2] = E[||\triangledown g_i(\textbf{w}^{(t)})||^2_2] \leq C^2
\end{equation}

Then by combining Equation \eqref{eq: linear_regression_derivation} and Equation \eqref{eq: linear_regression_second_derivative_bound}, Equation \eqref{eq: linear_regression_convergence} can be derived recursively as below:
\begin{align}
    \begin{split}
        & E(||\increw^{(t+1)} - \increw^*||_2^2) \\
        & \leq (1-2\lambda \eta)||\increw^{(t)} - \increw^*||_2^2 + C^2\eta^2 \\
        & \leq (1-2\lambda \eta)^{(t+1)}||\increw^{(0)} - \increw^*||_2^2 + O(\eta)
    \end{split}
\end{align}

So in the case of constant $\eta_t$ across different iterations, under the convergence conditions in Lemma \ref{lemma: convergence_conditions2}, i.e. $\eta_t = \eta < \frac{1}{2\lambda}$, $E(g(\increw^{(t)}))$ should be converged and thus $\increw^{(t)}$ should get converged.
}

In what follows, we will only prove the convergence of binary logistic regression, i.e. the convergence of $\linearincreprov^{(t)}$, which is the same as proving the convergence of $\linearincrew^{(t)}$. The convergence of linear regression and multi-nomial logistic regression can be proven in the similar way. 

According to Lemma \ref{lemma: piecewise_interpolation_bound}, the following equation holds:
% and Equation \eqref{eq: mini-SGD_logistic_regression_para_update2}, which satisfies based on Equation \eqref{eq: piecewise_approx_rate}:
\begin{align}
    &\|\triangledown T^{(t)}_i(\textbf{w}^{(t)}) - \triangledown h^{(t)}_i(\textbf{w}^{(t)}) \| = \|y_i \textbf{x}_i(s(y_i\textbf{w}^{(t)T}\textbf{x}_i) - f(y_i\textbf{w}^{(t)T}\textbf{x}_i))\| = O((\Delta x)^2)\label{eq: approx_bound_1}\\
    & \|\triangledown^2 T^{(t)}_i(\textbf{w}^{(t)}) - \triangledown^2 h^{(t)}_i(\textbf{w}^{(t)})\| = \|-a^{i, (t)}\textbf{x}_i\textbf{x}_i^T + f'(y_i\textbf{w}^{(t)T}\textbf{x}_i)\textbf{x}_i\textbf{x}_i^T\| = O((\Delta x))\label{eq: approx_bound_2}
\end{align}
for all $\w^{(t)}$ (rather than $\logistlinearincrew^{(t)}$ or $\linearincrew^{(t)}$ since $a^{i,(t)}$ and $b^{i, (t)}$ are evaluated when $\w$ is $\w^{(t)}$). Since $\lambda \leq \|\triangledown^2 h^{(t)}(\textbf{w}^{(t)})\| \leq L$, then $\lambda - O(\Delta x) \leq \|\triangledown^2 T^{(t)}(\textbf{w}^{(t)})\| \leq L + O(\Delta x)$. Then by bringing in the definition of $\triangledown^2 T^{(t)}(\textbf{w}^{(t)})$, we know that:
\begin{align}
    & \lambda - O(\Delta x) \leq \|\triangledown^2 T^{(t)}(\textbf{w}^{(t)})\| = \|-a^{i, (t)}\textbf{x}_i\textbf{x}_i^T + \lambda\textbf{I} \| \leq L +O(\Delta x)
\end{align}

By using the fact that $\|*\|_2 = C_{max}(*)$, the following formula also holds:
\begin{align}
    & \lambda - O(\Delta x) \leq C_{max}(-a^{i, (t)}\textbf{x}_i\textbf{x}_i^T + \lambda\textbf{I}) \leq L +O(\Delta x)
\end{align}

But note that since every $a^{i,(t)}$ is a negative value, then $-a^{i, (t)}\textbf{x}_i\textbf{x}_i^T$ is a semi-positive definite matrix and thus:
\begin{align}
    \|-a^{i, (t)}\textbf{x}_i\textbf{x}_i^T + \lambda\textbf{I}\| \geq \lambda
\end{align}

Then we bound $\textbf{I} - \eta (-a^{i, (t)}\textbf{x}_i\textbf{x}_i^T + \lambda\textbf{I})$ as:
\begin{align}\label{eq: bound_t_w}
\begin{split}
& \|\textbf{I} - (-\eta a^{i, (t)}\textbf{x}_i\textbf{x}_i^T + \eta\lambda\textbf{I})\| \leq 1 - \eta \lambda
\end{split}
\end{align}

Similarly we know that the following inequality holds for any $\textbf{w}$:
\begin{align}\label{eq: h_hessian_bound}
    \|\textbf{I} - \eta\triangledown^2 h^{(t)}(\textbf{w})\| = \|(1-\eta \lambda)\textbf{I} - \eta [-\textbf{x}_i\textbf{x}_i^Tf'(y_i\textbf{w}^{T}\textbf{x}_i)]\| \leq 1-\eta \lambda
\end{align}

Then the convergence of $\linearincrew$ can be derived below (in the case of constant $\eta_t$ according to Lemma \ref{lemma: convergence_conditions2}):
\begin{align}\label{eq: w_LU_gap}
    \begin{split}
        &||\linearincrew^{(t+1)} - \linearincrew^*||_2\\
        & = ||\linearincrew^{(t)} - \eta \triangledown^{(t)} R^{(t)}(\linearincrew^{(t)}) - \linearincrew^* + \eta \triangledown^{(t)} R^{(t)}(\linearincrew^{*}) - \eta \triangledown^{(t)} R^{(t)}(\linearincrew^{*})||_2
        % & = ||\linearincrew^{(t)} - \linearincrew^*||_2^2 - 2\eta <\triangledown R(\linearincrew^{(t)}), \linearincrew^{(t)} - \linearincrew^*>\\
        % & + \eta^2 ||R^{(t)}(\linearincrew^{(t)})||_2^2\\
        % & \leq (1-2\eta \lambda )||\linearincrew^{(t)} - \linearincrew^*||_2^2 + C'^2\eta^2
    \end{split}
\end{align}

Then by using the fact that if $\linearincrew^{(t)}$ converges, $\|\triangledown^{(t)} R^{(t)}(\linearincrew^{*})\| \leq C$ for some constant value $C$ for all $t$. By using the triangle inequality and the definition of $\triangledown R^{(t)}(*)$, the formula above is bounded as:
\begin{align*}
    & \leq \|\{\textbf{I} - \frac{1}{\increB^{(t)}}\sum_{i \in \mathscr{B}^{(t)}, i \not\in \mathcal{R}}\eta[-a^{i, (t)}\textbf{x}_i\textbf{x}_i^T + \lambda\textbf{I} ]\}(\linearincrew^{(t)} - \linearincrew^{*})\| + \eta C\\
    & \leq \|\textbf{I} - \frac{1}{\increB^{(t)}}\sum_{i \in \mathscr{B}^{(t)}, i \not\in \mathcal{R}}\eta[-a^{i, (t)}\textbf{x}_i\textbf{x}_i^T + \lambda\textbf{I} ]\| \|(\linearincrew^{(t)} - \linearincrew^{*})\| + \eta C
\end{align*}

By using the result from Equation \eqref{eq: bound_t_w}, we get:
\begin{align*}
    & \leq (1-\eta \lambda)\| \|(\linearincrew^{(t)} - \linearincrew^{*})\| + \eta C 
\end{align*}

By applying the formula above recursively, we get:
\begin{align*}
    & ||\linearincrew^{(t+1)} - \linearincrew^*||_2 \leq \frac{C}{\lambda}
\end{align*}

This finishes the proof.

% The last inequality is by using Equation \eqref{eq: R_strong_convexity} and Equation \eqref{eq: R_bound}.

% Then by applying Equation \eqref{eq: w_LU_gap} recursively, we can get:
% \begin{equation}
%     E(||\linearincrew^{(t+1)} - \linearincrew^*||_2^2) \leq (1-2\eta\lambda)^{t+1} ||\linearincrew^{(0)} - \linearincrew^*||_2^2 + O(\eta)
% \end{equation}

% which indicates the convergence of $\linearincrew^{(t)}$ under the convergence conditions of Lemma \ref{lemma: convergence_conditions2}, i.e. $\eta \leq \frac{1}{2\lambda}$.

\eat{and the difference between $\triangledown^2 h^{(t)}(\textbf{w}^{(t)})$ and $\triangledown^2 T^{(t)}(\linearw^{(t)})$ is quantified by computing the $L2-$norm of the difference between Equation \eqref{eq: second_derivative_h_logistic_regression} and Equation \eqref{eq: T_second_derivative_logistic_regression}, i.e.:
\begin{align}\label{eq: gap_h_T_w_second_deri}
\begin{split}
    & ||\triangledown^2 h^{(t)}(\textbf{w}^{(t)}) - \triangledown^2 T^{(t)}(\linearw^{(t)})||_2\\
    & = ||\frac{1}{B}\sum_{i=r^{(t)}}^{r^{(t)}+B-1}(a^{i,(t)} - f'(y_i\textbf{w}^{(t)T}\textbf{x}_i))\textbf{x}_i\textbf{x}_i^T||_2\\
    & = ||\frac{1}{B}\sum_{i=r^{(t)}}^{r^{(t)}+B-1}(a^{i,(t)} - f'(y_i\textbf{w}^{(t)T}\textbf{x}_i))\textbf{x}_i\textbf{x}_i^T||_2 \\
    & \leq \frac{1}{B}\sum_{i=r^{(t)}}^{r^{(t)}+B-1}|a^{i,(t)} - f'(y_i\textbf{w}^{(t)T}\textbf{x}_i)|||\textbf{x}_i\textbf{x}_i^T||_2\\
    & \leq O(\Delta x)mD^2
\end{split}
\end{align}

By tuning the size of $\Delta x$, the difference above can be small enough, i.e. $||\triangledown^2 h^{(t)}(\textbf{w}^{(t)}) - \triangledown^2 T^{(t)}(\linearw^{(t)})||_2 \leq o(1)$.

Similarly, we can then quantify the difference between $\triangledown^2 R^{(t)}(\linearincrew^{(t)})$ (from Equation \eqref{eq: mini-SGD_updated_model_parameters}) and $\triangledown^2 T^{(t)}(\linearw^{(t)})$, i.e.:
\begin{align}\label{eq: gap_Q_T_second_deri}
    \begin{split}
        &||\triangledown^2 R^{(t)}(\linearincrew^{(t)}) - \triangledown^2 T^{(t)}(\linearw^{(t)})||_2\\
        & = ||\frac{1}{z^{(t)}}\sum_{i=r^{(t)}, i \in \{i_1, i_2,\dots, i_z\}}^{r^{(t)}+B-1}a^{i, (t)}(\textbf{x}_i\textbf{x}_i^T)\\
        & - \frac{1}{B}\sum_{i=r^{(t)}}^{r^{(t)}+B-1}a^{i, (t)}(\textbf{x}_i\textbf{x}_i^T)||_2\\
        & \leq \frac{4(B - z^{(t)})}{B} U^2 
    \end{split}
\end{align}

In the end, we can bound the $L2-$norm of $\triangledown^2 R^{(t)}(\linearincrew^{(t)})$ by using the result of Equation \eqref{eq: gap_h_T_w_second_deri} and Equation \eqref{eq: gap_Q_T_second_deri}, i.e.:
\begin{align}
\begin{split}
& ||\triangledown^2 R^{(t)}(\linearincrew^{(t)})||_2 = ||\triangledown^2 R^{(t)}(\linearincrew^{(t)})\\
& - \triangledown^2 T^{(t)}(\linearw^{(t)}) + \triangledown^2 T^{(t)}(\linearw^{(t)})\\
& - \triangledown^2 h^{(t)}(\textbf{w}^{(t)}) + \triangledown^2 h^{(t)}(\textbf{w}^{(t)})||_2\\
& \leq \frac{4(B - z^{(t)})}{B} U^2  + o(1) + L \leq L'
\end{split}
\end{align}

in which $L'$ is a constant. Besides, we can bound the distance between $\triangledown T_i(\textbf{w}^{(t)})$ and $\triangledown h_i(\textbf{w}^{(t)})$ which represent the gradient of $T(\textbf{w})$ and $h(\textbf{w})$ over each sample $i$, i.e.:
\begin{align}\label{eq: T_h_derivative_difference}
\begin{split}
&||\triangledown T_i(\textbf{w}^{(t)}) - \triangledown h_i(\textbf{w}^{(t)})||_2\\
    & = ||y_i\textbf{x}_i (f(y_i\textbf{w}^{(t)T}\textbf{x}_i) -s(y_i\textbf{w}^{(t)T}\textbf{x}_i))||_2 \\
    & \leq \sqrt{m}DO((\Delta x)^2) \approx o(1)    
\end{split}
\end{align}

Besides, according to Lemma \ref{lemma: expectation_incremental}, $E[||\triangledown R_i(\textbf{w}^{(t)})||_2^2] =  E[||\triangledown T_i(\textbf{w}^{(t)})||_2^2]$. So by considering the result of Equation \eqref{eq: T_h_derivative_difference}, $E[||\triangledown R_i(\textbf{w}^{(t)})||_2^2] = E[||\triangledown h_i(\textbf{w}^{(t)})||_2^2] = C^2 + o(1)$. So in the end, by following the same derivation procedure in Equation \eqref{eq: prove_convergence}, we can prove that $\E(\linearincrew^{(t)})$, i.e. the expectation of $\linearincrew^{(t)}$ in Equation \eqref{eq: mini_sgd_instantiation_approx2}, converges.}

\eat{the following equality holds:
\begin{align}
    \begin{split}
        ||E_{\textbf{w}} (\triangledown^2 T_{r..r+B}(\textbf{w})) - \triangledown^2 T(\textbf{w})||_2 \leq o(1)
    \end{split}
\end{align}

in which $\triangledown^2 T(\textbf{w})$ and $(\triangledown^2 T_{r..r+B}(\textbf{w}))$ are both semi-positive definite matrices. 
% Then we can compare $\triangledown T(\textbf{w})$ against $\triangledown h(\textbf{w})$, i.e.:
% \begin{align}
%     \begin{split}
%     &||\triangledown h(\textbf{w}) - \triangledown T(\textbf{w})||_2\\
%     & = \frac{1}{n} \sum_{i=1}^{n} y_i\textbf{x}_i f(y_i\textbf{w}\textbf{x}_i)\\  
%     \end{split}
% \end{align}
$\triangledown^2 T(\textbf{w})$ against $\triangledown^2 h(\textbf{w})$, i.e.:
\begin{align*}
        &||\triangledown^2 h(\textbf{w}) - \triangledown^2 T(\textbf{w})||_2 = ||\frac{1}{n}\sum_{i=1}^{n} (f'(y_i\textbf{w}^T\textbf{x}_i) - a^{i,(t)})\textbf{x}_i\textbf{x}_i^T||_2\\
        % & = ||\frac{1}{B}\sum_{i=r}^{r+B-1}\textbf{x}_i\textbf{x}_i^T (f'(y_i\textbf{w}^T\textbf{x}_i) - a^{i,(t)})||_2\\
        % & = \frac{1}{B}||\begin{bmatrix} \textbf{x}_1 & \textbf{x}_2 & \dots & \textbf{x}_n
        % \end{bmatrix}diag(\{f'(y_i\textbf{w}^T\textbf{x}_i) - a^{i,(t)}\}_{i=1}^n)\begin{bmatrix} \textbf{x}_1^T \\ \textbf{x}_2^T \\ \dots \\ \textbf{x}_n^T
        % \end{bmatrix}||_2\\
        % &\csineq \frac{1}{B}||\begin{bmatrix} \textbf{x}_1 & \textbf{x}_2 & \dots & \textbf{x}_n
        % \end{bmatrix}||_2\\
        % &||diag(\{f'(y_i\textbf{w}^T\textbf{x}_i) - a^{i,(t)}\}_{i=1}^n)||_2||\begin{bmatrix} \textbf{x}_1^T \\ \textbf{x}_2^T \\ \dots \\ \textbf{x}_n^T
        % \end{bmatrix}||_2\\
        % &= \frac{1}{B} ||\textbf{X}||_2^2 ||diag(\{f'(y_i\textbf{w}^T\textbf{x}_i) - a^{i,(t)}\}_{i=1}^n)||_2\\
        % & \leq \frac{1}{B}||\textbf{X}||_2^2 max_{i}|f'(y_i\textbf{w}^T\textbf{x}_i) - a^{i,(t)}|\\
        &\myleqone \frac{1}{n}||\textbf{X}||_2^2 |O(\Delta x)| \leq \frac{1}{n}(\sum_i\sum_j x_{ij}^2)|O(\Delta x)|\numberthis\label{eq: second_derivative_gap}
\end{align*}

and compare $\triangledown^2 T_{r..r+B}(\textbf{w})$ against $\triangledown^2 h_{r..r+B}(\textbf{w})$, i.e.:
\begin{align*}
        &||\triangledown^2 h_{r..r+B}(\textbf{w}) - \triangledown^2 T_{r..r+B}(\textbf{w}))||_2\\
        &= ||(\frac{1}{B}\sum_{i=r}^{r+B} (f'(y_i\textbf{w}^T\textbf{x}_i) - a^{i,(t)})\textbf{x}_i\textbf{x}_i^T||_2\\
        % & = ||\frac{1}{B}\sum_{i=r}^{r+B-1}\textbf{x}_i\textbf{x}_i^T (f'(y_i\textbf{w}^T\textbf{x}_i) - a^{i,(t)})||_2\\
        % & = \frac{1}{B}||\begin{bmatrix} \textbf{x}_1 & \textbf{x}_2 & \dots & \textbf{x}_n
        % \end{bmatrix}diag(\{f'(y_i\textbf{w}^T\textbf{x}_i) - a^{i,(t)}\}_{i=1}^n)\begin{bmatrix} \textbf{x}_1^T \\ \textbf{x}_2^T \\ \dots \\ \textbf{x}_n^T
        % \end{bmatrix}||_2\\
        % &\csineq \frac{1}{B}||\begin{bmatrix} \textbf{x}_1 & \textbf{x}_2 & \dots & \textbf{x}_n
        % \end{bmatrix}||_2\\
        % &||diag(\{f'(y_i\textbf{w}^T\textbf{x}_i) - a^{i,(t)}\}_{i=1}^n)||_2||\begin{bmatrix} \textbf{x}_1^T \\ \textbf{x}_2^T \\ \dots \\ \textbf{x}_n^T
        % \end{bmatrix}||_2\\
        % &= \frac{1}{B} ||\textbf{X}||_2^2 ||diag(\{f'(y_i\textbf{w}^T\textbf{x}_i) - a^{i,(t)}\}_{i=1}^n)||_2\\
        % & \leq \frac{1}{B}||\textbf{X}||_2^2 max_{i}|f'(y_i\textbf{w}^T\textbf{x}_i) - a^{i,(t)}|\\
        &\myleqone \frac{1}{B}||\textbf{X}_{r..r+B}||_2^2 |O(\Delta x)| \leq \frac{1}{B}(\sum_{i=r}^{r+B}\sum_j x_{ij}^2)|O(\Delta x)|\numberthis\label{eq: second_derivative_gap}
\end{align*}

In mose cases, there should be a upper bound for each feature. Suppose $X_{ij}$ is bounded by a value $D$, then $(\sum_i\sum_j x_{ij}^2)$ can be bounded by $mn D^2$ and thus:
\begin{equation}
    ||\triangledown^2 h(\textbf{w}) - \triangledown^2 T(\textbf{w})||_2 \leq mD^2|O(\Delta x)|
\end{equation}
\begin{equation}
    ||\triangledown^2 h_{r..r+B}(\textbf{w}) - \triangledown^2 T_{r..r+B}(\textbf{w})||_2 \leq mD^2|O(\Delta x)|
\end{equation}

which also implies that:
\begin{align}
    & ||E(\triangledown^2 h_{r..r+B}(\textbf{w}) - \triangledown^2 T_{r..r+B}(\textbf{w}))||_2\\
    & ||\triangledown^2 h(\textbf{w}) - E(\triangledown^2 T_{r..r+B}(\textbf{w}))||_2
    \leq mD^2|O(\Delta x)|
\end{align}

for which we can tune the size of $\Delta x$ to make $mD^2|O(\Delta x)|$ as small as possible, which means that:
\begin{equation}
    ||\triangledown^2 h(\textbf{w}) - \triangledown^2 T(\textbf{w})||_2 \leq o(1)
\end{equation}
\begin{equation}
    ||\triangledown^2 h_{r..r+B}(\textbf{w}) - \triangledown^2 T_{r..r+B}(\textbf{w})||_2 \leq o(1)
\end{equation}
\begin{equation}
    ||\triangledown^2 h(\textbf{w}) - E(\triangledown^2 T_{r..r+B}(\textbf{w}))||_2 \leq o(1)
\end{equation}

Then we can bound the $L2-$norm of $E(\triangledown^2 T_{r..r+B}(\textbf{w}))$, i.e.:
\begin{align}\label{eq: T_w_upper_bound}
\begin{split}
&||E(\triangledown^2 T_{r..r+B}(\textbf{w}))||_2 = ||E(\triangledown^2 T_{r..r+B}(\textbf{w})) - \triangledown^2 h(\textbf{w}) + \triangledown^2 h(\textbf{w})||_2\\
& \csineq ||E(\triangledown^2 T_{r..r+B}(\textbf{w})) - \triangledown^2 h(\textbf{w})||_2\\
& + ||\triangledown^2 h(\textbf{w})||_2 \leq L + o(1)
\end{split}
\end{align}

Since $-\frac{1}{B}\sum_{i=r}^{r+B-1} \textbf{x}_i\textbf{x}_i^Tf'(y_i\textbf{w}^T\textbf{x}_i)$ is a semi-positive definite matrix and $\lambda > 0$, then $E(\triangledown^2 T_{r..r+B}(\textbf{w}))$ should be a positive definite matrix. Suppose the eigenvalues of $\E(\triangledown^2 T_{r..r+B}(\textbf{w}))$ is $\{c_i'\}_{i=1}^n$, by referencing Equation \eqref{eq: T_w_upper_bound}, we can also bound every $c_i'$:
\begin{align}\label{eq: T_w_eigen_bound}
    0<c_i' \leq max\{c_i'\}_{i=1}^n = ||E(\triangledown^2 T_{r..r+B}(\textbf{w}))||_2 \leq L + o(1)
\end{align}

So the bound of the $L2-$norm of $(\textbf{I} - \eta_t E(\triangledown^2 T_{r..r+B}(\textbf{w})))$, i.e.:
\begin{align}
    \begin{split}
        ||\textbf{I} - \eta_t E(\triangledown^2 T_{r..r+B}(\textbf{w}))||_2 = max\{|1-\eta_t c_i'|\}
    \end{split}
\end{align}

According to Equation \eqref{eq: T_w_eigen_bound}, since $c_i' \leq L+o(1)$, we can further bound the formula above, i.e. :
\begin{align}\label{eq: T_w_bound}
    \begin{split}
        &(\forall i) -\eta_t o(1)\leq 1-\eta_tL-\eta_t o(1) \leq 1-\eta_t c_i' \leq 1\\
        & \Rightarrow (\forall i) |1-\eta_t c_i'| < 1 \Rightarrow max\{|1-\eta_t c_i'|\} < 1\\
        & \Rightarrow ||\textbf{I} - \eta_t E(\triangledown^2 T_{r..r+B}(\textbf{w}))||_2 = max\{|1-\eta_t c_i'|\} \leq 1
    \end{split}
\end{align}

We also observe that Equation \eqref{eq: mini_sgd_instantiation_approx2} can be rewritten as:
\begin{align}\label{eq: iterative_terms}
    \begin{split}
         \linearw^{(t+1)}& = ((1-\eta_t\lambda)\textbf{I} + \frac{\eta_t}{B}\sum_{i=r}^{r+B-1}a^{i, (t)}\textbf{x}_i\textbf{x}_i^T) \linearw^{(t)}\\
        & + \frac{\eta_t}{B} \sum_{i=r}^{r+B-1} b^{i, (t)}y_i\textbf{x}_i\\
        & = (\textbf{I} - \eta_t \triangledown^2 T_{r..r+B}(\textbf{w})) \linearw^{(t)} + \frac{\eta_t}{B} \sum_{i=r}^{r+B-1} b^{i, (t)}y_i\textbf{x}_i\\
    \end{split}
\end{align}

for which we can apply expectation over both sides, i.e.:
\begin{align}\label{eq: iterative_terms_expect}
    \begin{split}
        E( \linearw^{(t+1)}) &= (E(\textbf{I} - \eta_t \triangledown^2 T_{r..r+B}(\textbf{w}))) \linearw^{(t)}\\
        & + E(\frac{\eta_t}{B} \sum_{i=r}^{r+B-1} b^{i, (t)}y_i\textbf{x}_i)\\
        & = (\textbf{I} - \eta_t E(\triangledown^2 T_{r..r+B}(\textbf{w})))\textbf{w}^{(t)} + \frac{\eta_t}{n}\sum_{i=1}^n b^{i, (t)}y_i\textbf{x}_i
    \end{split}
\end{align}

which can be derived recursively. Since $\triangledown^2 T(\textbf{w})$ is a positive definite matrix, it should be a invertible matrix. If $\textbf{I} - \eta_t E(\triangledown^2 T_{r..r+B}(\textbf{w}))$ in Equation \eqref{eq: iterative_terms} is regarded as $\textbf{B}$ in Lemma \ref{lm: linear_system_convergence}, then $\textbf{w}^{(t)'}$ should be converged, which is satisfied for any batch of any size $B$.

Similarly, if we replace $\textbf{X}_{r..r+B}$ with any subset of it,i .e. $\Delta \textbf{X}_{r..r+B}$, i.e.:
\begin{align}\label{eq: iterative_terms_subset}
    \begin{split}
        \textbf{w}^{(t+1)''}& = ((1-\eta_t\lambda)\textbf{I} + \frac{\eta_t}{B}\sum_{i \in \mathcal{R}_{r..r+B}}a^{i, (t)}\textbf{x}_i\textbf{x}_i^T) \linearw^{(t)}\\
        & + \frac{\eta_t}{B} \sum_{i \in \mathcal{R}_{r..r+B}, y_i \in \Delta \textbf{y}_{r..r+B}} b^{i, (t)}y_i\textbf{x}_i\\
    \end{split}
\end{align}

which should be still converged.

Then we need to build connections between the provenance expression $\mathcal{W}^{(t)'}$ and $\textbf{w}^{(t+1)''}$.

First of all,  . So $\mathcal{W}^{(t)'}$ in Equation \eqref{eq: mini-SGD_prov} will be written as:

\begin{align}\label{eq: mini-SGD_prov_idempotence}
    \begin{split}
        \mathcal{W}^{(t+1)'} & \leftarrow ((1-\eta_t\lambda)(1_K*\textbf{I}) + \frac{\eta_t}{\mathcal{B}}\sum_{i=r}^{r+B-1}a^{i, (t)}(p_i*\textbf{x}_i\textbf{x}_i^T))\mathcal{W}^{(t)'}\\
        & + \frac{\eta_t}{\mathcal{B}} \sum_{i=r}^{r+B-1} ({p_i}*b^{i, (t)}y_i\textbf{x}_i)\\        
    \end{split}
\end{align}

The proof of the convergence of every tensor product in $\mathcal{W}^{(t+1)'}$ relies on the following theorems, i.e.:

\begin{theorem}
For any subset of provenance tokens, i.e. $\{p_{i_1}, p_{i_2}, \dots, p_{i_k}\} (\subset \{p_{1}, p_{2}, \dots, p_{n}\})$, the sum of the matrices associated with those tokens in $\mathcal{W}^{(t)}$ will be the same as the $\textbf{w}^{(t)'}$ where $\bigcup \Delta \textbf{X}_{r..r+B}$ is composed of $\textbf{x}_{i_1}, \textbf{x}_{i_2},\dots, \textbf{x}_{i_k}$.
\end{theorem}

By denoting the tensor product with provenance monomial $p_{i_1}p_{i_2}\dots p_{i_k}$ in $\mathcal{W}^{(t)}$ as $\mathcal{W}^{(t)}_{p_{i_1}p_{i_2}\dots p_{i_k}}$ and $\textbf{w}^{(t)'}$ where $\bigcup \Delta \textbf{X}_{r..r+B}$ is composed of $\textbf{x}_{i_1}, \textbf{x}_{i_2},\dots, \textbf{x}_{i_k}$ as $\textbf{w}^{(t)}_{\{\textbf{x}_{i_1}, \textbf{x}_{i_2},\dots, \textbf{x}_{i_k}\}}$, then the following equality hold:

\begin{align}
\begin{split}
    &\mathcal{W}^{(t)}_{p_{i_1}p_{i_2}\dots p_{i_k}} = \textbf{w}^{(t)}_{\textbf{x}_{i_1}, \textbf{x}_{i_2},\dots, \textbf{x}_{i_k}}\\
    & - \sum_{\{\textbf{x}_{j_1}, \textbf{x}_{j_2},\dots, \textbf{x}_{j_r}\} \subsetneq  \{\textbf{x}_{i_1}, \textbf{x}_{i_2},\dots, \textbf{x}_{i_k}\}}\textbf{w}^{(t)}_{\textbf{x}_{j_1}, \textbf{x}_{j_2},\dots, \textbf{x}_{j_r}}    
\end{split}
\end{align}

Since every $\textbf{w}^{(t)}_{\textbf{x}_{j_1}, \textbf{x}_{j_2},\dots, \textbf{x}_{j_r}}$ and $\textbf{w}^{(t)}_{\textbf{x}_{i_1}, \textbf{x}_{i_2},\dots, \textbf{x}_{i_k}}$ are converged, then $\mathcal{W}^{(t)}_{p_{i_1}p_{i_2}\dots p_{i_k}}$ should be converged.
}
\eat{In the beginning, $\mathcal{W}^{(0)}' = 1_K*\textbf{w}^{(0)}$, which satisfies the form above. Then suppose in the $t_{th}$ iteration, $\mathcal{W}^{(t)}' = \sum_{\substack{\{p_{i_1}, p_{i_2}, \dots, p_{i_r}\}\\ \in 2^{\{p_1, p_2, \dots, p_n\}}}} (p_{i_1}p_{i_2}\dots, p_{i_r})*\textbf{u}_{i_1i_2\dots i_r}^{(t)}$. By plugging it into Equation \eqref{eq: mini-SGD_prov_idem} and using the properties in Definition \ref{def: tensor_prod}, every tensor product in $\mathcal{W}^{(t + 1)}'$ should still be in the form of $(p_{i_1}p_{i_2}\dots p_{i_r})*\textbf{u}_{i_1i_2\dots i_r}^{(t + 1)}$
\eat{\begin{align}
    \begin{split}
        &\mathcal{W}^{(t+1)'} = ((1-\eta_t\lambda)(1_K*\textbf{I}) + \frac{\eta_t}{B}\sum_{i=r}^{r+B-1}a^{i, (t)}(p_i*\textbf{x}_i\textbf{x}_i^T))\mathcal{W}^{(t)'}\\
        & + \frac{\eta_t}{B} \sum_{i=r}^{r+B-1} ({p_i}*b^{i, (t)}y_i\textbf{x}_i)\\
        & = ((1-\eta_t\lambda)(1_K*\textbf{I}) + \frac{\eta_t}{B}\sum_{i=r}^{r+B-1}a^{i, (t)}(p_i*\textbf{x}_i\textbf{x}_i^T))\\
        &\sum_{\substack{\{p_{i_1}, p_{i_2}, \dots, p_{i_r}\}\\ \in 2^{\{p_1, p_2, \dots, p_n\}}}} (p_{i_1}p_{i_2}\dots, p_{i_r})*\textbf{u}_{i_1i_2\dots i_r}^{(t)}\\
        &+ \frac{\eta_t}{B} \sum_{i=r}^{r+B-1} ({p_i}*b^{i, (t)}y_i\textbf{x}_i)\\
        & = (1-\eta_t\lambda)\sum_{\substack{\{p_{i_1}, p_{i_2}, \dots, p_{i_r}\}\\ \in 2^{\{p_1, p_2, \dots, p_n\}}}} (p_{i_1}p_{i_2}\dots, p_{i_r})*\textbf{u}_{i_1i_2\dots i_r}^{(t)}\\
        &+ \frac{\eta_t}{B}\sum_{i=r}^{r+B-1}a^{i, (t)}(p_i*\textbf{x}_i\textbf{x}_i^T)\sum_{\substack{\{p_{i_1}, p_{i_2}, \dots, p_{i_r}\}\\ \in 2^{\{p_1, p_2, \dots, p_n\}}}} (p_{i_1}p_{i_2}\dots, p_{i_r})*\textbf{u}_{i_1i_2\dots i_r}^{(t)}\\
        &\\
        &+ \frac{\eta_t}{B} \sum_{i=r}^{r+B-1} ({p_i}*b^{i, (t)}y_i\textbf{x}_i) = \sum_{\substack{\{p_{i_1}, p_{i_2}, \dots, p_{i_r}\}\\ \in 2^{\{p_1, p_2, \dots, p_n\}}}} (p_{i_1}p_{i_2}\dots, p_{i_r})*\textbf{u}_{i_1i_2\dots i_r}^{(t+1)}\\
    \end{split}
\end{align}}

Then according to the definition of convergence of tensor product (See Definition \ref{def: convergence_tensor_prod}), we need to prove that every $\textbf{u}^{(t)}_{i_1i_2\dots i_r}$ should be converged when $t \rightarrow \infty$, which can be still proven by induction.

In the simplest case where the tensor product has the form of $p_j*\textbf{u}_{j}^{(t)}$, $\textbf{u}_{j}^{(t)}$ can be derived by setting $p_j$ as $1_K$ while setting all the other tokens as $0_K$, which will be:
\begin{align}
\begin{split}
    \textbf{u}_j^{(t+1)} & = ((1-\eta_t\lambda)(1_K*\textbf{I}) + \frac{\eta_t}{B}\mathbb{I}_{r\leq j < r+B}a^{j, (t)}(p_j*\textbf{x}_j\textbf{x}_j^T))\textbf{u}_j^{(t)}\\
        & + \frac{\eta_t}{B} \sum_{i=r}^{r+B-1} ({p_i}*b^{i, (t)}y_i\textbf{x}_i)\\    
\end{split}
\end{align}
}
\qed

\begin{theorem}\label{theorem: aproximation_bound2}(This is Theorem \ref{theorem: aproximation_bound} in the paper)
$||E(\textbf{w}^{(t)} - \linearw^{(t)})||_2$ is bounded by $O((\Delta x)^2)$ where $\Delta x$ is an arbitrarily small value representing the length of the longest sub-interval used in piecewise linear interpolations. 
\end{theorem}

% \subsection{Proof of theorem \ref{theorem: aproximation_bound2}}\label{sec: aproximation_bound_proof}

% It would be proved inductively.

% First of all, when $t=0$, $\textbf{w}^{(0)}$ should be equal to $\textbf{w}^{(0)'}$ since both $\textbf{w}^{(t)}$ and $ \linearw^{(t)}$ should have the same initialization. 

\begin{proof}

By subtracting Equation \eqref{eq: mini_sgd_logistic_regression2} by Equation \eqref{eq: update_rule_w_l} and taking the matrix norm, we get:

\begin{align*}
    &||\E(\textbf{w}^{(t+1)} - \linearw^{(t+1)})||_2 \\
    & = \E(\|\textbf{w}^{(t)} -\eta \triangledown h^{(t)}(\textbf{w}^{(t)}) - (\linearw^{(t)} -\eta \triangledown T^{(t)}(\linearw^{(t)}))\|_2) \\
    & = \E(\|\textbf{w}^{(t)} - \linearw^{(t)} - \eta[\triangledown T^{(t)}(\textbf{w}^{(t)}) - \triangledown T^{(t)}(\linearw^{(t)})] - \eta [\triangledown h^{(t)}(\textbf{w}^{(t)}) - \triangledown T^{(t)}(\textbf{w}^{(t)})]\|_2) \\
    & = \E(\|[\textbf{I} - (- \frac{\eta}{B}\sum_{i \in \mathscr{B}^{(t)}} a^{i, (t)}\textbf{x}_i\textbf{x}_i^T + \eta\lambda\textbf{I})](\textbf{w}^{(t)} - \linearw^{(t)}) - \eta [\triangledown h^{(t)}(\textbf{w}^{(t)}) - \triangledown T^{(t)}(\textbf{w}^{(t)})]\|_2)
\end{align*}

Then by using triangle inequality and the bound from Equation \eqref{eq: bound_t_w}, the formula above is further bounded as:
\begin{align*}
    & \leq E(\|[\textbf{I} - (- \frac{\eta}{B}\sum_{i \in \mathscr{B}^{(t)}} a^{i, (t)}\textbf{x}_i\textbf{x}_i^T + \eta\lambda\textbf{I})]\|\|(\textbf{w}^{(t)} - \linearw^{(t)})\|_2 + \eta \| [\triangledown h^{(t)}(\textbf{w}^{(t)}) - \triangledown T^{(t)}(\textbf{w}^{(t)})]\|_2)\\
    & \leq (1-\eta \lambda)\|(\textbf{w}^{(t)} - \linearw^{(t)})\|_2 + \eta \| [\triangledown h^{(t)}(\textbf{w}^{(t)}) - \triangledown T^{(t)}(\textbf{w}^{(t)})]\|_2
\end{align*}

Then by using the result from Equation \eqref{eq: approx_bound_1}, the formula above is rewritten as:
\begin{align*}
    = (1-\eta \lambda)\|(\textbf{w}^{(t)} - \linearw^{(t)})\|_2 + \eta O((\Delta x)^2)
\end{align*}

Then by applying the formula above recursively, we have:
\begin{align*}
    & = (1-\eta \lambda)^t\|(\textbf{w}^{(0)} - \linearw^{(0)})\|_2 + \frac{1-(1-\eta\lambda)^t}{\eta\lambda}\eta O((\Delta x)^2)
\end{align*}

Since $\textbf{w}^{(0)} = \linearw^{(0)}$ and $\eta \leq \frac{1}{L}$, then the formula above is bounded as:
\begin{align*}
    & \leq \frac{1}{\lambda}O((\Delta x)^2) = O((\Delta x)^2)
\end{align*}

\end{proof}

% \begin{align}\label{eq: w_gap}
%     \begin{split}
%         &\textbf{w}^{(t+1)} -  \linearw^{(t+1)}\\
%         & = (1-\eta_t\lambda )(\textbf{w}^{(t)} -  \linearw^{(t)})\\
%         & + \frac{\eta_t}{B}\sum_{i=r}^{r+B-1}y_i\textbf{x}_i( f(y_i\textbf{w}^{(t)T}\textbf{x}_i)-(a^{i,(t)}(y_i\textbf{w}^{(t)T'}\textbf{x}_i)) + b^{i,(t)})\\
%         & =((1-\eta_t\lambda)\textbf{I} + \frac{\eta_t}{B}\sum_{i=r}^{r+B-1}a^{i,(t)}\textbf{x}_i\textbf{x}_i^T )(\textbf{w}^{(t)} -  \linearw^{(t)})\\
%         &+ \frac{\eta_t}{B}\sum_{i=r}^{r+B-1}y_i\textbf{x}_i(f(y_i\textbf{w}^{(t)T}\textbf{x}_i) - s(y_i\textbf{w}^{(t)T}\textbf{x}_i))
%     \end{split}
% \end{align}

\eat{According to Section \ref{sec: theorem_convergency_proof}, $||E((1-\eta_t\lambda)\textbf{I} + \frac{\eta_t}{B}\sum_{i=r}^{r+B-1}a^{i,(t)}\textbf{x}_i\textbf{x}_i^T)||_2 < 1$. By applying expectation over the both side of the Equation \eqref{eq: w_gap} and apply Equation \eqref{eq: w_gap} recursively, the result will be:

\begin{align}\label{eq: w_gap}
    \begin{split}
        &E(\textbf{w}^{(t+1)} -  \linearw^{(t+1)})\\
        & =(E((1-\eta_t\lambda)\textbf{I} + \frac{\eta_t}{B}\sum_{i=r}^{r+B-1}a^{i,(t)}\textbf{x}_i\textbf{x}_i^T))(\textbf{w}^{(t)} -  \linearw^{(t)})\\
        &+ E(\frac{\eta_t}{B}\sum_{i=r}^{r+B-1}y_i\textbf{x}_i(f(y_i\textbf{w}^{(t)T}\textbf{x}_i) - s(y_i\textbf{w}^{(t)T}\textbf{x}_i)))\\
        & = \Pi_{j=1}^t(E((1-\eta_j\lambda)\textbf{I} + \frac{\eta_t}{B}\sum_{i=r}^{r+B-1}a^{i,(j)}\textbf{x}_i\textbf{x}_i^T))(\textbf{w}^{(0)}-\textbf{w}^{(0)'})\\
        & + (\sum_{r=1}^{t-1}\Pi_{j=r}^{t-1}E((1-\eta_j\lambda)\textbf{I} + \frac{\eta_t}{B}\sum_{i=r}^{r+B-1}a^{i,(j)}\textbf{x}_i\textbf{x}_i^T))\\
        &E(\frac{\eta_t}{B}\sum_{i=r}^{r+B-1}y_i\textbf{x}_i(f(y_i\textbf{w}^{(t)T}\textbf{x}_i) - s(y_i\textbf{w}^{(t)T}\textbf{x}_i)))\\
    \end{split}
\end{align}

Since $(\textbf{w}^{(0)}-\textbf{w}^{(0)'}) = 0$, $||E((1-\eta_j\lambda)\textbf{I} + \frac{\eta_t}{B}\sum_{i=r}^{r+B-1}a^{i,(j)}\textbf{x}_i\textbf{x}_i^T))||_2 < 1$ and $||E(\frac{\eta_t}{B}\sum_{i=r}^{r+B-1}y_i\textbf{x}_i(f(y_i\textbf{w}^{(t)T}\textbf{x}_i) - s(y_i\textbf{w}^{(t)T}\textbf{x}_i)))||_2 \leq E(\frac{\eta_t}{B}\sum_{i=r}^{r+B-1}||y_i\textbf{x}_i(f(y_i\textbf{w}^{(t)T}\textbf{x}_i) - s(y_i\textbf{w}^{(t)T}\textbf{x}_i)))||_2 \leq O(\Delta x)$, then:
\begin{equation}
||E(\textbf{w}^{(t+1)} -  \linearw^{(t+1)})||_2 \leq O(\Delta x)   
\end{equation}}

% \subsection{Proof to theorem \ref{theorem: aproximation_bound_change2}}\label{sec: aproximation_bound_change_proof}

According to Assumption \ref{assp: bounded_grad} and Equation \eqref{eq: approx_bound_1} and by using the triangle inequality and the Theorem above, we have:
\begin{align}\label{eq: delta_t_bound}
    \begin{split}
        & \|\nabla T^{(t)}_i(\linearw^{(t)})\| = \|\nabla T^{(t)}_i(\linearw^{(t)}) - \nabla T^{(t)}_i(\w^{(t)}) + \nabla T^{(t)}_i(\w^{(t)}) - \nabla h^{(t)}_i(\w^{(t)}) + \nabla h^{(t)}_i(\w^{(t)})\|\\
    & \leq \|\nabla T^{(t)}_i(\linearw^{(t)}) - \nabla T^{(t)}_i(\w^{(t)})\| + \|\nabla T^{(t)}_i(\w^{(t)}) - \nabla h^{(t)}_i(\w^{(t)})\| + \|\nabla h^{(t)}_i(\w^{(t)})\| \\
    & \leq \|-a^{i, (t)}\textbf{x}_i\textbf{x}_i^T + \lambda\textbf{I}\| O((\Delta x)^2) + O((\Delta x)^2) + c_1 :=c_2
    \end{split}    
\end{align}

\begin{theorem}\label{theorem: leave_one_out_bound}
$||\E(\logistlinearincrew^{(t)}- \w^{(t)})||_2$ is bounded by $O(\frac{\Delta n}{n})$.
\end{theorem}

\begin{proof}
By using the definition of $\logistlinearincrew^{(t)}$ and $\w^{(t)}$, i.e. Equation \eqref{eq: mini-SGD_updated_model_parameters_expected2} and Equation \eqref{eq: mini_sgd_logistic_regression2}, we have:
\begin{align*}
    &\E(\|\logistlinearincrew^{(t+1)} - \w^{(t+1)}\|)\\
    & = \E(\|\logistlinearincrew^{(t)} -\eta(\lambda \logistlinearincrew^{(t)} + \frac{1}{B}\sum_{i \in \mathscr{B}^{(t)}}y_i\textbf{x}_if(y_i\logistlinearincrew^{(t)T}\textbf{x}_i) )  - [\w^{(t)} - \eta \lambda \w^{(t)}  - \eta \frac{1}{\increB^{(t)}}\sum_{\substack{i \in \mathscr{B}^{(t)}, i \not\in \mathcal{R}}}y_i\textbf{x}_if(y_i\textbf{w}^{(t)T}\textbf{x}_i)]\|)\\
    & = \E(\|(1-\eta \lambda)(\logistlinearincrew^{(t)} - \w^{(t)}) + \frac{\eta}{B}\sum_{i \in \mathscr{B}^{(t)}} y_i\textbf{x}_i[f(y_i\logistlinearincrew^{(t)T}\textbf{x}_i) - f(y_i\w^{(t)T}\textbf{x}_i)] + \frac{\eta}{B}\sum_{i \in \mathscr{B}^{(t)}}y_i\textbf{x}_if(y_i\w^{(t)T}\textbf{x}_i)\\
    & - \frac{\eta}{\increB^{(t)}}\sum_{\substack{i \in \mathscr{B}^{(t)}, i \not\in \mathcal{R}}}y_i\textbf{x}_if(y_i\textbf{w}^{(t)T}\textbf{x}_i) \|)
\end{align*}

Then by using the cauchy mean value theorem over $f(y_i\logistlinearincrew^{(t)T}\textbf{x}_i) - f(y_i\w^{(t)T}\textbf{x}_i)$, the formula above is bounded as:
\begin{align*}
    & \leq \E(\|[(1-\eta \lambda)\textbf{I} + \frac{\eta}{B}\sum_{i \in \mathscr{B}^{(t)}} \textbf{x}_i \textbf{x}_i^T f'(p)] (\logistlinearincrew^{(t)}-\w^{(t)})\|) \\
    & + \E(\|\frac{\eta}{B}\sum_{i \in \mathscr{B}^{(t)}}y_i\textbf{x}_if(y_i\w^{(t)T}\textbf{x}_i) - \frac{\eta}{\increB^{(t)}}\sum_{\substack{i \in \mathscr{B}^{(t)}, i \not\in \mathcal{R}}}y_i\textbf{x}_if(y_i\textbf{w}^{(t)T}\textbf{x}_i) \|) \\
    & =\E(\|[(1-\eta \lambda)\textbf{I} + \frac{\eta}{B}\sum_{i \in \mathscr{B}^{(t)}} \textbf{x}_i \textbf{x}_i^T f'(p)] (\logistlinearincrew^{(t)}-\w^{(t)})\|) + \E(\|\frac{\eta}{B}\sum_{i \in \mathscr{B}^{(t)}}\triangledown h_i(\w^{(t)}) - \frac{\eta}{\increB^{(t)}}\sum_{\substack{i \in \mathscr{B}^{(t)},i \not\in \mathcal{R}}}\triangledown h_i(\w^{(t)})\|)
\end{align*}

By rewriting the formula above and using the upper bound on $\|\triangledown h_i(\w^{(t)})\|$, we get:
\begin{align*}
    & = \E(\|[(1-\eta \lambda)\textbf{I} + \frac{\eta}{B}\sum_{i \in \mathscr{B}^{(t)}} \textbf{x}_i \textbf{x}_i^T f'(p)] (\logistlinearincrew^{(t)}-\w^{(t)})\|) + \E(\|\frac{\eta}{B}\sum_{\substack{i \in \mathscr{B}^{(t)},\\ i \in \mathcal{R}}}\triangledown h_i(\w^{(t)}) + (\frac{\eta}{B} - \frac{\eta}{\increB^{(t)}})\sum_{\substack{i \in \mathscr{B}^{(t)},\\ i \not\in \mathcal{R}}}\triangledown h_i(\w^{(t)})\|)\\
    & \leq \E(\|[(1-\eta \lambda)\textbf{I} + \frac{\eta}{B}\sum_{i \in \mathscr{B}^{(t)}} \textbf{x}_i \textbf{x}_i^T f'(p)] (\logistlinearincrew^{(t)}-\w^{(t)})\|) + \E(\frac{\eta}{B}\sum_{\substack{i \in \mathscr{B}^{(t)},\\ i \in \mathcal{R}}}c_1 + (\frac{\eta}{B} - \frac{\eta}{\increB^{(t)}})\sum_{\substack{i \in \mathscr{B}^{(t)}, i \not\in \mathcal{R}}}c_1)\\
    & = \E(\|[(1-\eta \lambda)\textbf{I} + \frac{\eta}{B}\sum_{i \in \mathscr{B}^{(t)}} \textbf{x}_i \textbf{x}_i^T f'(p)] (\logistlinearincrew^{(t)}-\w^{(t)})\|) + \E(\frac{2\eta\Delta B^{(t)}}{B}c_1)
\end{align*}

Then by using the result from Equation \eqref{eq: h_hessian_bound} and Equation \eqref{eq: removed_num_exp}, the formula above is bounded as:
\begin{align*}
    & \leq (1-\eta \lambda)\|\logistlinearincrew^{(t)}-\w^{(t)}\| + 2\eta c_1\frac{\Delta n}{n}
\end{align*}

By applying the formula recursively, we get:
\begin{align*}
    \leq 2 \frac{1}{c_1\lambda}\frac{\Delta n}{n} = O(\frac{\Delta n}{n})
\end{align*}
\end{proof}

\begin{theorem}\label{theorem: aproximation_bound_change2} (It is Theorem \ref{theorem: aproximation_bound_change} in the paper)
$||E(\linearincrew^{(t)}- \logistlinearincrew^{(t)})||_2$ is bounded by $O(\frac{\Delta n}{n}\Delta x) + O((\frac{\Delta n}{n})^2) + O((\Delta x)^2)$, where $\Delta n$ is the number of the removed samples and $\Delta x$ is defined in Theorem \ref{theorem: aproximation_bound2}
\end{theorem}

\begin{proof}

By using the definition of $\linearincrew$ and $\logistlinearincrew$ and subtracting the former one from the latter one, we have:
\begin{align*}
    &\E(\|\linearincrew^{(t+1)} - \logistlinearincrew^{(t+1)}\|_2)\\
    & = \E(\|\linearincrew^{(t)} - \lambda \eta \linearincrew^{(t)} - \eta \triangledown R^{(t)}(\linearincrew^{(t)})  - (\logistlinearincrew^{(t)} - \lambda \eta \logistlinearincrew^{(t)} - \eta \triangledown g^{(t)}(\logistlinearincrew^{(t)}))\|_2)\\
    & = \E(\|\linearincrew^{(t)} - \logistlinearincrew^{(t)} - \eta (\triangledown R^{(t)}(\linearincrew^{(t)}) - \triangledown R^{(t)}(\logistlinearincrew^{(t)})) - \lambda \eta (\linearincrew^{(t)} - \logistlinearincrew^{(t)}) - \eta (\triangledown R^{(t)}(\logistlinearincrew^{(t)}) - \triangledown g^{(t)}(\logistlinearincrew^{(t)}))\|)\\
    & \leq \E(\|[\textbf{I} - \eta \frac{1}{\increB^{(t)}}\sum_{\substack{ i \in \mathscr{B}^{(t)}, i \not\in \mathcal{R}}} (-a^{i,(t)}\textbf{x}_i\textbf{x}_i^T + \lambda \textbf{I})](\linearincrew^{(t)} - \logistlinearincrew^{(t)})\| + \eta \|\triangledown R^{(t)}(\logistlinearincrew^{(t)}) - \triangledown g^{(t)}(\logistlinearincrew^{(t)})\|)
\end{align*}

Then by using the result from Equation \eqref{eq: bound_t_w}, the formula above is bounded as:
\begin{align}\label{eq: bound_deriv_1}
\begin{split}
    & \leq (1-\eta\lambda)\|\linearincrew^{(t)} - \logistlinearincrew^{(t)}\| + \eta \E(\|\triangledown R^{(t)}(\logistlinearincrew^{(t)}) - \triangledown g^{(t)}(\logistlinearincrew^{(t)})\|)\\
    & = (1-\eta\lambda)\|\linearincrew^{(t)} - \logistlinearincrew^{(t)}\| + \eta \E(\|\frac{1}{\increB^{(t)}}\sum_{\substack{ i \in \mathscr{B}^{(t)} \\ i \not\in \mathcal{R}}}y_i\textbf{x}_i[s(y_i\logistlinearincrew^{(t)T}\textbf{x}_i) - f(y_i\logistlinearincrew^{(t)T}\textbf{x}_i)]\|)
\end{split}
\end{align}

in which we bound $y_i\textbf{x}_is(y_i\logistlinearincrew^{(t)T}\textbf{x}_i) - y_i\textbf{x}_if(y_i\logistlinearincrew^{(t)T}\textbf{x}_i)$ as below:
\begin{align*}
    &\|y_i\textbf{x}_is(y_i\logistlinearincrew^{(t)T}\textbf{x}_i) - y_i\textbf{x}_if(y_i\logistlinearincrew^{(t)T}\textbf{x}_i)\|\\
    & =\|y_i\textbf{x}_is(y_i\logistlinearincrew^{(t)T}\textbf{x}_i) - y_i\textbf{x}_is(y_i\w^{(t)T}\textbf{x}_i) + y_i\textbf{x}_is(y_i\w^{(t)T}\textbf{x}_i) - y_i\textbf{x}_if(y_i\w^{(t)T}\textbf{x}_i) + y_i\textbf{x}_if(y_i\w^{(t)T}\textbf{x}_i) - y_i\textbf{x}_if(y_i\logistlinearincrew^{(t)T}\textbf{x}_i)\| \\
    & \leq \|y_i\textbf{x}_is(y_i\logistlinearincrew^{(t)T}\textbf{x}_i) - y_i\textbf{x}_is(y_i\w^{(t)T}\textbf{x}_i) + y_i\textbf{x}_if(y_i\w^{(t)T}\textbf{x}_i) - y_i\textbf{x}_if(y_i\logistlinearincrew^{(t)T}\textbf{x}_i)\| + \|y_i\textbf{x}_is(y_i\w^{(t)T}\textbf{x}_i) - y_i\textbf{x}_if(y_i\w^{(t)T}\textbf{x}_i)\| \\
    & = \|a^{i, (t)}\textbf{x}_i\textbf{x}_i^T(\logistlinearincrew^{(t)} - \w^{(t)}) + y_i\textbf{x}_if(y_i\w^{(t)T}\textbf{x}_i) - y_i\textbf{x}_if(y_i\logistlinearincrew^{(t)T}\textbf{x}_i)\| + O((\Delta x)^2)
\end{align*}

The last step uses the result from Equation \eqref{eq: approx_bound_1}. Then by using the Cauchy mean value theorem on $f(y_i\w^{(t)T}\textbf{x}_i) - f(y_i\logistlinearincrew^{(t)T}\textbf{x}_i)$, we know that:
\begin{align*}
    & = \|a^{i, (t)}\textbf{x}_i\textbf{x}_i^T(\logistlinearincrew^{(t)} - \w^{(t)}) + y_i\textbf{x}_i [\int_0^1 f'(y_i\w^{(t)T}\textbf{x}_i + x(y_i\logistlinearincrew^{(t)T}\textbf{x}_i - y_i\w^{(t)T}\textbf{x}_i))dx ](y_i\w^{(t)T}\textbf{x}_i - y_i\logistlinearincrew^{(t)T}\textbf{x}_i)\| + O((\Delta x)^2) \\
    & \leq \|[a^{i, (t)}-\int_0^1 f'(y_i\w^{(t)T}\textbf{x}_i + x(y_i\logistlinearincrew^{(t)T}\textbf{x}_i - y_i\w^{(t)T}\textbf{x}_i))dx]\textbf{x}_i\textbf{x}_i^T\|\|(\logistlinearincrew^{(t)} - \w^{(t)})\| + O((\Delta x)^2)
\end{align*}

Then by adding and subtracting  $f'(y_i\textbf{w}^{(t)T}\textbf{x}_i)$ in the first term and using the fact from Equation \eqref{eq: approx_bound_2} and Assumption \ref{assp: continuous}, the formula above is bounded as:
\begin{align*}
    & = \|[a^{i, (t)}-f'(y_i\textbf{w}^{(t)T}\textbf{x}_i) + f'(y_i\textbf{w}^{(t)T}\textbf{x}_i)-\int_0^1 f'(y_i\w^{(t)T}\textbf{x}_i + x(y_i\logistlinearincrew^{(t)T}\textbf{x}_i - y_i\w^{(t)T}\textbf{x}_i))dx]\textbf{x}_i\textbf{x}_i^T\|\|(\logistlinearincrew^{(t)} - \w^{(t)})\| + O((\Delta x)^2)\\
    & \leq [\|(a^{i, (t)}-f'(y_i\textbf{w}^{(t)T}\textbf{x}_i))\textbf{x}_i\textbf{x}_i^T\| + \|\textbf{x}_i\textbf{x}_i^T\int_0^1[f'(y_i\textbf{w}^{(t)T}\textbf{x}_i) - f'(y_i\w^{(t)T}\textbf{x}_i + x(y_i\logistlinearincrew^{(t)T}\textbf{x}_i - y_i\w^{(t)T}\textbf{x}_i))]dx\|]\|(\logistlinearincrew^{(t)T} - \w^{(t)T})\|\\
    & + O((\Delta x)^2)\\
    & \leq O((\Delta x))\|(\logistlinearincrew^{(t)} - \w^{(t)})\| + \|\textbf{x}_i\textbf{x}_i^T \int_0^1 c_2x(y_i\logistlinearincrew^{(t)T}\textbf{x}_i - y_i\w^{(t)T}\textbf{x}_i)dx \|\|(\logistlinearincrew^{(t)} - \w^{(t)})\| + O((\Delta x)^2)\\
    & \leq O((\Delta x))\|(\logistlinearincrew^{(t)} - \w^{(t)})\| + \frac{c_2}{2}\|\textbf{x}_i\textbf{x}_i^T y_i\textbf{x}_i^T\|\|(\logistlinearincrew^{(t)} - \w^{(t)})\|^2 + O((\Delta x)^2)
\end{align*}

Then by using the result from Theorem \ref{theorem: leave_one_out_bound}, the formula above is bounded as:
\begin{align*}
    \leq O((\Delta x))O(\frac{\Delta n}{n}) + \frac{c_2}{2}\|\textbf{x}_i\textbf{x}_i^T y_i\textbf{x}_i^T\|(O(\frac{\Delta n}{n}))^2 + O((\Delta x)^2) = O(\frac{\Delta n}{n}\Delta x) + O((\frac{\Delta n}{n})^2) + O((\Delta x)^2)
\end{align*}

which is then plugged into Equation \eqref{eq: bound_deriv_1}, we have:
\begin{align*}
& \E(\|\linearincrew^{(t+1)} - \logistlinearincrew^{(t+1)}\|_2)\\
& \leq (1-\eta\lambda)\|\linearincrew^{(t)} - \logistlinearincrew^{(t)}\| + \eta[O(\frac{\Delta n}{n}\Delta x) + O((\frac{\Delta n}{n})^2) + O((\Delta x)^2)]
\end{align*}

which is then used recursively. Then we have:
\begin{align*}
    & \leq \frac{1}{\lambda} [O(\frac{\Delta n}{n}\Delta x) + O((\frac{\Delta n}{n})^2) + O((\Delta x)^2)] = O(\frac{\Delta n}{n}\Delta x) + O((\frac{\Delta n}{n})^2) + O((\Delta x)^2)
\end{align*}

\end{proof}

\begin{theorem}\textbf{Approximation ratio}\label{theorem: aproximation_bound_svd2} (It is Theorem \ref{theorem: aproximation_bound_svd} in the paper)
Under the convergence conditions for $\textbf{w}^{(t)}$, $||\textbf{w}^{(t)}||$ should be bounded by some constant $C$. Suppose $\frac{||\textbf{U}^{(t)}_{1..r}\textbf{S}^{(t)}_{1..r}\textbf{V}^{T,(t)}_{1..r}||_2}{||\textbf{U}^{(t)}\textbf{S}^{(t)}\textbf{V}^{T,(t)}||_2} \geq 1-\epsilon$ where $\epsilon$ is a small value, then the change of model parameters caused by the approximation will be bounded by $O(\epsilon)$.
\end{theorem}
\begin{proof}
The approximate update rule for linear regression by using SVD after removing subsets of training samples is:
\begin{align}\label{eq: gbm_linear_regression_incremental_updates_svd2_0}
    \begin{split}
        &\increw^{(t+1)'} \leftarrow [(1-\eta_t\lambda)\textbf{I} - \frac{2\eta_t}{\increB^{(t)}}(\textbf{U}^{(t)}_{1..r}\textbf{S}^{(t)}_{1..r}\textbf{V}^{T,(t)}_{1..r}\\
        &\hspace{-2mm}-\Delta \textbf{X}^T_{\subids}\Delta \textbf{X}_{\subids})]\increw^{(t)'}+\frac{2\eta_t}{\increB^{(t)}}(\sum_{\substack{ i \in \mathscr{B}^{(t)}}}\textbf{x}_i y_i - \sum_{\substack{ i \in \mathscr{B}^{(t)}, i \in \mathcal{R}}}\eat{\cdot} \textbf{x}_i y_i)
    \end{split}
\end{align}

By comparing Equation \eqref{eq: gbm_linear_regression_incremental_updates2} against this approximate update rule, the only difference is $\textbf{U}^{(t)}_{1..r}\textbf{S}^{(t)}_{1..r}\textbf{V}^{T,(t)}_{1..r}$ in Equation \eqref{eq: gbm_linear_regression_incremental_updates_svd2_0} and $\sum_{\substack{ i \in \mathscr{B}^{(t)}}} \textbf{x}_i\textbf{x}_i^T$ in Equation \eqref{eq: gbm_linear_regression_incremental_updates2}. Then according to the condition in this theorem, $||\sum_{\substack{ i \in \mathscr{B}^{(t)}}} \textbf{x}_i\textbf{x}_i^T - \textbf{U}^{(t)}_{1..r}\textbf{S}^{(t)}_{1..r}\textbf{V}^{T,(t)}_{1..r}||_2 = ||\textbf{U}^{(t)}\textbf{S}^{(t)}\textbf{V}^{T,(t)} - \textbf{U}^{(t)}_{1..r}\textbf{S}^{(t)}_{1..r}\textbf{V}^{T,(t)}_{1..r}||_2 = O(\epsilon)$. So by subtracting Equation \eqref{eq: gbm_linear_regression_incremental_updates2} by Equation \eqref{eq: gbm_linear_regression_incremental_updates_svd2_0}, the result becomes:
\begin{align}\label{eq: gbm_linear_regression_incremental_updates_svd2_1}
    \begin{split}
        &||\increw^{(t+1)'}-\increw^{(t+1)}||_2 \leftarrow ||[(1-\eta_t\lambda)\textbf{I} - \frac{2\eta_t}{\increB^{(t)}}(\sum_{\substack{ i \in \mathscr{B}^{(t)}}} \textbf{x}_i\textbf{x}_i^T -\Delta \textbf{X}^T_{\subids}\Delta \textbf{X}_{\subids})](\increw^{(t)'} - \increw^{(t)})\\
        & + (\frac{2\eta_t}{\increB^{(t)}}(\textbf{U}^{(t)}_{1..r}\textbf{S}^{(t)}_{1..r}\textbf{V}^{T,(t)}_{1..r} - \sum_{\substack{ i \in \mathscr{B}^{(t)}}} \textbf{x}_i\textbf{x}_i^T)) \increw^{(t)}||_2\\
        & \leq ||[(1-\eta_t\lambda)\textbf{I} - \frac{2\eta_t}{\increB^{(t)}}(\sum_{\substack{ i \in \mathscr{B}^{(t)}}} \textbf{x}_i\textbf{x}_i^T -\Delta \textbf{X}^T_{\subids}\Delta \textbf{X}_{\subids})]||_2||(\increw^{(t)'} - \increw^{(t)})||_2 + O(\epsilon)
    \end{split}
\end{align}

Then we evaluate the expectation between both sides, which ends up with:
\begin{align}\label{eq: gbm_linear_regression_incremental_updates_svd2_2}
    \begin{split}
        &\E(||\increw^{(t+1)'}-\increw^{(t+1)}||_2) \\
        & = ||[(1-\eta_t\lambda)\textbf{I} - \frac{2\eta_t}{n -\Delta n}(\textbf{X}^T\textbf{X} -\Delta \textbf{X}^T_{\subids}\Delta \textbf{X}_{\subids})]||_2||(\increw^{(t)'} - \increw^{(t)})||_2 + O(\epsilon)
    \end{split}
\end{align}

in which $\lambda \textbf{I} + \frac{2}{n -\Delta n}(\textbf{X}^T\textbf{X} -\Delta \textbf{X}^T_{\subids}\Delta \textbf{X}_{\subids})$ equals to the hessian matrix for the object function over the remaining samples after removing subsets of samples, which should be still $\lambda-$strong convex and $L-$lipschitz continuous. It indicates that:
\begin{equation}
    1-\eta_t\lambda \geq ||(1-\lambda\eta_t) \textbf{I} - \frac{2\eta_t}{n -\Delta n}(\textbf{X}^T\textbf{X} -\Delta \textbf{X}^T_{\subids}\Delta \textbf{X}_{\subids})||_2 \geq 1-\eta_tL
\end{equation}

Since according to Lemma \ref{lemma: convergence_conditions2}, $\eta_tL < 1$, then:
\begin{equation}
    1 > 1-\eta_t\lambda \geq ||(1-\lambda\eta_t) \textbf{I} - \frac{2\eta_t}{n -\Delta n}(\textbf{X}^T\textbf{X} -\Delta \textbf{X}^T_{\subids}\Delta \textbf{X}_{\subids})||_2 \geq 1-\eta_tL > 0
\end{equation}

So we can compute Equation \eqref{eq: gbm_linear_regression_incremental_updates_svd2_2} recursively and thus the following inequality holds:
\begin{equation}
    E(||\increw^{(t+1)'}-\increw^{(t+1)}||_2) \leq O(\epsilon)
\end{equation}

\end{proof}

\begin{theorem}(\textbf{Approximation ratio})\label{theorem: aproximation_bound_eigen2} (It is Theorem \ref{theorem: aproximation_bound_eigen} in the paper)
The approximation of \proopt\ over the model parameters is bounded by $O(||\Delta \textbf{X}^T \Delta \textbf{X}||)$
\end{theorem}
\begin{proof}

The update rule through gradient descent for linear regression model is as below:
\begin{align}\label{eq: linear_regression_gradient_descent}
\begin{split}
        \increw^{(t+1)}& \leftarrow ((1-\eta_t\lambda)\textbf{I} - \frac{2\eta_t}{n-\Delta n}(\textbf{X}^T\textbf{X}-\Delta\textbf{X}^T\Delta \textbf{X}))\increw^{(t)}\\
        &+\frac{2\eta_t}{n - \Delta n}(\textbf{X}^T\textbf{Y} - \Delta\textbf{X}^T\Delta\textbf{Y})
    \end{split}
\end{align}

while the approximated update rule through the approximations by incremental updates over eigenvalue is:
\begin{align}\label{eq: incremental_updates}
\begin{split}
        \increw^{(t+1)'}& \leftarrow ((1-\eta_t\lambda)\textbf{I} - \frac{2\eta_t}{n-\Delta n}\textbf{Q}^{-1}diag[c_1',c_2',\dots,c_m']\textbf{Q})\increw^{(t)'}\\
        &+\frac{2\eta_t}{n - \Delta n}(\textbf{X}^T\textbf{Y} - \Delta\textbf{X}^T\Delta\textbf{Y})
    \end{split}
\end{align}

According to \cite{ning2010incremental}, the difference between $\textbf{Q}^{-1}diag[c_1',c_2',\dots,c_m']\textbf{Q}$ and $\textbf{X}^T\textbf{X}-\Delta\textbf{X}^T\Delta \textbf{X}$ is bounded by $O(\Delta\textbf{X}^T\Delta \textbf{X})$. So by subtracting Equation \eqref{eq: linear_regression_gradient_descent} by Equation \eqref{eq: incremental_updates}, the results become:
\begin{align}\label{eq: gap_incremental_updates}
    \begin{split}
        &||\increw^{(t+1)'}-\increw^{(t+1)}||_2 \leftarrow ||[(1-\eta_t\lambda)\textbf{I} - \frac{2\eta_t}{\increB^{(t)}}(\textbf{X}^T\textbf{X} -\Delta \textbf{X}^T_{\subids}\Delta \textbf{X}_{\subids})](\increw^{(t)'} - \increw^{(t)})\\
        & + (\frac{2\eta_t}{\increB^{(t)}}(\textbf{Q}^{-1}diag[c_1',c_2',\dots,c_m']\textbf{Q} - \Delta \textbf{X}^T_{\subids}\Delta \textbf{X}_{\subids})) \increw^{(t)}||_2\\
        & \leq ||[(1-\eta_t\lambda)\textbf{I} - \frac{2\eta_t}{\increB^{(t)}}(\sum_{\substack{ i \in \mathscr{B}^{(t)}}} \textbf{x}_i\textbf{x}_i^T -\Delta \textbf{X}^T_{\subids}\Delta \textbf{X}_{\subids})]||_2||(\increw^{(t)'} - \increw^{(t)})||_2 + O(\Delta\textbf{X}^T\Delta \textbf{X})
    \end{split}
\end{align}

Through the similar analysis to Theorem \ref{theorem: aproximation_bound_svd2}, the above formula is computed recursively, which ends up with:
\begin{equation}
    ||\increw^{(t+1)'}-\increw^{(t+1)}||_2 \leq O(\Delta\textbf{X}^T\Delta \textbf{X})
\end{equation}
\end{proof}

\begin{theorem} \textbf{(Approximation ratio)}\label{theorem: aproximation_bound_svd_logistic2} (It is Theorem \ref{theorem: aproximation_bound_svd_logistic} in the paper)
Similar to Theorem \ref{theorem: aproximation_bound_svd2}, the deviation caused by the SVD approximation will be bounded by $O(\epsilon)$, given \eat{the upper bound of $||\textbf{w}^{(t)}||$ and }the ratio  $\frac{||\textbf{P}^{(t)}_{1..r}\textbf{V}^{T,(t)}_{1..r}||_2}{||\textbf{P}^{(t)}\textbf{V}^{T,(t)}||_2} \geq 1-\epsilon$. So using Theorem \ref{theorem: aproximation_bound_change2}, $||E(\linearincrew^{(t)}- \logistlinearincrew^{(t)})||_2$ is bounded by $O(\frac{\Delta n}{n}\Delta x) + O((\frac{\Delta n}{n})^2) + O((\Delta x)^2) + O(\epsilon)$.
\end{theorem}
\begin{proof}
Let's assume that the incremental updated model parameter without SVD approximation is ${\linearincrew^{(t)}}_0$. According to the results in Theorem \ref{theorem: aproximation_bound_change2}, the $||{\linearincrew^{(t)}}_0 -\logistlinearincrew^{(t)}|| \leq O(\frac{\Delta n}{n}\Delta x) + O((\frac{\Delta n}{n})^2) + O((\Delta x)^2)$. After the SVD approximation, similar analysis to Theorem \ref{theorem: aproximation_bound_svd2} can be done such that $||{\linearincrew^{(t)}}_0 - \linearincrew^{(t)}||_2 \leq O(\epsilon)$. So $||E(\linearincrew^{(t)}- \logistlinearincrew^{(t)})||_2 \leq O(\epsilon) + O(\frac{\Delta n}{n}\Delta x) + O((\frac{\Delta n}{n})^2) + O((\Delta x)^2)$
\end{proof}

\begin{theorem}\textbf{(Approximation ratio)}\label{theorem: aproximation_bound_eigen_logistic2} (It is Theorem \ref{theorem: aproximation_bound_eigen_logistic} in the paper)
Suppose that after iteration $t_s$ the gradient of the objective function is smaller than $\delta$, then the approximations of \proopt\ can lead to deviations of the model parameters bounded by $O((\tau - t_s)\delta) + O(||\Delta \textbf{X}^T\Delta \textbf{X}||)$. By combining the analysis in Theorem \ref{theorem: aproximation_bound_change2}, $||E(\linearincrew^{(t)}- \logistlinearincrew^{(t)})||_2$ is bounded by $O(\frac{\Delta n}{n}\Delta x) + O((\frac{\Delta n}{n})^2) + O((\Delta x)^2) + O((\tau - t_s)\delta) + O(||\Delta \textbf{X}^T\Delta \textbf{X}||)$
\end{theorem}
\begin{proof}
Let's assume that after $t_s^{th}$ iteration, the incremental updated model parameter without only linearization approximation is ${\linearincrew^{(t)}}_0$. Then $||\E({\linearincrew^{(t)}}_0- \logistlinearincrew^{(t)})||_2 \leq O(\frac{\Delta n}{n}\Delta x) + O((\frac{\Delta n}{n})^2) + O((\Delta x)^2)$ based on Theorem \ref{theorem: aproximation_bound_change2}. Also let's assume that after $t_s^{th}$ iteration, the incremental updated model parameter with both linearization approximation and SVD approximation is ${\linearincrew^{(t)}}_1$. Then:
\begin{align}
    \begin{split}
        &||{\linearincrew^{(t)}}_1 - {\linearincrew^{(t)}}_0||_2 = ||{\linearincrew^{(t-1)}}_1 - \triangledown R^{(t-1)}({\linearincrew^{(t-1)}}_1) - ({\linearincrew^{(t)}}_0-\triangledown R^{(t-1)}({\linearincrew^{(t-1)}}_0))||_2 \\
        & \leq ||{\linearincrew^{(t-1)}}_1 - {\linearincrew^{(t-1)}}_0||_2 + 2\delta \leq ||{\linearincrew^{(t_s)}}_1 - {\linearincrew^{(t_s)}}_0||_2 + 2(t-t_s-1)\delta = O((t-t_s)\delta)
    \end{split}
\end{align}
\end{proof}

Finally, by using Theorem \ref{theorem: aproximation_bound_eigen2}, $||\linearincrew^{(t)} - {\linearincrew^{(t)}}_1||_2 \leq O(||\Delta \textbf{X}^T\Delta \textbf{X}||)$. By combining those results together, we have:
\begin{align}
    \begin{split}
        ||E(\linearincrew^{(t)}- \logistlinearincrew^{(t)})||_2 \leq O(\frac{\Delta n}{n}\Delta x) + O((\frac{\Delta n}{n})^2) + O((\Delta x)^2) + O((\tau - t_s)\delta) + O(||\Delta \textbf{X}^T\Delta \textbf{X}||)
    \end{split}
\end{align}

\end{appendix}

\end{document}